\documentclass[12pt,leqno]{amsart} 

\usepackage{graphicx,amssymb,colordvi,textcomp,latexsym} 

\usepackage{color}
\usepackage{soul}

\newcommand{\examend}{\hfill \mbox{\textreferencemark}}
\newcommand{\rend}{\hfill $\maltese$}

\newcommand{\fps}{$\text{FPS }$}
\newcommand{\wbxi}{{\widetilde{\boldsymbol{\xi}}}}

\newcommand{\bxi}{{\boldsymbol{\xi}}}

\newcommand{\grad}[1]{{\text{\rm grad}({#1})}}
\newcommand{\erf}[1]{{\text{\rm erf}({#1})}}

\newcommand{\erfc}[1]{{\text{\rm erfc}({#1})}}
\newcommand{\terfc}{{\text{\emph{erfc}}()}}
\newcommand{\erfcb}[1]{{\text{\rm erfc}\left({#1}\right)}}
\newcommand{\erfb}[1]{{\text{\rm erf}\left({#1}\right)}}

\newcommand{\JJ}{{\mathbf{J}}}

\newcommand{\jj}{{\mathbf{j}}}
\newcommand{\TT}{{\mathbf{T}}}

\newcommand{\is}[1]{{\mathbf{#1}}}

\newcommand{\dd}{\mbox{$\;|\;$}}

\newcommand{\Refb}[1]{(\ref{#1})}

\newcommand{\real}{{\mathbb{R}}}

\newcommand{\xx}{{\mathbf{x}}}

\newcommand{\bBB}{{\mathbf{B}}}

\newcommand{\bOmega}{{\boldsymbol{\Omega}}}
\newcommand{\bbbeta}{{\boldsymbol{\bbeta}}}
\newcommand{\bbgamma}{{\boldsymbol{\bgamma}}}

\newcommand{\defoo}{\stackrel{\mathrm{def}}{=}}

\newcommand{\vv}{{\is{v}}}
\newcommand{\bvv}{\overline{{\is{v}}}}

\newcommand{\VV}{{\is{V}}}
\newcommand{\ww}{{\is{w}}}
\newcommand{\bww}{\overline{{\is{w}}}}
\newcommand{\tww}{\widetilde{{\is{w}}}}
\newcommand{\bbeta}{\overline{\beta}}
\newcommand{\bgamma}{\overline{\gamma}}
\newcommand{\bdelta}{\overline{\delta}}

\newcommand{\WW}{{\is{W}}}
\newcommand{\wWW}{{\widetilde{\WW}}}
\newcommand{\uu}{{\is{u}}}
\newcommand{\UU}{{\is{U}}}

\newcommand{\RS}{{\is{r_s}}}
\newcommand{\RA}{{\is{r_a}}}
\newcommand{\RM}{{\is{r_m}}}
\newcommand{\wLambda}{{\widetilde{\Lambda}}}
\newcommand{\wlambda}{{\widetilde{\lambda}}}
\newcommand{\wTheta}{{\widetilde{\Theta}}}
\newcommand{\wtheta}{{\widetilde{\theta}}}
\newcommand{\wZZ}{{\widehat{Z}}}
\newcommand{\wzz}{{\widehat{z}}}

\newcommand{\wTT}{{\widehat{T}}}

\newcommand{\pint}{\mbox{$\mathbb{N}$}}
\newcommand{\sgn}{\mbox{sgn}}
\newcommand{\arr}{{\rightarrow}}

\newcommand{\sgstar}{{\boldsymbol{\Sigma^\star}}}

\newtheorem{lemma}{Lemma}[section]

\newtheorem{prop}[lemma]{Proposition}
\newtheorem{thm}[lemma]{Theorem}

\theoremstyle{definition}
\newtheorem{Def}[lemma]{Definition}
\newtheorem{exam}[lemma]{Example}
\newtheorem{exams}[lemma]{Examples}

\theoremstyle{remark}
\newtheorem{rem}[lemma]{Remark}
\newtheorem{rems}[lemma]{Remarks}

\title[Population loss in shallow ReLU networks]{Population loss in shallow ReLU networks:\\
Bias \& families of critical points} 
\author{Michael Field}
\address{Michael Field, Department of Mechanical Engineering, UC
Santa Barbara, CA 93106}
\email{mikefield@gmail.com}
\date{\today}
\begin{document}
\begin{abstract}
The main result presented is a formula for the population loss in the student-teacher kernel model 
that is applicable to shallow ReLU networks with bias. This extends previous work of Choo \& Saul (2009) and Brutzkus and Globerson (2017).
The formula makes essential use of Owen's $T$-function. The necessary theory of the $T$ function is given and a high precision coding using \emph{MPFR} for the $T$ function, 
based on an algorithm of Komelj (2023), is available on request. It is shown that various families of spurious minima described in past papers of 
Arjevani and the author extend to
biased networks and that the loss is always strictly decreased when bias is added.  The change in landscape geometry caused by adding bias appears to be relatively mild. 
Only the simplest examples  are described in this paper where it is assumed that the number of inputs is equal to the number of neurons (this restriction is for reasons of length). A review of relevant previous results 
on unbiased networks is included. 
Aside from Gaussian statistics, the main mathematical tools and ideas come from analytic geometry (analytic and subanalytic sets,
the Curve Selection Lemma). 
\end{abstract}
\maketitle

\tableofcontents

\section{Introduction}
In a past collaboration with Yossi Arjevani, a detailed analysis was made of families of spurious minima that appeared in a bias-free shallow student-teacher 
ReLU network with the identity matrix as target~\cite{ArjevaniField2019a,ArjevaniField2020b,ArjevaniField2021,ArjevaniField2021x,ArjevaniField2022b,ArjevaniField2022a}.
The investigation, initiated by Arjevani, was motivated in part by earlier work of Safran \& Shamir~\cite{SafranShamir2018} who undertook a numerical investigation 
of spurious minima in low dimensional bias free ReLU networks with $d \le 20$ inputs, $k$ neurons,  $d \le k \le d+2$, and the identity matrix $I_d$ as target. The authors remarked that
the spurious minima they found always had a symmetric structure (formally, after a permutation of rows, the solutions were fixed by a subgroup of the symmetric group $S_d$ acting diagonally on the rows and columns of the parameter matrix $\WW \in M(d,d)$). As Arjevani observed, the symmetric structure was 
associated to group invariance properties of the loss inherited from the symmetry of the target.

Using this symmetric structure, one can apply representation theory of the symmetric group to determine the spectrum of the Hessian; 
in particular, the asymptotics as $d \arr \infty$~\cite{ArjevaniField2020b,ArjevaniField2022b}. These results made essential use of a result in \cite{ArjevaniField2021x} that each
family of spurious minima had (computable) fractional power series expansions in $d^{-\frac{1}{2}}$ or $d^{-\frac{1}{4}}$ (the latter only if the number of neurons $k$ was greater than $d$). 
Along similar lines, it was possible to give precise asymptotics for the loss and show that there were `good' (type II) and `bad' (type I) spurious minima. 
The type II spurious minima had rapid decay of the loss to zero, while for the type I families 
the loss converged to a strictly positive constant as $d \arr \infty$. Some of this work is reviewed in more detail in Section~\ref{sec: revsymm}.

Taking account of the applications, the discussion of symmetry breaking will be in the framework of gradient dynamics and bifurcation. Specifically by gradient dynamics of the loss in the student-teacher model and the creation or annihilation of local minima.
Given a symmetric model, there are two different types of symmetry breaking that can occur. 
If the symmetry of the model is fixed, then as model parameters are varied new local minima may appear typically with less symmetry. This is characteristic of \emph{symmetry breaking}.
In symmetric bifurcation theory this phenomenon occurs 
when a group orbit of high symmetry critical points defining, for example, local minima, bifurcates into a group orbit of critical points defining local minima but each with less symmetry. 
However, this does \emph{not} happen in an obvious way in our setting: the global minima do not bifurcate. Instead, new 
(local) minima appear through, for example, bifurcation on subspace representations associated to the symmetric group on the parameter space (see~\cite[\S 4]{ArjevaniField2022a} for examples). Basically, this
is bifurcation from a saddle to a local minimum and is characteristic of all the examples described later in Section~15.

On the other hand, \emph{forced} symmetry breaking occurs when the symmetry of the \emph{model} is reduced, for example by a generic perturbation of the target $I_d$. This will sharply reduce the number of critical points
defining the \emph{global} minimum with similar effects on all the local minima. As long as the original critical points defining local minima are non-degenerate and the perturbation is not too
large, the perturbed critical points will continue to define local minima. 

Because the critical points defining local minima in the symmetric case are typically non-degenerate for sufficiently large $d$, the spectral properties persist under perturbations
of the target matrix $I_d$.  More formally, pick a $p \times p$ matrix $A$ which is not too far from the identity $I_p$ and for $ d \ge p$, embed $A$ as the lower
$p \times p$ diagonal block in $I_d$ to obtain a `target' matrix $T$ which has 
less symmetry than $I_d$ (typically $S_{d-p}$ under the diagonal action on the upper $d-p$ diagonal block). 
In this setup, basic results such as the existence of fractional power series for families of critical points, carry over to the partially asymmetric case. 
Computation is straightforward using standard path following techniques.

However, there is a problem. As pointed out in the recent paper by Zhang \emph{et al.}~\cite{ZSL2024}, the space of functions that 
can be modelled by a bias free ReLU network is always a \emph{proper} closed vector subspace of the 
space of continuous functions\footnote{This result is presumably well known~\cite{Pinkus1999}---and is not the main focus of the paper by Zhang \emph{et al.}---but 
reading their preprint was the trigger for this paper.}.
In summary,  without further work,  no inferences can be drawn about stability in a biased ReLU network under forced symmetry breaking on the basis of results on bias free ReLU networks. 
The aim of this paper is to develop the appropriate tools that show that the examples found previously for bias free networks do carry over to biased networks. But the 
correspondence is non-trivial since a local minimum of a bias free network is generally not a critical point of a biased network, with zero bias. 

So as to keep the paper to a reasonable length, the focus will be on examples with $k = d$  
(not `over-specified' in the terminology of \cite{ArjevaniField2022b}) and there are no computations relating to the Hessian spectrum.

\begin{rems} 
(1) Adding bias to the network adds strong constraints on the existence of families of spurious minima. This is on account of multiple terms in the formula for the loss which contain the
expressions $-\bgamma \cot(\lambda) + \bbeta \csc(\lambda)$, where $\bgamma,\bbeta$ are respectively (normalized) biases for the teacher and student network and $\lambda$ is the angle between  the corresponding rows of the target and parameter matrix.
Restricting to the case $k = d$, the typical behavior as $d \arr \infty$ is that every row of the parameter matrix converges to $\pm$ a row of the target (this holds for targets in a neighbourhood of $I_d$). Since $ \cot(\lambda) \arr \pm \infty$ and
$\csc(\lambda) \arr +\infty$ as $\theta \arr 0+$ or $\pi-$, this imposes strong constraints on $\bbeta = \bbeta(d)$ as $d \arr \infty$. In particular, $|\bgamma(d) - \bbeta(d)| \arr 0$ fast enough as $d \arr \infty$ else the
term $-\bgamma \cot(\lambda) + \bbeta \csc(\lambda)$ will diverge to $\pm \infty$. A natural question is to ask about the effects of noise and random perturbations even when $\bbeta(d)-\bgamma(d)$ has good asymptotics.\\
(2) From a mathematical perspective, the framework described in this paper shows that there are interesting 
connections to analytic geometry and potentially complex analysis. 
Further, unlike when sigmoidal 
functions are used to build the activation function, much of the geometry is attractive and  natural. Perhaps this setting has the potential to form 
a bridge between the limiting regimes, when $d$ and/or $k$ $\arr \infty$, 
and the natural regimes $d,k < \infty$.
\rend
\end{rems}

We conclude with a summary of the paper by section.

Section~\ref{sec2 review} starts with a short review of shallow neural networks (with and without bias), notational conventions (which are strictly followed and avoid use of the transpose) 
and the fundamental approximation theorem for shallow neural networks.  
The \emph{population loss} or just the \emph{loss} is defined as an ergodic limit of the empirical loss, using Gaussian statistics. The section concludes with a 
review of the results of Choo \& Saul~\cite{ChooSaul2009} and Brutzkus \& Globerson~\cite{BrutzkusGloberson2017} together with remarks about the subanalyticity of the loss function in the
bias free case. 

Sections 3 through 11 are devoted to obtaining explicit formulas for the loss and gradient of the loss in a shallow ReLU network with bias. Section~3 starts with an introduction to the problem of finding an expression for the loss. This amounts to 
evaluation of the expectation
\[
f((\ww,\beta),(\vv,\gamma)) = \|\ww\|\|\vv\| \mathbb{E}_{\xx \sim \mathcal{N}}(\sigma(\bww\xx-\bbeta)\sigma(\bvv\xx -\bgamma))
\]
where we use the \emph{normalized} biases $\bbeta = \beta/\|\ww\|$,  $\bgamma = \gamma/\|\vv\|$,  $\bww = \ww/|\ww\| $, $\bvv = \vv/\|\vv\|$. and 
$\mathcal{N}$ denotes the standard normal distribution $N(0,I_d)$ on $\real^d$. 
Rather than using polar coordinates (the method used in~\cite{ChooSaul2009}), we present a proof in 3.1 of Choo \& Saul's result using Cartesian coordinates on $\real^2$ as this 
indicates some of the difficulties of handling the general case in a simple situation where the result is known. The reader will note how all the terms involving error functions, exponentials,
cosecants and cotangents disappear in the final steps of the computation (remarkable cancellations also occur in the computation of the gradient of the loss in a biased ReLU network).

Sections 3.2, 3.3 evaluate the defining integrals for $f(\ww,\vv)$ with one exception:
\[
\sqrt{\frac{\pi}{2}}\int_{\bgamma}^\infty e^{-\frac{t^2}{2}} \erfb{\frac{-t\cot(\theta) +\bbeta \csc(\theta)}{\sqrt{2}}} dt
\]
This integral cannot be evaluated in terms of `standard functions'. However, it can be evaluated in terms of two instances of Owen's $T$-function, three inverse tangents and one error function and 
this is the topic of Section~\ref{sec: 34} and Appendix C. In Appendix B we indicate how high precision computation of the $T$-function can be done rapidly using an algorithm due to Komelj~\cite{Komelj2023} (see Appendix C).

Using results on the $T$-function, we complete the derivation of a computable formula for $f(\ww,\vv)$ in Sections 5 and 6.

In Sections 7, 8 and 9, the gradient of $f$ is computed with respect to $\ww$ and $\bbbeta$. A few brief notes are 
included on changes in the formulas if we use $\beta$ rather than $\bbeta = \beta/\|\ww\|$ as the bias parameter. 

In Section 10, the formula for the gradient of the loss $\mathcal{L}$ with respect to $\ww$ is given; in Section 11 is the formula for the gradient of the loss with respect to
$\boldsymbol{\bbeta}$.  In addition, formulas for the $\boldsymbol{\bbeta}$ derivative of the loss at $\boldsymbol{\bbeta} = 0$ are given. This is approximately the half-way point of the paper and 
readers not interested in applications related to symmetry can safely omit reading the remaining sections (but perhaps not the appendices).

In Section~12, we review notation and terminology for symmetry in a bias free ReLU network with the high symmetry target $I_d$. Next is a description of symmetries of the loss $\mathcal{L}$ in a biased network 
with target $(I_d, \bbgamma)\in M(d,d+1)$, where $\bbgamma \in M(d,1)$ is the matrix of biases which are assumed equal. The remainder of the section is devoted to a description of how we quantify the
symmetry of a critical point using the isotropy group of the action of $\Delta S_d$ on $M(d,d+1)$ and the vital role played by the associated fixed point 
spaces which allow a big reduction in the dimension of the space needed
when we search for families of critical points of the loss.  For example, the fixed point space $M(d,d+1)^{\Delta S_d}$ is 3-dimensional and is used to find type I families of critical points with isotropy $\Delta S_d$.
If we search for families with isotropy $\Delta(S_{d-q} \times S_q)$, where $q \ge 2$ and fixed, then the associated fixed point space is of dimension $8$ compared with the $d^2 + d$, the dimension of $M(d,d+1)$.
Typically, Newton's method is used to find critical points by treating $d$ as a real parameter (it is, if we restrict to the fixed point space and is fast). An alternative, is to use gradient descent on the fixed point space.

In the remainder of Section~12, we develop the necessary results to give the formula for the gradient of the loss (Proposition~\ref{prop: sym}). 
Although the presence of symmetry greatly simplifies the computation of critical points, the formula for the gradient is significantly more complex than the general formula for the asymmetric case given towards the end of Section~10. The section ends with necessary details about the computation of $\mathcal{L}$ in the symmetric case as well as the formula for the derivative of $\mathcal{L}$ with respect to $\bbeta_i$. 

Section~13 is a brief review of relevant past work on the bias free case and the role of fractional power series in parameterizing $d$-dependant families of critical points (types I and II). 

In Section~14, examples are given of the computation of $\frac{\partial \mathcal{L}}{\partial \bbeta}(\WW,\mathbf{0})$, where $\WW$ is a critical point of the loss in the zero bias network. 
Included are examples where $k > d$ or the target has an asymmetric diagonal block. In every case, $\frac{\partial \mathcal{L}}{\partial \bbeta}(\WW,\mathbf{0})\ne 0$.  
These results  imply that the extension of results obtained for unbiased networks to biased networks is non-trivial. For example, adding bias might annihilate the family of spurious minima and result in the point
$(\WW,\is{0})$ lying on a gradient trajectory convergent to a critical point defining the global minimum. However, it does not---at least in the cases so far analysed.

In Section~15, we give the results of numerical investigations that show that families
of spurious minima, and critical points generally, do persist when we add bias and limiting values as $d\arr \infty$ are the same as those in the unbiased case.   Moreover, adding bias reduces the loss (as indeed it should) but
as $d \arr \infty$, the difference converges rapidly to zero,

Appendices A, B \& C respectively review needed formulas for the error function, give details on the high precision computation of the $T$-function and lastly describe an alternative derivation 
of the key formula~\Refb{EQ: Tgen}~\cite{Prezmo} for the "gneralized $T$-function using only Owen's results~\cite{Owen1980}.

\section{Background on neural networks and machine learning}\label{sec2 review}
\subsection{Shallow Neural Networks}
Roughly speaking, a shallow neural network consists of an \emph{input} layer, one\footnote{More generally, a `small' number of hidden layers} \emph{hidden layer} comprised of ``neurons'', 
and an \emph{output} layer. The number of layers is defined as the number of hidden layers plus one. If the output layer is not counted--that is, the weighted summation over the hidden neuron outputs is not included---then the network is said to have $1\frac{1}{2}$ layers.
See Figure~\ref{Shallow} for a 2-layer neural network with one hidden layer.  
\begin{figure}\label{Shallow}
\centering
\includegraphics[width=0.9\textwidth]{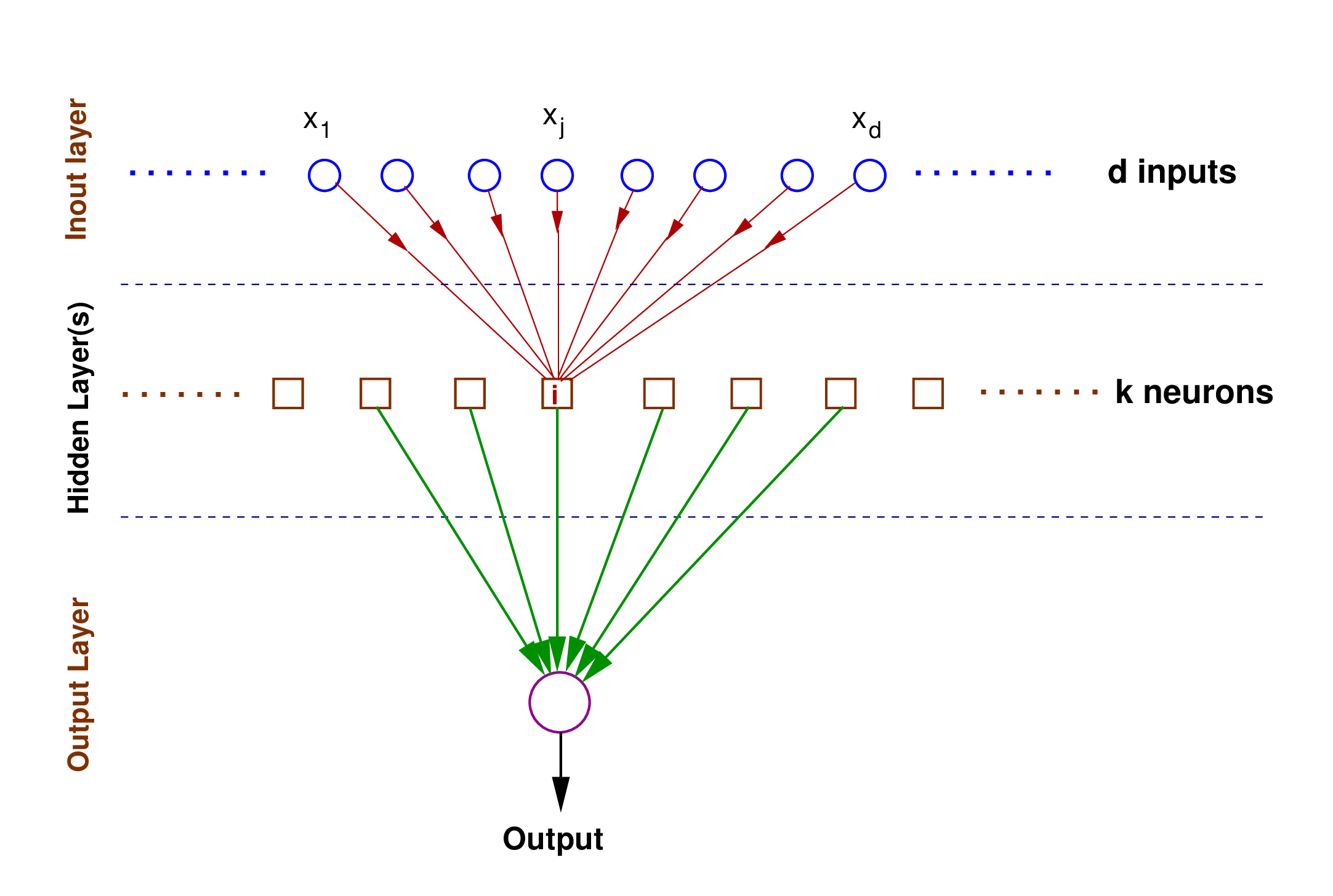}
\end{figure}
The \emph{width} of the network is the maximal number of nodes in a layer. Deep neural networks typically have many hidden layers of relatively small widths whereas 
shallow networks may have very large widths. In this work the emphasis is on shallow networks though results are likely relevant for deep networks. 

Referring to Figure~\ref{Shallow}, suppose that the input is $\xx = (x_1,\ldots,x_d) \in \real^d$, and that we are given
a set of $k$ affine linear maps $A_i : \real^d \arr \real$, $\mathbf{a} \in  (\real^k)^\star\approx M(1,k)$ 
and a continuous $\sigma:\real\arr \real$---the \emph{activation}\footnote{In this work, only ReLU activation is used: $\sigma(t) = \max\{0,t\}$.}.  The neural network defines the map $\widehat F: \real^d \arr \real$ by
\[
\widehat{F}(\xx) = \sum_{i \in [k]} a_i\sigma(A_i(\xx)) 
\]
Suppose that $A_i(\xx) = \sum_j w_{ij} x_j - \beta_i$. Then $[w_{ij}] = \WW \in M(k,d)$ defines the \emph{parameter matrix} and
$[\beta_i]=\boldsymbol{\beta} \in M(k,1)$ is the vector of neuron \emph{biases} or \emph{offsets}. Note that it is common to write $\sigma(\ww \xx + \beta)$ 
rather than $\sigma(\ww \xx -\beta)$ as done here (following \cite{Pinkus1999}).

Let $\ww^i \in (\real^d)^\star$ denote row $i$ of $\WW$, $i \in [k]$. Then 
\[
\widehat{F}(\xx) = \sum_{i \in [k]} a_i \sigma(\ww^i \xx - \beta_i) = \is{a} \sigma(\WW\xx - \boldsymbol{\beta}), 
\]
where $\sigma(\WW\xx - \boldsymbol{\beta}) \defoo (\sigma(\ww^1\xx - \beta_1),\ldots,\sigma(\ww^k\xx - \beta_k))\in \real^k$ 
and $\is{a} \sigma(W\xx - \boldsymbol{\beta})$ is matrix multiplication (or evaluation of a linear functional).  
In summary, neuron $i$ has bias $\beta_i\in \real$ and (input) parameters $\ww^i = (w_{i1},\cdots,w_{id})\in (\real^d)^\star$.
The notation generalizes straightforwardly to maps $F:\real^{n_1}\arr \real^{n_2}$, $n_2 > 1$, as well as to 
deep neural networks.  

Given the activation function $\sigma$, let $\mathcal{M}(\sigma)\subset C^0(\real^d)$ be the vector space spanned by $\{\sigma(\ww\xx - \beta)\dd \ww \in (\real^d)^\star, \beta \in \real\}$.
We recall the fundamental approximation theorem for shallow neural nets.
\begin{thm}[{\cite[Theorem 3.1]{Pinkus1999}}]
Let $\sigma:\real\arr\real$ be a continuous function and give $C^0(\real^d)$ the topology of uniform convergence on compact subsets. Then $\mathcal{M}(\sigma)$ is dense in $C^0(\real^d)$ iff $\sigma $ is \emph{not} a polynomial.
\end{thm} 
\begin{rem}
The approximation theorem fails if we replace \emph{affine linear} maps by \emph{linear} maps---that is, all biases are set to zero and $\mathcal{M}^\star(\sigma)$ 
is defined to be the vector space spanned by $\{\sigma(\ww\xx)\dd \ww \in (\real^d)^\star\}$:  the closure of $\mathcal{M}^\star(\sigma)$ has no interior points. 
For details we refer to the recent paper by Zhang, Saxe and Latham~\cite{ZSL2024}.  The main aim of the present work is to address the question of whether or not results obtained for the restricted 
class of bias free ReLU networks extend to class of ReLU networks with bias covered by the approximation theorem. For example, the existence of families of 
spurious minima of various types in bias free student teacher 
models~\cite{SafranShamir2018,ArjevaniField2020b,ArjevaniField2021x,ArjevaniField2022b} (see Section~\ref{sec: revsymm} 
for a review of theory for examples with symmetry). Since  
the closure of $\mathcal{M}^\star(\sigma)$ has no interior points, and non-degenerate critical points of a bias free network appear never to be critical points of
the associated biased network, with bias set to zero (see Section~\ref{Sec: 11}), the question concerning the existence of these families in a ReLU network  with bias 
cannot be inferred using simple stability (persistence under perturbation) arguments.
The answer, as indicated at the end of the paper, is positive but non-trivial. 
\rend
\end{rem}

\subsection{Statistical learning and a reduction}
Suppose given a physical system that we believe can be modelled by a function $F:\real^d \arr \real$.  Based on experimental 
data, it is required to construct a `good' approximation to the conjectured model $F$. To this end, assume given a `random' sequence of points 
$(\xx^i) \in \real^d$ and corresponding value sequence $(y^i) = (F(\xx^i))$ and train a neural network with this data. Roughly speaking, 
initialize the network with an appropriate\footnote{For example, Xavier~\cite{SafranShamir2018}} random distribution of 
parameters $\WW, \mathbf{a}$ and biases $\boldsymbol{\beta}$. Compute $\widehat{F}(\xx^i)$,  $i \ge 1$,
and then improve the estimates of $\widehat{F}(x^i)-y_i$ using gradient descent and back propagation on an appropriate cost function (see below). 

This process is difficult to handle theoretically since the function $F$ is unknown and the implied optimization is non-convex. There is also the matter of 
a lack of analyticity (the ReLU activation function $\sigma(t) = \max\{0,t\}$ most commonly used 
is only piecewise analytic). 
A more realistic approach is to study the \emph{realizable} case where the 
data is generated by a known function $F$. To this end, assume a trained neural network (the `teacher') and use that to train a second 
neural network (the `student'). Here we are given the parameter matrix $\VV$, bias vector $\boldsymbol{\gamma}$ and linear functional 
$\mathbf{b}$ defining the output for the teacher network. 
Assume the teacher and student networks both have $d$ inputs. 
If the teacher network has $k'$ neurons, then $(\VV,\boldsymbol{\gamma}) \in M(k',d+1)$, where the column matrix$\boldsymbol{\gamma}\in M(k',1)$ gives the 
neuron biases and 
$\mathbf{b} \in M(1,k)$.  Define $F:\real^d \arr \real$ by
\[
F(\xx) = F(\xx;\VV,\boldsymbol{\gamma},\mathbf{b}) = \sum_{i \in [k]} b_i \sigma(\vv^i \xx - \gamma_i) = \is{b} \sigma(\VV\xx - \boldsymbol{\gamma}) = y
\]
Here $\xx$ will be sampled according to a given distribution $\mathcal{D}$ on $\real^d$, typically the standard norm distribution $\mathcal{N}(0,I_d)$. If we assume $n$ training points $\xx_1,\ldots,\xx_n$
and set $F(\xx_i) = y_i$,
then the problem is to find  $\WW,\boldsymbol{\beta},\mathbf{a}$  for the student network that minimize the 
\emph{training} (or \emph{empirical}) loss:
\[
\frac{1}{n}\left(\frac{1}{2} \sum_i \big(\widehat{F}(\xx_i;\WW,\boldsymbol{\beta})-y_i\big)^2\right)
\]
Assume that the student has $k$ neurons and so $\WW \in M(k,d)$, $\boldsymbol{\beta} \in M(k,1)$ and $\mathbf{a} \in M(1,k)$. Letting $n\arr \infty$, 
leads to the problem of minimizing the \emph{population loss}\footnote{Also called the \emph{population risk}. Typically we just refer to the `loss'}: 
\[
\mathcal{L}(\WW,\VV) = \frac{1}{2}\mathbb{E}_{\xx\sim\mathcal{D}} \big(\widehat{F}(\xx,\WW) - F(\xx;\VV)\big)^2,
\]
where $\widehat{F}(\xx) = \widehat{F}(\xx;\WW,\boldsymbol{\beta},\mathbf{a})$. Here, as elsewhere
we usually write $\widehat{F}(\xx)$ or $\widehat{F}(\xx;\WW)$ rather than $\widehat{F}(\xx;\WW,\boldsymbol{\beta},\mathbf{a})$; similarly for $F(\xx)$.

This reduces the problem to one of non-convex optimization of $\mathcal{L}=\mathcal{L}(\WW)$ using gradient descent.  
Further simplifications are possible and often used. 
For example, assuming biases are constant and set to zero and the functional $\mathbf{a}$ is constant.  
By way of illustration, in Brutzkus \& Globerson's work on no-overlap convolutional networks~\cite{BrutzkusGloberson2017}, 
the networks studied are bias free, with ReLU activation, the $k$ neurons have the same number of inputs, none of which overlap, 
and the output is given by $F(\xx,\ww)=\frac{1}{k} \sum_i \sigma(\ww,\xx_i)$---average pooling.  

It is also natural to take $\mathbf{a}=\mathbf{b} = \mathbf{1}_k$~\cite{SafranShamir2018,ArjevaniField2020b}. Since ReLU
activation is positively homogeneous we can rescale rows of $\WW$ and $\boldsymbol{\beta}$ so that
$\mathcal{L}(\WW,\VV)$ is unchanged but $a_i \in \{\pm1, 0\}$, for all $i\in [k]$.  
Under these conditions, $\WW, \boldsymbol{\beta}$ and $\mathbf{a}$ are uniquely determined up to row permutations\footnote{Of course, 
we may remove rows $\ww^i$ if $a_i = 0$.}. 

Assume the target network is bias free and the target matrix of parameters is $\VV \in M(d.d)$ and $\VV$ has no zero rows. 
Take $\mathbf{b} = \mathbf{1}_d$.  This choice simplifies the optimization process (the network is said to have $1\frac{1}{2}$ layers).
Let $k \ge d$ and $\WW \in M(k,d)$ be a choice of initialization of the parameter matrix for the student network. Again assume no zero rows for $\WW$, bias 
free and $\mathbf{a}=\mathbf{1}_k$. Assume ReLU activation and bias free. The loss is given by
\[
\mathcal{L}(\WW,\VV)= \frac{1}{2} \mathbb{E}_{\xx\sim  \mathcal{N}(0,I_d)} \left(\sum_{i \in [k]} \sigma (\ww^i\xx) - \sum_{j \in [d]} \sigma (\vv^j\xx)\right)^2,
\]
where $\mathbb{E}$ is the expectation over the zero mean, unit variance Gaussian distribution on $\real^d$. Obviously $\mathcal{L}(\WW,\VV) = 0$ if
$\WW = \VV$. The same is true if we replace $\WW$ by any row permutation of $\WW$; similarly for $\VV$. 

Following Brutzkus \& Globerson, define 
\[ 
f(\ww,\vv) = \mathbb{E}_{\xx \sim \mathcal{N}(0,I_d)}\; \sigma(\ww\xx)\sigma(\vv\xx),\; \ww,\vv\in (\real^d)^\star
\] and note that 
\[
\mathcal{L}(\WW,\VV)= \frac{1}{2}\left[\sum_{i,ji\in [k]} f(\ww^i,\ww^j) -2\sum_{i\in [k], j \in [d]}f(\ww^i,\vv^j) + \frac{1}{2}\sum_{i,j \in [d]}f(\vv^i,\vv^j) \right]
\]
Using elementary methods, Choo \& Saul~\cite{ChooSaul2009} proved that  
\begin{equation}\label{EQ:  ChooSaul}
f(\ww,\vv) = \frac{\|\ww\|\|\vv\|}{2\pi}\left(\sin(\theta) + (\pi - \theta)\cos(\theta)\right),
\end{equation}
where $\theta$ is the angle between $\ww$ and $\vv$. Of course, if either $\vv$ or $\ww$ is zero, $f(\ww,\vv)$ is zero. Obviously, $f$ is a continuous function of
$(\ww,\vv)$ on $\real^d\times \real^d$ and is analytic precisely 
on the open and dense domain $\{(\ww,\vv) \dd \langle \ww,\vv\rangle \ne \pm \|\ww\|\|\vv\|\}\subset \real^d\times \real^d$. In particular, if $\ww,\vv$ are parallel, $f$ is not analytic at $(\ww,\vv)$.

Moreover, $f$ is a \emph{subanalytic} function. That is, the closure of the graph of $f$ is locally the projection of a relatively compact semianalytic 
set~\cite[\S 2]{Parusinski1994},\cite{BierMil1998}. This follows easily since the sum and product of subanalytic functions are subanalytic and the Euclidean norm
$\|\;\|$ is trivially a subanalytic function of $\ww \in \real^d$. Hence, it suffices to show that $\langle \ww,\vv\rangle/\|\ww\|\|\vv\|$ is subanalytic 
and this is elementary using the Cauchy-Schwartz inequality.  
For our applications, $\ww, \vv$ will never be zero and so proving subanalyticity amounts to analyzing limiting behaviour as $\langle \ww,\vv\rangle \arr 0$, where $\ww,\vv$ are bounded away from zero.

The subanalyticity of $f$ implies strong regularity results on $f$ that are exploited in
\cite{ArjevaniField2021x,ArjevaniField2022b,ArjevaniField2022a} when the target $\VV = I_d$ and methods from the
analytic geometry of stratified sets can be used to prove the existence of fractional power series (\fps) representations of families of critical points  
parametrized by $d$ and precise results on the Hessian spectrum (the \fps results are discussed in more detail in Section~\ref{sec: revsymm}).

The main aim of this work is to provide the framework for extending results on spurious minima and fractional power series to biased networks. The key steps involve generalizing the
formula of Choo \& Saul~\Refb{EQ:  ChooSaul} and quantifying the regularity of the associated population loss which is quite different from the bias free case.

\section{Bias}\label{sec: bias}   
We find an explicit formula for the loss when bias is included.  Our approach initially follows that of Choo \& Saul~\cite{ChooSaul2009} 
for zero bias (see~\cite{ArjevaniField2021x} for details in the case of `leaky ReLU'). 

Let $\|\;\|$ and $\langle\;,\;\rangle$ denote the standard Euclidean norm and inner product. 
Given $\ww\in (\real^d)^\star$, $\xx\in \real^d$, $\ww\xx$ is evaluation of $\ww$ at $\xx$ (matrix multiplication;  $\ww(\xx)$ if it is necessary to emphasize that $\ww\xx$ 
is the linear functional $\ww$ evaluated at $\xx$). If $(\real^d)^\star$ is identified with $\real^d$ using the inner product,  then
$\ww\xx = \langle \ww, \xx\rangle = \ww^T\xx$ (matrix multiplication) but we avoid the use of the transpose here and only use the inner product
to  define the angle between $\ww,\vv$ (row vectors). 
To emphasize: $\xx$ is a column vector and $\ww$ is a row vector (linear functional or dual vector). 
Care needs to be taken with the notation especially because of the use of (matrix) group representation theory where the formal theory is
most transparently expressed in terms of matrices and vectors rather than vectorized quantities.

Let  $\ww,\vv \in (\real^d)^\star$ and $\beta,\gamma\in \real$. Following the approach for the bias free case, we want to evaluate 
\[
f((\ww,\beta),(\vv,\gamma)) \defoo \mathbb{E}_{\xx \sim \mathcal{N}}(\sigma(\ww\xx-\beta)\sigma(\vv\xx -\gamma)),
\]
where $\xx\in\real^d$ and 
$\mathcal{N} = \mathcal{N}(0,I_d)$ (standard normal). Usually we write $f(\ww,\vv)$ and omit explicit dependence on the bias parameters. 
\begin{rem}\label{rem: fnotzero}
Unlike the bias free case, if $\ww=\is{0}$, $\vv\ne \is{0}$, then  $f(\ww,\vv)=0$ only if $\beta \ge 0$; similarly if $\vv=\is{0}$.  In particular, 
$f(\is{0},\is{0}) \ne 0$ iff $\beta,\gamma < 0$.  We generally assume that $\WW$ and $\VV$ have no zero rows. 
\rend
\end{rem}
Assume $\ww, \vv \ne \is{0}$. Write
\begin{equation}
\ww\xx - \beta = \|\ww\|(\bww\xx - \bar{\beta})\quad \vv\xx - \gamma = \|\vv\|(\bvv\xx - \bgamma)
\end{equation}
where 
\[
\bww = \frac{\ww}{\|\ww\|}, \;\;\bvv = \frac{\vv}{\|\vv\|},\;\;\bbeta= \frac{\beta}{\|\ww\|},\;\;\bgamma= \frac{\gamma}{\|\vv\|}
\]

It follows that
\begin{equation}\label{EQ: fundEx}
f((\ww,\beta),(\vv,\gamma)) = \|\ww\|\|\vv\| \mathbb{E}_{\xx \sim \mathcal{N}}(\sigma(\bww\xx-\bbeta)\sigma(\bvv\xx -\bgamma))
\end{equation}
\begin{rem}
In future, we regard $\bbeta,\bgamma$ as the (normalized) bias parameters. This will not change the critical point structure but gives a significant simplification
of computations and formulas. Of course, since $\beta = \|\ww\|\bbeta$, we can always recover the bias $\beta$. \rend
\end{rem}
If $\ww,\vv$ are parallel and not zero, computations are elementary, see Lemmas~\ref{lem: zero}, \ref{lem: pi}. So assume $\ww,\vv$ are not parallel; in particular,  $\ww,\vv \neq \mathbf{0}$.
Choose the orthogonal coordinate system on the 2-plane spanned by $\{\ww,\vv\}$ so that 
\[\bvv = (1,0,\ldots,0),\quad \bww = (\cos(\theta),\sin(\theta),0,\ldots,0)
\]
where $\theta=\theta_{\ww,\vv} \in (0,\pi)$ is the angle between $\ww$ and $\vv$ defined by
\[
\theta = \cos^{-1}\left(\frac{\langle\ww,\vv\rangle}{\|\ww\|\|\vv\|}\right)
\]
Given $\xx \in \real^d$,
\[
\ww \xx = \|\ww\|(x_1\cos(\theta)+x_2\sin(\theta)),\;\;\vv\xx = \|\vv\|x_1
\]
The term $\sigma(\ww\cdot\xx-\beta)\sigma(\vv\cdot\xx -\gamma)$ contributes to the expectation iff 
$\sigma(\ww\cdot\xx-\beta), \sigma(\vv\cdot\xx -\gamma)$ are both strictly positive. Setting 
\[
F(x_1,x_2) = (x_1\cos(\theta)+x_2\sin(\theta)-\bar{\beta})(x_1-\bgamma)
\]
and $p_{\mathcal{N}} = e^{-\frac{x_1^2+x_2^2}{2}}$ (so $\int_{\real^2} p_{\mathcal{N}}dx_1dx_2 = 2\pi$), we have 
\begin{equation}\label{EQ: fint}
f(\ww,\vv) =  \frac{\|\ww\|\|\vv\|}{2\pi}\int_{D_\theta} F(x_1,x_2)p_{\mathcal{N}}(x_1,x_2) dx_1dx_2,
\end{equation}
where 
\[
D_\theta = D = \{(x_1,x_2) \in \real^2 \dd x_1\cos(\theta)+x_2\sin(\theta) -\bar{\beta}, x_1 - \bgamma \ge 0\}
\]

In figure~\ref{D1}, we show the domain of integration if $\bgamma < 0 < \bbeta$ and $\theta \in (0,\pi/2]$. The boundary of $D$ is defined by the lines
$x_1 = \bgamma$ and $x_1\cos(\theta) +x_2\sin(\theta) = \bbeta$ (denoted $L$ in the figure).
\begin{figure}
\centering
\includegraphics[width=0.9\textwidth]{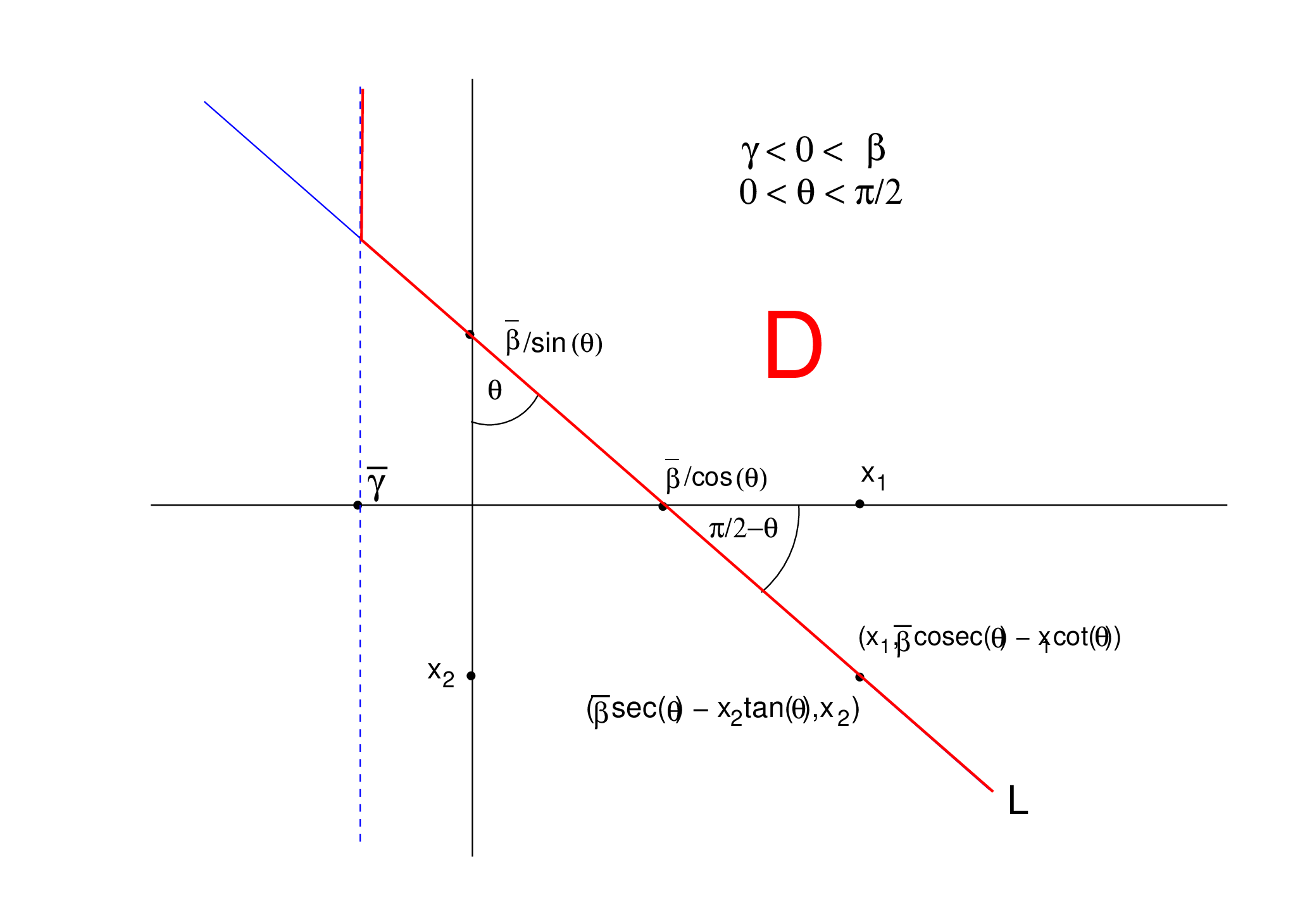}
\caption{The region of integration in the $x_1,x_2$-plane with boundary shown in red and $\bar{b} < 0 < \bar{\beta}$ and $\theta \in (0,\pi/2]$. As $\theta $ increases, the line $L$ defined by 
$x_1\cos(\theta)+x_2\sin(\theta) =\bar{\beta}$ rotates counter clockwise so that when $\theta \in (\pi/2,\pi)$, $L$ intersects the negative $x_1$-axis. }
\label{D1}
\end{figure}

\begin{rems}\label{rem: intdom}
(1) It is clear from \Refb{EQ: fint} that if $\theta \in (0,\pi)$, then $f(\ww,\vv)$ 
is a continuous function of $\ww,\vv, \bbeta$ and $\bgamma$. 
Just as for zero bias,  it is easy to see that $f$ will not be differentiable with respect to $\ww$ if $\ww = 0$ and will be analytic (in $\ww,\vv, \bbeta$ and $\bgamma$) provided that $\ww,\vv$
are not parallel: $\|\ww\| \|\vv\| \ne \langle\ww,\vv\rangle$ (which implies $\ww,\vv$ are non-zero). For differentiability see also Remarks~\ref{rem: intdom}(3). \\
(2) Suppose that $\ww,\vv$ are are non-zero and parallel with $\theta = 0$. Observe that as $\theta \arr 0+$, the line $L$ converges to $x_1 = \bbeta$. Consequently,
\[
D_0 = \{(x_1,x_2) \dd x_1 \ge \max\{\bgamma,\bbeta\}\}
\]
and \emph{not} $x_1 \ge \bgamma$ unless $\bgamma \ge \bbeta$. \\
If instead we consider the limiting domain $D_\pi$ obtained by letting $\theta \arr \pi-$, then $L$ converges to $x_1 = -\bbeta$ and
\[
D_\pi = \begin{cases} 
& \emptyset,\;\text{if } \bgamma + \bbeta > 0\\
& [\bgamma,-\bbeta] \times \real \subset \real^2,\;\text{if } \bgamma + \bbeta \le 0. 
\end{cases}
\]
In particular, $f(\ww,\vv) = 0$ if $\theta = \pi$ and $\bgamma + \bbeta \ge 0$. \\
However, although the limits of $D_\theta$ as $\theta \arr 0+,\pi-$ might suggest a discontinuity, it will be shown later that
$\|\ww\|\|\vv\|\int_{D_\theta} F p_{\mathcal{N}}$ is a continuous function of $\ww,\vv, \beta,\gamma$. In particular, the 
limit of $\int_{D_\theta} F p_{\mathcal{N}}$ as $\theta\arr 0+$ exists and is equal to $\int_{D_0} F p_{\mathcal{N}}$; similarly for the limit as $\theta\arr \pi-$.  \\
(3) In the case of zero bias, Safran, Yelundai \& Shamir~\cite{Safranetal2021} prove that the loss $\mathcal{L}$ is $C^2$ at $\WW\in M(k,d)$ provided that no row of $\WW$ is zero and $k \ge d$. However,
if $k > d$, there will be zero eigenvalues at critical points defining the global minima \emph{op.~cit.}, \cite[Appendix A.7]{ArjevaniField2022b}. This is so even though the Hessian is continuous 
and all eigenvalues are positive. This degeneracy is quite apparent when doing numerical investigations in the over-parametrized case. Of course, all eigenvalues at a global minimum are positive.
However, if $k \ge d+1$, there typically exist many non-degenerate critical points including spurious minima. \\
(4) Both domain $D_\theta$ and integrand $F$ in \Refb{EQ: fint} depend on the choice $\vv = \|\vv\|(1,0,\ldots,0)$. 
If instead we take $\ww = \|\ww\|(1,0,\ldots,0)$,  then $\theta$ is unchanged and 
$\bvv = \|\vv\|(\cos(\theta),\sin(\theta),0,\ldots,0)$.  The effect is to swap $\bbeta$ and $\bgamma$ in the expressions for $F$ and $D_\theta$. 
This leads to two terms in the expression we obtain for $f(\ww,\vv)$ depending
on which vector was chosen to be a multiple of the unit vector (of course, the value of $f(\ww,\vv)$ is unchanged). The associated identity leads to a surprising and important symmetry in $(\bbeta,\bgamma)$ that underpins much of our analysis.
\rend
\end{rems}
\begin{exam}\label{ex: simp}
If $\vv = \ww$ and $\beta = \gamma$. then  $\theta = 0$, $\bbeta = \bgamma$ and $D_0 = [\bgamma,\infty)\times \real$. The problem of computing $f$ reduces to computing
$\frac{1}{\sqrt{2\pi}}\int_{\bbeta}^\infty (x^2 -2\bbeta x-\bbeta^2) e^{-\frac{x^2}{2}} dx$ giving 
\begin{equation}\label{EQ: fww}
f(\ww,\ww) = \frac{\|\ww\|^2(\bbeta^2 +1)}{2}\erfcb{\frac{\bbeta}{\sqrt{2}}} -\frac{\|\ww\|^2}{\sqrt{2\pi}}\bbeta e^{-\frac{\bbeta^2}{2}},
\end{equation}
where $\terfc$ is the complementary error function (see Appendix~\ref{append: error}).  Since $\bbeta = \bgamma$, $D_0 = \lim_{\theta\arr 0+}D_\theta\subset \real^2$.  
\examend
\end{exam}
\begin{lemma}\label{lem: zero}
Suppose that $\ww,\vv$ are parallel and non-zero with $\theta = 0$. If $\beta,\gamma \in \real$, then
\begin{equation}
f(\ww,\vv) =  \frac{\|\ww\|\|\vv\|(\bbeta\bgamma + 1)}{2}\erfcb{\frac{M}{\sqrt{2}}} - \frac{m\|\vv\|\|\ww\|}{\sqrt{2\pi}}e^{-\frac{M^2}{2}},
\end{equation}
where $M = \max\{\bgamma,\bbeta\}$, $m  = \min\{\bgamma,\bbeta\}$. 
\end{lemma}
\begin{proof}  As in Example~\ref{ex: simp}, the result follows by evaluating the elementary integral
$\int_{\max\{\bgamma,\bbeta\}}^\infty (x_1^2 - (\bgamma+\bbeta)x_1 +\bbeta\bgamma) e^{-\frac{x^2}{2}}dx$.
\end{proof}
\begin{lemma}\label{lem: pi}
Suppose that $\ww,\vv$ are parallel and non-zero with $\theta = \pi$. If $\beta,\gamma \in \real$, then
\begin{enumerate}
\item For $\bgamma \le -\bbeta$,
\begin{multline}
\begin{aligned}
f(\ww,\vv) = & \frac{\|\ww\|\|\vv\|(1-\bbeta\bgamma)}{4}\left[\erfb{\frac{\bbeta}{\sqrt{2}}}+\erfb{\frac{\bgamma}{\sqrt{2}}}\right]\\ 
	&-\frac{\|\ww\|\|\vv\|}{2\sqrt{2\pi}}\left[\bgamma e^{-\frac{\bbeta^2}{2}}+\bbeta e^{-\frac{\bgamma^2}{2}}\right] 
\end{aligned}
\end{multline} 
\item For $\bgamma \ge -\bbeta$, $f(\ww,\vv) = 0$.
\end{enumerate}
\end{lemma}
\proof Similar to that of Lemma~\ref{lem: zero} using Remarks~\ref{rem: intdom}(2). \qed
\begin{rems} 
(1) Lemmas~\ref{lem: zero}, \ref{lem: pi} imply that $f(\ww,\vv)$ is a continuous function of $\bbeta,\bgamma$ (and $\beta,\gamma$) on $D_0$ and $D_\pi$. It may be shown that $f$ is
$C^1$ in $\bbeta,\bgamma$ on these subspaces but not $C^2$. \\
(2) Note the vanishing of $f$ if $\bgamma \ge -\bbeta$ in Lemma~\ref{lem: pi}(2). If the network is not biased, then
$f(\ww,\vv)$ is zero iff $\ww$ is a non-zero negative multiple of $\vv$.
\rend
\end{rems}

\subsection{Computation of $f$ if the network is not biased$^\star$}\label{sec: zerobias}
Before completing the computation of $f(\ww,\vv)$ for biased ReLU networks in section~\ref{sec: 22}--\ref{sec: 27}, 
the reader may find it helpful to review 
the computation for bias free networks without the use of polar coordinates~\cite{ChooSaul2009}. 

If $\beta = \gamma = 0$, then
\begin{eqnarray*}
f(\ww,\vv) & = & \frac{\|\ww\|\|\vv\|}{2\pi}\int_D (x_1^2\cos(\theta)+x_1x_2\sin(\theta))e^{-\frac{x_1^2+x_2^2}{2}}dx_1dx_2\\
& = & \frac{\|\ww\|\|\vv\|}{2\pi}\int_0^\infty \left(\int_{\widetilde{x}_2}^\infty x_1^2\cos(\theta)+x_1x_2\sin(\theta)e^{-\frac{x_2^2}{2}}dx_2\right)e^{-\frac{x_1^2}{2}}dx_1,
\end{eqnarray*}
where $\widetilde{x}_2 = -x_1\cot(\theta)$ (see Figure~\ref{D1}).  We evaluate the double integral directly, rather than use the 
transformation to polar coordinates, so that we can see how the angle term $(\pi-\theta)\cos(\theta)$ enters into the equation for $f(\ww,\vv)$. 
Of course, there is no scope for a simple (or useful) application of polar coordinates if $\beta\gamma \ne 0$.

Write  
\begin{equation*}
\int_{\widetilde{x}_2}^\infty x_1^2\cos(\theta)+x_1x_2\sin(\theta)e^{-\frac{x_2^2}{2}}dx_2=\cos(\theta)x_1^2 A_{11} + \sin(\theta)x_1A_{12}
\end{equation*}
where
\[
A_{11} =  \int_{\widetilde{x}_2}^\infty e^{-\frac{x_2^2}{2}}dx_2 \quad A_{12} = \int_{\widetilde{x}_2}^\infty x_2e^{-\frac{x_2^2}{2}}dx_2 
\]
An elementary integration gives
\[
A_{12}=\int_{\widetilde{x}_2}^\infty x_2e^{-\frac{x_2^2}{2}}dx_2 = e^{-\frac{x_1^2\cot^2(\theta)}{2}}
\] 
and so the contribution to the integral over $D$ will be
\begin{equation}
\label{EQ8}
\sin(\theta)\int_0^\infty x_1 e^{-\frac{x_1^2\cot^2(\theta)}{2}} e^{-\frac{x_1^2}{2}}dx_1 = \sin(\theta)\int_0^\infty x_1  e^{-\frac{x_1^2\csc^2(\theta)}{2}} dx_1 = \sin^3(\theta).
\end{equation}
Turning to the first term in the integral over $D$, 
\[
A_{11} = \sqrt{\frac{\pi}{2}}\erfcb{\frac{-x_1\cot(\theta)}{\sqrt{2}}} = \sqrt{\frac{\pi}{2}}\left[1- \erfb{\frac{-x_1\cot(\theta)}{\sqrt{2}}}\right]
\]
and so we must evaluate 
\begin{eqnarray*}
F(\theta)&=& \sqrt{\frac{\pi}{2}}\int_0^\infty x_1^2 \big(1 - \erfb{\frac{-x_1\cot(\theta)}{\sqrt{2}}}\big) e^{-\frac{x_1^2}{2}}dx_1\\
& = & \frac{\pi}{2} - \sqrt{\frac{\pi}{2}}\int_0^\infty x_1^2  \erfb{\frac{-x_1\cot(\theta)}{\sqrt{2}}} e^{-\frac{x_1^2}{2}}dx_1,
\end{eqnarray*}
where the last equality follows from~\Refb{EQ4}. Since $\cot(\pi/2) = 0$, 
\begin{equation}
\label{EQ9}
F(\pi/2) = \frac{\pi}{2}
\end{equation}
By \Refb{EQ0a}, $\erf{ax} = \frac{2}{\sqrt{\pi}} \int_0^{ax} e^{-t^2}dt$ and so since $\frac{d}{d\theta}\cot(\theta) = -\csc^2(\theta)$
\[
\frac{\partial}{\partial \theta} \erfb{\frac{-x_1\cot(\theta)}{\sqrt{2}}} = \sqrt{\frac{2}{\pi}}x_1\csc^2(\theta)e^{\frac{-x^2_1\cot^2(\theta)}{2}}
\]
Using $1+\cot^2(\theta) = \csc^2(\theta)$, 
\begin{eqnarray*}
F'(\theta)& = & -\csc^2(\theta)\int_0^\infty x_1^3 e^{-\frac{x_1^2\csc^2(\theta)}{2}} dx_1\\
& = &  -2\sin^2(\theta)=-1+\cos(2\theta).  
\end{eqnarray*}
where the last line follows from \Refb{EQ4a}.  Integrating and using $F(\pi/2) = \frac{\pi}{2}$, gives
\[
F(\theta) = (\pi-\theta) + \sin(\theta)\cos(\theta)
\]
Combining the computations, the integral involving $A_{12}$ (resp.~$A_{11}$) contributes $\sin^3(\theta)$ (resp.~$(\pi-\theta) \cos(\theta) + \sin(\theta)\cos^2(\theta)$)
and so 
\begin{eqnarray*}
\label{EQ10}
 \int_D (x_1^2\cos(\theta)+x_1x_2\sin(\theta))e^{-\frac{x_1^2+x_2^2}{2}}dx_1dx_2& =& \sin(\theta) + (\pi-\theta) \cos(\theta) 
\end{eqnarray*}
\begin{rems}
(1) The key to the computation lies in evaluating the integral $\int_0^\infty \erfb{\frac{-x\cot(\theta)}{\sqrt{2}}} x^2 e^{-\frac{x^2}{2}}dx$. 
Using integration by parts, this is easily reduced to evaluating $\int_0^\infty \erfb{\frac{-x\cot(\theta)}{\sqrt{2}}} e^{-\frac{x^2}{2}}dx$ which can be done by the same method used above.\\
(2) For the bias free network, it sufficed to compute the integrals of $x_1x_2$ and $x_1^2$ over $D$ with respect to the normal distribution $p_{\mathcal{N}}$. 
When we add bias, there are five integrals to compute over $D$:  $x_1^2, x_1x_2, x_1, x_2$ and $1$. Here the more difficult integrals are those of $x_1^2$ and $1$ both of which use 
Owen's $T$-function (see sections~\ref{sec: 34}, \ref{sec: 35}). 
\rend
\end{rems}

\subsection{Evaluation of the defining integrals for $f(\ww,\vv)$: I}\label{sec: 22}
Our approach will be to represent $\int_D$ as the double integral
\[
\int_{\bgamma}^\infty \left(\int_{\widetilde{x}_2}^\infty F (x_1,x_2) e^{-\frac{x_2^2}{2}}dx_2\right)e^{-\frac{x_1^2}{2}}dx_1,
\]
where $\widetilde{x}_2 = \bbeta\, \csc(\theta) - x_1 \cot(\theta)$ (see Figure~\ref{D1}), $\theta \in (0,\pi)$, and
\[
F(x_1,x_2) = (x_1\cos(\theta)+x_2\sin(\theta)-\bbeta)(x_1 - \bgamma).
\]
In order to compute the double integral, it suffices, by linearity of the integral, to compute the integral over $D$ for each of the polynomial functions $x_1^2, x_1x_2, x_1, x_2,1$ since
\[
F(x_1,x_2) = \cos(\theta)x_1^2 + \sin(\theta)x_1x_2 - (\bgamma\cos(\theta) + \bbeta)x_1 -\bgamma\sin(\theta)x_2 + \bbeta \bgamma 
\]

\begin{exam}
If $\theta = \pi/2$, the integral of $F$ is easily computed since 
$
\int_D F = \int_{\bgamma}^\infty \left(\int_{\bbeta}^\infty [x_1x_2-\bbeta x_1-\bgamma x_2 + \bgamma\bbeta] e^{-\frac{x_2^2}{2}}dx_2\right)e^{-\frac{x_1^2}{2}}dx_1
$
and  
\begin{enumerate}
\item $\int_D x_1x_2 = e^{-\frac{\bbeta^2+\bgamma^2}{2}}$.
\item $\int_D x_1 = \sqrt{\frac{\pi}{2}}e^{-\frac{\bgamma^2}{2}}\erfcb{\frac{\bbeta}{\sqrt{2}}}$.
\item $\int_D x_2 = \sqrt{\frac{\pi}{2}}e^{-\frac{\bbeta^2}{2}}\erfcb{\frac{\bgamma}{\sqrt{2}}}$.
\item $\int_D 1 = \frac{\pi}{2}\erfcb{\frac{\bgamma}{\sqrt{2}}}\erfcb{\frac{\bbeta}{\sqrt{2}}}$.
\end{enumerate}
Hence  if $\theta = \pi/2$, $f(\ww,\vv) = \frac{\|\ww\|\|\vv\|}{2\pi} \int_D F$ is given by
\begin{multline*}
\begin{aligned}
f(\ww,\vv) =\;\; & \frac{\|\ww\|\|\vv\|\bbeta \bgamma}{4}\erfcb{\frac{\bgamma}{\sqrt{2}}}\erfcb{\frac{\bbeta}{\sqrt{2}}}\\
& - \frac{\|\ww\|\|\vv\|}{2\sqrt{2\pi}}\left( \bbeta e^{-\frac{\bgamma^2}{2}}\erfcb{\frac{\bbeta}{\sqrt{2}}}
+ \gamma e^{-\frac{\bbeta^2}{2}}\erfcb{\frac{\bgamma}{\sqrt{2}}}\right)\\
&+\frac{\|\ww\|\|\vv\|}{2\pi}e^{-\frac{\bbeta^2+\bgamma^2}{2}}
\end{aligned}
\end{multline*}
The symmetry $f(\ww,\vv) = f(\vv,\ww)$ is reflected in the $(\bbeta,\bgamma) \leftrightarrow (\bgamma,\bbeta)$ invariance of the formula. 
\examend
\end{exam}

In the remainder of this section $\int_D$ is computed for the integrands $x_2,x_1,x_1x_2$. Results are given in terms of standard functions: the exponential and error function. 
Computations are valid for $\theta \in (0,\pi)$.  The limiting cases $\theta\arr 0+,\pi-$ are given only for the integrand $x_1$ 
since the contribution to the integral of $F$ from the
the terms $x_1x_2\sin(\theta), -\bgamma x_2\sin(\theta)$ is zero in the limit on account of the factor $\sin(\theta)$. 

Throughout, we set $- x_1 \cot(\theta)+\bbeta \csc(\theta) = \widetilde{x}_2(x_1) = \widetilde{x}_2$ ($x_1$ is omitted if clear from the context).
For future reference,  we record the limits of $\widetilde{x}_2(x_1)$ as $\theta \arr 0+,\pi-$.
\[
\lim_{\theta \arr 0+} \widetilde{x}_2(x_1) = 
\begin{cases}
&+\infty,\;  x_1 < \bbeta \\
&0,\;\;\quad  x_1 = \bbeta \\
&-\infty,\;  x_1 > \bbeta 
\end{cases}
\qquad\lim_{\theta \arr \pi-} \widetilde{x}_2(x_1) = 
\begin{cases}
&+\infty,\;  x_1 > -\bbeta \\
&0,\;  x_1 = -\bbeta \\
&-\infty,\;  x_1 < -\bbeta 
\end{cases}
\]
We frequently use the abbreviations $A = -\cot(\theta)$ and $B = \csc(\theta)$. 
\begin{lemma}\label{LEM: X2}
\[
\int_D x_2 = \sqrt{\frac{\pi}{2}}e^{-\frac{\bbeta^2}{2}}\sin(\theta)\erfcb{\frac{\bgamma\csc(\theta) -\bbeta\cot(\theta)}{\sqrt{2}}}
\]
\end{lemma}
\begin{proof}  $\int_D x_2 = \int_{\bgamma}^\infty \left(\int_{\widetilde{x}_2}^\infty x_2 e^{-\frac{x_2^2}{2}}dx_2\right)e^{-\frac{x_1^2}{2}}dx_1$.
Clearly,  
\begin{eqnarray*}
\int_{\widetilde{x}_2}^\infty x_2 e^{-\frac{x_2^2}{2}}dx_2& =& e^{-\frac{\widetilde{x}_2^2}{2}} 
= e^{-\frac{(x_1\cot(\theta)-\bbeta\csc(\theta))^2}{2}}
\end{eqnarray*}
Hence 
\begin{equation*}
\int_D x_2 = \int_{\bgamma}^\infty e^{-\frac{(x_1\cot(\theta)-\bbeta\csc(\theta))^2}{2}}e^{-\frac{x_1^2}{2}}dx_1 = \int_{\bgamma}^\infty e^{-(Bx_1+A\bbeta)^2/2}e^{-\frac{\bbeta^2}{2}}dx_1
\end{equation*}
The result follows from \Refb{EQ2b} since $\csc(\theta)=B > 0$ on $(0,\pi)$. \end{proof}
\begin{lemma}\label{LEM: X1}
\mbox{ } \\
\vspace*{-0.15in}

\begin{enumerate}
\item If $\theta \in (0,\pi)$, then
\begin{multline*}
\begin{aligned}
\hspace*{0.35in}\int_{D_\theta} x_1 =&\sqrt{\frac{\pi}{2}}e^{-\frac{\bgamma^2}{2}}\erfcb{\frac{-\bgamma\cot(\theta) + \bbeta\csc(\theta)}{\sqrt{2}}} \\  
&+ \sqrt{\frac{\pi}{2}}e^{-\frac{\bbeta^2}{2}}\cos(\theta) \erfcb{\frac{\bgamma\,\csc(\theta)-\bbeta\cot(\theta)}{\sqrt{2}}}
\end{aligned}
\end{multline*}
\item Set $M = \max \{\bgamma,\bbeta\}$. If $\theta = 0$, then $D_0 = [M,\infty) \times \real$ and
\[
\int_{D_0} x_1 = \lim_{\theta \arr 0+} \int_{D_{\theta}} x_1 = \sqrt{2\pi}e^{-\frac{M^2}{2}} 
\]
\item If $\theta = \pi$ and $\bgamma \le -\bbeta$, then $D_\pi = [\bgamma,-\bbeta] \times \real$ and 
\[
\int_{D_\pi} x_1 = \lim_{\theta \arr \pi-} \int_{D_{\theta}} x_1 =  \sqrt{2\pi} \big[e^{-\frac{\bgamma^2}{2}} -e^{-\frac{\bbeta^2}{2}}\big]
\]
If  $\theta = \pi$ and $\bgamma \ge -\bbeta$, then $\int_{D_\pi} x_1 = \lim_{\theta \arr \pi-} \int_{D_{\theta}} x_1 = 0$.
\end{enumerate}
For (2) (resp.~(3)) see also Lemma~\ref{lem: zero} (resp.~Lemma~\ref{lem: pi}).
\end{lemma}
\begin{proof} (1)  $\int_D x_1 = \int_{\bgamma}^\infty \left(\int_{\widetilde{x}_2}^\infty e^{-\frac{x_2^2}{2}}dx_2\right)x_1e^{-\frac{x_1^2}{2}}dx_1$. 
By~\Refb{EQ2b},
\begin{eqnarray*}
\int_{\widetilde{x}_2}^\infty e^{-\frac{x_2^2}{2}}dx_2 & = & \sqrt{\frac{\pi}{2}}\big[1 - \erfb{\frac{\widetilde{x}_2}{\sqrt{2}}}\big]\\
& = & \sqrt{\frac{\pi}{2}}\left[1 - \erfb{\frac{-x_1\cot(\theta) + \bbeta\csc(\theta)}{\sqrt{2}}}\right]
\end{eqnarray*}
It remains to compute
\begin{eqnarray*}
\sqrt{\frac{\pi}{2}} \int^\infty_{\bgamma} x_1\left[1 - \erfb{\frac{-x_1\cot(\theta) + \bbeta\csc(\theta)}{\sqrt{2}}}\right]e^{-\frac{x_1^2}{2}}dx_1 &=&\sqrt{\frac{\pi}{2}}e^{-\frac{\bgamma^2}{2}} -\\
\hspace*{0.3in} \sqrt{\frac{\pi}{2}} \int_{\bgamma}^\infty x_1 e^{-\frac{x_1^2}{2}} \erfb{\frac{-x_1\cot(\theta) + \bbeta\csc(\theta)}{\sqrt{2}}}dx_1&\defoo & I_1-I_2,
\end{eqnarray*}
where $I_1 = \sqrt{\frac{\pi}{2}}e^{-\frac{\bgamma^2}{2}}$.
Integrating by parts  and using \Refb{EQ1a}.
\begin{eqnarray*}
I_2& =& \sqrt{\frac{\pi}{2}} e^{-\frac{\bgamma^2}{2}}\erfb{\frac{A\bgamma + B\bbeta}{\sqrt{2}}}   -
 \cot(\theta)\int_{\bgamma}^\infty e^{-\frac{x_1^2}{2}} e^{-\frac{(Ax_1 + B\bbeta)^2}{2}}dx_1
\end{eqnarray*} 
Now $ e^{-\frac{x_1^2}{2}} e^{-\frac{(Ax_1 + B\bbeta)^2}{2}} = e^{-(Bx_1+A\bbeta)^2/2}e^{-\frac{\bbeta^2}{2}} $ and so,
by \Refb{EQ2b}, with $a = \csc(\theta)/\sqrt{2}>0$ on $(0,\pi)$ and $b = -\cot(\theta)/\sqrt{2}$,
\[
\hspace*{-0.3in}I_2  = \sqrt{\frac{\pi}{2}}\left[e^{\frac{-\bgamma^2}{2}}\erfb{\frac{A\bgamma + B\bbeta}{\sqrt{2}}}  - e^{-\frac{\bbeta^2}{2}}\cos(\theta) \erfcb{\frac{B\bgamma+A\bbeta}{\sqrt{2}}}\right]
\]
and the result follows.   \\
(2) For the evaluation of $\lim_{\theta \arr 0+}\int_{D_\theta} x_1$, the main step is to show that if $\bgamma < \bbeta$, then 
$\lim_{\theta \arr 0+}\int_{D_\theta \cap([\bgamma,\bbeta]\times\real)} x_1 = 0$. Referring to Figure~\ref{fig: intR}, $L_\theta$ is the line
defining the lower boundary of $D_\theta$ and is given by the equation $x_1 \cos(\theta) + x_2\sin(\theta) = \bbeta$. The line $L_\theta$
intersects $x_1 = \bbeta$ at $X_2(\theta) = X_2 = \bbeta(\cot(\theta) - \csc(\theta))$. Hence 
$X_2(\theta)$ is monotone function of $\theta$ and $\lim_{\theta \arr 0+} X_2(\theta) = 0$. In the figure, $\bbeta > 0$ and so $X_2(\theta) < 0$ and $X_2(\theta)$ is strictly monotone increasing.  On the other hand, $L_\theta$ meets 
$x_1 = \bgamma$ at the point
$Y_2(\theta) = \bbeta \csc(\theta) - \bgamma \cot(\theta)$ and since $\bbeta\csc(\theta) >  \bgamma \cot(\theta)$ on $(0,\pi/2)$, $\lim_{\theta \arr 0+} Y_2(\theta) = +\infty$.  

There are two approaches. (1) A direct proof that if $\bbeta \ge \bgamma$ then
\[
\lim_{\theta \arr 0+}\int_{\bgamma}^{\bbeta} \left(\int_{\widetilde{x}_2}^\infty e^{-\frac{x_2^2}{2}}dx_2\right)x_1e^{-\frac{x_1^2}{2}}dx_1 = 0
\]
Alternatively, (2) A Proof that $\lim_{\theta \arr 0+}\int_{\bgamma}^{\bbeta}\erfcb{\frac{-x_1\cot(\theta) + \bbeta\csc(\theta)}{\sqrt{2}}} = 0$ if $\bbeta > \bgamma$.
We use the first method here; later we use the second method.

\begin{figure}[h]
\centering
\includegraphics[width=0.7\textwidth]{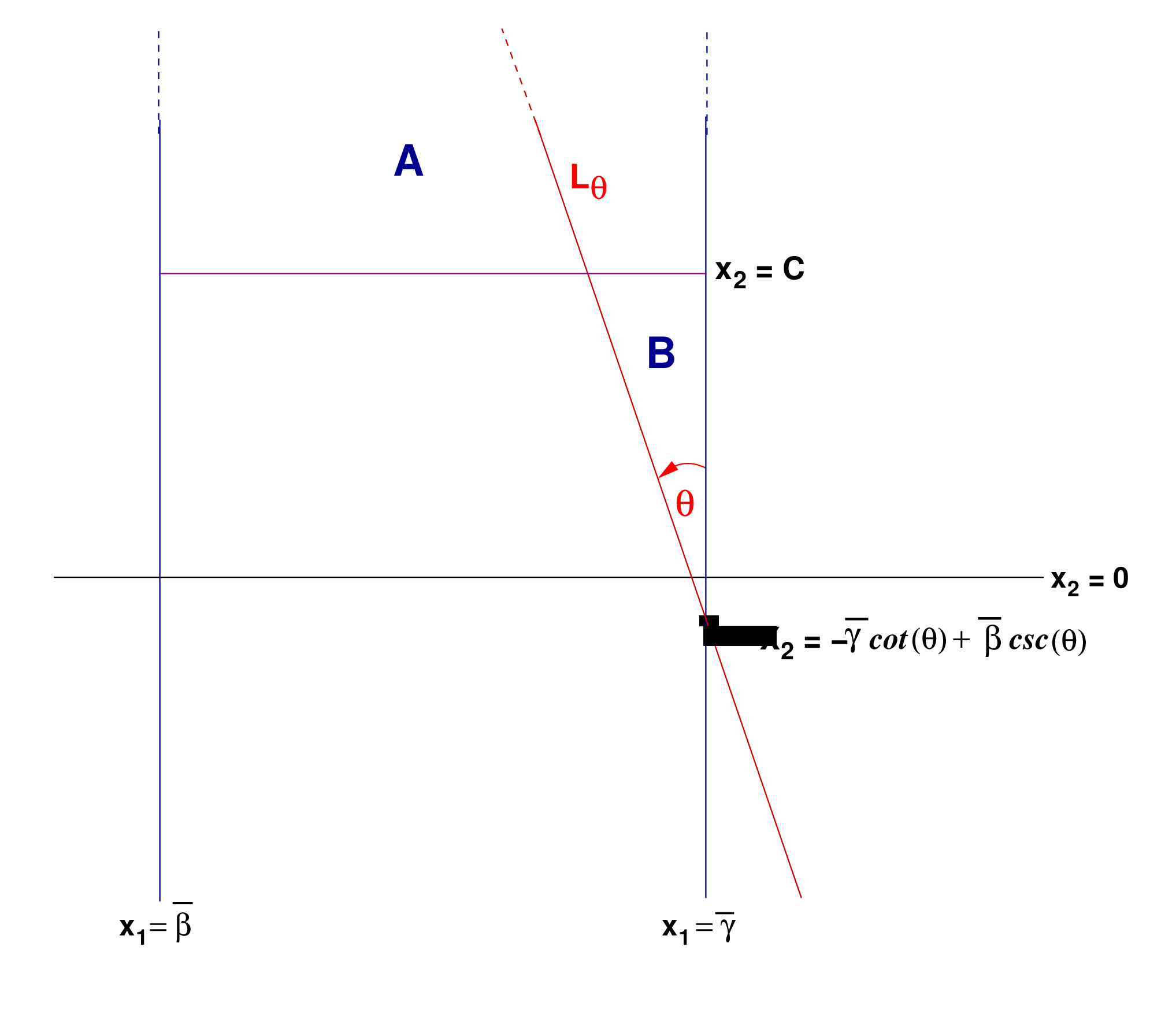}
\caption{The region of integration on the strip $[\bgamma,\bbeta]\times \real$, $\theta$ near zero and $\bbeta < \bgamma$.}
\label{fig: intR}
\end{figure}

For the first approach, it is easy to see that the integral over the region $A$ in  Figure~\ref{fig: intR} is bounded by $k e^{-C^2/2}$, where $k$ is a constant independent of $C$, and
the integral over the triangular region $B$ converges to zero as $\theta \arr 0+$. Formally, given $\varepsilon > 0$, choose $C$ so that the integral over $A$ is bounded by $\varepsilon/2$, then choose $\kappa(\varepsilon)  > 0$, so that 
the integral over the triangular region $B$ is bounded by $\varepsilon/2$ for $\theta \in (0,\kappa(\varepsilon))$. 

Alternatively, $\lim_{\theta\arr 0+} \erfcb{\frac{-x_1\cot(\theta) + \bbeta\csc(\theta)}{\sqrt{2}}} = 0$ uniformly for $x_1 \in [\bgamma,\bbeta-\varepsilon]$, $\varepsilon > 0$, and the error arising from the integral
over $[\bbeta-\varepsilon,\bbeta]$ is bounded by a constant multiple of $\varepsilon$. 

Either approach gives the explicit formula for $D_0$. \\
(3) The proof is similar to that of (2) and is omitted. 
\end{proof}

Neither of the terms $x_1$ or $x_2$ occurs in the bias free case. However,  $x_1 x_2$ does occur if there is zero bias. 
\begin{lemma}\label{LEM: X1X2}
\begin{multline}
\begin{aligned}
\int_D x_1x_2  =&e^{-\frac{\bbeta^2}{2}}\sin(\theta)\sqrt{\frac{\pi}{2}}\bbeta \cos(\theta)\erfcb{\frac{\bgamma\csc(\theta)-\bbeta \cot(\theta)}{\sqrt{2}}}\\
& +e^{-\frac{\bbeta^2}{2}}\sin^2(\theta)e^{-\frac{(\bgamma\csc(\theta) - \bbeta \cot(\theta))^2}{2}}
\end{aligned}
\end{multline}
\end{lemma}
\begin{proof} We have
\[
\int_D x_1x_2 = \int_{\bgamma}^\infty \left(\int_{\widetilde{x}_2}^\infty x_2 e^{-\frac{x_2^2}{2}}dx_2\right)x_1 e^{-\frac{x_1^2}{2}}dx_1
\]
As in Lemma~\ref{LEM: X2},
\[
\int_{\widetilde{x}_2}^\infty x_2 e^{-\frac{x_2^2}{2}}dx_2 =  e^{-\frac{(-x_1\cot(\theta)+\bbeta\csc(\theta))^2}{2}}
\]
Hence
\[ 
\int_D x_1 x_2  =  \int_{\bgamma}^\infty x_1 e^{-\frac{x_1^2}{2}}e^{-\frac{(-x_1\cot(\theta)+\bbeta\csc(\theta))^2}{2}} dx_1 
\]
Now
\[
-\frac{x_1^2}{2} - \frac{(-x_1\cot(\theta)+\bbeta\csc(\theta))^2}{2} = -\frac{\beta^2}{2} - \frac{(x_1\csc(\theta) - \bbeta \cot(\theta))^2}{2}
\]
and so 
\[
\int_D x_1x_2 =e^{-\frac{\bbeta^2}{2}} \int_{\bgamma}^\infty x_1 e^{-\frac{(x_1\csc(\theta) - \bbeta \cot(\theta))^2}{2}} dx_1
\]
Apply Example~\ref{Ex: EQ3} with $a = \frac{\csc(\theta)}{\sqrt{2}}$ and $b = -\frac{\bbeta\cot(\theta)}{\sqrt{2}}$ to deduce that
\[
\int_D x_1x_2 =e^{-\frac{\bbeta^2}{2}}\sin(\theta)\left[\sqrt{\frac{\pi}{2}}\bbeta \cos(\theta)\erfcb{\frac{B\bgamma +A\bbeta}{\sqrt{2}}} +\sin(\theta)e^{-\frac{(B\bgamma + A\bbeta)^2}{2}}\right]
\]
since $\csc(\theta) > 0$ on $(0,\pi)$.  

On account of the factor $\sin(\theta)$, there is no need to analyze the limits as $\theta \arr 0+,\pi-$.
\end{proof}
\subsection{Evaluation of the defining integrals for $f(\ww,\vv)$: II}
In this section, we start the analysis of $\int_D 1$ and $\int_D x_1^2$ for $\theta \in (0,\pi)$ and find the limits of these integrals as $\theta \arr 0+,\pi-$.  

The methods of the previous section show that for $\theta\in (0,\pi)$
\begin{eqnarray*}
\int_D 1 & = & \int_{\bgamma}^\infty \left(\int^\infty_{\widetilde{x}_2} 1e^{-\frac{x_2^2}{2}}dx_2\right)e^{-\frac{x_1^2}{2}}dx_1 \\
& = & \sqrt{\frac{\pi}{2}}\int_{\bgamma}^\infty \left(1 - \erfb{\frac{-x_1\cot(\theta) +\bbeta \csc(\theta)}{\sqrt{2}}}\right)e^{-\frac{x_1^2}{2}}dx_1 \\
& = & \frac{\pi}{2}\erfcb{\frac{\bgamma}{\sqrt{2}}}- \sqrt{\frac{\pi}{2}}\int_{\bgamma}^\infty e^{-\frac{x_1^2}{2}}\erfb{\frac{-x_1\cot(\theta) +\bbeta \csc(\theta)}{\sqrt{2}}}dx_1
\end{eqnarray*}

For the integral of $x_1^2$ over $D$ we have
\begin{eqnarray*}
\int_D x_1^2 & = & \int_{\bgamma}^\infty \left(\int^\infty_{\widetilde{x}_2} 1e^{-\frac{x_2^2}{2}}dx_2\right)x_1^2e^{-\frac{x_1^2}{2}}dx_1 \\
& = & \sqrt{\frac{\pi}{2}}\int_{\bgamma}^\infty \left(1 - \erfb{\frac{-x_1\cot(\theta) +\bbeta \csc(\theta)}{\sqrt{2}}}\right)x_1^2e^{-\frac{x_1^2}{2}}dx_1 \\
& = & \sqrt{\frac{\pi}{2}}\bgamma e^{-\frac{\bgamma^2}{2}} + \frac{\pi}{2}\erfcb{\frac{\bgamma}{\sqrt{2}}} -\\
&& \sqrt{\frac{\pi}{2}} \int_{\bgamma}^\infty x_1^2e^{-\frac{x_1^2}{2}}\erfb{\frac{-x_1\cot(\theta) +\bbeta \csc(\theta)}{\sqrt{2}}} dx_1
\end{eqnarray*}

Using~\Refb{EQ5a}, we have
 
\begin{multline*}
\sqrt{\frac{\pi}{2}}\int_{\bgamma}^\infty x_1^2e^{-\frac{x_1^2}{2}} \erfb{\frac{-x_1\cot(\theta) +\bbeta \csc(\theta)}{\sqrt{2}}} dx_1  =  \\
\begin{aligned}
&\sqrt{\frac{\pi}{2}}\bgamma e^{-\frac{\bgamma^2}{2}}\erfb{\frac{-\bgamma\cot(\theta)+\bbeta \csc(\theta)}{\sqrt{2}}} \\
&-\sqrt{\frac{\pi}{2}}\bbeta\cos^2(\theta)e^{-\frac{\bbeta^2}{2}}\left[1 -\erfb{\frac{\bgamma\csc(\theta)-\bbeta\cot(\theta)}{\sqrt{2}}}\right] \\
&-\sin(\theta)\cos(\theta)e^{-\frac{\bbeta^2}{2}}e^{-\frac{(\bgamma\csc(\theta)-\bbeta \cot(\theta))^2}{2}} \\
&+ \sqrt{\frac{\pi}{2}}\int_{\bgamma}^\infty e^{-\frac{x_1^2}{2}}\erfb{\frac{-x_1\cot(\theta) + \bbeta\csc(\theta)}{\sqrt{2}}}dx_1
\end{aligned}
\end{multline*}
Hence
\begin{eqnarray*}
\int_D x_1^2& =& \sqrt{\frac{\pi}{2}}\bgamma e^{-\frac{\bgamma^2}{2}}\erfcb{\frac{-\bgamma\cot(\theta)+\bbeta \csc(\theta)}{\sqrt{2}}} +\frac{\pi}{2}\erfcb{\frac{\bgamma}{\sqrt{2}}}\\
&& + \sqrt{\frac{\pi}{2}}\bbeta\cos^2(\theta)e^{-\frac{\bbeta^2}{2}}\erfcb{\frac{\bgamma\csc(\theta)-\bbeta\cot(\theta)}{\sqrt{2}}} \\
&& + \sin(\theta)\cos(\theta)e^{-\frac{\bbeta^2}{2}}e^{-\frac{(\bgamma\csc(\theta)-\bbeta \cot(\theta))^2}{2}} \\
&& - \sqrt{\frac{\pi}{2}} \int_{\bgamma}^\infty e^{-\frac{x_1^2}{2}}\erfb{\frac{-x_1\cot(\theta) +\bbeta \csc(\theta)}{\sqrt{2}}} dx_1
\end{eqnarray*}

The integral 
\begin{equation}\label{EQ: integral} 
\sqrt{\frac{\pi}{2}}\int_{\bgamma}^\infty e^{-\frac{t^2}{2}} \erfb{\frac{-t\cot(\theta) +\bbeta \csc(\theta)}{\sqrt{2}}} dt
\end{equation}
occurs in both formulas obtained above for $\int_D 1$ and $\int_D x_1^2$. 
In general, $\int_b^\infty  e^{-\frac{t^2}{2}}\erfb{\frac{at+b}{\sqrt{2}}} dt$ cannot be expressed in terms of standard elementary functions. However, there are some exceptions: 
\begin{eqnarray}
\label{genT1}
\sqrt{\frac{\pi}{2}}\int_0^\infty e^{-\frac{t^2}{2}}\erfb{\frac{at}{\sqrt{2}}}dt&=& \tan^{-1}(a) \\
\label{genT2}
\sqrt{\frac{\pi}{2}}\int_{-\infty}^\infty e^{-\frac{t^2}{2}}\erfb{\frac{at+b}{\sqrt{2}}}dt&=& \pi\,\erfb{\frac{b}{\sqrt{2(a^2+1)}}}  
\end{eqnarray}
\begin{multline} \label{genT3}
\lim_{\theta\arr 0+} \sqrt{\frac{\pi}{2}} \int_{\bbeta}^\infty e^{-\frac{t^2}{2}}\erfb{\frac{-t\cot(\theta)+\bbeta \csc(\theta)}{\sqrt{2}}}dt 
 = -\frac{\pi}{2}\erfcb{\frac{\bbeta}{\sqrt{2}}}
\end{multline}
If $a = -\cot(\theta)$, $b = \bbeta \csc(\theta)$, then \Refb{genT1}
equals $\theta - \frac{\pi}{2}$ and \Refb{genT2} is equal to $\pi\erfb{\frac{\bbeta}{\sqrt{2}}}$. 

Notwithstanding the lack of a simple representation of~\Refb{EQ: integral}, we can evaluate the limit of the integral as $\theta \arr 0+,\pi-$.

\begin{lemma}\label{lem: ITlim}
Set $I_\theta = \sqrt{\frac{\pi}{2}}\int_{\bgamma}^\infty e^{-\frac{t^2}{2}} \erfb{\frac{-t\cot(\theta) +\bbeta \csc(\theta)}{\sqrt{2}}}dt$.
\begin{enumerate}
\item The limit as $\theta \arr 0+$ 
\[
\lim_{\theta\arr 0+} I_\theta = 
\begin{cases}
& = \frac{\pi}{2}\left[ 2 \erfb{\frac{\bbeta}{\sqrt{2}}} - \erfb{\frac{\bgamma}{\sqrt{2}}} -1\right], \;\text{if } \bbeta \ge \bgamma. \\
& = -\frac{\pi}{2} \erfcb{\frac{\bgamma}{\sqrt{2}}}, \; \text{if } \bbeta \le  \bgamma. 
\end{cases}
\]
\item The limit as $\theta \arr \pi-$ 
\[
\lim_{\theta\arr \pi-} I_\theta = 
\begin{cases}
& = \frac{\pi}{2}\left[ 2 \erfb{\frac{\bbeta}{\sqrt{2}}} + \erfb{\frac{\bgamma}{\sqrt{2}}} +1\right], \;\text{if }  \bgamma + \bbeta \le 0. \\
& = \frac{\pi}{2} \erfcb{\frac{\bgamma}{\sqrt{2}}},\;\text{if } \bgamma + \bbeta \ge  0.
\end{cases}
\]
\end{enumerate}
\end{lemma}
\begin{proof} We prove (1); details for (2) are similar.  For $\theta \in (0,\pi/4]$, define
\begin{eqnarray*}
f_\theta(t)& =& \frac{-t\cot(\theta) +\bbeta \csc(\theta)}{\sqrt{2}}\\
& =& \frac{(\bbeta -t) \csc(\theta)}{\sqrt{2}} + \frac{t(\csc(\theta) - \cot(\theta))}{\sqrt{2}} , \; t \in \real.
\end{eqnarray*}
Observe that 
\[
0 < t \frac{(\csc(\theta) - \cot(\theta))}{\sqrt{2}} = t\frac{\tan(\theta/2)}{\sqrt{2}} \le C t \theta, \theta \in (0,\pi/4],
\]
where $C>0$ is independent of $t$ and $\theta \in (0,\pi/4]$. Consequently, on any closed interval $J$ not containing $\bbeta$, we can always choose $\theta_0 \in (0,\pi/4]$ such that
$f_\theta$ is strictly monotone decreasing on $J$ for all $\theta \in (0,\theta_0]$.

Suppose that $\bbeta \ge \bgamma$. Given $\varepsilon > 0$, choose $\delta_\varepsilon > 0$ and $\theta_0 \in (0,\pi/4]$ so that $$\sqrt{\frac{\pi}{2}}\int_{[\bbeta-\delta_\varepsilon,\bbeta]} e^{-\frac{t^2}{2}} dt < \varepsilon/2$$
and $f_\theta(t)$ is strictly monotone decreasing on $[\bgamma,\bbeta-\delta_\varepsilon]$ for all $\theta \in (0,\theta_0]$. Hence the contribution to $I_\theta$ from the integral over  
$[\bbeta-\delta_\varepsilon,\bbeta]$ is less than $ \varepsilon/2$. Using the monotonicity of $f_\theta$ on $[\bgamma,\bbeta-\delta_\varepsilon]$, choose $\theta_\varepsilon \in (0,\theta_0]$ so that
$\sqrt{\frac{\pi}{2}}|f_\theta(t) - 1| < \varepsilon/2|\bbeta-\bgamma|$ for $\theta \in (0,\theta_\varepsilon]$. Hence for all $\theta\in (0,\theta_\varepsilon)$, 
\[
\sqrt{\frac{\pi}{2}}\left|\int_{[\bgamma,\bbeta]}e^{-\frac{t^2}{2}} \erfb{\frac{-t\cot(\theta) +\bbeta \csc(\theta)}{\sqrt{2}}}dt  - \int_{[\bgamma,\bbeta]} e^{-\frac{t^2}{2}} dt\right| < \varepsilon,
\]
proving the first part of (1). The proof of the second part of (1) is similar but simpler.
\end{proof}
\begin{rem}
Let $I_0 = I_0(\bbeta,\bgamma)$ denote the limiting value $\lim_{\theta \arr 0+} I_\theta$ and similarly define $I_\pi$, then
$I_\theta$ is a continuous function of $(\bbeta,\bgamma,\theta)\in\real^2\times [0,\pi]$. \rend
\end{rem}

\begin{lemma}\label{lem: spcases}
If $\theta \in \{0,\pi\}$, the integrals $\int_{D_\theta} 1, x_1^2$ can be explicitly computed and 
\begin{enumerate}
\item 
\begin{eqnarray*}
\int_{D_0} 1&=&\lim_{\theta \arr 0+} \int_{D_\theta} 1 = \pi\,\erfcb{\frac{M}{\sqrt{2}}}\\
\int_{D_\pi} 1&=&\lim_{\theta \arr \pi-} \int_{D_\theta} 1 = 
\begin{cases}
&-\pi \left[\erfb{\frac{\bbeta}{\sqrt{2}}}+\erfb{\frac{\bgamma}{\sqrt{2}}}\right],\;\text{if } \bgamma+\bbeta \le 0\\
& 0,\; \text{if } \bgamma+\bbeta \ge 0
\end{cases}
\end{eqnarray*}
\item
\begin{eqnarray*}
\int_{D_0} x_1^2&=&\lim_{\theta \arr 0+} \int_{D_\theta} x_1^2 = \sqrt{2\pi} Me^{-\frac{M^2}{2}} +\pi\,\erfcb{\frac{M}{\sqrt{2}}}\\
\int_{D_\pi} x_1^2 &=&\lim_{\theta \arr \pi-} \int_{D_\theta} x_1^2 = \begin{cases}
& \sqrt{2\pi}\big(\bgamma e^{-\frac{\bgamma^2}{2}}+\bbeta e^{-\frac{\bbeta^2}{2}}\big)\\
& -\pi \left[\erfb{\frac{\bbeta}{\sqrt{2}}}+\erfb{\frac{\bgamma}{\sqrt{2}}}\right],\;\text{if } \bgamma+\bbeta \le 0\\
& 0, \; \text{if } \bgamma+\bbeta \ge 0
\end{cases}
\end{eqnarray*}
where $M = \max\{\bbeta,\bgamma\}$.
\end{enumerate}
\end{lemma}
\begin{proof}
The proof follows along the lines of the proof of Lemma~\ref{LEM: X1} starting from the integral \Refb{EQ: integral}.
\end{proof}
\begin{rem}
Using the results of Lemmas~\ref{LEM: X1} and \ref{lem: spcases}, we recover the results of Lemmas~\ref{lem: zero} and \ref{lem: pi}. \rend
\end{rem}

\section{The Owen $T$-function}\label{sec: 34}   
The T-function $T(h,a)$ was introduced by Owen in 1956~\cite{Owen1956} and is defined by
\[
T(h,a) = \frac{1}{2\pi} \int_0^a \frac{e^{-\frac{h^2(1+x^2)}{2}}}{1+x^2} dx,\;\;h,a \in \real
\]
The T-function gives the probability of the event $X >0 \wedge 0 < Y < aX$ where $X$ and $Y$ are independent standard normal variables.
Let $G(x)$ denote the CDF $\frac{1}{\sqrt{2\pi}} \int_{-\infty}^xe^{-\frac{t^2}{2}}dt $ for the standard normal density.  Some important properties of the $T$-function
from Owen~\cite[Table II]{Owen1980}: 
\[
\begin{aligned}
&T(h,a)=T(-h,a),\; T(h.-a) = -T(h,a),\;
T(0,a) =\frac{1}{2\pi}\tan^{-1}(a) \\
&T(h,a) + T(ha,a^{-1}) = \begin{cases}
& G(h)/2+G(ah)/2 -G(h)G(ah),\; \text{if } a \ge 0 \\
& G(h)/2+G(ah)/2 -G(h)G(ah) - 1/4,\; \text{if } a < 0
\end{cases}
\end{aligned} 
\]
\begin{eqnarray}\label{EQ: OwenEval1}
T(h,1) & = &G(h)(1-G(h))/2=\frac{1}{8} \left(1 - \erfb{\frac{h}{\sqrt{2}}}^2\right)  \\
\label{EQ: OwenEval}
T(h,\infty)&=&(1-G(|h|)/2 = \frac{1}{4}\erfcb{\frac{|h|}{\sqrt{2}}}
\end{eqnarray}

In~\cite{Owen1980}, there are tables of integrals of functions related to the standard normal density, both univariate and multivariate. 
See also Appendix A (\ref{EQ0b}---\ref{EQ0c}) for the relation between the cumulative density (CDF) and error function. 

The result from~\cite{Owen1980} of special interest here is the indefinite integral from Table 1, formula 10,010.3 (p.~403). 
\begin{multline}
\label{EQT}
\int G'(x) G(ax +b) dx = T\left(x,\frac{b}{x\sqrt{1+a^2}}\right)\\ 
\begin{aligned}
&+ T\left(\frac{b}{\sqrt{1+a^2}},\frac{x\sqrt{1+a^2}}{b}\right)  -T\left(x,\frac{ax +b}{x}\right)\\
& - T\left(\frac{b}{\sqrt{1+a^2}},\frac{ab+ x(1+a^2)}{b}\right) + G(x)G\left(\frac{b}{\sqrt{1+a^2}}\right)
\end{aligned}
\end{multline}
Note that for consistency with our notation, we have switched $a$ and $b$ in~\cite{Owen1980}. That is, we write $ax+b$ rather than the $a + bx$ used in~\cite{Owen1980}.

In our context, it is more natural to express results in terms of error functions as opposed to the CDF $G(x)$.
Indeed, this emphasis leads to a more attractive and concise formula that better fits our applications. 
We refer to Appendix~\ref{AppendixC} for more details on Owen's formula 
and how to derive the formula described below \emph{directly} from~\Refb{EQT}.  

\subsection{The Generalized Owen $T$-function~\cite{Prezmo}}\label{sec: 35}
The generalized $T$-function is the analytic function $\wTT$ on $\real^3$ defined by
\begin{multline}\label{EQ: Tgen}
\begin{aligned}
\wTT(h,a,b)  = &\frac{1}{2\sqrt{2\pi}}\int_h^\infty e^{-\frac{t^2}{2}}\erfb{\frac{at + b}{\sqrt{2}}}dt 
\end{aligned}
\end{multline}

For $\theta \in (0,\pi)$, the integral we need for the evaluation of $\int_{D_\theta} 1$ and $\int_{D_\theta} x_1^2$ has a simple relationship to the generalized $T$-function: 
\[
\sqrt{\frac{\pi}{2}}\int_{\bgamma}^\infty e^{-\frac{t^2}{2}}\erfb{\frac{-t\cot(\theta) + \bbeta\csc(\theta)}{\sqrt{2}}}dt = 2\pi \wTT(\bgamma,-\cot(\theta),\bbeta\csc(\theta))
\]
The definition of the generalized $T$-function appears in~\cite{Prezmo}\footnote{The symbol used is $T$ rather than $\wTT$ but $\wTT$ is preferred here to avoid possible confusion with
Owen's $T$-function.} where it is shown that
\begin{multline}\label{EQ: fullT}
\begin{aligned}
\hspace*{-0.13in}\wTT(h,a,b)  =  &\frac{1}{2\pi} \left[\tan^{-1}(a) -\tan^{-1}\left(a+\frac{b}{h}\right) - \tan^{-1}\left(a+\frac{h(1+a^2)}{b}\right)\right] \\
& + T\left(h,a+\frac{b}{h}\right) +T\left(\frac{b}{\sqrt{1+a^2}},a +\frac{h(1+a^2)}{b}\right) \\
& + \frac{1}{4}\erfb{\frac{b}{\sqrt{2(1+a^2)}}} 
\end{aligned}
\end{multline}
Four of the six functions on the right hand side are not continuous functions on all of $\real^3$---the continuous (analytic) exceptions being $\tan^{-1}(a)$ and  $\frac{1}{4}\erfb{\frac{b}{\sqrt{2(1+a^2)}}} $.
Consequently, \Refb{EQ: fullT}  is a decomposition of an analytic function of $(h,a,b)$ into a sum of six functions, four of which are not continuous, let alone analytic, on $\real^3$.
\begin{lemma}\label{lem: ex2}
Regard $\wTT(h,a,b)$ as given by \Refb{EQ: fullT}. 
\begin{enumerate}
\item $\lim_{b \arr 0} \wTT(h,a,b)= T(h,a)$, $h \ne 0$.
\item $\lim_{h \arr 0} \wTT(h,a,b)= \frac{1}{4}\erf{\frac{b}{\sqrt{2(1+a^2)}}}+ T\big(\frac{b}{\sqrt{1+a^2}},a\big)$, $b \ne 0$.
\end{enumerate}
In particular, $\wTT(0,a,b) = \frac{1}{4}\erf{\frac{b}{\sqrt{2(1+a^2)}}}+ T\big(\frac{b}{\sqrt{1+a^2}},a\big)$ and $\wTT(h,a,0) = T(h,a)$, for all $a,b,h \in \real$.
\end{lemma}
\begin{proof} We prove (1), the proof of (2) is similar.
\begin{eqnarray*}
\lim_{b \arr 0_+} \wTT(h,a,b)&=&\frac{1}{2\pi} \left(\tan^{-1}(a)-\tan^{-1}(a)-\frac{\pi}{2} \sgn(h)\right)\\
&& +\; T(h,a) + \frac{1}{4} \sgn(h) + 0 
= T(h,a)
\end{eqnarray*}
The same argument proves $\lim_{b \arr 0_-} \wTT(h,a,b) = T(h,a)$ and
the final statement follows by the continuity of $\wTT(h,a,b)$, 
\end{proof}
\begin{rems}
(1) The formula for $\wTT(h,a,0)$ (resp.~$\wTT(0,a,b)$) is written in terms of the CDF $G$ and is in \cite[1.010.1]{Owen1980} (resp.~\cite[10.010.6]{Owen1980}). The argument for 
$\lim_{b \arr 0+} \wTT(h,a,b)$ is in \cite{Prezmo}. The result for $\wTT(h,a,0)$ can be proved directly by differentiating the integral for $\wTT(h,a,0)$ with respect to $a$, then integrating with respect to 
$t$, and finally with respect to $a$ to recover the formula for $T(h,a)$. \\
(2) Using the continuity of $\wTT(h,a,b)$, it follows easily from  Lemma~\ref{lem: ex2} that $\wTT(0,a,0) = \frac{1}{2\pi}\tan^{-1}(a)$ (cf.~\Refb{genT1}).\\
(3) Notwithstanding Lemma~\ref{lem: ex2}, the appearance of discontinuous terms in~\Refb{EQ: fullT}
indicates the need for care in numerical simulations involving
the $T$-function as well as in analytical studies of bifurcation from zero bias. See also Remarks~\ref{rems: sub}(3). \rend
\end{rems}
\section{Evaluation of the defining integrals for $f(\ww,\vv)$: III}\label{Sec: Owen}   
For $\theta \in (0,\pi)$, 
\begin{equation}\label{EQ: defint}
\sqrt{\frac{\pi}{2}}\int_{\bgamma}^\infty e^{-\frac{t^2}{2}}\erfb{\frac{-t\cot(\theta) + \bbeta\csc(\theta)}{\sqrt{2}}}dt = 2\pi \wTT(\bgamma,-\cot(\theta),\bbeta\csc(\theta)),
\end{equation}
where the integral on the left occurs twice in our computations of $f(\ww,\vv)$.  Simplifying notation, define $H(\bbeta,\bgamma,\theta)$ by
\[
H(\bbeta,\bgamma,\theta) = 2\pi \wTT(\bgamma,-\cot(\theta),\bbeta\csc(\theta)), \; (\bbeta,\bgamma) \in \real^2, \; \theta \in (0.\pi).
\]

Applying the results from the previous section, we have
\begin{multline}\label{EQ: GENT}
\begin{aligned}
H(\bbeta,&\bgamma,\theta)=\\&+\tan^{-1}(-\cot(\theta))\\
&-\left[\tan^{-1}\left(-\cot(\theta) + \frac{\bbeta\csc(\theta)}{\bgamma}\right) + \tan^{-1}\left(-\cot(\theta) + \frac{\bgamma\csc(\theta)}{\bbeta}\right)\right] \\
&+2\pi\left[T\left(\bgamma,-\cot(\theta) + \frac{\bbeta\csc(\theta)}{\bgamma}\right)+T\left(\bbeta,-\cot(\theta) + \frac{\bgamma\csc(\theta)}{\bbeta}\right)\right]\\
&+\frac{\pi}{2}\erfb{\frac{\bbeta}{\sqrt{2}}} 
\end{aligned}
\end{multline}

\begin{rems}\label{rems: sub}
(1) The sum of the first three rows of the expression for $H$ are unchanged if we permute $\bgamma, \bbeta$. Although the last row does not have this symmetry,
the symmetry will reappear when we evaluate $\int_D 1, \int_D x_1^2$: the symmetry is a reflection of the trivial symmetry $f(\ww,\vv) = f(\vv,\ww)$.
A consequence of the symmetry is the identity
\begin{equation}\label{EQ: bgsymm}
H(\bbeta,\bgamma,\theta)+\frac{\pi}{2}\erfb{\frac{\bgamma}{\sqrt{2}}} = H(\bgamma,\bbeta,\theta)+\frac{\pi}{2}\erfb{\frac{\bbeta}{\sqrt{2}}}
\end{equation}
See also Remarks~\ref{rem: intdom}(3).\\
(2) Although $H$ is a continuous function of $(\bbeta,\bgamma,\theta)$ (since this is true for the defining integral \Refb{EQ: defint}), the terms 2---5 in 
the above decomposition of the generalized $T$-function in terms of inverse tangents and Owen's $T$-function may not be defined when $\bbeta\bgamma = 0$.  However, the approach of Lemma~\ref{lem: ex2} can be used to handle cases where only one of $\bbeta,\bgamma$ is zero. See also Proposition~\ref{prop: spec} below. \\
(3) The function $H$ is \emph{not} subanalytic on $\real^2 \times (0,\pi)$ even though it is analytic (that is, the graph is not subanalytic). It is, however, subanalytic on $\real^2 \times [a,b]$ for any closed interval $[a,b] \subset (0,\pi)$. 
Although this appears to be a significant barrier to the extension of the fractional power series results for bias free networks (\cite{ArjevaniField2020b}---\cite{ArjevaniField2022b}) to networks with bias, this turns 
out not to be the case. \rend 
\end{rems}
It is sometimes  possible to extract useful information about the inverse tangent terms in the above expression for $H$. First, an elementary lemma.
\begin{lemma}\label{lemDetails} 
(Notation of \Refb{EQ: GENT})
\begin{enumerate}
\item $\tan^{-1}(-\cot(\theta)) = \theta - \frac{\pi}{2},\;\theta \in [0,\pi]$.
\item If $\theta \in [0,\pi]$, then
\[
\tan^{-1}(-\cot(\theta)+\csc(\theta))=\frac{\theta}{2},\;\;  
\tan^{-1}(-\cot(\theta)-\csc(\theta))=\frac{\theta-\pi}{2}
\]
where the left hand side is defined as $\lim_{\theta \arr 0+}$ if $\theta = 0$ and $\lim_{\theta \arr \pi-}$ if $\theta = \pi$.
\end{enumerate}
\end{lemma}
\proof For (2) observe that $-\cot(\theta) + \csc(\theta) = \tan(\theta/2)$ and
$-\cot(\theta) - \csc(\theta) = -\cot(\theta/2) = \tan(\frac{\theta-\pi}{2})$. \qed
\begin{prop}\label{prop: spec}
\mbox{ }\\
\vspace*{-0.1in}
\begin{enumerate}
\item $H(\bbeta,0,\theta) = \frac{\pi}{2}\erf{\frac{\bbeta}{\sqrt{2}}}+ 2\pi T(\bbeta,-\cot(\theta))$.
\item $H(0,\bgamma,\theta) = 2\pi T(\bgamma,-\cot(\theta))$.
\item $H(\bbeta,\bbeta,\theta) = 4\pi T(\bbeta,\tan(\frac{\theta}{2})) - \frac{\pi}{2}\erfc{\frac{\bbeta}{\sqrt{2}}}$ and
\begin{equation}\label{EQ: bderiv} 
\frac{\partial H}{\partial \bbeta}(\bbeta,\bbeta,\theta) = \sqrt{\frac{\pi}{2}}e^{-\frac{\bbeta^2}{2}}\left[2\, \erfcb{\frac{\bbeta \big(-\cot(\theta)+\csc(\theta)\big)}{\sqrt{2}}} - 1\right] 
\end{equation}
\end{enumerate}
\end{prop}
\begin{proof} Statements (1,2) are immediate from Lemma~\ref{lem: ex2} and (3) is substitution in \Refb{EQ: GENT} and Lemma~\ref{lemDetails}. Finally~\Refb{EQ: bderiv}
follows by noting that
\[
\frac{\partial H}{\partial \bbeta} = -2\bbeta e^{-\frac{\bbeta^2}{2}} \int_0^{\tan(\theta/2)} e^{-\bbeta t)^2/2}dt + \sqrt{\frac{\pi}{2}}e^{-\frac{\bbeta^2}{2}}
\]
The integral is equal to $\bbeta^{-1} \int_0^{\bbeta \tan(\theta/2)} e^{-\frac{s^2}{2}} ds$ and the result follows since $\bbeta \tan(\theta/2) = \bbeta \big(-\cot(\theta)+\csc(\theta)\big)$ (proof of Lemma~\ref{lemDetails}(2)).
\end{proof}

\begin{lemma}[Results for $\lim_{\theta \arr 0+,\pi-}$ I]\label{lemDetails2}
Let $\bbeta,\bgamma \in \real$ be non-zero. 
\begin{enumerate}
\item Limits as $\theta \arr 0+$. \vspace*{-0.15in} 
\[
\lim_{\theta \arr 0+}\tan^{-1}\big(-\cot(\theta)+\frac{\bbeta}{\bgamma}\csc(\theta)\big) =
\begin{cases}
& \frac{\pi}{2},\;\text{if } \frac{\bbeta}{\bgamma} > 1\\
& 0,\;\text{if } \bbeta = \bgamma \\
& -\frac{\pi}{2},\;\text{if } \frac{\bbeta}{\bgamma} < 1.
\end{cases}
\]
(If $\bbeta$ and $\bgamma$ are interchanged, cases 1 and 3 are switched.)
\item Limits as $\theta \arr \pi-$. \vspace*{-0.15in}
\[
\lim_{\theta \arr \pi-}\tan^{-1}(-\cot(\theta)+\frac{\bbeta}{\bgamma}\csc(\theta)) = 
\begin{cases}
& \frac{\pi}{2},\;\text{if } \frac{\bbeta}{\bgamma} > 0\\
& \frac{\pi}{2},\;\text{if } \frac{\bbeta}{\bgamma} \in (-1,0)\\
& 0,\;\text{if } \bbeta=-\bgamma \\
& -\frac{\pi}{2},\;\text{if }  \frac{\bbeta}{\bgamma} < -1
\end{cases}
\]
(If $\bbeta$ and $\bgamma$ are interchanged, cases 2 and 4 are switched.)
\end{enumerate}
\end{lemma}
\proof Elementary and omitted. \qed
\begin{rem}
For the limit as $\theta \arr 0+$. the sum of both terms is always $0$; for the limit as $\theta \arr \pi-$, the sum is zero if $\frac{\bbeta}{\bgamma} < 0$, otherwise it is $\pi$. \rend
\end{rem}

In the limiting cases $\theta \arr 0+,\pi-$, useful information can be given about the terms in the decomposition of the generalized $T$-function \Refb{EQ: GENT}.
Owen $T$-function terms.
\begin{lemma}[Results for $\lim_{\theta \arr 0+,\pi-}$ II]\label{lemDetails3}
Suppose that $\bgamma  \ne 0$.
\begin{enumerate}
\item Limit as $\theta \arr 0+$. \vspace*{-0.15in}
\[
\lim_{\theta \arr 0+}T\big(\bgamma,-\cot(\theta)+ \frac{\bbeta\csc(\theta)}{\bgamma}\big) = 
\begin{cases}
& \frac{1}{4}\,\erfcb{\frac{|\bgamma|}{\sqrt{2}}}\;\text{if } \bbeta/\bgamma > 1 \\
& 0,\;\text{if } \bbeta=\bgamma \\
& -\frac{1}{4}\,\erfcb{\frac{|\bgamma|}{\sqrt{2}}}\;\text{if }  \bbeta / \bgamma < 1
\end{cases}
\]
(If $\bbeta$ and $\bgamma$ are interchanged, cases 1 and 3 are switched---assuming $\bbeta \ne 0$.)
\item Limit as $\theta \arr \pi-$. \vspace*{-0.15in}
\[
\lim_{\theta \arr \pi-} T\big(\bgamma,-\cot(\theta)+ \frac{\bbeta\csc(\theta)}{\bgamma}\big) =
\begin{cases}
&\frac{1}{4}\,\erfcb{\frac{|\bgamma|}{\sqrt{2}}}\;\text{if } \bbeta/\bgamma > 0 \\
&\frac{1}{4}\,\erfcb{\frac{|\bgamma|}{\sqrt{2}}}\;\text{if } \bbeta/\bgamma \in (-1,0] \\
& 0\;\text{if } \bbeta = -\bgamma \\
& -\frac{1}{4}\,-\erfcb{\frac{|\bgamma|}{\sqrt{2}}}\;\text{if } \bbeta/\bgamma < -1
\end{cases}
\]
(If $\bbeta$ and $\bgamma$ are interchanged, cases 2 and 4 are switched---assuming $\bbeta \ne 0$.)
\end{enumerate}
\end{lemma}
\begin{proof} Straightforward using \Refb{EQ: OwenEval}. \end{proof}
\begin{rem}
Lemmas~\ref{lemDetails2}, \ref{lemDetails3} can be used to give a simple direct proof of the continuity of $H$ in 
$\bbeta,\bgamma$ when $\theta \in \{0,\pi\}$ and $\beta\bgamma \ne 0$. Note that analyticity of $H$ fails if $\theta \in \{0,\pi\}$, as in the bias free case. \rend
\end{rem}

It follows from~\Refb{EQ: GENT} that if we define $S(\bbeta,\bgamma,\theta) = H(\bbeta,\bgamma,\theta)- \frac{\pi}{2}\erfb{\frac{\bbeta}{\sqrt{2}}}$, then   for $\bbeta\bgamma \ne 0$,
\begin{equation}\label{EQN: Sform}
\begin{split}
&\hspace*{0.4in}S(\bbeta,\bgamma,\theta) = \tan^{-1}(-\cot(\theta)){}-\\
&\left[\tan^{-1}\left(-\cot(\theta) + \frac{\bbeta\csc(\theta)}{\bgamma}\right) + \tan^{-1}\left(-\cot(\theta) + \frac{\bgamma\csc(\theta)}{\bbeta}\right)\right] \\
&+2\pi\left[T\left(\bgamma,-\cot(\theta) + \frac{\bbeta\csc(\theta)}{\bgamma}\right)+T\left(\bbeta,-\cot(\theta) + \frac{\bgamma\csc(\theta)}{\bbeta}\right)\right]
\end{split}
\end{equation}
The function $S(\bbeta,\bgamma,\theta)$ is symmetric in $\bbeta,\bgamma$: $S(\bbeta,\bgamma,\theta)= S(\bgamma,\bbeta,\theta)$. 
Since  $\frac{\pi}{2}\erfb{\frac{\bbeta}{\sqrt{2}}}$ is real analytic in $\bbeta\in\real$ and the integral defining $H$ is real analytic on $\real^2 \times (0,\pi)$, we have
\begin{prop}\label{prop: det}
(Notation and assumptions as above.) $S(\bbeta,\bgamma,\theta)$ is real analytic and bounded on $\real^2 \times (0,\pi)$. Moreover,
\begin{enumerate}
\item $S(\bbeta,0,\theta) = 2\pi T(\bbeta,-\cot(\theta))$, $\bbeta \ne 0$.
\item $S(0,\bgamma,\theta) = 2 \pi T(\bgamma,-\cot(\theta))$, $\bgamma \ne 0$.
\item $S(\bbeta,\bbeta,\theta) = 4\pi T(\bbeta,\tan(\theta/2)) - \pi/2$, $\bbeta \ne 0$.
\item $S(0,0,\theta) = \theta - \frac{\pi}{2}$.
\end{enumerate}
\end{prop}
\begin{proof} For (1--4), use Lemma~\ref{lemDetails}. \end{proof}
\begin{rems}
(1) Although $S$ is real analytic and bounded on $\real^2 \times (0,\pi)$, it is not true that $S$ is subanalytic. In particular, the closure of the graph of $S$ is not subanalytic. \\
(2) The right hand side of~\Refb{EQN: Sform} comprises four \emph{discontinuous} functions and one analytic function---$\tan^{-1}(-\cot(\theta))$ (cf.~Lemmas~\ref{lemDetails2}, \ref{lemDetails3}).  \rend
\end{rems}

It is helpful and suggestive to express
$S(\bbeta,\bgamma,\theta)$ in terms of ``angles". To this end,  define $\widetilde \theta=\widetilde \theta(\bbeta,\bgamma,\theta)$ by  the equation
\begin{equation}
S(\bbeta,\bgamma,\theta) = \widetilde \theta - \frac{\pi}{2}
\end{equation}
and note that
{\small
\begin{eqnarray}\label{EQ: tildatheta}
\hspace*{0.5in} \wtheta&=&H(\bbeta,\bgamma,\theta)+\frac{\pi}{2}\erfcb{\frac{\bbeta}{\sqrt{2}}} \\
\label{EQ: tildatheta2}
& = &\sqrt{\frac{\pi}{2}}\int_{\bgamma}^\infty e^{-\frac{t^2}{2}}\erfb{\frac{-t\cot(\theta) + \bbeta\csc(\theta)}{\sqrt{2}}}dt +\frac{\pi}{2}\erfcb{\frac{\bbeta}{\sqrt{2}}} \\
& = &\sqrt{\frac{\pi}{2}}\int_{\bbeta}^\infty e^{-\frac{t^2}{2}}\erfb{\frac{-t\cot(\theta) + \bgamma\csc(\theta)}{\sqrt{2}}}dt +\frac{\pi}{2}\erfcb{\frac{\bgamma}{\sqrt{2}}} 
\end{eqnarray}
}\normalsize


If $\bbeta = \bgamma = 0$, then $S(0,0,\theta) = \theta - \frac{\pi}{2}$ (Proposition~\ref{prop: det}) and so $\wtheta = \theta$. In the sequel, the variable dependence of $\widetilde \theta$ will be clear from the context and
we always write  $\widetilde \theta$ and omit the variable dependence.

\begin{exam}
Using either Lemmas~\ref{lemDetails} and \ref{lemDetails2} or Proposition~\ref{prop: det},  
\begin{eqnarray*}
\lim_{\theta \arr 0+} H(\bbeta,\bbeta,\theta)& =& S(\bbeta,\bbeta,0) + \frac{\pi}{2}\erfb{\frac{\bbeta}{\sqrt{2}}}\\
& =& -\frac{\pi}{2}\, \erfcb{\frac{\bbeta}{\sqrt{2}}}  
\end{eqnarray*}
and $\widetilde \theta = 0$.
Here, exactly two of the terms in \Refb{EQ: GENT} are non-zero: term 1 (which is $-\frac{\pi}{2}$) and the final term. A similar analysis can be done for 
$\lim_{\theta \arr \pi-} H(\bbeta,-\bbeta, \theta) = 0$ (see also~Lemma~\ref{lem: pi}). \examend
\end{exam}

\section{Evaluation of $f(\ww,\vv)$ concluded}\label{sec: 27}  

Throughout this section,  take $A = -\cot(\theta)$, $B = \csc(\theta)$ and assume $\theta \in (0,\pi)$.
\subsubsection{The constant term $ \bbeta\bgamma\int_D 1$}
We have 
\begin{eqnarray*}
\bbeta\bgamma\int_D  1&=&\bbeta\bgamma\left[\frac{\pi}{2}\erfcb{\frac{\bgamma}{\sqrt{2}}} - \big(\widetilde \theta - \frac{\pi}{2}\big) - \frac{\pi}{2}\erfb{\frac{\bbeta}{\sqrt{2}}}\right]\\
 & = & \bbeta\bgamma\left(\frac{\pi}{2}\left[\erfcb{\frac{\bgamma}{\sqrt{2}}}+\erfcb{\frac{\bbeta}{\sqrt{2}}}\right]  - \widetilde \theta \right)
\end{eqnarray*}
\emph{Check:} $\bbeta\bgamma\int_D  1$ is $(\bbeta,\bgamma)$-symmetric.
\subsubsection{The linear term $-\int_D (\bgamma\cos(\theta)+\bbeta)x_1 + \bgamma\sin(\theta)x_2$}
\begin{multline*}
\begin{aligned}
-\int_D& (\bgamma\cos(\theta)+\bbeta)x_1 + \bgamma\sin(\theta)x_2  = \\
& - \sqrt{\frac{\pi}{2}}e^{-\frac{\bgamma^2}{2}}\left[(\bgamma\cos(\theta) + \bbeta) \erfcb{\frac{A\bgamma + B\bbeta}{\sqrt{2}}}\right] \\
& -\sqrt{\frac{\pi}{2}}e^{-\frac{\bbeta^2}{2}}\left[(\bgamma + \bbeta\cos(\theta))\erfcb{\frac{A\bbeta + B\bgamma}{\sqrt{2}}}\right]
\end{aligned}
\end{multline*}
\emph{Check:} $-\int_D (\bgamma\cos(\theta)+\bbeta)x_1 + \bgamma\sin(\theta)x_2$ is $(\bbeta,\bgamma)$-symmetric.
\subsubsection{The first quadratic term $\int_D \sin(\theta)x_1x_2$}
\[
\int_D \sin(\theta)x_1x_2 = \sin^2(\theta)e^{-\frac{\bbeta^2}{2}}\left[\sqrt{\frac{\pi}{2}}\bbeta\cos(\theta)\erfcb{\frac{A\bbeta+B\bgamma}{\sqrt{2}}}+\sin(\theta)e^{-\frac{(A\bbeta+B\bgamma)^2}{2}}\right]
\]
This is consistent with the result for $\int_D \sin(\theta)x_1x_2$ when the bias is zero. 
Obviously this term is not $(\bbeta,\bgamma)$-symmetric. Consequently, the sum of the quadratic terms
$\int_D \sin(\theta)x_1x_2 + \int_D \cos(\theta) x_1^2$ must be $(\bbeta,\bgamma)$-symmetric.
\subsubsection{The second quadratic term $\int_D \cos(\theta)x_1^2$}
\begin{multline*}
\begin{aligned}
\int_D\cos(\theta) x_1^2 =& \cos(\theta)\frac{\pi}{2} \big(1 - \erfb{\frac{\bgamma}{\sqrt{2}}}\big) \\
&+\sqrt{\frac{\pi}{2}}\left[\cos(\theta) \bgamma e^{-\frac{\bgamma^2}{2}}\erfcb{\frac{A\bgamma + B\bbeta}{\sqrt{2}}}\right] \\
& +\sqrt{\frac{\pi}{2}}\left[ \cos^3(\theta) \bbeta e^{-\frac{\bbeta^2}{2}} \erfcb{\frac{A\bbeta + B\bgamma}{\sqrt{2}}}\right] \\
& + \sin(\theta)\cos^2(\theta) e^{-\frac{\bbeta^2}{2}} e^{-\frac{(A\bbeta + B\bgamma)^2}{2}}\\
& -\cos(\theta)\big(\widetilde\theta - \frac{\pi}{2}\erfcb{\frac{\bbeta}{\sqrt{2}}}\big)- \cos(\theta)\frac{\pi}{2} \erfb{\frac{\bbeta}{\sqrt{2}}}
\end{aligned}
\end{multline*}
It is convenient to define functions $k, K$ of two biases and one angle by
\begin{eqnarray*}
k(\bbeta,\bgamma,\theta)&=&\frac{\bbeta^2}{2} + \frac{(A\bbeta + B\bgamma)^2}{2} = \frac{\bgamma^2}{2} +\frac{(A\bgamma + B\bbeta)^2}{2}\\
&=& \frac{(\bbeta^2+\bgamma^2)\csc^2(\theta) - 2 \bbeta\bgamma\cot(\theta)\csc(\theta)}{2}
\end{eqnarray*}
We use capital $K$ to indicate a function $K(\bbeta_i,\bbeta_j, \Theta)$ where $\Theta$ will now be an angle between two rows $\ww^i,\ww^j$ of $\WW$ 
as opposed to the angle between a row of $\WW$ and a row of $\VV$. 

After using some elementary trigonometric identities, we find that if $\|\ww\|,\|\vv\| \ne 0$, then
\begin{multline*}
\begin{aligned}
\int_D&\cos(\theta) x_1^2 + \sin(\theta) x_1x_2 = \\
& \cos(\theta)\frac{\pi}{2}\left[ \erfcb{\frac{\bgamma}{\sqrt{2}}} +\erfcb{\frac{\bbeta}{\sqrt{2}}}\right] -\cos(\theta)\widetilde\theta \\
&+\sqrt{\frac{\pi}{2}}\cos(\theta)\left[\bbeta  e^{-\frac{\bbeta^2}{2}} + \bgamma e^{-\frac{\bgamma^2}{2}}\right] \\
&-\sqrt{\frac{\pi}{2}}\cos(\theta)\left[\bbeta  e^{-\frac{\bbeta^2}{2}}\erfb{\frac{A\bbeta+B\bgamma}{\sqrt{2}}} +\bgamma  e^{-\frac{\bgamma^2}{2}}\erfb{\frac{A\bgamma+B\bbeta}{\sqrt{2}}}\right] \\
&+\sin(\theta)e^{-k(\bbeta,\bgamma,\theta)}
\end{aligned}
\end{multline*}
\emph{Checks:} $\int_D\cos(\theta) x_1^2 + \sin(\theta) x_1x_2$ is $(\bbeta,\bgamma)$-symmetric and if
$\bgamma,\bbeta = 0$, $\int_D\cos(\theta) x_1^2 + \sin(\theta) x_1x_2= \sin(\theta)+(\pi-\theta)\cos(\theta)$.

\subsubsection{The formula for $f(\ww,\vv)$}\label{sec: notcon}
For this we use the expressions computed above and the abbreviations $A = -\cot(\theta), B = \csc(\theta)$.   Many terms cancel and we find that 
\begin{multline}\label{EQ: fwv}
\begin{aligned}
f(\ww,&\vv)=  \\
& \frac{\|\ww\|\|\vv\|[\bbeta\bgamma +  \cos(\theta)]}{2\pi}\left(\frac{\pi}{2}\left[\erfcb{\frac{\bgamma}{\sqrt{2}}}+\erfcb{\frac{\bbeta}{\sqrt{2}}}\right]  - \widetilde\theta \right)\\
&-\frac{\|\ww\|\|\vv\|}{2\sqrt{2\pi}}\left[\bgamma e^{-\frac{\bbeta^2}{2}} \erfcb{\frac{A\bbeta + B\bgamma}{\sqrt{2}}} + \bbeta e^{-\frac{\bgamma^2}{2}}\erfcb{\frac{A\bgamma + B\bbeta}{\sqrt{2}}} \right]\\
&+\frac{\|\ww\|\|\vv\|}{2\pi}\sin(\theta)e^{-k(\bbeta,\bgamma,\theta)}
\end{aligned}
\end{multline}

\begin{rem}
Using \Refb{EQ: fwv} for $f(\ww,\vv)$, it is a useful exercise to deduce the result for $f(\ww,\ww)$ given in \Refb{EQ: fww} as well as Lemmas~\ref{lem: zero}, \ref{lem: pi} covering the cases $\theta\in\{0,\pi\}$. \rend
\end{rem}
The section concludes with some examples and computations used later. Some preliminary notational conventions are needed first.
\subsection{Notational conventions \& formulas for $f(\ww^i,\vv^j), f(\ww^i,\ww^j)$}\label{sec: notconv}.
Given $d \in \pint$, define $\is{d} = \{0,1,\ldots,d-1\}$ and $[d] = \{1,2,\ldots,d\}$. We fix a matrix $\VV \in M(d,d)$ and denote the rows by $\vv^0,\ldots,\vv^{d-1}$. Later we 
refer to $\VV$ as the ``target'': the hoped for end result of the learning process (gradient descent in our set up). Generally, it is assumed 
that $\VV$ has no parallel rows.  Typically, we take $\VV$ to be the identity matrix and $\{\vv^i\}_{i \in \is{d}}$ to be the standard orthonormal basis of $(\real^d)^\star \approx \real^d$ (by the Euclidean inner product).

Next, let $\WW \in M(k,d)$, $k \ge d\in \pint$. Assume $\WW$ has no parallel rows and label rows of $\WW$ by $\is{k}$: $\ww^0,\ldots,\ww^{k-1}$.  Given rows $\ww^i, \ww^j$ in $\WW$, $\vv^j$ in $\VV$, define
\[
\Lambda_{ij} = \cos^{-1}\left(\frac{\langle \ww^i,\ww^j\rangle}{\|\ww^i\|\|\ww^j\|}\right),\; \lambda_{ij} = \cos^{-1}\left(\frac{\langle \ww^i,\vv^j\rangle}{\|\ww^i\|\|\vv^j\|}\right).
\]
Note that $\Lambda_{ij} = \Lambda_{ji}$, all $i,j \in \is{k}$ and that $\Lambda_{ii} = 0$. This is not the case for the angles $\lambda_{ij}$;  indeed, $\lambda_{ji}$ is not be defined if $j > d$.

It is helpful to develop a more concise notation for the terms that occur in $f(\ww,\vv)$. Continuing with our assumptions on $\WW, \VV$, 
\begin{enumerate}
\item[(a1)] For $i,j \in \is{k}$, $j \ne i$, define
\[
Z_{ij} =  -\cot(\Lambda_{ij}) \bbeta_j + \csc(\Lambda_{ij}) \bbeta_i, \qquad  \wZZ_{ij} = -\cot(\Lambda_{ij}) \bbeta_i + \csc(\Lambda_{ij}) \bbeta_j
\]
\[
K(\bbeta_i,\bbeta_j,\Lambda_{ij}) = \frac{1}{2}\big((\bbeta_i^2 + \bbeta_j^2)\csc^2(\Lambda_{ij}) - 2\bbeta_i\bbeta_j  \csc(\Lambda_{ij})\cot(\Lambda_{ij})\big)
\]
\item[(a2)] For $i \in \is{k}, j \in \is{d}$, define
\[
z_{ij} =  -\cot(\lambda_{ij}) \bgamma_j + \csc(\lambda_{ij}) \bbeta_i, \qquad  \wzz_{ij} = -\cot(\lambda_{ij}) \bbeta_i + \csc(\lambda_{ij}) \bgamma_j
\]
\[
k(\bbeta_i,\bgamma_j,\lambda_{ij}) = \frac{1}{2}\big((\bbeta_i^2 + \bgamma_j^2)\csc^2(\lambda_{ij}) - 2\bbeta_i\bgamma_j  \csc(\lambda_{ij})\cot(\lambda_{ij})\big)
\]
\end{enumerate}
With these conventions,  
\begin{multline*}
\begin{aligned}
f(\ww^i,&\vv^j)= \\
&\frac{\|\ww^i\|\|\vv^j\|[\bbeta_i\bgamma_j +  \cos(\lambda_{ij})]}{2\pi}\left(\frac{\pi}{2}\left[\erfcb{\frac{\bbeta_i}{\sqrt{2}}}+\erfcb{\frac{\bgamma_j}{\sqrt{2}}}\right]-\wlambda_{ij}\right) \\
& -\frac{\|\ww^i\|\|\vv^j\|}{2\sqrt{2\pi}}\left[\bgamma_j e^{-\frac{\bbeta_i^2}{2}} \erfcb{\frac{\wzz_{ij}}{\sqrt{2}}} + \bbeta_i e^{-\frac{\bgamma_j^2}{2}}\erfcb{\frac{z_{ij}}{\sqrt{2}}} \right]\\
&+\frac{\|\ww^i\|\|\vv^j\|}{2\pi}\sin(\lambda_{ij})e^{-k(\bbeta_i,\bgamma_j,\lambda_{ij})},\;\; i \in \is{k}, j \in \is{d}, \\
f(\ww^i,&\ww^j)= \\
&\frac{\|\ww^i\|\|\ww^j\|[\bbeta_i\bbeta_j +  \cos(\Lambda_{ij})]}{2\pi}\left(\frac{\pi}{2}\left[\erfcb{\frac{\bbeta_i}{\sqrt{2}}}+\erfcb{\frac{\bbeta_j}{\sqrt{2}}}\right]-\wLambda_{ij}\right) \\
& -\frac{\|\ww^i\|\|\ww^j\|}{2\sqrt{2\pi}}\left[\bbeta_j e^{-\frac{\bbeta_i^2}{2}} \erfcb{\frac{\wZZ_{ij}}{\sqrt{2}}} + \bbeta_i e^{-\frac{\bbeta_j^2}{2}}\erfcb{\frac{Z_{ij}}{\sqrt{2}}} \right]\\
&+\frac{\|\ww^i\|\|\ww^j\|}{2\pi}\sin(\Lambda_{ij})e^{-K(\bbeta_i,\bbeta_j,\Lambda_{ij})}, \;\; i,j \in \is{k}, j \ne i
\end{aligned}
\end{multline*}
(For $i = j$, see Example~\ref{ex: simp}, equation \Refb{EQ: fww} and also \Refb{EQ: iisj} below.)

If $\bgamma_i = 0$, for all $i \in \is{d}$---the only case considered in this paper---then for $i \in \is{k}$ and $j \in \is{d}$ 
\begin{multline*}
\begin{aligned}
f(\ww^i,&\vv^j)= \frac{\|\ww^i\|\|\vv^j\|\cos(\lambda_{ij})}{2\pi}\left(\pi -\wlambda_{ij} - \frac{\pi}{2} \erfb{\frac{\bbeta_i}{\sqrt{2}}}\right) \\
& -\frac{\|\ww^i\|\|\vv^j\|}{2\sqrt{2\pi}}\bbeta_i \erfcb{\frac{\csc(\lambda_{ij}) \bbeta_i}{\sqrt{2}}} 
 +\frac{\|\ww^i\|\|\vv^j\|}{2\pi}\sin(\lambda_{ij})e^{-\frac{\bbeta_i^2\csc^2(\lambda_{ij})}{2}} 
\end{aligned}
\end{multline*}
When symmetry is present, we may have $\bbeta_i = \bbeta_j$, $i \ne j$. Setting $\bbeta_i,\bbeta_j = \bbeta$ and $\Lambda_{ij} = \Theta$,  the formula for $f(\ww^i,\ww^j)$ is 
\begin{multline*}
\begin{aligned}
f(\ww^i,\ww^j)= &
\frac{\|\ww^i\|\|\ww^j\|[\bbeta^2 + \cos(\Theta)]}{2\pi}\left(\pi\erfcb{\frac{\bbeta}{\sqrt{2}}}-\wTheta\right) \\
& \hspace*{-0.4in}-\frac{\|\ww^i\|\|\ww^j\|}{\sqrt{2\pi}}\bbeta e^{-\frac{\bbeta^2}{2}} \erfcb{\frac{\beta\big(-\cot(\Theta)+\csc(\Theta)\big)}{\sqrt{2}}}\\
&\hspace*{-0.4in}+\frac{\|\ww^i\|\|\ww^j\|}{2\pi}\sin(\Theta)e^{-\bbeta^2\csc(\Theta)\big(\csc(\Theta) -\cot(\Theta)\big)} \;\; i,j \in \is{k}, j \ne i
\end{aligned}
\end{multline*}

\begin{exam}\label{EX: on_set_target}
Let $\{\uu^i\}_{i \in \is{d}}$ be an orthogonal basis of $\real^d$ such that each $\uu^i$ is a non-zero multiple of $\vv^i$, $i \in \is{d}$. For $i \ne j$, the angle $\phi_{ij}$ between
$\uu^i$ and $\uu^j$ (or $\vv^j$) is $\pi/2$.  
For $i \in \is{d}$,
\begin{equation}\label{EQ: iisj}
f(\uu^i,\uu^i)  = \frac{\|\uu^i\|^2}{2}\big(\bbeta_i^2+1\big)\erfcb{\frac{\bbeta_i}{\sqrt{2}}} - \frac{\|\uu^i\|^2}{\sqrt{2\pi}}\bbeta_i e^{-\frac{\bbeta_i^2}{2}}
\end{equation}
For $i,j \in \is{d}$, $i \ne j$,
\begin{multline}
\begin{aligned}
f(\uu^i,&\uu^j)  = \frac{\|\uu^i\|\|\uu^j\|}{4}\bbeta_i\bbeta_j\erfcb{\frac{\bbeta_i}{\sqrt{2}}}\erfcb{\frac{\bbeta_j}{\sqrt{2}}} \\
& -\frac{\|\uu^i\|\|\uu^j\|}{2\sqrt{2\pi}}\left[\bbeta_j e^{-\frac{\bbeta_i^2}{2}} \erfcb{\frac{\bbeta_j}{\sqrt{2}}} + \bbeta_i e^{-\frac{\bbeta_j^2}{2}}\erfcb{\frac{\bbeta_i}{\sqrt{2}}} \right]\\
&+\frac{\|\uu^i\|\|\uu^j\|}{2\pi}e^{-K(\bbeta_i,\bbeta_j,\pi/2)},
\end{aligned}
\end{multline}
where $K(\bbeta_i,\bbeta_j,\pi/2) = \frac{1}{2}(\bbeta_i^2+\bbeta_j^2)$ and we used $\frac{2}{\pi}H(\bbeta_i,\bbeta_j,\pi/2) = \erfcb{\frac{\bbeta_j}{\sqrt{2}}}\erfb{\frac{\bbeta_i}{\sqrt{2}}}$.  
The reader may confirm that the result is consistent with the alternative description in terms of $\widetilde{\phi}_{ij}$.
\examend
\end{exam}
\section{Gradient of $f(\ww,\vv)$ with respect to $(\ww,\bbeta)$: I}\label{sec: sec4}
Our strategy will be to minimize the number of evaluations of the $T$ function. In particular, when we consider derivatives of $H(\bbeta,\bgamma,\theta) =2\pi T(\bgamma,-\cot(\theta), \bbeta\csc(\theta))$, we differentiate
the defining integral~\Refb{EQ: integral} rather than the Owen $T$-functions and inverse tangents that appear in \Refb{EQ: GENT}. 
\subsubsection{Notational conventions}
If there are $d$ inputs and $k$-neurons and there is non-zero bias, then the parameter space will be $M(k,d+1)$. If $i \in \is{k}$, and $\WW \in M(k,d)\subset M(k,d+1)$, then row $i$ of $\WW$ is denoted by $\ww^i$ and
\[
\ww^i = (w_{i0},\ldots,w_{id-1})\in (\real^d)^\star
\]
Note that $\ww^i$ is a row vector---linear functional. Thus the matrix multiplication $\ww^i\xx$ is well defined for all vectors $\xx \in \real^d$. With bias included, we sometimes define $\wWW \in M(k,d+1)$ by 
$\wWW = [\WW,\boldsymbol{\bbeta}]$ where $\boldsymbol{\bbeta}$ is the $k \times 1$ column matrix of biases and row $i$ of $\wWW$ is denoted by $\tww^i$ so that
\[
\tww^i = (w_{i0},\ldots,w_{id-1},\bbeta_i) \in  (\real^d)^\star \times \real,
\]
where $\bbeta_i$ is the (normalized) bias associated to neuron $i$.  Mostly we stick to $\WW, \ww$ when the meaning is clear from the context.

We regard $\partial/\partial \ww$ as shorthand for the gradient operator in the variable $\ww$. That is, $\frac{\partial f}{\partial \ww}$ 
is the derivative $Df_\ww \in (\real^\star)^{d}$ followed by identification of $Df_\ww$ with the point in $\real^{d}\subset \real^{d+1}$ defined by the Euclidean inner product.
Note that absence of bold face for $w$ indicates the standard partial derivative. Set $\grad{f}(\tww) = (\frac{\partial f}{\partial \ww},\frac{\partial f}{\partial \boldsymbol{\bbeta}})$. 

\begin{rem}
From Section~\ref{sec: symmsec9}, the focus will be on networks with a symmetric target and then $\lambda$ is used for the angle between $\ww$ and $\vv$ in $f(\ww,\vv)$ and $\Lambda$ for angle between $\ww, \ww'$ in $f(\ww,\ww')$ when there is no symmetry relation.  Moreover, $\theta$ (resp.~$\Theta$) will be used for the angle between $\ww$ and $\vv$ (resp.~$\ww$ and $\ww'$) when the associated rows have the same symmetry or \emph{rowtype} (see Section~\ref{sec: symmsec9}).  Since 
$\theta$ is often used to denote angles between rows in the literature (for example, \cite{ChooSaul2009,BrutzkusGloberson2018} where the networks are bias free), 
it seemed best to start by following this convention. However, from now on $\lambda$ will generally be used for the angle between a row of $\WW$ and a row of $\VV$ and
$\Lambda$ for angles between (different) rows of $\WW$. 
\rend
\end{rem}
\begin{exam}\label{EX: fwwderivex}
Recall from Example~\ref{ex: simp} that 
\[
f(\ww,\ww) = \frac{\|\ww\|^2(\bbeta^2+1)}{2}\erfcb{\frac{\bbeta}{\sqrt{2}}} - \frac{\|\ww\|^2}{\sqrt{2\pi}}\bbeta e^{-\frac{\bbeta^2}{2}}
\]
Computing we find that
\begin{eqnarray}
\label{EQ: iiw_deriv}
\frac{\partial f}{\partial \ww}(\ww,\ww)&=&\left[(\bbeta^2+1)\erfcb{\frac{\bbeta}{\sqrt{2}}} - \sqrt{\frac{2}{\pi}}\bbeta e^{-\frac{\bbeta^2}{2}}\right]\ww\\
\label{EQ: iibeta_deriv}
\frac{\partial f}{\partial \bbeta} & = & \|\ww\|^2\left[\bbeta\erfcb{\frac{\bbeta}{\sqrt{2}}}-\sqrt{\frac{2}{\pi}}e^{-\frac{\bbeta^2}{2}}\right]
\end{eqnarray}
and so, setting $(\ww,\bbeta) = \widetilde{\ww} \in M(d+1,k)$.
\[
\grad{f}(\tww) = \erfcb{\frac{\bbeta}{\sqrt{2}}}\big((\bbeta^2+1)\ww, \|\ww\|^2\bbeta) - \sqrt{\frac{2}{\pi}}e^{-\frac{\bbeta^2}{2}}\big(\bbeta \ww,1\big) 
\]
\examend
\end{exam}
If $\uu \in \real^n$ (or $(\real^n)^\star$) is non-zero, define $\overline{\uu} = \frac{\uu}{\|\uu\|}$.

\subsubsection{Some frequently used derivatives}\label{deriv: useful}
Suppose $\ww,\vv\in\real^d$ are not parallel (and therefore non-zero), $\cos(\theta) = \frac{\langle \ww,\vv\rangle}{\|\ww\|\|\vv\|}$ and 
\begin{equation}\label{def: Jvector}
\jj =  \cot(\theta) \overline{\ww} - \csc(\theta)\overline{\vv}
\end{equation}
Similarly define $\JJ = \JJ(\ww,\ww') \in \real^d$ when $\ww,\ww'$ are distinct rows of $\WW$.
The following derivatives will be used repeatedly in this and subsequent sections.
\begin{eqnarray*}
\frac{\partial}{\partial \ww} \|\ww\|&=&\overline{\ww},\quad \frac{\partial}{\partial \ww} \frac{1}{\|\ww\|} = -\frac{\overline{\ww}}{\|\ww\|^2}, \quad
\frac{\partial \theta}{\partial \ww}=\frac{\jj}{\|\ww\|}\\
\frac{\partial\, \cos(\theta)}{\partial \ww} & = &  -\frac{\sin(\theta)}{\|\ww\|}\jj, \quad
\frac{\partial\, \sin(\theta)}{\partial \ww}  =  \frac{\cos(\theta)}{\|\ww\|}\jj \\
\frac{\partial\, \cot(\theta)}{\partial \ww} & = & -\frac{\csc^2(\theta)}{\|\ww\|}\jj ,\quad
\frac{\partial\, \csc(\theta)}{\partial \ww}  =   -\frac{\cot(\theta)\csc(\theta)}{\|\ww\|}\jj 
\end{eqnarray*}

\begin{lemma}\label{lem: Jlemma}
(Notation and assumptions as above.)
\begin{enumerate}
\item $\jj(\ww,\vv) \perp \ww$, $\|\jj\| =1$,
\item $\JJ(\ww,\ww') \perp \ww$, $\|\JJ\| =1$.
\item $\{\overline{\ww},\jj\}$ is an orthonormal basis for the space spanned by $\ww,\vv$.
\end{enumerate}
\end{lemma}
The proof is left as an exercise for the reader.
\subsection{Derivatives of $H(\bbeta,\bgamma,\lambda)$ and $H(\bbeta_i,\bbeta_j,\Lambda_{ij})$, $i \ne j$}
The variable $\bgamma$ is not used in computing the gradient of $f(\ww,\vv)$ since the $\vv$-variables are part of the target and fixed.
For the term $f(\ww^i,\ww^j)$, $i \ne j$, we need to consider the derivatives of $H(\bbeta_i,\bbeta_j,\Lambda_{ij})$ with respect to
$\bbeta_i$ and $\bbeta_j$---including the case when $\bbeta_i = \bbeta_j$ which is important when there is symmetry.

\begin{prop}\label{lem: Hderiv}
(Notation and assumptions as above.)
\begin{enumerate}
\item Derivatives of $H(\bbeta,\bgamma,\lambda)$.
\begin{eqnarray}
\label{EQ: Hwdervb}
\frac{\partial H}{\partial \bbeta} & = & \sqrt{\frac{\pi}{2}}e^{-\frac{\bbeta^2}{2}}\erfcb{\frac{-\cot(\lambda)\bbeta+\csc(\lambda)\bgamma}{\sqrt{2}}}\\
\label{EQ: Hwdervw}
\frac{\partial H}{\partial \ww}&=& \frac{e^{-k(\bbeta,\bgamma,\lambda)}}{\|\ww\|}\jj
\end{eqnarray}
\item Derivatives of $H(\bbeta_i,\bbeta_j,\Lambda_{ij})$, $i \ne j$.
\begin{eqnarray}
\label{EQ: Hwdervb2}
\frac{\partial H}{\partial \bbeta_i} & = &   \sqrt{\frac{\pi}{2}}e^{-\frac{\bbeta_i^2}{2}}\erfcb{\frac{-\cot(\Lambda_{ij})\bbeta_i+\csc(\Lambda_{ij})\bbeta_j}{\sqrt{2}}}\\
\label{EQ: Hwdervb3}
\frac{\partial H}{\partial \bbeta_j} & = & -\sqrt{\frac{\pi}{2}}e^{-\frac{\bbeta_j^2}{2}}\erfb{\frac{-\cot(\Lambda_{ij})\bbeta_j+\csc(\Lambda_{ij})\bbeta_i}{\sqrt{2}}} \\
\label{EQ: Hwdervw2}
\frac{\partial H}{\partial \ww^i}&=& \frac{e^{-K(\bbeta_i,\bbeta_j,\Lambda_{ij})}}{\|\ww^i\|}\JJ_{ij} 
\end{eqnarray}
\item Derivative of $H(\bbeta,\bbeta,\Theta)$.
\begin{eqnarray}
\label{EQ: Hwdervw4}
\frac{\partial H}{\partial \bbeta} & = & \sqrt{\frac{\pi}{2}}e^{-\frac{\bbeta^2}{2}}\left[2\,\erfcb{\frac{\bbeta\big(-\cot(\Theta) +\csc(\Theta)\big)}{\sqrt{2}}} - 1\right]
\end{eqnarray}
\end{enumerate}
\end{prop}
\begin{proof}
(1) We have
\begin{multline*}
\begin{aligned}
\frac{\partial H}{\partial \bbeta}(\bbeta,\bgamma,\lambda) & = \frac{\partial}{\partial \bbeta}\sqrt{\frac{\pi}{2}}\int^\infty_{\bgamma} e^{-t^2/2} \erfb{\frac{-t\cot(\lambda) + \csc(\lambda)\bbeta}{\sqrt{2}}}dt\\
&= \sqrt{\frac{\pi}{2}} \int^\infty_{\bgamma} e^{-t^2/2} \frac{2\csc(\lambda)}{\sqrt{2\pi}}e^{-(t\cot(\lambda)-\bbeta\csc(\lambda))^2/2} dt\\
& = e^{-\frac{\bbeta^2}{2}} \int^\infty_{\bgamma} \csc(\lambda) e^{-(t \csc(\lambda) - \bbeta\cot(\lambda))^2/2}dt\\
& =\sqrt{\frac{\pi}{2}}e^{-\frac{\bbeta^2}{2}}\erfcb{\frac{-\cot(\lambda)\bbeta+ \csc(\theta)\bgamma}{\sqrt{2}}}
\end{aligned}
\end{multline*}
where the last line follows from \Refb{EQ2b}. 

In order to prove \Refb{EQ: Hwdervw},
differentiate $H$ with respect to $\ww$ to obtain
\[
\frac{\partial H}{\partial \ww} = \int_{\bgamma}^\infty e^{-\frac{t^2}{2}}\big [ tL - M\big] e^{-\frac{(-t\cot(\lambda)+\csc(\lambda)\bbeta)^2}{2}} dt = I_L -I_M,
\]
where
\[
L  =  \csc^2(\lambda)\frac{\partial \lambda}{\partial \ww} = \frac{\csc^2(\lambda)}{\|\ww\|}\jj\; \text{and }
M  =  \bbeta \csc(\lambda)\cot(\lambda)\frac{\partial \lambda}{\partial \ww} =  \bbeta \frac{\csc(\lambda)\cot(\lambda)}{\|\ww\|}\jj 
\]
Setting $A = -\cot(\lambda)$, $B = \csc(\lambda)$, evaluate the integrals $I_L, I_M$ using  Example~\ref{Ex: EQ3} and (\ref{EQ2b}):
\begin{eqnarray*}
I_L & = & \frac{e^{-\frac{\bbeta^2}{2}}}{\|\ww\|}e^{-\frac{(A\bbeta + B\bgamma)^2}{2}}\jj 
  -\sqrt{\frac{\pi}{2}}\bbeta\frac{e^{-\frac{\bbeta^2}{2}}}{\|\ww\|}\erfcb{\frac{A\bbeta + B\bgamma}{\sqrt{2}}}\cot(\lambda)\,\jj\\
-I_M & = & \sqrt{\frac{\pi}{2}}\bbeta\frac{e^{-\frac{\bbeta^2}{2}}}{\|\ww\|}\erfcb{\frac{A\bbeta + B\bgamma}{\sqrt{2}}}\cot(\lambda)\,\jj  
\end{eqnarray*}
and so
$
\frac{\partial H}{\partial \ww} = 
\frac{e^{-k(\bbeta,\bgamma,\lambda)}}{\|\ww\|}\jj. 
$

\noindent (2)  The proofs of (\ref{EQ: Hwdervb2},\ref{EQ: Hwdervb3},\ref{EQ: Hwdervw2}) are similar to those of (1) and omitted (note that \Refb{EQ: Hwdervb3} follows easily from
\Refb{EQ: Hwdervb2} using \Refb{EQ: bgsymm}).\\

\noindent (3)  Simply deduced from (\ref{EQ: Hwdervb2},\ref{EQ: Hwdervb3}). Alternatively, the result follows from Proposition~\ref{prop: spec}(3). 
\end{proof}

Since $H(\bbeta,\bgamma,\lambda) = \widetilde \lambda -\frac{\pi}{2}\erfcb{\frac{\bbeta}{\sqrt{2}}}$ (see \Refb{EQ: tildatheta}), the next result
follows trivially from Proposition~\ref{lem: Hderiv}. 
\begin{prop}\label{prop: tildederiv}
(Notation and assumptions as above.)
\begin{enumerate}
\item $\bbeta$ and $\ww$ derivatives of $\wlambda(\bbeta,\bgamma,\lambda)$
\begin{equation*}
\frac{\partial  \wlambda}{\partial \bbeta}  =  -\sqrt{\frac{\pi}{2}}e^{-\frac{\bbeta^2}{2}}\erfb{\frac{-\cot(\lambda)\bbeta + \csc(\lambda)\bgamma}{\sqrt{2}}},\quad
\frac{\partial  \wlambda}{\partial \ww}  =  \frac{e^{-k(\bbeta,\bgamma,\lambda)}}{\|\ww\|}\jj
\end{equation*}
\item $\bbeta_i,\bbeta_j$ and $\ww$ derivatives of $\wLambda_{ij}(\bbeta_i,\bbeta_j,\Lambda_{ij})$
\begin{eqnarray*}
\frac{\partial  \wLambda_{ij}}{\partial \bbeta_i} & = & -\sqrt{\frac{\pi}{2}}e^{-\frac{\bbeta_i^2}{2}}\erfb{\frac{-\cot(\Lambda_{ij})\bbeta_i + \csc(\Lambda_{ij})\bbeta_j}{\sqrt{2}}}\\
\frac{\partial  \wLambda_{ij}}{\partial \bbeta_j}&  = & -\sqrt{\frac{\pi}{2}}e^{-\frac{\bbeta_j^2}{2}}\erfb{\frac{-\cot(\Lambda_{ij})\bbeta_j + \csc(\Lambda_{ij})\bbeta_i}{\sqrt{2}}} \\
\frac{\partial  \wLambda_{ij}}{\partial \ww} &=& \frac{e^{-K(\bbeta_i,\bbeta_j,\Lambda_{ij})}}{\|\ww\|}\JJ_{ij}
\end{eqnarray*}

\end{enumerate}

\end{prop}
\begin{rems}
(1) The expression for $\frac{\partial H}{\partial \ww}$ is singular if $\theta \in \{0,\pi\}$ even for zero bias: a similar term appears if we follow the same analysis in 
section~\ref{sec: zerobias} where the kernel for zero bias is derived without using polar coordinates. That singular term is cancelled out by another $\ww$-derivative in the 
expression for $\frac{\partial f}{\partial \ww}$; if bias is non-zero it will not be. \\
(2) The additional term
$-\sqrt{\frac{\pi}{2}}\,\bbeta\frac{e^{-\frac{\bbeta^2}{2}}}{\|\ww\|}\erfcb{\frac{A\bbeta + B\bgamma}{\sqrt{2}}}\bww$ appears in the formula for $\frac{\partial H}{\partial \ww}$ if
$\beta$ is the bias parameter. Of course, $\frac{\partial H}{\partial \beta} = \frac{1}{\|\ww\|}\frac{\partial H}{\partial \bbeta}$.\rend
\end{rems}

\section{Computation of $\frac{\partial f}{\partial \bbeta}$}\label{sec8: betederivoff}
\begin{lemma}\label{EQ:betaderv}
(Notations and assumptions of previous section). \\
(1) Assuming $\ww^i,\vv^j$ are not parallel,  
\begin{multline*}
\begin{aligned}
\hspace*{0.2in}\frac{\partial f}{\partial \bbeta_i}(\ww^i,\vv^j)&= \frac{\|\ww^i\|\|\vv^j\|\bgamma_j}{2\pi}\left(\frac{\pi}{2}\left[\erfcb{\frac{\bgamma_j}{\sqrt{2}}}+\erfcb{\frac{\bbeta_i}{\sqrt{2}}}\right]  - \wlambda_{ij}\right)\\
&\hspace*{-0.6in} -\frac{\|\ww^i\|\|\vv^j\|}{2\sqrt{2\pi}}\left[e^{-\frac{\bgamma_j^2}{2}}\erfcb{\frac{A\bgamma_j+B\bbeta_i}{\sqrt{2}}}+\cos(\lambda_{ij})e^{-\frac{\bbeta_i^2}{2}}\erfcb{\frac{A\bbeta_i+B\bgamma_j}{\sqrt{2}}}\right]
\end{aligned}
\end{multline*}
\noindent (2) If $\bgamma_j = 0$, then
\begin{multline*}
\begin{aligned}
\hspace*{-0.2in}\frac{\partial f}{\partial \bbeta_i}(\ww^i,\vv^j)&= 
-\frac{\|\ww^i\|\|\vv^j\|}{2\sqrt{2\pi}}\left[\erfcb{\frac{\csc(\lambda_{ij})\bbeta_i}{\sqrt{2}}}+\cos(\lambda_{ij})e^{-\frac{\bbeta_i^2}{2}}\erfcb{\frac{-\cot(\lambda_{ij})\bbeta_i}{\sqrt{2}}}\right]
\end{aligned}
\end{multline*}
\end{lemma}
\begin{proof} Omitting sub- and superscripts, a straightforward computation gives
\begin{multline*}
\begin{aligned}
\frac{\partial f}{\partial \bbeta}(\ww,\vv)&= \frac{\|\ww\|\|\vv\|\bgamma}{2\pi}\left(\frac{\pi}{2}\left[\erfcb{\frac{\bgamma}{\sqrt{2}}}+\erfcb{\frac{\bbeta}{\sqrt{2}}}\right]  - \widetilde{\lambda}\right) \\
& -\frac{\|\ww\|\|\vv\|(\bbeta\bgamma + \cos(\lambda))}{2\sqrt{2\pi}}e^{-\frac{\bbeta^2}{2}}\erfcb{\frac{A\bbeta+B\bgamma}{\sqrt{2}}} \\
& -\frac{\|\ww\|\|\vv\|}{2\sqrt{2\pi}}\left[-\bbeta\bgamma e^{-\frac{\bbeta^2}{2}}\erfcb{\frac{A\bbeta+B\bgamma}{\sqrt{2}}}+e^{-\frac{\bgamma^2}{2}}\erfcb{\frac{A\bgamma+B\bbeta}{\sqrt{2}}}\right] \\
& + \frac{\|\ww\|\|\vv\|}{2\pi} \left[A\bgamma e^{-\frac{\bbeta^2}{2}}e^{-\frac{(A\bbeta+B\bgamma)^2}{2}}+B\bbeta e^{-\frac{\bgamma^2}{2}}e^{-\frac{(A\bgamma+B\bbeta)^2}{2}}\right] \\
& -  \frac{\|\ww\|\|\vv\|}{2\pi}\big(\bbeta \csc(\lambda) -\bgamma \cot(\lambda)\big)e^{-k(\bbeta,\bgamma,\lambda)}
\end{aligned}
\end{multline*}
After some cancellation---the last two lines and terms involving $\bbeta\bgamma$---the result follows.
\end{proof}
\begin{prop}\label{prop: lossbeta}
\mbox{ }\\
(1) Assuming $i \ne j$ and $\ww^i,\ww^j$ are not parallel,
\begin{multline*}
\begin{aligned}
\hspace*{0.2in}\frac{\partial f}{\partial \bbeta_i}(\ww^i,\ww^j)&= \frac{\|\ww^i\|\|\ww^j\|\bbeta_j}{2\pi}\left(\frac{\pi}{2}\left[\erfcb{\frac{\bbeta_i}{\sqrt{2}}}+\erfcb{\frac{\bbeta_j}{\sqrt{2}}}\right]  - \wLambda_{ij}\right)\\
&\hspace*{-0.6in} -\frac{\|\ww^i\|\|\ww^j\|}{2\sqrt{2\pi}}\left[e^{-\frac{\bbeta_j^2}{2}}\erfcb{\frac{A\bbeta_j+B\bbeta_i}{\sqrt{2}}}+\cos(\Lambda_{ij})e^{-\frac{\bbeta_i^2}{2}}\erfcb{\frac{A\bbeta_i+B\bbeta_j}{\sqrt{2}}}\right]
\end{aligned}
\end{multline*}
(2) If $\bbeta_i = \bbeta_j = \bbeta$, then
\begin{multline*}
\begin{aligned}
\hspace*{0.2in}\frac{\partial f}{\partial \bbeta}(\ww^i,\ww^j)&= \frac{\|\ww^i\|\|\ww^j\|\bbeta}{\pi}\left(\pi\,\erfcb{\frac{\bbeta}{\sqrt{2}}} - \wLambda_{ij}\right)\\
&\hspace*{-0.6in} -\frac{\|\ww^i\|\|\ww^j\|}{\sqrt{2\pi}}\left[e^{-\frac{\bbeta^2}{2}}\erfcb{\frac{\bbeta(-\cot(\Lambda_{ij})+\csc(\Lambda_{ij})}{\sqrt{2}}}(1+\cos(\Lambda_{ij}))\right]
\end{aligned}
\end{multline*}
\end{prop}
\begin{proof} The proof of (1) is formally identical to that of Lemma~\ref{EQ:betaderv}(1). The proof of (2) amounts to adding the derivatives of $f$ with respect to $\bbeta_i$ and $\bbeta_j$ and then setting 
$\bbeta,\bbeta_j = \bbeta$. Of special interest is the term $\frac{\partial \wLambda_{ij}}{\partial \bbeta}$. It is shown in section~\ref{Sec: Owen} that
\[
\wLambda_{ij} = S(\bbeta,\bbeta,\Lambda_{ij}) + \frac{\pi}{2} = H(\bbeta,\bbeta,\Lambda_{ij}) + \frac{\pi}{2}\erfcb{\frac{\bbeta}{\sqrt{2}}}
\]
and so 
\[
\frac{\partial \wLambda_{ij}}{\partial \bbeta}  = -2\sqrt{\frac{\pi}{2}}e^{-\frac{\bbeta^2}{2}}\erfb{\frac{\bbeta(-\cot(\Lambda_{ij})+\csc(\Lambda_{ij}))}{\sqrt{2}}}
\]
by Proposition~\ref{prop: spec}, \Refb{EQ: bderiv}.  Adding the derivative of $\pi\,\erfcb{\frac{\bbeta}{\sqrt{2}}}$ to $-\frac{\partial \wLambda_{ij}}{\partial \bbeta}$ gives the coefficient of $\cos(\Lambda_{ij})$ in the displayed formula. 
\end{proof}
\begin{exams}\label{EX: dervbeta}
(1) If $\bbeta = \bgamma = 0$, then
\[
\frac{\partial f}{\partial \bbeta}(\ww,\vv) =-\frac{\|\ww\|\|\vv\|}{2\sqrt{2\pi}}(1+\cos(\lambda)) 
\]
If $\vv = \ww$ and $\bgamma = \bbeta$, it follows from Example~\ref{ex: simp} that if $\bbeta = 0$ then
\[
\frac{\partial f}{\partial \bbeta}(\ww,\ww) =-\frac{2\|\ww\|^2}{\sqrt{2\pi}}
\]
(2) Suppose $\ww,\vv \ne 0$ are parallel. Using Lemma~\ref{lem: ITlim}, we have
\begin{enumerate}
\item  If $\lambda = 0$:
\[
\frac{\partial f}{\partial \bbeta} = \begin{cases}
& \frac{\|\ww\|\|\vv\|\bgamma}{2} \erfcb{\frac{\bbeta}{\sqrt{2}}}
 -\frac{\|\ww\|\|\vv\|}{\sqrt{2\pi}}e^{-\frac{\bbeta^2}{2}},\;\text{if } \bbeta \ge \bgamma \\
& \frac{\|\ww\|\|\vv\|\bgamma}{2} \erfcb{\frac{\bgamma}{\sqrt{2}}}-\frac{\|\ww\|\|\vv\|}{\sqrt{2\pi}}e^{-\frac{\bgamma^2}{2}},\;\text{if } \bbeta \le \bgamma
\end{cases}
\]
\item  If $\lambda = \pi$:
\[
\frac{\partial f}{\partial \bbeta} = \begin{cases}
& \frac{\|\ww\|\|\vv\|\bgamma}{2}\left[\erfb{\frac{\bbeta}{\sqrt{2}}}-\erfb{\frac{\bgamma}{\sqrt{2}}}\right]+ \frac{\|\ww\|\|\vv\|}{\sqrt{2\pi}}e^{-\frac{\bbeta^2}{2}},\; \text{if } \bbeta + \bgamma \ge 0\\
& -\|\ww\|\|\vv\|\bgamma \erfb{\frac{\bgamma}{\sqrt{2}}} + \frac{\|\ww\|\|\vv\|}{\sqrt{2\pi}}e^{-\frac{\bgamma^2}{2}},\;\text{if } \bbeta + \bgamma \le 0
\end{cases}
\]
\end{enumerate}
\examend
\end{exams}
\begin{rem}\label{rem: fC1beta}
It follows from the last example that $f$ is a $C^1$ function of $\bbeta$ (or $\beta$) if $\lambda \in \{0,\pi\}$ and $f$ is analytic in $\bbeta$ (or $\beta$) if $\lambda  =0$ (resp.~$\lambda  =\pi$) provided that
$ \bbeta \ne \bgamma $ (resp.~$ \bbeta \ne -\bgamma $). Of course, real analyticity of $f$ in $\bbeta$ holds without restriction if $\lambda\in(0,\pi)$ and $\ww,\vv \ne 0$. \rend 
\end{rem} 
\section{Computation of $\frac{\partial f}{\partial \ww}$ (normalized bias $\bbeta$)}\label{sec: sec9}
We supplement the list of derivatives given in Section~\ref{sec: sec4} before Lemma~\ref{lem: Jlemma} with
\begin{multline*}
\begin{aligned}
\frac{\partial}{\partial \ww}  \erfb{\frac{A\bgamma + B\bbeta}{\sqrt{2}}} =& 
 \sqrt{\frac{2}{\pi}}\frac{e^{-(A\bgamma+B\bbeta)^2/2}}{\|\ww\|}\big[\csc(\lambda)\left(A\bbeta + B\bgamma\right)\big]\jj  \\
\frac{\partial}{\partial \ww}  \erfb{\frac{A\bbbeta + B\bgamma}{\sqrt{2}}} =& 
 \sqrt{\frac{2}{\pi}}\frac{e^{-(A\bbeta+B\bgamma)^2/2}}{\|\ww\|}\big[\csc(\lambda)\left(A\bgamma + B\bbeta\right)\big]\jj  \\
 \frac{\partial k}{\partial \ww}(\bbeta,\bgamma,\lambda) = & 
 \frac{\csc(\lambda)}{\|\ww\|} \left[(A\bbeta + B\bgamma)(A\bgamma + B\bbeta)\right]\jj\\
\frac{\partial}{\partial \ww}e^{-k}  =& -\frac{\partial k}{\partial \ww}e^{-k},\quad \frac{\partial H}{\partial \ww} =\frac{e^{-k(\bbeta,\bgamma,\lambda)}}{\|\ww\|}\jj 
\end{aligned}
\end{multline*}
($A = -\cot(\lambda), B = \csc(\lambda)$ and $\lambda$ is the angle between $\ww$ and $\vv$.)
\begin{eqnarray*}
\frac{\partial  \wlambda_{ij}}{\partial \bbeta_i}& =& -\sqrt{\frac{\pi}{2}}e^{-\frac{\bbeta_i^2}{2}}\erfb{\frac{-\cot(\lambda_{ij})\bbeta_i + \csc(\lambda_{ij})\bgamma_j}{\sqrt{2}}}\quad
\frac{\partial  \wlambda_{ij}}{\partial \ww} = \frac{e^{-k(\bbeta_i,\bgamma_j,\lambda_{ij})}}{\|\ww\|}\jj \\
\frac{\partial  \wLambda_{ij}}{\partial \bbeta_i} & = & -\sqrt{\frac{\pi}{2}}e^{-\frac{\bbeta_i^2}{2}}\erfb{\frac{-\cot(\Lambda_{ij})\bbeta_i + \csc(\Lambda_{ij})\bbeta_j}{\sqrt{2}}}\quad
\frac{\partial  \wLambda_{ij}}{\partial \ww} = \frac{e^{-K(\bbeta_i,\bbeta_j,\Lambda_{ij})}}{\|\ww\|}\JJ_{ij}
\end{eqnarray*}
Here $j \in \is{d}, i\in \is{k}$ and $\lambda_{ij}$ (resp.~$\Lambda_{ij}$) is the angle between $\ww^i$ and $\vv^j$ (resp.~$\ww^j$, $j \ne i$). Note that the partial derivatives of $\wLambda_{ij}$ with respect to
$\bbeta_i,\bbeta_j$ will not generally be equal unless $\bbeta_i=\bbeta_j$ (see Section~\ref{sec8: betederivoff} for more details and note that the $\ww$-derivatives of $\wLambda_{ij},\wLambda_{ji}$ are equal)).

\subsection{The formula for $\frac{\partial f}{\partial \ww}(\ww,\vv)$ (normalized bias $\bbeta$)}\label{sec: Glob}.
A straightforward computation of the gradient gives
\begin{multline*}
\begin{aligned}
\frac{\partial f}{\partial \ww}  = &\frac{\|\vv\|}{2\pi}\left(\frac{\pi}{2}\left[\erfcb{\frac{\bgamma}{\sqrt{2}}}+\erfcb{\frac{\bbeta}{\sqrt{2}}}\right]  - \widetilde{\lambda}\right)\big[(\bbeta\bgamma + \cos(\lambda))\overline{\ww}-\sin(\lambda)\jj\big]\\
& -\frac{\|\vv\|}{2\pi}e^{-k}\big(\bbeta\bgamma + \cos(\lambda)\big)\jj\\
& - \frac{\|\vv\|}{2\sqrt{2\pi}}\left[\bbeta e^{-\frac{\bgamma^2}{2}}\erfcb{\frac{A\bgamma + B\bbeta}{\sqrt{2}}} +\bgamma e^{-\frac{\bbeta^2}{2}}\erfcb{\frac{A\bbeta + B\bgamma}{\sqrt{2}}}\right]\overline{\ww} \\
& + \frac{\|\vv\|e^{-k}}{2\pi}\big[\bgamma \csc(\lambda)(A\bgamma+B\bbeta) + \bbeta \csc(\lambda)(A\bbeta+B\bgamma)\big]\jj \\
& +\frac{\|\vv\|}{2\pi}e^{-k}\big(\sin(\lambda)\overline{\ww} + \cos(\lambda)\jj\big) \\
& -\frac{\|\vv\|e^{-k}}{2\pi}\big[(A\bbeta + B\bgamma)(A\bgamma + B\bbeta)\big]\jj
\end{aligned}
\end{multline*}
In the last line, the $\csc(\lambda)$ term that comes from $\frac{\partial k}{\partial \ww}$ is multiplied by $\sin(\lambda)$ and so the final expression is quadratic in both $\bbeta,\bgamma$ and $\csc(\lambda)$ and
$\cot(\lambda)$, as is the corresponding expression on line 4.
After some rearrangement and two major cancellations---most notably lines 4 \& 6 together with the first term in line 2---we obtain
\begin{prop}\label{propgen}
(Notation and assumptions as above.)
\begin{multline*}
\begin{aligned}
\frac{\partial f}{\partial \ww}  = &\frac{\|\vv\|}{2\pi} \left(\frac{\pi}{2}\left[\erfcb{\frac{\bgamma}{\sqrt{2}}}+\erfcb{\frac{\bbeta}{\sqrt{2}}}\right]  - \widetilde{\lambda}\right) (\bbeta\bgamma \bww +\bvv)\\
& +\frac{\|\vv\|}{2\pi}e^{-k}\sin(\lambda)\bww \\
& - \frac{\|\vv\|}{2\sqrt{2\pi}}\left[\bbeta e^{-\frac{\bgamma^2}{2}}\erfcb{\frac{A\bgamma + B\bbeta}{\sqrt{2}}} +\bgamma e^{-\frac{\bbeta^2}{2}}\erfcb{\frac{A\bbeta + B\bgamma}{\sqrt{2}}}\right]\overline{\ww} 
\end{aligned}
\end{multline*}
\end{prop}
\begin{rems}
(1) If we take $\bbeta = \bgamma = 0$ (zero bias), then the last three terms are zero and the first two terms give the formula for grad $f(\ww,\vv)$ derived by Brutzkus and Globerson~\cite{BrutzkusGloberson2017}:
\[
\frac{\partial f}{\partial \ww}  = \frac{\|\vv\|}{2\pi}\big( \sin(\lambda) \bww + (\pi - \lambda) \bvv \big)
\]
The computation uses $H(0,0,\lambda) = \lambda - \frac{\pi}{2}$. Note that in Brutzkus and Globerson the second term is written $\frac{1}{2\pi}(\pi-\lambda)\vv$.  \\
(2) Terms in $\jj$ do not appear. This is because of the cancellation of lines 4 \& 6 together with the first term in line 2. The only place where $\jj$ 
is expanded as $\cot(\lambda)\bww - \csc(\lambda) \bvv$ is in line 1. This results in the replacement of the factor $(\bbeta\bgamma + \cos(\lambda))\overline{\ww}-\sin(\lambda)\jj)$ by
$(\bbeta\bgamma \bww +\bvv)$.  Thus the gradient  is a linear combination of $\bvv$ and $\bww$. That is not a surprise since $\jj$ is a linear combination of $\bvv$ and $\bww$ and also, of course, since 
the gradient $\frac{\partial f}{\partial \ww} $ is obviously a point in the space spanned by $\ww,\vv$. What is surprising is the relative simplicity of the formula for  $\frac{\partial f}{\partial \ww}$ together 
with the absence of products and powers of $\cot(\lambda)$ and $\csc(\lambda)$ except for the exponential term $e^{-k}$.
\rend
\end{rems}
All examples in this paper assume $\bgamma = 0$ and, for future reference, the next proposition gives the gradients of $f(\ww^i,\vv^j)$ and $f(\ww^i,\ww^j), i \ne j$ under the 
assumption that $\bgamma = 0$.
\begin{prop}\label{prop: gradgen}
Assume $\bgamma = 0$ and follow the notation conventions of section~\ref{sec: notconv}.
\begin{enumerate}
\item For $i \in \is{k}, j \in \is{d}$:
\begin{multline*}
\begin{aligned}
\frac{\partial f}{\partial \ww}(\ww^i,\vv^j)  = &\frac{\|\vv^j\|}{2\pi} \left[ \pi - \wlambda_{ij} - \frac{\pi}{2}\erfb{\frac{\bbeta_i}{\sqrt{2}}}\right]\bvv^j  \\
& +\frac{\|\vv^j\|}{2\pi}\left[e^{-k_{ij}}\sin(\lambda_{ij})- \sqrt{\frac{\pi}{2}} \bbeta_i \erfcb{\frac{\csc(\lambda_{ij})\bbeta_i}{\sqrt{2}}}\right]\overline{\ww}^i 
\end{aligned}
\end{multline*}
where $k_{ij} = \frac{1}{2}\bbeta_i^2 \csc^2(\lambda_{ij})$. 
\item For $i,j \in \is{k}$, $i \ne j$:
\begin{multline*}
\begin{aligned}
\frac{\partial f}{\partial \ww}(\ww^i,\ww^j)  = &\frac{\|\ww^j\|}{2\pi} \left[\pi-\wLambda_{ij} -\frac{\pi}{2}\erfb{\frac{\bbeta_i}{\sqrt{2}}}-\frac{\pi}{2}\erfb{\frac{\bbeta_j}{\sqrt{2}}} \right] (\bbeta_i\bbeta_j\bww^i +\bww^j) \\
 +\frac{\|\ww^j\|}{2\pi}&e^{-K_{ij}}\sin(\Lambda_{ij})\bww^i \\
 - \frac{\|\ww^j\|}{2\sqrt{2\pi}}&\left[\bbeta_i e^{-\frac{\bbeta_j^2}{2}}\erfcb{\frac{Z_{ij}}{\sqrt{2}}} +\bbeta_j e^{-\frac{\bbeta_i^2}{2}}\erfcb{\frac{\wZZ_{ij}}{\sqrt{2}}}\right]\overline{\ww}^i 
\end{aligned}
\end{multline*}
where $Z_{ij} = -\cot(\Lambda_{ij})\bbeta_j +\csc(\Lambda_{ij})\bbeta_i = \wZZ_{ji}$.
\item For $i \in \is{k}$,
\begin{eqnarray*}
\frac{\partial f}{\partial \bbeta_i}(\ww^i,\ww^i)& = & \left[1 + \bbeta_i^2\erfcb{\frac{\beta_i}{\sqrt{2}}} -\erfb{\frac{\beta_i}{\sqrt{2}}} -\sqrt{\frac{2}{\pi}}\bbeta_i e^{-\frac{\bbeta_i^2}{2}}\right]\ww^i
\end{eqnarray*}
\end{enumerate}
\end{prop}
\begin{proof} For the $\bbeta$-derivative of $f(\ww,\ww)$ use \Refb{EQ: fww}. \end{proof}

\begin{exam}
Take $\bgamma = 0$ and recall from Lemma~\ref{EQ:betaderv}(2) that 
\begin{multline*}
\begin{aligned}
\hspace*{-0.2in}\frac{\partial f}{\partial \bbeta}(\ww,\vv)&= 
-\frac{\|\ww\|\|\vv\|}{2\sqrt{2\pi}}\left[\erfcb{\frac{\csc(\lambda)\bbeta}{\sqrt{2}}}+\cos(\lambda)e^{-\frac{\bbeta^2}{2}}\erfcb{\frac{-\cot(\lambda)\bbeta}{\sqrt{2}}}\right]
\end{aligned}
\end{multline*}
where we have set $\ww^i = \ww, \vv^j = \vv$ and $\lambda_{ij} = \lambda$, Computing $\frac{\partial^2 f}{\partial\ww \partial \bbeta}$ gives
\begin{multline*}
\begin{aligned}
\frac{\partial^2 f}{\partial\ww \partial \bbeta} = & -\frac{\|\vv\|}{2\sqrt{2\pi}}\left[\erfcb{\frac{\csc(\lambda)\bbeta}{\sqrt{2}}}+ \cos(\lambda)e^{-\frac{\bbeta^2}{2}}\erfcb{\frac{-\cot(\lambda)\bbeta}{\sqrt{2}}} \right]\bww \\
& + \frac{\|\vv\|}{2\sqrt{2\pi}}\left[\sin(\lambda)e^{-\frac{\bbeta^2}{2}}\erfcb{\frac{-\cot(\lambda)\bbeta}{\sqrt{2}}}\right]\jj \\
=&-\frac{\|\vv\|}{2\sqrt{2\pi}}\left[\erfcb{\frac{\csc(\lambda)\bbeta}{\sqrt{2}}}\bww +\frac{\|\vv\|}{2\sqrt{2\pi}}e^{-\frac{\bbeta^2}{2}}\erfcb{\frac{-\cot(\lambda)\bbeta}{\sqrt{2}}}\bvv\right]
\end{aligned}
\end{multline*}
It is a good exercise/check to verify that $\frac{\partial^2 f}{\partial\ww \partial \bbeta} = \frac{\partial^2 f}{\partial\bbeta \partial \ww}$.  \examend
\end{exam}

It follows from Propositions~\ref{propgen}, ~\ref{prop: gradgen} that 
\begin{eqnarray}\label{EQ: GenGlob}
\frac{\partial f}{\partial \ww}(\ww^i,\vv^j)&=&\frac{\|\vv\|}{2\pi} \mathcal{U}_{ij}\bww^i + \frac{1}{2\pi}\mathcal{V}_{ij} \vv^j,\; i \in \is{k}, j \in \is{d}\\
\frac{\partial f}{\partial \ww}(\ww^i,\ww^j)&=&\frac{\|\vv\|}{2\pi} \mathcal{W}_{ij}\bww^i + \frac{1}{2\pi}\mathcal{Y}_{ij} \ww^j,\; i,j \in \is{k}, i \ne j
\end{eqnarray}
where
\begin{multline*}
\begin{aligned}
\mathcal{U}_{ij} = & -\sqrt{\frac{\pi}{2}}\left[\bbeta_i e^{-\frac{\bgamma_j^2}{2}}\erfcb{\frac{z_{ij}}{\sqrt{2}}}+\bgamma_j e^{-\frac{\bbeta_i^2}{2}}\erfcb{\frac{\wzz_{ij}}{\sqrt{2}}}\right] \\
& +\bbeta_i\bgamma_j \left(\frac{\pi}{2}\left[\erfcb{\frac{\bgamma_j}{\sqrt{2}}}+\erfcb{\frac{\bbeta_i}{\sqrt{2}}}\right] - \wlambda_{ij}\right)+e^{-k_{ij}}\sin(\lambda_{ij})\\
\mathcal{V}_{ij} = & \left(\frac{\pi}{2}\left[\erfcb{\frac{\bgamma_j}{\sqrt{2}}}+\erfcb{\frac{\bbeta_i}{\sqrt{2}}}\right] - \wlambda_{ij}\right) \\
\mathcal{W}_{ij} = & -\sqrt{\frac{\pi}{2}}\left[\bbeta_i e^{-\frac{\bbeta_j^2}{2}}\erfcb{\frac{Z_{ij}}{\sqrt{2}}}+\bbeta_j e^{-\frac{\bbeta_i^2}{2}}\erfcb{\frac{\wZZ_{ij}}{\sqrt{2}}}\right] \\
& +\bbeta_i\bbeta_j \left(\frac{\pi}{2}\left[\erfcb{\frac{\bbeta_j}{\sqrt{2}}}+\erfcb{\frac{\bbeta_i}{\sqrt{2}}}\right] - \wLambda_{ij}\right)+e^{-K_{ij}}\sin(\Lambda_{ij})\\
\mathcal{Y}_{ij} = & \left(\frac{\pi}{2}\left[\erfcb{\frac{\bgamma_j}{\sqrt{2}}}+\erfcb{\frac{\bbeta_i}{\sqrt{2}}}\right] - \wLambda_{ij}\right) 
\end{aligned}
\end{multline*}
(Notational conventions of section~\ref{sec: notconv}.)

\subsection{Angles close to $0$ and $\pi$} 
In the examples studied in this paper, the angle between two rows of the parameter matrix $\WW$, or between two rows of the target matrix $\VV$, at a critical point of 
the loss will be bounded away from $0$ and $\pi$.  However, angles between a row of $\WW$ and a row of $\VV$ may be close to $0$ or $\pi$. Indeed, for $d$ dependent families of critical points, 
it is often the case that each $\WW$ row converges to $\pm$ a $\VV$ row as $d \arr\infty$. This has the potential to cause problems with the exponential factors $e^{-k_{ij}}$ (see the remark below). 
The next two lemmas show that under reasonable hypotheses this potential pathology can be controlled.
\begin{lemma}
Assume $\theta \in (0,\pi)$ and is either near zero or $\pi$.
\begin{enumerate}
\item $\csc(\theta) - \cot(\theta) = \tan\left(\frac{\theta}{2}\right) = \frac{\theta}{2}+\frac{\theta^3}{24} + O(\theta^5)$.
\item $\csc(\pi-\theta) + \cot(\pi-\theta) = \tan\left(\frac{\pi-\theta}{2}\right) = \frac{(\pi-\theta)}{2}+\frac{(\pi-\theta)^3}{24} + O((\pi-\theta)^5)$.
\end{enumerate}
\end{lemma}
\begin{proof} Elementary trigonometry and the Maclaurin series for $\tan$. \end{proof}

\begin{lemma}\label{lem: CSL}
Set $\bdelta_\pm =  \bbeta \pm \bgamma$, $\Delta_\pm = \csc(\theta) \pm cot(\theta)$, $z = -\bgamma\cot(\theta)+\bbeta \csc(\theta)$ and $\wzz = -\bbeta\cot(\theta)+\bgamma\csc(\theta)$.  
\begin{enumerate}
\item  For $\theta$ near zero
\begin{eqnarray*} 
\bgamma z + \bbeta\wzz&=&  \bbeta\bgamma \Delta_- - \bdelta_-^2 \cot(\theta) 
\end{eqnarray*}
\item  For $\pi-\theta$ near zero
\begin{eqnarray*} 
\bgamma z + \bbeta\wzz&=&  \bbeta\bgamma \Delta_+ -  \bdelta_+^2 \cot(\theta) 
\end{eqnarray*}
\end{enumerate}
\end{lemma}
\begin{proof} Straightforward substitution. \end{proof}
\begin{rem}
Lemma~\ref{lem: CSL} is important as it implies the Curve Selection Lemma~\cite{BierMil1998,Milnor1968} (analytic arc, limiting at boundary point) will hold in cuspidal regions with cusp vertices at $\lambda = 0$ or $\lambda = \pi$) and 
cusps tangent at the vertex to $\bgamma- \bbeta = 0$, or $\bgamma+\bbeta = 0$, even though there is the exponential factor $e^{-k}$. 
In general, this is \emph{not} so for o-minimal structures where practical computability also appears to be extremely challenging if we leave the (sub)analytic domain. In other words, we do not expect to see infinite (in $d$) 
(analytic) families of critical points if we get too close to points where subanalyticity breaks down. 
The result suggests analytically why (spurious) minima in symmetric and asymmetric systems are often 
characterized by convergence of rows to $\pm$ a target row as $d \arr \infty$. More generally convergence can be to a 
non-zero row which is parallel to more than one target row (when $k > d$~\cite[Example 2(2)]{ArjevaniField2022b}).  
Possibly adding bias, and so breaking subanalyticity 
for the equations, is an advantage as it makes it harder to have convergence to a row which is not parallel to a target row. Perhaps the word `spurious' should be reserved for critical point families that break down for large $d$. \rend
\end{rem}

\subsection{Computation of $\frac{\partial f}{\partial \ww}$ (bias $\beta$) I: derivatives}\label{sec: sbias}
We list the derivatives used that differ from those in the list for normalized bias. 
\begin{eqnarray*}
\frac{\partial \bbeta}{\partial \ww} & = & -\frac{\bbeta}{\|\ww\|} \overline{\ww},\quad 
\frac{\partial e^{-\frac{\bbeta^2}{2}}}{\partial \ww} = \frac{\bbeta^2 e^{-\frac{\bbeta^2}{2}}}{\|\ww\|} \overline{\ww} \\
\frac{\partial\, \erf{\bbeta/\sqrt{2}}}{\partial \ww} & = & -\sqrt{\frac{2}{\pi}} \frac{\bbeta e^{-\frac{\bbeta^2}{2}}}{\|\ww\|} \overline{\ww}\\
\end{eqnarray*}
\begin{multline*}
\begin{aligned}
\frac{\partial}{\partial \ww} & \erfb{\frac{A\bbeta + B\bgamma}{\sqrt{2}}}  =
 \sqrt{\frac{2}{\pi}} \frac{e^{-\frac{(A\bbeta + B\bgamma)^2}{2}}}{\|\ww\|} \left(\csc(\lambda)\left[A\bgamma + B\bbeta\right]\jj +\cot(\lambda)\bbeta\,\overline{\ww}\right) \\
\frac{\partial}{\partial \ww} & \erfb{\frac{A\bgamma + B\bbeta}{\sqrt{2}}}  = 
 \sqrt{\frac{2}{\pi}} \frac{e^{-\frac{(A\bgamma + B\bbeta)^2}{2}}}{\|\ww\|} \left(\csc(\lambda)\left[A\bbeta + B\bgamma\right]\jj -\csc(\lambda)\bbeta\,\overline{\ww}\right) \\
 \frac{\partial}{\partial \ww}& k(\bbeta,\bgamma,\lambda)  = 
 -\frac{\bbeta\csc(\lambda)}{\|\ww\|}\left[A\bgamma+B\bbeta\right]\overline{\ww}+ \frac{\csc(\lambda)}{\|\ww\|} \left[(A\bbeta + B\bgamma)(A\bgamma + B\bbeta)\right]\jj\\
\frac{\partial H}{\partial \ww}& = \frac{e^{-k(\bbeta,\bgamma,\lambda)}}{\|\ww\|}\jj 
-\sqrt{\frac{\pi}{2}}\frac{\bbeta e^{-\frac{\bbeta^2}{2}}}{\|\ww\|}\erfcb{\frac{A\bbeta + B\bgamma}{\sqrt{2}}}\overline{\ww}
\end{aligned}
\end{multline*}
\subsection{The computation of $\frac{\partial f}{\partial \ww}$ (bias $\beta)$: II}
Although our emphasis is on normalized bias, it is straightforward to give a formula for $\frac{\partial f}{\partial \ww}$ when the bias is $\beta$. 
If $\widetilde{\frac{\partial f}{\partial \ww}}$ denotes the gradient when the bias is $\beta$, then 
\begin{equation}\label{EQ: biases}
\widetilde{\frac{\partial f}{\partial \ww}}=\frac{\partial f}{\partial \ww} -\frac{\bbeta \bww}{\|\ww\|} \frac{\partial f}{\partial \bbeta} =\frac{\partial f}{\partial \ww} -\bbeta \bww \frac{\partial f}{\partial \beta}
\end{equation}
This can be proved straightforwardly using the formula for $\frac{\partial f}{\partial \ww}$ when the bias is $\bbeta$ or directly using the results listed in
section~\ref{sec: sbias}. 
\section{The formula for $\frac{\partial \mathcal{L}}{\partial \ww}$ (no symmetry assumption)}\label{sec: coeffs}
Assume given $\WW \in M(k,d)$ and a target matrix $\VV \in M(d,d)$. Denote row $i$ of $\WW$ (resp.~row $j$ of $\vv$) by $\ww^i$ (resp.~$\vv^j$).
We extend the notational conventions developed at the end of section~\ref{sec: bias}. Recall that
$\lambda_{ij}$ denotes the angle between $\ww^i$ and $\vv^j$, $i\in\is{k}, j\in \is{d}$, and $\Lambda_{ij}$ denotes the angle between $\ww^i$ and $\ww^j$, $i,j \in \is{k}$. 
Assume no parallel rows. In particular:  $\lambda_{ij} \notin \{0,\pi\}$, $i \in \is{k}, j \in \is{d}$ and  $\Lambda_{ij}\notin \{0,\pi\}$,  $i,j \in \is{k},i \ne j$.

\subsection{Table defining terms used in the formula for $\grad{\mathcal{L}}(\WW)$}\label{SEC: table}
\begin{eqnarray*}
Z_{ij}&=&-\cot(\Lambda_{ij})\bbeta_j + \csc(\Lambda_{ij})\bbeta_i\quad \wZZ_{i j} = -\cot(\Lambda_{ij})\bbeta_i + \csc(\Lambda_{ij})\bbeta_j \\
z_{ij}&=&-\cot(\lambda_{ij})\bgamma_j + \csc(\lambda_{ij})\bbeta_i\quad \wzz_{i j} = -\cot(\lambda_{ij})\bbeta_i + \csc(\lambda_{ij})\bgamma_j \\
K_{ij} & = & K(\bbeta_i,\bbeta_j,\Lambda_{ij})= \frac{1}{2}\left[(\bbeta_i^2+\bbeta_j^2)\csc^2(\Lambda_{ij})-2\bbeta_i\bbeta_j\csc(\Lambda_{ij})\cot(\Lambda_{ij})\right]\\
k_{ij} & = & k(\bbeta_i,\bgamma_j,\lambda_{ij})= \frac{1}{2}\left[(\bbeta_i^2+\bgamma_j^2)\csc^2(\lambda_{ij})-2\bbeta_i\bgamma_j\csc(\lambda_{ij})\cot(\lambda_{ij})\right]\\
L_{ij} & = & -\frac{\pi}{2}\left[\erfb{\frac{\bbeta_i}{\sqrt{2}}}+\erfb{\frac{\bbeta_j}{\sqrt{2}}}\right],\;\ell_{ij}= -\frac{\pi}{2}\left[\erfb{\frac{\bbeta_i}{\sqrt{2}}}+\erfb{\frac{\bgamma_j}{\sqrt{2}}}\right]\\
E^X_{ij} & = & -\sqrt{\frac{\pi}{2}}\left[\bbeta_i e^{-\frac{\bbeta_j^2}{2}}\erfcb{\frac{Z_{ij}}{\sqrt{2}}}+\bbeta_j e^{-\frac{\bbeta_i^2}{2}}\erfcb{\frac{\wZZ_{ij}}{\sqrt{2}}}\right]\\
e^X_{ij} & = & -\sqrt{\frac{\pi}{2}}\left[\bbeta_i e^{-\frac{\bgamma_j^2}{2}}\erfcb{\frac{z_{ij}}{\sqrt{2}}}+\bgamma_j e^{-\frac{\bbeta_i^2}{2}}\erfcb{\frac{\wzz_{ij}}{\sqrt{2}}}\right] \\
X_i & = & \bbeta_i^2\erfcb{\frac{\bbeta_i}{\sqrt{2}}}-\erfb{\frac{\bbeta_i}{\sqrt{2}}} - \sqrt{\frac{2}{\pi}}\bbeta_ie^{-\frac{\bbeta_i^2}{2}}
\end{eqnarray*}
The capitalized terms are defined for $i,j \in \is{k}$, $j \ne i$. Lower case terms are defined for $i \in \is{k}$, $j \in \is{d}$.
The terms $K_{ij},L_{ij}, E_{ij}^X$ are symmetric in $i,j$ and, with the exception of $L_{ij}$, 
the diagonal terms are undefined---though it is safe to set these terms (apart from $L_{ii}$) to zero, which is the limiting value as $\Lambda_{ii} \arr 0$.
Note that $Z_{ji} = \wZZ_{i j}$ but that $z_{ji}$ is generally \emph{not} equal to $\wzz_{i j}$ unless, for example, $\lambda_{ij} = \lambda_{ji}$ ($\lambda_{ji}$ is not even defined if $i \notin \is{d}$).  

\subsection{The terms $\mathcal{U}_{ij},\ldots,\mathcal{Y}_{ij}$}\label{sec: Uij}
From Section~\ref{sec: sec9}   we have
\begin{eqnarray*}
\frac{\partial f}{\partial \ww}(\ww^i,\vv^j)&=& \frac{\|\vv^j\|}{2\pi}\mathcal{U}_{ij} \bww^i + \frac{1}{2\pi}\mathcal{V}_{ij} \vv^j,\; i \in \is{k}, j \in \is{d}\\
\frac{\partial f}{\partial \ww}(\ww^i,\ww^j)&=& \frac{\|\ww^j\|}{2\pi}\mathcal{W}_{ij} \bww^i +  \frac{1}{2\pi}\mathcal{Y}_{ij} \ww^j,\; i,j \in \is{k}, i\ne j \\
\frac{\partial f}{\partial \ww^i}(\ww^i,\ww^i)&=& \mathcal{X}_i\ww^i, \;i\in \is{k} 
\end{eqnarray*}
where 
\begin{multline*}
\begin{aligned}
\mathcal{U}_{ij} = & -\sqrt{\frac{\pi}{2}}\left[\bbeta_i e^{-\frac{\bgamma_j^2}{2}}\erfcb{\frac{z_{ij}}{\sqrt{2}}}+\bgamma_j e^{-\frac{\bbeta_i^2}{2}}\erfcb{\frac{\wzz_{ij}}{\sqrt{2}}}\right] \\
& +\bbeta_i\bgamma_j \left(\frac{\pi}{2}\left[\erfcb{\frac{\bgamma_j}{\sqrt{2}}}+\erfcb{\frac{\bbeta_i}{\sqrt{2}}}\right] - \wlambda_{ij}\right)+e^{-k_{ij}}\sin(\lambda_{ij})\\
& = e_{ij}^X + \bbeta_i\bgamma_j\big(\ell_{ij} + \pi - \wlambda_{ij}\big) + e^{-k_{ij}}\sin(\lambda_{ij})\\
\mathcal{V}_{ij} = & \left(\frac{\pi}{2}\left[\erfcb{\frac{\bgamma_j}{\sqrt{2}}}+\erfcb{\frac{\bbeta_i}{\sqrt{2}}}\right] - \widetilde{\lambda}_{ij}\right)\\
 = & \ell_{ij} + \pi - \widetilde{\lambda}_{ij} \\
\mathcal{W}_{ij} =
& -\sqrt{\frac{\pi}{2}}\left[\bbeta_i e^{-\frac{\bbeta_j^2}{2}}\erfcb{\frac{Z_{ij}}{\sqrt{2}}}+\bbeta_j e^{-\frac{\bbeta_i^2}{2}}\erfcb{\frac{\wZZ_{ij}}{\sqrt{2}}}\right] \\
& +\bbeta_i\bbeta_j \left(\frac{\pi}{2}\left[\erfcb{\frac{\bbeta_j}{\sqrt{2}}}+\erfcb{\frac{\bbeta_i}{\sqrt{2}}}\right] - \wLambda_{ij}\right)+e^{-K_{ij}}\sin(\Lambda_{ij})\\
& = E_{ij}^X + \bbeta_i\bbeta_j\big(L_{ij} + \pi - \wLambda_{ij}\big) + e^{-K_{ij}}\sin(\Lambda_{ij})\\
\mathcal{Y}_{ij} = & \frac{\pi}{2}\left[\erfcb{\frac{\bbeta_i}{\sqrt{2}}}+\erfcb{\frac{\bbeta_j}{\sqrt{2}}}\right] - \widetilde{\Lambda}_{ij} \\
 = & L_{ij} + \pi - \widetilde{\Lambda}_{ij} \\
\mathcal{X}_i  =& (\bbeta_i^2 +1)\erfcb{\frac{\bbeta_i}{\sqrt{2}}} - \sqrt{\frac{2}{\pi}}\bbeta_ie^{-\frac{\bbeta_i^2}{2}}
= 1+X_i
\end{aligned}
\end{multline*}
For $i \in \is{k}$, denote component $i$ of $\grad{\mathcal{L}}(\WW)$ by $\grad{\mathcal{L}} (\WW)^i$ (a row vector). Then 
\begin{multline*}
\begin{aligned}
\grad{\mathcal{L}} (\WW)^i& = \frac{\mathcal{X}_i}{2}\ww^i 
+ \frac{1}{2\pi}\left(\sum_{j \in \is{k},j \ne i} \frac{\|\ww^j\|\mathcal{W}_{ij}}{\|\ww^i\|}-\sum_{j \in \is{d}} \frac{\|\vv^j\|\mathcal{U}_{ij}}{\|\ww^i\|}\right)\ww^i\\
&\hspace*{0.6in}+\frac{1}{2\pi}\left(\sum_{j\in\is{k},j \ne i} \mathcal{Y}_{ij} \ww^j - \sum_{j \in \is{d}} \mathcal{V}_{ij} \vv^j\right) 
\end{aligned}
\end{multline*}
Using the expressions for $\mathcal{U}_{ij},\ldots,\mathcal{Y}_{ij}$ and $\mathcal{X}_i$, we  may write
\begin{equation}
\grad{\mathcal{L}} (\WW)^i =\Gamma_i\ww^i -\mathbf{A}^i + \mathbf{B}^i+ \bOmega
\end{equation}
where $\grad{\mathcal{L}} (\WW)^i$ is the $i$th row of $\grad{\mathcal{L}} (\WW) \in M(k,d)$ and  
\begin{eqnarray*}
\Gamma_i & = &\frac{X_i}{2} + \frac{1}{2\pi}\left[\sum_{j\in\is{k},j\ne i}\frac{\|\ww^j\| e^{-K_{ij}}\sin(\Lambda_{ij})}{\|\ww^i\|}-\sum_{j\in\is{d}}\frac{\|\vv^j\| e^{-k_{ij}}\sin(\lambda_{ij})}{\|\ww^i\|}\right]\\
&& +\frac{1}{2\pi}\left[\sum_{j\in\is{k},j\ne i}\frac{\|\ww^j\|E_{ij}^X}{\|\ww^i\|}-\sum_{j\in\is{d}}\frac{\|\vv^j\|e_{ij}^X}{\|\ww^i\|}\right] \\
&& +\frac{\bbeta_i}{2\pi}\left[\sum_{j\in\is{k},j\ne i}\frac{\bbeta_j\|\ww^j\|\big(L_{ij}+(\pi-\wLambda_{ij})\big)}{\|\ww^i\|} 
-\sum_{j\in\is{d}}\frac{\bgamma_j\|\vv^j\|(\ell_{ij} +(\pi-\wlambda_{ij}))}{\|\ww^i\|}\right] \\
\is{A}^i& = &\frac{1}{2\pi}\left(\sum_{j\in\is{k},j\ne i}\wLambda_{ij}\ww^j - \sum_{j \in \is{d}}\wlambda_{ij}\vv^j\right)\quad
\is{B}^i =  \frac{1}{2\pi}\left(\sum_{j\in\is{k},j\ne i}L_{ij}\ww^j-\sum_{j \in \is{d}}\ell_{ij}\vv^j \right)\\
\bOmega & = & \frac{1}{2}\left(\sum_{j\in\is{k}}\ww^j - \sum_{j \in \is{d}}\vv^j\right)
\end{eqnarray*}

Setting all biases to zero, $\mathcal{U}_{ij} = \sin(\lambda_{ij})$, $\mathcal{V}_{ij} = \pi-\lambda_{ij}$, $\mathcal{W}_{ij} = \sin(\Lambda_{ij})$, $\mathcal{Y}_{ij} = \pi-\Lambda_{ij}$, $X_i = 0$, 
and $\is{B}^i = \mathbf{0}$.
Substitution in the formula for $\grad{\mathcal{L}} (\WW)^i$ gives the result for the bias free case~\cite{BrutzkusGloberson2017} (see also~\cite{SafranShamir2018} and \cite[Prop.~4.11]{ArjevaniField2021x}):
\begin{multline*}
\begin{aligned}
\grad{\mathcal{L}}&(\WW)^i = \frac{1}{2\pi}\left[\sum_{j \in\is{k}}\frac{\|\ww^j\|\sin(\Lambda_{ij})}{\|\ww^i\|}\ -\sum_{j\in\is{d}}\frac{\|\vv^j\|\sin(\lambda_{ij})}{\|\ww^i\|}\right]\ww^i\\
& -\frac{1}{2\pi}\left(\sum_{j \in \is{k}} \Lambda_{ij}\ww^j - \sum_{j\in \is{d}}\lambda_{ij}\vv^j\right) + \frac{1}{2}\left(\sum_{j\in\is{k}}\ww^j - \sum_{j \in \is{d}}\vv^j\right) 
\end{aligned}
\end{multline*}
\begin{rems}\label{rem: fossil}
(1) Note that $\Lambda_{ii} = 0$, so that $\|\ww^i\|\sin(\Lambda_{ii})=0$ in the first sum. The row vector $\boldsymbol{\Omega}$ is common to both the biased and bias free equations and the ``angle'' terms $\is{A}^i$ have the same structure for both biased and bias free equations.\\ 
(2) Write $\mathbf{A}^i = \mathbf{U}^i -\mathbf{V}^i$ where
\begin{eqnarray*}
\mathbf{T}^i&=&(T^i_0,\ldots,T^i_{d-1}) = \frac{1}{2\pi}\sum_{\ell\in\is{k},\ell\ne i}\wLambda_{i\ell}\ww^\ell\\
\mathbf{U}^i&=& (U^i_0,\ldots,U^i_{d-1}) = \frac{1}{2\pi}\sum_{\ell\in\is{d}}\wlambda_{i\ell}\vv^\ell
\end{eqnarray*}
The components of $\mathbf{T}^i$ and $\mathbf{U}^i$ are given by
\[
T^i_j= \frac{1}{2\pi}\sum_{\ell\in\is{k},\ell\ne i}\wLambda_{i\ell}w_{\ell j}\quad
U^i_j= \frac{1}{2\pi}\sum_{\ell\in\is{d}}\wlambda_{i\ell}v_{\ell j}, \;j\in \is{d}
\]
If the target $\VV$ is the identity matrix, then
$
U^i_j= \frac{1}{2\pi}\wlambda_{ij}, \;j\in\is{d}.
$
Similar remarks hold for the row vectors $\bBB^i$. All of this is straightforward for the general case but becomes less obvious, 
but more interesting, in the presence of symmetry when the target matrix typically has a diagonal block $I_p$, $2 \le p \le d$. We refer Section~\ref{Sec: 11} 
for more details and examples, including ones where $\VV$ contains an asymmetric block. In Section~\ref{sec: sec12} we give details for the case when the target is $I_d$ and families of spurious minima
have a symmetry (isotropy) ranging from $\Delta S_d$ to  $\Delta(S_{d-q} \times S_q)$, $q \ge 1$. \\
(3) The usual convexity property holds for solutions of the critical point equation. That is, suppose $(\WW,\boldsymbol{\bbeta})\in M(k,d+1)$ is a solution of the 
critical point equation and $\ww^i$ is a row of $\WW$ with associated bias $\bbeta_i$. Let $\lambda,\mu > 0$ with $\lambda + \mu = 1$.  Define $(\WW^\star,\boldsymbol{\bbeta}^\star) \in M(k+1,d+1)$ by replacing $\ww^i$ by the pair of rows $(\lambda\ww^i,\mu\ww^i)$, both with bias $\bbeta_i$.  Then $(\WW^\star,\boldsymbol{\bbeta}^\star) \in M(k+1,d+1)$ 
is a solution of the corresponding critical point equation. For more generalities on this well-known phenomenon,
which implies that the set of critical points defining global minima form a simplicial complex of degenerate solutions when $k > d$, see \cite[A.7]{ArjevaniField2022b} where the phenomenon is referred to as ``fossilization'': every time a neuron is added the old critical points remain but in a degenerate state spread out over a simplicial complex.  
\rend
\end{rems}

\section{The formula for $\frac{\partial \mathcal{L}}{\partial \bbeta}$}\label{sec: betaderiv}
We continue with the previous notation and assumptions. In particular, let $\mathcal{L}$ denote the loss and regard $\mathcal{L}$ as a function of
$\WW \in M(k,d)$, $\boldsymbol{\bbeta} \in M(k,1)$.
\begin{prop}\label{prop: lossderivbeta}
For $i \in \is{k}$,
{\small 
\begin{multline*}
\begin{aligned}
\frac{\partial \mathcal{L}}{\partial \bbeta_i}& = \frac{\|\ww^i\|^2}{2}\left[\bbeta_i\erfcb{\frac{\bbeta_i}{\sqrt{2}}} - \sqrt{\frac{2}{\pi}}e^{-\frac{\bbeta_i^2}{2}}\right]\\
& +\frac{\|\ww^i\|}{2\pi}\sum_{j\in\is{k},j \ne i}\|\ww^j\|\bbeta_j \left(\frac{\pi}{2}\left[\erfcb{\frac{\bbeta_i}{\sqrt{2}}}+\erfcb{\frac{\bbeta_j}{\sqrt{2}}}\right]-\wLambda_{ij}\right)\\
& -\frac{\|\ww^i\|}{2\sqrt{2\pi}}\sum_{j\in\is{k},j \ne i}\|\ww^j\|\left[e^{-\frac{\bbeta_j^2}{2}}\erfcb{\frac{Z_{ij}}{\sqrt{2}}}+\cos(\Lambda_{ij})e^{-\frac{\bbeta_i^2}{2}}\erfcb{\frac{\wZZ_{ij}}{\sqrt{2}}}\right]\\
& -\frac{\|\ww^i\|}{2\pi}\sum_{j\in\is{d}}\|\vv^j\|\bgamma_j \left(\frac{\pi}{2}\left[\erfcb{\frac{\bbeta_i}{\sqrt{2}}}+\erfcb{\frac{\bgamma_j}{\sqrt{2}}}\right]-\wlambda_{ij}\right)\\
& +\frac{\|\ww^i\|}{2\sqrt{2\pi}}\sum_{j\in\is{d}}\|\vv^j\|\left[e^{-\frac{\bgamma_j^2}{2}}\erfcb{\frac{z_{ij}}{\sqrt{2}}}+\cos(\lambda_{ij})e^{-\frac{\bbeta_i^2}{2}}\erfcb{\frac{\wzz_{ij}}{\sqrt{2}}}\right]
\end{aligned}
\end{multline*}
} \normalsize 
\end{prop}
\begin{proof}
Using \Refb{EQ: iibeta_deriv} and Lemma~\ref{EQ:betaderv}, we have
{\small 
\begin{multline*}
\begin{aligned}
\frac{\partial f}{\partial \bbeta_i} (\ww^i,\ww^i) & = \|\ww^i\|^2\left[\bbeta_i\erfcb{\frac{\bbeta_i}{\sqrt{2}}} - \sqrt{\frac{2}{\pi}}e^{-\frac{\bbeta_i^2}{2}}\right]\\
\frac{\partial f}{\partial \bbeta_i} (\ww^i,\ww^j) & = \frac{\|\ww^i\|\|\ww^j\|\bbeta_j}{2\pi}\left(\frac{\pi}{2}\left[\erfcb{\frac{\bbeta_i}{\sqrt{2}}} +\erfcb{\frac{\bbeta_j}{\sqrt{2}}}\right]-\wLambda_{ij}\right)\\
&  -\frac{\|\ww^i\|\|\ww^j\|}{2\sqrt{2\pi}}\left[e^{-\frac{\bbeta_j^2}{2}}\erfcb{\frac{Z_{ij}}{\sqrt{2}}}+\cos(\Lambda_{ij})e^{-\frac{\bbeta_i^2}{2}}\erfcb{\frac{\wZZ_{ij}}{\sqrt{2}}}\right]\\
\frac{\partial f}{\partial \bbeta_i} (\ww^i,\vv^j) & = \frac{\|\ww^i\|\|\vv^j\|\bgamma_j}{2\pi}\left(\frac{\pi}{2}\left[\erfcb{\frac{\bbeta_i}{\sqrt{2}}}+\erfcb{\frac{\bgamma_j}{\sqrt{2}}}\right] -\wlambda_{ij}\right) \\
& -\frac{\|\ww^i\|\|\vv^j\|}{2\sqrt{2\pi}}\left[e^{-\frac{\bgamma_j^2}{2}}\erfcb{\frac{z_{ij}}{\sqrt{2}}}+\cos(\lambda_{ij})e^{-\frac{\bbeta_i^2}{2}}\erfcb{\frac{\wzz_{ij}}{\sqrt{2}}}\right]
\end{aligned}
\end{multline*} }\normalsize
The result follows straightforwardly. 
\end{proof}

\begin{rem}
Since $\frac{\partial f}{\partial \beta}= \frac{1}{\|\ww\|}\frac{\partial f}{\partial \bbeta}$,  it follows from \Refb{EQ: biases} that
\begin{eqnarray}\label{A1}
\widetilde{\frac{\partial f}{\partial \ww^i}}&=&\frac{\partial f}{\partial \ww^i} -\frac{\bbeta \bww^i}{\|\ww^i\|} \frac{\partial f}{\partial \bbeta_i} \\
\label{A2}
\frac{\partial f}{\partial \beta_j}& =& \frac{1}{\|\ww^j\|}\frac{\partial f}{\partial \bbeta_j}
\end{eqnarray}

Define the coordinate transformation $C$ of $M(k,d+1)$ by
\[
C(\WW,\boldsymbol{\bbeta}) = (\WW,(\|\ww^i\|\bbeta_i)) = (\WW,\boldsymbol{\beta}), \;(\WW,\boldsymbol{\bbeta})\in M(k,d+1),
\]
where $\WW$ is always assumed to have no zero rows.
Equations~(\ref{A1},\ref{A2}) imply that non-degenerate critical points of $\mathcal{L}$ ($(\WW,\boldsymbol{\bbeta})$-variables) are mapped by $C$ to non-degenerate critical points of
$\widetilde{\mathcal{L}}$ ($(\WW,\boldsymbol{\beta})$-variables). The linearizations of $\grad{\mathcal{L}}$ and $\grad{\widetilde{\mathcal{L}}}$ at corresponding critical points are similar matrices and 
so have the same spectrum---this a standard result about linearizations of vector fields at a zero (rather than the Hessian of a gradient vector field at a critical point) and is appropriate here as we reduce to studying vector 
fields on fixed point spaces. An important consequence is that it will be no loss of generality to work with $(\WW,\boldsymbol{\bbeta})$-variables rather than $(\WW,\boldsymbol{\beta})$-variables. 
If we identify the fixed point space with $\real^n$, \emph{Euclidean inner product}, the induced vector fields on $\real^n$ will  not be gradient unless we either choose a non-standard  inner product
on $\real^n$ or we modify the identification (best option)---see Remarks~\ref{isomremark}(2). 
\rend
\end{rem}
\begin{rem}[Discussion]
Of special interest is the value of $\frac{\partial \mathcal{L}}{\partial \boldsymbol{\bbeta}}$ when the bias vectors $\boldsymbol{\bbeta}, \boldsymbol{\bgamma}$ are zero.
The problem is to see how the critical point and gradient structure change when we add bias to the activation function.  In particular, what happens when the bias is zero in a biased network?\footnote{Abuse of language---a ``biased network'' is a ``network with bias'', parameterized by $M(k,d+1)$.}.

For example, is a critical point of the bias free network\footnote{A network without bias variable, parameterized by $M(k,d)$.}, a critical point of the biased network, with bias set to zero? 
Is the space of bias free networks an invariant subspace of $M(k,d+1)$ under gradient dynamics?  If invariance holds, then a critical point $\WW \in M(k,d)$ of the loss (no bias) 
necessarily defines a critical point $(\WW,\mathbf{0})\in M(k,d+1)$ for the loss of the biased network. Note that if $\WW \in M(k,d)$ is a critical point of the bias free network, 
then $(\WW,\mathbf{0}) \in M(k,d+1)$ is a critical point for the loss of the biased network iff $\partial \mathcal{L}/\partial \boldsymbol{\bbeta}(\WW,\mathbf{0}) = 0$ (a necessary condition for invariance). 

As it turns out,  in \emph{all} the examples we have investigated,  if $\WW$ is a non-degenerate critical point of $\mathcal{L}$, then $(\WW,\mathbf{0})\in M(k,d+1)$ is 
\emph{not} a critical point in the biased network. In particular,  the invariance property fails. For details and examples see Section~\ref{Sec: 11}.

As a consequence the only tractable transition from critical point of a bias free network to the corresponding 
biased network would appear to be if \emph{the derivatives 
$\frac{\partial \mathcal{L}}{\partial \bbeta}$ are non-zero but `small'.}  Otherwise put, no drastic changes in the landscape geometry near the critical point.
That at least allows the possibility of a 
critical point of the biased network nearby. As we show in Section~\ref{Sec: 11}, the derivatives $\frac{\partial \mathcal{L}}{\partial \bbeta}$ are small---asymptotically they decay to zero 
like a fractional power of $d^{-1}$---and the derivatives for small values of $d$ are typically significantly smaller than suggested by the asymptotics.  

If there is a natural connection between spurious minima in corresponding bias free and biased networks, then one would expect that the loss at a spurious 
minimum in the biased network should be smaller than that in the bias free network. Indeed, it would be counter intuitive if adding extra structure, 
in this case bias, resulted in degraded performance (increased loss). The point is illustrated with examples in Section~\ref{sec: 15}.\rend 
\end{rem}

\subsection{The $\boldsymbol{\bbeta}$-derivative of $\mathcal{L}$ at $\boldsymbol{\bbeta} = \mathbf{0}$}\label{sec: betaderive}
\begin{lemma}\label{LEM: formulabeta}
(Notation and assumptions as above.) Assume $\boldsymbol{\bgamma} = \mathbf{0}$. For all $i \in \is{k}$, $\WW\in M(k,d)$,
\begin{equation*}
\frac{\partial \mathcal{L}}{\partial \bbeta_i}(\WW,\mathbf{0}) = 
\frac{\|\ww^i\|}{2\sqrt{2\pi}}\left[ \sum_{j\in \is{d}}\|\vv^j\|(1+\cos(\lambda_{ij})) -\sum_{j\in \is{k}}\|\ww^j\|(1+\cos(\Lambda_{ij}))\right]
\end{equation*}
\end{lemma}
\begin{proof} Substitute $\boldsymbol{\bgamma} =\mathbf{0}$, $\boldsymbol{\bbeta}=\mathbf{0}$ in the formula for the $\bbeta$-derivative of 
$\mathcal{L}$ (Proposition~\ref{prop: lossderivbeta}, 
The term $-\frac{\|\ww\|^2}{\sqrt{2\pi}}$ corresponding to $\frac{\partial f}{\partial \bbeta_i}(\ww^i,\ww^i)$ (see Examples~\ref{EX: dervbeta}(1)) is accounted for 
since $\Lambda_{ii} = 0$ and so $\cos(\Lambda_{ii}) = 1$.
\end{proof}
\begin{rem}\label{REM: formulabeta}
The formula for  $\frac{\partial \mathcal{L}}{\partial \bbeta_i}(\WW,\mathbf{0})$ can be written as a sum and difference of inner products of rows of $\WW$ and $\VV$ and a sum and difference of norms of rows. 
This is trivial since the cosine of the angle between row $\vv$ and $\ww$ is $\frac{\langle \ww,\vv\rangle}{\|\ww\|\|\vv\|}$ and the denominator is cleared since it is multiplied by $\|\ww\|$, $\|\vv\|$; similarly for angles between
rows of $\WW$. That is, the expression for $\frac{\partial \mathcal{L}}{\partial \bbeta_i}(\WW,\mathbf{0})$ is equal to
\[
\frac{1}{2\sqrt{2\pi}}
\left(\|\ww^i\|\big[\sum_{j\in \is{d}}\|\vv^j\|-\sum_{j\in \is{k}}\|\ww^j\|\big]+\big[\sum_{j\in \is{d}}\langle \ww^i,\vv^j\rangle - \sum_{j\in \is{k}}\langle\ww^i,\ww^j\rangle\big]\right)
\]
Computationally this is better to work with than using the formula with angles since no division or angle computation is required.
\rend
\end{rem}

\section{Symmetry and $\grad{\mathcal{L}}$ in a network with bias}\label{sec: symmsec9}
This section starts by reviewing terminology and notation for symmetry in shallow ReLU networks with a high symmetry target 
(proofs for bias free networks can be found 
in \cite[Appendix A]{ArjevaniField2022b} and \cite{ArjevaniField2020b,ArjevaniField2021}). 
The remainder of the section is devoted to developing definitions and basic theory for 
symmetric biased networks; in particular, the critical point equations for a large class of symmetric biased networks.

The mathematical background and overhead required for the analytic investigation of networks with a high symmetry target is significant. However, the advantage of
working in this setting is the ability to prove sharp analytical results which persist under symmetry breaking perturbation. As one example, we mention computation of the Hessian spectrum at critical points for large (very large) values of $d$ and
applications to networks with \emph{asymmetric} targets which are not too far from the high symmetry target used in this section (see Section~\ref{Sec: 11} for examples).
However, in the present paper we do not present any detailed analysis of the Hessian spectrum (for the basic theory see~\cite{ArjevaniField2021} and also \cite{ArjevaniField2022b}).

Symmetry properties of the loss $\mathcal{L}$ related to symmetries of the target were first described 
in~\cite{ArjevaniField2019a} (see also \cite[\S4.3]{ArjevaniField2021x}).
Generally we follow the notation previously described and developed from~\cite[Appendix A]{ArjevaniField2022b}. In particular, each row $\ww$ of a matrix $\WW \in M(k,d)$ is regarded as a linear functional 
(row vector) and no use is made of the transpose operator. That is, $\ww \in (\real^d)^\star$ and
$\ww(\xx) = \ww\xx$ (matrix multiplication) is a well-defined real number for all $\xx \in \real^d$. 
This is the natural approach in the context of group representation theory (a key tool in the analysis of the Hessian) and also avoids unnecessary notational complexity resulting from the use of the transpose.

\subsection{Symmetries of the loss $\mathcal{L}$ in a biased network}
Recall from Section 3 that if $\WW\in M(k,d)$, $\VV\in M(d,d)$ have no zero rows---a standing assumption in what follows---$\boldsymbol{\bgamma} \in M(1,d)$ and
$\boldsymbol{\bbeta} \in M(1,k)$ are (normalized) bias vectors, then 
\begin{scriptsize}
\begin{eqnarray}\label{EQ:  Loss}
\hspace*{0.4in}\mathcal{L}(\WW,\VV)& =& \frac{1}{2} \mathbb{E}_{\xx \in \mathcal{N}(0,I_d)} \left(\sum_{i\in \is{k}}\|\ww^i\| \sigma(\bww^i\xx-\bbeta_i) - \sum_{j \in \is{d}}\|\vv^j\|\sigma(\bvv^j\xx - \bgamma_j)\right)^2\\
& = & \frac{1}{2}\sum_{i,j \in \is{k}} f(\ww^i,\ww^j) - \sum_{i \in\is{k},j\in \is{d}} f(\ww^i,\vv^j) + \frac{1}{2}\sum_{i,j \in \is{d}} f(\vv^i,\vv^j)
\end{eqnarray} 
\end{scriptsize}

\noindent where $f(\ww^i,\vv^j) = \|\ww^i\|\|\vv^j\|\mathbb{E}_{\xx \in \mathcal{N}(0,I_d)}\big(\sigma(\bww^i\xx-\bbeta_i)\sigma(\bvv^j\xx-\bgamma^j)\big) $ 
with similar expressions for the $f(\ww^i,\ww^j)$ and $f(\vv^i,\vv^j)$ terms (for explicit formulas, see Section~\ref{sec: 27}).

When the target $(\VV,\boldsymbol{\bgamma})$ is fixed, we regard $\mathcal{L}$ as a function of $(\WW,\boldsymbol{\beta})$ 
and write $\mathcal{L}(\WW)$ rather than $\mathcal{L}(\WW,\VV)$.

Note that $(\WW,\boldsymbol{\bbeta}) \in M(k,d+1)$, $(\VV,\boldsymbol{\bgamma}) \in M(d,d+1)$.  The natural action of the symmetric group $S_d$ on $M(k,d)$ permuting columns extends
to an $S_d$-action on $M(k,d+1)$ which fixes the last column $\boldsymbol{\bbeta}$. Similarly for the column action $S_d$ on $M(d,d)$ and $M(d,d+1)$ and $\boldsymbol{\bgamma}$. We sometimes write $S_d^c$ to emphasize the action permutes columns.
The natural $S_k$ action permuting rows of $M(k,d)$ extends to an $S_k$ action on $M(k,d+1)$ which permutes the entries of 
$\boldsymbol{\bbeta})$. Similarly for the $S_k$ actions permuting rows of $M(k,d)$ and $M(k,d+1)$.  Here we sometimes write $S_k^r$ or $S_d^r$ to emphasize the action
permutes rows.

The orthogonal group $O(d)$ acts on the acts on the right of $M(k,d)$ and $M(d,d)$:  
\[
g\WW = \WW \circ g^{-1}, \;\WW \in M(k,d), g \in \text{O}(d).
\]
In particular, since the group of $d\times d$ permutation matrices is a subgroup of $\text{O}(d)$, we can represent the action of
$S^c_d$ on $M(k,d)$ and $M(d,d)$ as an orthogonal action by the group of permutation matrices. The orthogonal action by $\text{O}(d)$ on $M(k,d)$ 
extends in the obvious way to the orthogonal action by  $\text{O}(d) \subset \text{O}(d+1)$ on $M(k,d+1)$ that fixes $\boldsymbol{\beta})$. 
Similar remarks hold for the $\text{O}(d)$ actions on $M(d,d)$, $M(d,d+1)$.

Since the Gaussian distribution on $\real^d$ is invariant by the orthogonal group $O(d)$, $\mathcal{L}$ is $O(d)$-invariant:
\[
\mathcal{L}(g\WW,g\VV) = \mathcal{L}(\WW,\VV),\; g \in O(d), 
\]
where $O(d)$ acts on the right as defined above. Hence $\mathcal{L}$ is $S^c_d$-invariant \emph{as a function of $(\WW,\VV)$}.

Trivially, $\mathcal{L}(\WW,\VV)$ is invariant by the action of $S_k^r \times S_d^r$ on $M(k,d)\times M(d,d)$: the finite sums in \Refb{EQ:  Loss} are independent of the order of summation of the terms. In particular $\mathcal{L}(\WW)$ is $S_k^r$-invariant.
Although  $\mathcal{L}(\WW,\VV)$ is invariant by the action of $S^c_d$,
$\mathcal{L}(\WW)$ will not be invariant under the action of a non-trivial column symmetry unless, for example, a column 
symmetry of $\VV$ can be represented as a row symmetry.
If $\VV = I_d$, then this is so for every column symmetry since if the transposition $(a,b) \in S_d$, $a \ne b$, then 
$(a,b)^cI_d=(a,b)^r I_d$. 
Making use of the right action by $O(d)$
it follows that if $\VV = I_d$, then $\mathcal{L}: M(k,d) \arr \real$ is $S_k^r \times S_d^c$-invariant.  Consequently, $\grad{\mathcal{L}}:  M(k,d) \arr M(k,d)$ is $S_k^r \times S_d^c$-equivariant:
\[
\grad{\mathcal{L}}(g \WW) = g\, \grad{\mathcal{L}}(\UU),\;\text{for all }g \in S_k^r \times S_d^c,\; \WW \in M(k,d).
\]
(The gradient is defined and continuous provided $\WW$ has no zero rows~\cite[\S 4.4]{ArjevaniField2021x}.)
If the network is biased and the gradient is taken with respect to $\ww$ and $\boldsymbol{\bbeta}$ variables, then the gradient is
equivariant with respect to the $S_k^r \times S_d^c$-action on $M(k,d+1)$  (the biases for $\VV=I_d$ must all be equal for this to be true).

\subsection{Isotropy groups of critical points when $\VV = I_d$}
Assume that $\VV = I_d\in M(d,d)$ and that $\VV$ is augmented  by a column $\boldsymbol{\bgamma}$ 
of identical constants $\bgamma_0$ if the network is biased,

Set $G_{k,d} = G = S^r_k \times S^c_d$.
If $\WW \in M(k,d)$, respectively $(\WW,\boldsymbol{\bbeta}) \in M(k,d+1)$, define the \emph{isotropy subgroup} $G_\WW$ for the action of $G$ at $\WW$, respectively 
$(\WW,\boldsymbol{\bbeta})$, by
\[
G_\WW = \{g \in G\dd g\WW = \WW\},\;\text{resp.}\;  \{g \in G\dd g(\WW,\boldsymbol{\bbeta}) = (\WW,\boldsymbol{\bbeta})\}
\]
The group $G_\WW$ measures the symmetry of the point $\WW$ with respect to the action of $G = S^r_k \times S^c_d$ and is an important invariant of critical points
of the loss when $\VV = I_d$. 
With $\VV\in M(d,d)$, 
\[
G_\VV = \Delta S_d  \defoo \{(g,g)\dd g \in S_d\} \subset S^r_d \times S^c_d
\]
and $G_\VV$ is maximal (up to conjugacy) for isotropy subgroups of points $\WW \in  M(d,d)$ provided that neither two rows of $\WW$ nor two columns of $\WW$ are equal
(for example, the $d\times d$ matrix $\mathcal{I}_{d,d}$ all of whose entries are $1$ has isotropy $G_{d,d}$). Note that under our conventions for non parallel rows,
$\VV$ is not defined as an element of $M(k,d)$ if $k > d$. However, $\Delta S_d$ is naturally embedded as a subgroup of $S^r_k \times S^c_d$ for $k > d$ (that is,
by embedding $S^r_d$ as the subgroup of $S^r_k$ permuting the first $d$ rows.
\begin{rem}
Suppose $k = d$ and $\WW$ is a non-degenerate critical point. If a single neuron is added to the network then typically $\WW$ will 
be fossilized in the new parameter space 
(see Remarks~\ref{rem: fossil}(3)) as a 1-dimensional simplex of degenerate critical points and $\WW$ will be replaced by a 
new non-degenerate critical point in $M(d+1,d)$ (this can be path connected back to the original point by deforming the gradient 
for $k =d+1$ back to $k=d$). Similar results hold for all $k > d$.
However, this is not the case for a critical point
giving the global minimum zero of the loss; this is replaced by a 1-dimensional connected simplicial complex each point of which defines the 
global minimum and no new non-degenerate critical points defining the global minimum are created. 
\rend
\end{rem}

Our interest here lies in $d$-dependent families of critical points of $\grad{\mathcal{L}}$  
which have isotropy a subgroup of $\Delta S_d$ (strictly, conjugate to a subgroup of $\Delta S_d$).  

For any subgroup $H$ of $G$, define the \emph{fixed point space} for the action of $H$ on $M(k,d)$ by
\[
M(k,d)^H = \{\WW \dd h\WW = \WW,\;\text{for all }h \in H\} \subset M(k,d).
\]
The fixed point space $M(k,d)^H$ is a vector subspace of $M(k,d)$  
which is invariant under the flow of any equivariant vector field on $M(k,d)$. 
More precisely, wherever the vector field is defined on on the fixed point space $M(k,d)$ , it will be tangent to $M(k,d)^H$. 
Consequently, gradient descent for $\grad{\mathcal{L}}$ which is initialized on $M(k,d)^H$ stays on $M(k,d)^H$. 

We similarly define the
fixed point space $M(k,d+1)^H$ for biased networks. All the properties stated above for $M(k,d)^H$ continue to hold for $M(k,d+1)^H$. 

The invariance of the fixed point space by gradient dynamics is crucial: in the search for critical points, it allows the search to start on the typically low dimensional fixed point space.

\subsection{Families of subgroups of $\Delta S_d \subset S^r_k \times S^c_d$}\label{sec: fixsp}
Given $f \in\pint$, suppose $q_i \in \pint$, $i \in [f-1]$ and that $q_1 \ge q_2 \ge \ldots \ge q_{f-1} \ge 1$ 
(the decreasing assumption is not strictly necessary but it makes the exposition simpler and can always be obtained by an 
appropriate permutation of rows and columns). 
Given $d >  \sum_{i \in [f-1]} q_i$, set $q_0 = d - \sum_{i \in [f-1]} q_i$.  Choose $d_0 \in \pint$ such that $q_0 \ge q_1$, all $d \ge d_0$---in particular, 
$q_1,\ldots,q_{f-1}$ will be constants and $q_0$ depends on $d$ and $q_1,\ldots,q_{f-1}$. 

Note that if $1=q_0\ge q_1 \ge \ldots \ge q_{f-1}\ge 1$, then 
the isotropy is trivial---the \emph{asymmetric case}. Although we give examples of asymmetry later, we usually assume that
$q_0 > 1$ and so there will be non-trivial symmetries. 

Given $f \in\pint$, $q_0,\ldots,q_{f-1}\ge 1$ define 
\begin{equation}\label{EQ: isotropy}
G^\Delta_{\mathbf{q}} = G^\Delta_{q_0,\cdots,q_{f-1}} = \Delta\left( \prod_{i \in \is{f}} S_{q_i}\right) \subset \Delta S_d,
\end{equation}
and regard $G^\Delta_{\mathbf{q}}$ as acting on $M(k,d), k\ge d$. 
\begin{rems}\label{rem: act}
(1) If we assume that $q_1 > \ldots > q_{f-1} \ge 1$ and $d_0$ is chosen sufficiently large so that $q_0 > q_1$, 
the row order is uniquely determined up to row $d$ (see (4) below). \\
(2) If $f = 2$, and $q_0 \ne q_1$, then $G^\Delta_{\mathbf{q}}$ is a maximal proper subgroup of $\Delta S_d$. A necessary condition for maximality is $q_0 \ne q_1$  (see~\cite[\S5.3]{ArjevaniField2021x} for precise statements and references). 
If $q_0 = q_1 > 1$, then the maximal proper subgroup is the wreath product $C_2 \wr S_p$, where $C_2\approx S_2$ 
is the cyclic group of order $2$ and $p = q_0=q_1$.  If there are more than two factors, $G^\Delta_{\mathbf{q}}$ 
is not a maximal proper subgroup. It is possible to have non-degenerate critical points $\WW \in M(2d,2d)$ 
such that $q_0 = q_1 = d$, 
$G_\WW = G^\Delta_{\mathbf{q}}$, the first $d$ rows of $\WW$ are not parallel to any of the remaining $d$ rows and there is no
symmetry relation between the two sets of rows. \\
(3) If $\WW$ is a
critical point of $\grad{\mathcal{L}}$ with isotropy $G^\Delta_{\mathbf{q}}$ then, by the $G$-equivariance of $\grad{\mathcal{L}}$, there are 
$|G|/|G^\Delta_{\mathbf{q}}| = k!d!/(q_0!\ldots q_{f-1}!)$ critical points on the $\Gamma$-orbit of $\WW$. \\
(4) If $k > d+1$ and $q_{f-1} \ge 2$, set $p = k-d$, and if  $k \ge d+1$ and $q_2 > q_1 = 1$, set $p=k-d-1$.  Noting the standing assumption of no
parallel rows, if $G_\WW = G^\Delta_{\mathbf{q}}$, then there is a free action by 
$S^r_p$ on the last $p$ rows of $\WW$ (and so on $M(k,d)^{G^\Delta_{\mathbf{q}}}$).\\
(5) If $q_{f-R_a} = q_{f-R_a+1} = \ldots = q_{f-1} = 1$ and $q_i >1$, $i < f-R_a$, there is an $R_a\times R_a$-diagonal 
asymmetric block in $\WW$. Here it is natural to set  
$G_\WW = \Delta\left( \prod_{i \in \is{f-R_a}} S_{q_i}\right) \subset \Delta S_{d-a}\subset \Delta S_d$. \rend
\end{rems}
\subsubsection{Review of notation and conventions}
Assume $f \in \pint$ and we are given integers $q_1 \ge q_2 \ge \ldots \ge q_{f-1} \ge 1$. Choose $d_0 \in \pint$ so that
$q_0 = d - \sum_{i \in [f-1]} q_i \ge 2$ for all $d \ge d_0$. Let $G^\Delta_{\mathbf{q}}\subset \Delta S_d$ be defined
as in \Refb{EQ: isotropy}. Set $m = k-d$ and assume that $m$ is constant for $d \ge d_0$.

Set $R = f + m$---the number of \emph{rowtypes}: one for each factor in $G^\Delta_{\mathbf{q}}$ and one for each additional row.
Let $R_s$ be the number of values of $i \in \is{f}$ for which $q_i > 1$. Since $q_0 \ge 2$, $R_s \in [f]$.
Let $\RS = \{i \in \is{f} \dd q_i > 1\}$. Since we assume  $(q_i)_{i\in \is{f}}$ is a decreasing sequence,  $\RS = \is{R_s}$.
Let $\RA = \{i\in \is{f} \dd q_i = 1\}$ and let $R_a=f - R_s$ count the number of factors with $q_i = 1$. Obviously, 
$\RA\cup\RS = \is{f}$ and $\RA =\is{f}\smallsetminus\RS$. The rowtypes in $\RA$ define the asymmetric diagonal block.
Finally, define $\RM = \{i\in \is{R}\dd i \ge f\}$; that is, the rowtypes associated to neurons with rowtype at least $f$.

Recall from Section~\ref{sec: sec4} that we label rows of $\WW \in M(k,d)$ from $0$ to $k-1$ rather than from $1$ to $k$.
That is, for $i \in \is{k}$, $\ww^i$ is row $i$ of $\WW$. We have a similar convention for rowtypes. For example, if
$\WW \in M(d,d)^{G^\Delta_{\is{q}}}$ (or $M(d,d+1)^{G^\Delta_{\is{q}}}$ if the network is biased), there is 1 rowtype $0$ 
and the first row of $\WW$ is $\ww^0$. In general, $\ww^i$ typically represents the first $\WW$ row of rowtype $i$; that is $\ww^{I(i)}$. 
The context will always make the meaning clear.

If $\text{dim}(M(k,d)^{G^\Delta_{\is{q}}}) = N$ (unbiased), then
\[
\text{dim}\left(M(k,d+1)^{\Delta_{\mathbf{q}}}\right) = N + R,\;\;\text{biased}
\]
\begin{rem}
The convention of indexing from $0$ rather than $1$ is made to reduce confusion and error when coding arrays.
If the reader prefers, $\is{d}$ can be interpreted as $[d]$; this results in few changes to the text. \rend
\end{rem}

\begin{lemma}[{\cite[Appendix A]{ArjevaniField2022b}} \& below]\label{LEM: FP}
(Notation and assumptions as above.)
For $d \ge d_0$,
\begin{eqnarray*}
\text{dim}\big(M(k,d)^{G^\Delta_\mathbf{q}}\big)&=&f(f+m) + R_s \\
\text{dim}\big(M(k,d+1)^{G^\Delta_\mathbf{q}}\big)&=&(f+1)(f+m) + R_s
\end{eqnarray*}
In particular,
\begin{enumerate}
\item  If $f = 1$, then 
\begin{eqnarray*}
\text{dim}\big(M(k,d)^{G^\Delta_\mathbf{q}}\big)&=&2+m \;\text{\rm (unbiased)} \\ 
\text{dim}\big(M(k,d+1)^{G^\Delta_\mathbf{q}}\big)& = &3+2m \;\text{\rm (biased)}
\end{eqnarray*}
\item  If $f =2$ and $q_1 = 1$, then 
\begin{eqnarray*}
\text{dim}\big(M(k,d)^{G^\Delta_\mathbf{q}}\big)&=&5+2m \;\text{\rm (unbiased)} \\ 
\text{dim}\big(M(k,d+1)^{G^\Delta_\mathbf{q}}\big)& = &7+3m \;\text{\rm (biased)}
\end{eqnarray*}
\item  If $f =2$ and $q_1 > 1$, then 
\begin{eqnarray*}
\text{dim}\big(M(k,d)^{G^\Delta_\mathbf{q}}\big)&=&6+2m \;\text{\rm (unbiased)} \\ 
\text{dim}\big(M(k,d+1)^{G^\Delta_\mathbf{q}}\big)& = &8+3m \;\text{\rm (biased)}
\end{eqnarray*}
\end{enumerate}
\end{lemma}
\subsubsection{Sketch of the proof of Lemma~\ref{LEM: FP}} Start by assuming bias free.
Observe that $S_m^r\subset S_k^r$ acts (freely) on the last $m$ rows of $M(k,d)$ and so acts on the fixed point space and does not enter into the 
computation of the dimension.  Suppose that $k = d$.
We have a natural block decomposition of 
$\mathcal{B}=\{{B}_{ij}\dd i \in \is{R}, j \in \is{f}\}$ of $M(d,d)$ which is preserved by the action of $G^\Delta_\mathbf{q}$.
Specifically, the diagonal blocks $B_{ii}$ will be of size $q_i \times q_i$ and invariant by 
$\Delta S_{q_i} \subset G^\Delta_\mathbf{q}\times S^r_m$. The off diagonal blocks $B_{ij}$, $i \ne j$, will be invariant by
$S_{q_i}^r \times S^c_{q_j}$. Since the action of $S_{q_i}^r \times S^c_{q_j}$ on $B_{ij}$ is transitive if $i \ne j$, it follows that
$B_{ij}$ is fixed by $S_{q_i}^r \times S^c_{q_j}$ iff all entries are equal. On the other hand, the
action on diagonal blocks $B_{ii}$ has two group orbits if $q_i > 1$: the set of diagonal elements and the set of off-diagonal elements. If $k > d$,
we have to consider the action induced by $S_d$ on columns. Here each of the $k-d$ additional rows has a block decomposition  into
$1 \times q_j$ matrices $B_{ij}$ and $S^c_{q_j}$ acts transitively on $B_{ij}$. Each of these blocks contributes 1 to the dimension of the fixed point space.

In general, we have
\[
\text{dim}\big(B_{ij}^{G^\Delta_\mathbf{q}}\big) = \begin{cases}
&2,\;\text{if $i = j \in \RS$} \\
&1,\; \text{otherwise} 
\end{cases}
\]
This argument suffices to prove the bias free case. Adding bias increases the dimension by the number of rowtypes which is $f+m$. \qed
\begin{rem}
Suppose that $k = d$ and $R_a \ge 1$ so that there is at least one $q_i$ equal to $1$. After a permutation of rows and columns 
we may assume that $q_0\ge q_1 \ge \ldots \ge q_{f-1} \ge 1$. Let $P = \cup_{f-R_a \le i,j < f} B_{ij}$. Then $P$
is an $R_a \times R_a$ diagonal block for which
$ \text{dim}\big(P^{G^\Delta_\mathbf{q}}\big) = R_a^2 $.  Since there are no symmetry constraints on $P$---the row and column actions of $(S_1)^{f-R_s}$ on $P$ are trivial---$P$ can be regarded as an asymmetric block with no symmetry.
In particular, if $R_s = 0$, there are no symmetry constraints at all on $P \in M(d,d)$ ($d = f$).  If we have a non-degenerate symmetric critical point $\WW$ of $\grad{\mathcal{L}}$ and deform $\VV$ to an
asymmetric target (no non-trivial symmetries of $\Delta S_d$), then as we deform $\VV$, $\WW$ will deform to an asymmetric 
non-degenerate critical point.  An example is given later. Similar comments hold when $m > 0$ and $\WW \in M(k,d)$ and/or 
when the network is biased.
\rend
\end{rem}
\subsubsection{Properties of rowtypes} 
Assume first that $k = d$ (so $m = 0$) and the network is unbiased.
Let $\WW \in M(d,d)$ have isotropy subgroup $G = G^{\Delta}_\mathbf{q}\subset\Delta S_d$ where $q_0\ge q_1 \ge \ldots \ge q_{f-1}\ge 1$. As described above, there is natural block decomposition 
$\mathcal{B} = \{B_{ij}\dd i, j \in \is{f}\}$ of
$\WW$ with diagonal blocks $B_{ii}$ consisting of $q_i \times q_i$-matrices which are
fixed by $\Delta S_{q_i}$. The action is non-trivial iff $i \in \RS$ (equivalently, $i \in [0,R_s-1]$). 
A row $\ww$ of $W \in M(d,d)^G$ is of rowtype $i \in \is{f}$ iff $\ww$ meets the diagonal block $B_{ii}$.  Rows of rowtype $i$ are permuted by $\Delta S_{q_i}$. 

Suppose now that $k = d+m$, $m > 0$ and that the isotropy subgroup of $\WW$ is $G = G^{\Delta}_\mathbf{q}$ as above. 
It is always assumed that there are no parallel rows and so the last $m$ rows of $\WW$ are not symmetrically related to any of the first $d$ rows of $\WW$.

The $f+m$ rowtypes of $\WW$ are indexed by $\is{R}$. If $i\in\RS$, then $q_i > 1$ and the set of $q_i$ rows of rowtype $i$ are permuted by $\Delta S_{q_i}$.   The set of rowtypes for which there is exactly one row of each rowtype is
parameterized by \{$i \in \is{R} \dd i \ge R-R_s = f+m -R_s\}$. Note that if $q_i = 1$, $i \in \is{f}$, then $i \in [\RS,f)$ and if $i \ge f$, then  $\ww^{d-f+i}$ will be the unique $\WW$ row of rowtype $i$. 

Suppose $i \in \RS$ so that $q_i > 1$ and that $\ww^a,\ww^b$ are both of rowtype $i$. Since $\Delta S_{q_i}$ acts transitively on the set of rows of rowtype $i$, it follows by equivariance that (a) $\|\ww^a \| = \|\ww^b\|$. (b)
the angle $\Theta_i$ between $\ww^a,\ww^b$ is independent of the choice of $a \ne b$; similar results hold for $\Lambda_{ij}$, where $j$ is not of rowtype $i$. Similarly, the angle $\theta_i$ between $\ww^a,\vv^b$, where $b\ne a$ but
$\ww^a,\ww^b$ are of the same rowtype, is independent of the choice of $a \ne b$. Similar results hold for the angles $\lambda_{ij}$, where $j$ is not of rowtype $i$ or $i=j$ (that is $\ww^i, \vv^i$ are in the same row).

\subsection{Coordinatization of $M(k,d)^{G_\mathbf{q}^\Delta}$ \&  $M(k,d+1)^{G_\mathbf{q}^\Delta}$}\label{sec: coordinatization}
First assume the unbiased case. There is a natural linear isomorphism $\phi$ (but \emph{not} isometry---see Remark~\ref{rem: isometry} at the end of this section):
\[
\phi: M(k,d)^{G_\mathbf{q}^\Delta} \approx \left(\prod_{i \in \is{R}} \real^f\right)  \times \real^{r_s}
\]
and $\phi$ is used to coordinatize $M(k,d)^{G_\mathbf{q}^\Delta}$. In particular, we compute vector fields on the fixed point space using $\phi$ coordinates.

Given $\WW \in M(k,d)^{G^\Delta_\mathbf{q}}$, let $i \in \is{R}$, $j \in \is{f}$. If $i \ne j$, define $\xi^i_j$ to be any element of $B_{ij}^{G_\mathbf{q}^\Delta}$.
If $i = j$,  define $\xi^i_i$ to be any diagonal element of $B_{ii}^{G_\mathbf{q}^\Delta}$.
If $i = j$ and $q_i > 1$---that is, $i \in \RS$---define $\xi_{\star i}^i$ to be any off-diagonal element of $B_{ii}^{G_\mathbf{q}^\Delta}$,

The construction defines the required linear isomorphism 
\[
\WW \mapsto (\bxi,\bxi_{\star})\in \left(\prod_{i \in \is{R}} \real^f\right)  \times \real^{r_s}
\]
All rows of $\WW$ of rowtype $i\in \is{f}$ get mapped to the same point $((\xi^i_j)_{j \in \is{f}},\xi^i_{i\star})$ 
(the `star' term is omitted if $q_i = 1$). In particular the superscript $i$ is the rowtype of the rows
parameterized by $((\xi_0^i,\ldots,\xi_{f-1}^i),\xi_{\star i}^i)$ or $(\xi_0^i,\ldots,\xi_{f-1}^i)$, if $i\notin\RS$.

For biased networks, $M(k,d+1)^{G^\Delta_\mathbf{q}} \approx M(k,d)^{G^\Delta_\mathbf{q}}\oplus M(1,d)^{G_{\is{q}}^r}$, where $G_{\is{q}}^r$
is the projection of $G_\mathbf{q}^\Delta$ onto $\prod S_{q_i}^r \subset S_d^r$.
\begin{exam}\label{EX: simple}
Suppose that $\WW \in M(k,d+1)$, $G_\WW = \Delta(S_{d-1} \times S_1)$ and $m = 1$. Assuming normalized bias,
$\text{dim}( M(k,d+1)^{G^\Delta_{(q_0,1)}}) = 10$ and 
\[
\WW = \left[\begin{matrix}
\xi^0_{0}& \xi_{\star 0}^0&\ldots&\xi_{\star 0}^0&\xi^0_1&\bbeta_0  \\
\xi^0_{\star 0}& \xi_{0}^0&\ldots&\xi^0_{\star 0}&\xi^0_1&\bbeta_0 \\
\ldots&\ldots&\ldots&\ldots&\ldots&\ldots\\
\ldots&\ldots&\ldots&\ldots&\ldots&\ldots\\
\xi^0_{\star 0}& \xi_{\star 0}^{0}&\ldots&\xi^0_{0}&\xi^0_1&\bbeta_0 \\
\xi^1_{0}& \xi^1_{0}&\ldots&\xi^1_0&\xi^1_1&\bbeta_1 \\
\xi^2_0& \xi^2_0&\ldots&\xi^2_0&\xi^2_1&\bbeta_2 
\end{matrix} \right]
\]
The penultimate row corresponds to the $S_1$ factor and the final row to the additional neuron. In this case
\[
\phi(\WW,\bbeta) = ((\xi^0_0,\xi^0_1,\bbeta_0,\xi^1_0,\xi^1_1,\bbeta_1,\xi^2_0,\xi^2_1,\bbeta_2),\xi_{\star 0}^0) \in \real^9 \times \real
\]
Ignoring the bias, the number of coordinates needed is $3$ (for the first row) and $2$ each for the last two rows. A total of $7$ coordinates, as required by Lemma~\ref{LEM: FP}(2).
Three additional parameters are needed when bias is added.
\examend
\end{exam}
\begin{rem}\label{rem: isometry} As indicated in the definition of the coordinatization of $M(k,d)^{G_\mathbf{q}^\Delta}$ \&  $M(k,d+1)^{G_\mathbf{q}^\Delta}$, 
the coordinatization map $\phi$ is not an isometry and so does not preserve the inner product. In particular, we do not recover $\grad{\mathcal{L}}$ if we compute the gradient of $\mathcal{L}$ in $\phi$ (Euclidean) coordinates.
If we arrange matters so that $\phi$ is an isometry, then we do recover $\grad{\mathcal{L}}$.
There are two ways of doing this: (a) leave $\phi$ unchanged but change the inner product on the range space so that $\phi$ is an isometry, (b) redefine $\phi$ so that
$\phi$ is an isometry. In case (b) the coordinate $\xi$ corresponding to a $p \times q$ block B, with all entries $\xi$, will be mapped to $\sqrt{pq}\xi$ rather than $\xi$. A diagonal block $\xi I_p $ will be mapped to $\sqrt{p}\xi$. In Example~\ref{EX: simple}, we modify $\phi$ to the isometry $\widetilde{\phi}:M(k,d+1)^{G^\Delta_{(q_0,1)}}\arr\real^9 \times \real$ defined by
\[
\widetilde{\phi}(\WW,\bbeta) = ((\alpha_0\xi_0^0,\alpha_0\xi_q^0,\alpha_0\bbeta_0,\alpha_0\xi^1_0,\xi_1^1,\bbeta_1,\alpha_0\xi^2_0,\xi^2_1,\bbeta_2),\alpha_1\xi_{\star 0}^0)
\]
where $\alpha_0 = \sqrt{d-1}$, $\alpha_1 = \sqrt{(d-1)(d-2)}$ and, as usual, we take the Frobenius (Euclidean) norm on $M(k,d+1) \approx \real^{k(d+1)}$. 
However, we do not follow either scheme since both (a) and (b) add some complexity but yield no additional information.  
Indeed, the linearization of  $\grad{\mathcal{L}}$ at a critical point is similar to the linearization of the vector field in $\phi$ coordinates by a diagonal matrix. 
Moreover, Sylvester's test for strictly positive spectrum applies to the linearization of $\grad{\mathcal{L}}$ in $\phi$-coordinates 
with no change even though there is no Hessian (that is, the linearization matrix is not usually symmetric). \rend
\end{rem}

\subsection{Critical point equations for $\grad{\mathcal{L}}$ on $M(k,d+1)^{G^\Delta_{\mathbf{q}}}$}\label{sec: sec9.6}
In sections~\ref{sec: coeffs} \& \ref{sec: betaderiv} formulas were given for the components of $\grad{\mathcal{L}}$ when there is bias but no symmetry. 
Using the crucial notation of rowtype, the aim is to extend these results to take account of symmetry. In this way, relatively simple critical point equations are obtained that are amenable 
to analysis and simulation for large values of $d$ (assuming, for example, that $q_1,\ldots,q_{f-1}$ are fixed if $f > 1$). The first step is to expand the table of coefficients given in section~\ref{SEC: table} 
to allow for rowtypes so that we can rewrite the formula for the components of $\grad{\mathcal{L}}$ in terms of rowtypes. 
Since there will generally be multiple rows with the same rowtype, this will lead to a reduction of the number of critical point equations defined on the fixed point space. 

The coefficients defined in section~\ref{SEC: table} were parametrized by a pair of rows (one from $\WW$, the other from $\WW$ or $\VV$). 
We now work in terms of rowtypes and define families of coefficients for rows of rowtype $i$ where $i \in \RS$. There will be one set of critical 
point equations for each rowtype. The number of equations for rowtype $i$ will depend on
whether or not $i \in \RS$---there will be one `extra' equation if $i \in \RS$ otherwise there will be $f$ equations. We start with the easiest  case $i \notin \RS$.
\subsection{Critical point equations associated to rowtype $i \notin \RS$}

If $k = d$, the integers $q_i$ are defined for all the rowtypes. If $k = m+d$, $m > 0$,  define $q_j= 1$ for each of the $m$ additional rows (strictly, rowtypes) $\ww^j$, $j \in \RM$. 
In what follows, networks are assumed to be biased and we usually write $\WW$ (or $\widetilde{\WW}$) rather than $(\WW,\boldsymbol{\bbeta})$.

Given $\WW \in M(k,d+1)^{G^\Delta_{\mathbf{q}}}$, denote the rowtype of $\ww^j$ (row $j$ of $\WW$) by $r(j)\in\is{R}$.  
Conversely, if $i$ is a rowtype, let $\is{i}_R$ denote the ordered set parameterizing the rows $\ww^a$ for which $r(a) = i$. For example, taking $G_{\mathbf{q}}^\Delta$ and $q_2 \ge q_1 > q_2 = 1$,
\begin{eqnarray*}
\is{0}_R& = &\{0,1,\ldots,q_0-1\}\\
\is{1}_R& = &\{q_0,1,\ldots,q_0+q_1-1\}\\
\is{2}_R& = &\{q_0+q_1\}
\end{eqnarray*}
Given the rowtype $i$, let $I(i)$ denote the initial  element of $\is{i}_R$. Thus 
the rows of rowtype $i$ in $\WW\in M(k,d)^{G^\Delta_{\mathbf{q}}}$ are $\ww^{I(i)},\ldots,\ww^{I(i)+q_i-1}$ and
$ \cup_{i \in \is{R}}\is{i}_R = \is{k}$.
 
Generally, we label specific rows of $\WW$ by $a,b,\ldots \in \is{k}$ \emph{except that for the first row $I(i)\in \is{i}_R$ we often write $\ww^i$}. That is, $\ww^i$ may denote 
either a general row of rowtype $i$ \emph{or}, more usually, the first row of rowtype $i$ in $\WW$.  In order to find the critical point equations on 
the fixed point space in $\bxi,\bxi_\star$ coordinates, it suffices to determine the
components of $\grad{\mathcal{L}}(\WW)^i$, for $i \in \is{R}$ (in terms of row number, for $i =0,q_0,q_0+q_1,\ldots$). 
Here $\ww^i$ will have coordinates $\big((\xi^i_j)_{j\in\is{f}},\xi_{\star i}^i\big)$ if $i \in\RS$, otherwise $(\xi^i_j)_{j\in\is{f}}$. In other words, when working with $\bxi$ coordinates on 
$\left(\prod_{i \in \is{R}} \real^f\right)  \times \real^{r_s}$, rowtype is the natural parameter as opposed to the actual row number.

We continue with the same definitions of $\mathcal{U}_{ij},\ldots,\mathcal{V}_{ij}$ used in section~\ref{sec: Uij} except that we use an overline to emphasize that everything is in terms of rowtype and take account of 
multiplicities. 

Thus, for $i \in \is{R}$, define
\begin{multline*}
\begin{aligned}
\overline{\mathcal{U}}_{ij} = & = e_{ij}^X + \bbeta_i\bgamma_j\big(\ell_{ij} + \pi - \wlambda_{ij}\big) + e^{-k_{ij}}\sin(\lambda_{ij}),\; j \in \is{f}\\
\overline{\mathcal{V}}_{ij} = & \ell_{ij} + \pi - \widetilde{\lambda}_{ij},\; j \in \is{f} \\
\overline{\mathcal{W}}_{ij} =& E_{ij}^X + \bbeta_i\bbeta_j\big(L_{ij} + \pi - \wLambda_{ij}\big) + e^{-K_{ij}}\sin(\Lambda_{ij}),\; j \ne i\\
\overline{\mathcal{Y}}_{ij} = & L_{ij} + \pi - \widetilde{\Lambda}_{ij},\; j \ne i \\
\overline{\mathcal{X}}_i  =&  1+X_i
\end{aligned}
\end{multline*}
Note that the condition $j \ne i$ is on rowtypes and so if $i \in \RS$ many terms are missed. This point is addressed shortly.

\begin{Def}
For $j \in \RS$ define
\begin{enumerate}
\item $\ww^j_\Sigma= \sum_{r(a) = j} \ww^a$. 
\item $\vv^j_\Sigma= \sum_{r(a) = j} \vv^a$. 
\end{enumerate}
(If $j \in \is{R}\smallsetminus \RS$, then $\ww^j_\Sigma$ is defined to be $\ww^j$; similarly for $\vv^j_\Sigma$).
\end{Def}
\begin{rem}\label{rem: block_decomp}
The row vector $\ww^j_\Sigma$ may be written as a $1 \times d$-matrix with block decomposition $[(a_{j0})^{q_0}],\ldots,[(a_{j {f-1}})^{q_{f-1}}]$ where
\[
a_{j \alpha} = \begin{cases}  
&q_j\xi^j_\alpha,\; \text{if } \alpha \ne j \\
&\xi^j_j + (q_j-1)\xi_{\star j}^j,\; \text{if } \alpha = j
\end{cases}
\]
Of course, if $q_j = 1$ (and so $j \notin \RS$), applying the definition, gives  $\ww^j_\Sigma= \ww^j$, $\vv^j_\Sigma = \vv^j$. \rend
\end{rem}

\begin{prop}\label{prop: asym}
(Notation and assumptions as above.) For $i \in \is{R}\smallsetminus \RS$, component $i$ of $\grad{\mathcal{L}}(\WW)$ is given by
\begin{equation}
\grad{\mathcal{L}} (\WW)^i =\Gamma_i\ww^i - \is{A}^i +\is{B}^i + \boldsymbol{\Omega}
\end{equation}
where
\begin{eqnarray*}
\Gamma_i & = &\frac{X_i}{2} + \frac{1}{2\pi}\left[\sum_{j\in\is{R},j\ne i}q_j\frac{\|\ww^j\| e^{-K_{ij}}\sin(\Lambda_{ij})}{\|\ww^i\|}-\sum_{j\in\is{f},j\ne i}q_j\frac{\|\vv^j\| e^{-k_{ij}}\sin(\lambda_{ij})}{\|\ww^i\|}\right]\\
&&+\frac{1}{2\pi}\left[\sum_{j\in\is{R},j\ne i}q_j\frac{\|\ww^j\|E_{ij}^X}{\|\ww^i\|}-\sum_{j\in\is{f},j\ne i}q_j\frac{\|\vv^j\|e_{ij}^X}{\|\ww^i\|}\right]-\frac{1}{2\pi}\left(\frac{\|\vv^i\|\big(e^{-k_{ii}}\sin(\lambda_{ii})+e_{ii}^X\big)}{\|\ww^i\|}\right)\\
&&+\frac{\bbeta_i}{2\pi}\left[\sum_{j\in\is{R},j\ne i}q_j\frac{\bbeta_j\|\ww^j\|\big(L_{ij}+(\pi-\wLambda_{ij})\big)}{\|\ww^i\|} 
-\sum_{j\in\is{f},j \ne i}q_j\frac{\bgamma_j\|\vv^j\|\big(\ell_{ij}+(\pi-\wlambda_{ij})\big)}{\|\ww^i\|}\right] \\
&& -\frac{\bbeta_i\bgamma_i}{2\pi}\frac{\|\vv^i\|\big(\ell_{ii} + (\pi -\wlambda_{ii})\big)}{\|\ww^i\|}
\end{eqnarray*}
\begin{eqnarray*}
\is{A}^i& = &\frac{1}{2\pi}\left(\sum_{j\in\is{R}, j \ne i}\wLambda_{ij}\ww_\Sigma^j - \sum_{j \in \is{f}}\wlambda_{ij}\vv_\Sigma^j\right)\\ 
\is{B}^i& = & \frac{1}{2\pi}\left(\sum_{j\in\is{R},j\ne i}L_{ij}\ww_\Sigma^j-\sum_{j \in \is{f}}\ell_{ij}\vv_\Sigma^j \right)\\
\bOmega & = & \frac{1}{2}\left(\sum_{j\in\is{R}}\ww_\Sigma^j - \sum_{j \in \is{f}}\vv_\Sigma^j\right)=\frac{1}{2}\left(\sum_{a\in\is{k}}\ww^a - \sum_{b \in \is{d}}\vv^b\right)
\end{eqnarray*}
\end{prop}
\begin{proof} Follows from the earlier description of $\grad{\mathcal{L}}(\WW)^i$ for $i\in\is{k}$. Note that since  $i \notin \RS$, $q_i = 1$ and so
none of the displayed terms involve the angles $\Theta_i,\theta_i$ which are only defined if $q_i > 1$. 
\end{proof}
\begin{rems}
(1) The proposition is valid for any row of $\grad{\mathcal{L}} (\WW)$ which is of rowtype $i \notin\RS$. However, it is easiest to think of $i$ as representing both the rowtype and the first row in
$\grad{\mathcal{L}}(\WW)$ which is of rowtype $i$; that is, row $I(i)$.  \\
(2) The only angles that appear in the formula for $\grad{\mathcal{L}} (\WW)^i$ when $i \notin \RS$ are the $\Lambda_{ij}, \lambda_{ij}$ angles defined in Section~\ref{sec: coeffs}.
\rend
\end{rems}
If $i \notin \RS$, then the row vectors $\is{A}^i, \is{B}^i$ have a block decomposition induced from the natural block decomposition on $M(k,d+1)^{G^\Delta_{\mathbf{q}}}$.
If $\VV = I_d$, then $\is{A}^i =[(A_0^i)^{q_0},\ldots, (A_{f-1}^i)^{q_{f-1}}]$ where  $(A^i_j)^{q_j}$ denotes the $j$th $1 \times q_j$ block. This consists of $A^i_j$ repeated $q_j$ times  where
\[
A^i_j = \begin{cases}
&\frac{1}{2\pi}\left[\sum_{\alpha \ne j} q_\alpha \wLambda_{i\alpha}\xi^\alpha_j+  \wLambda_{ij}((q_j-1)\xi^j_{\star j}+ \xi^j_j)-\wlambda_{ij}\right],\; j \ne i\\
& \frac{1}{2\pi}\left[\sum_{\alpha \ne i} q_\alpha \wLambda_{i\alpha}\xi^\alpha_i - \wlambda_{ii}\right],\; j = i
\end{cases}
\]
Similar, but simpler, expressions hold for $B^i_j$ provided $i \notin \RS$. Note that $\is{A}^i$ (resp.~$\is{B}^i$) contains the term $ -\frac{1}{2\pi} \wlambda_{ii} \vv^i$ (resp.~$-\frac{1}{2\pi} \ell_{ii}\vv^i$).  
 
\begin{exam}\label{EX: part1}
Suppose $G_\is{q}^\Delta = \Delta(S_{d-1} \times S_1)$, so that $q_0 = d - 1$, $q_1 = 1$. Assume $m = 0$ and $\VV = I_d$. Since $1 \notin \RS$, Proposition~\ref{prop: asym} may be used to 
compute $\mathbf{A}^1$ and $\mathbf{B}^1$.
We have
\begin{eqnarray*}
{A}^1_0& = & \frac{1}{2\pi}\left[\big(\xi^0_0 +(q_0-1)\xi^0_{\star 0}\big)\wLambda_{10}-\wlambda_{10}\right], \;
{A}^1_1  =  \frac{1}{2\pi}\left[q_0\xi^0_1\wLambda_{10} -\wlambda_{11}\right] 
\end{eqnarray*}
\begin{eqnarray*}
{B}^1_0& = & \frac{1}{2\pi}\left[\big(\xi^0_0 +(q_0-1)\xi^0_{\star 0}\big)L_{10}-\ell_{10}\right],\;
{B}^1_1  =  \frac{1}{2\pi}\left[q_0\xi^0_1L_{10} -\ell_{11}\right] 
\end{eqnarray*}
where 
\[
L_{10} = -\frac{\pi}{2}\left[\erfb{\frac{\bbeta_1}{\sqrt{2}}}+\erfb{\frac{\bbeta_0}{\sqrt{2}}}\right] 
\]
\[
\ell_{10} = -\frac{\pi}{2}\left[\erfb{\frac{\bbeta_1}{\sqrt{2}}}+\erfb{\frac{\bgamma_0}{\sqrt{2}}}\right], \quad \ell_{11} = -\frac{\pi}{2}\left[\erfb{\frac{\bbeta_1}{\sqrt{2}}}+\erfb{\frac{\bgamma_1}{\sqrt{2}}}\right]
\]
Note that the biases for $\VV$ are all equal if $(\VV,\bgamma)$ has isotropy $\Delta S_d$ and so $\boldsymbol{\bgamma}$ is fixed by $S_{d-1}^r \times S_1^r$---the isotropy group for 
$\boldsymbol{\bbeta}$.

Setting the biases to zero, $\bBB^1 = \is{0}$ and
\[
{A}^1_0  =  \frac{1}{2\pi}\big(\Lambda_{10}(\xi^0_0 +(q_0-1)\xi^0_{\star 0})-\lambda_{10}\big),\quad
{A}^1_1  =  \frac{1}{2\pi}\big(q_0\Lambda_{10}\xi^0_1 -\lambda_{11})
\]
which is seen to be correct using the formula for $\grad{\mathcal{L}}$~\cite{ArjevaniField2021x}. 
\examend
\end{exam}

\subsection{Critical point equations associated to rowtype $i \in \RS$}\label{subsec: rtype}
In the table below, we add new coefficients to the table of section~\ref{SEC: table}.
These will all be scalars with upper and lower case having the same meaning as before. For $i \in \RS$ define
\begin{eqnarray*}
Z_i&=& \big(-\cot(\Theta_i) + \csc(\Theta_i)\big)\bbeta_i \\
z_i&=&-\cot(\theta_i)\bgamma_i+ \csc(\theta_i)\bbeta_i,\; \wzz_i = -\cot(\theta_i)\bbeta_i+ \csc(\theta_i)\bgamma_i  \\
K_i & = & \bbeta_i^2\big(\csc^2(\Theta_i)-\csc(\Theta_i)\cot(\Theta_i)\big)\\
k_i & = & \frac{1}{2}\big[(\bbeta_i^2+\bgamma_i^2)\csc^2(\theta_i)-2\bbeta_i\bgamma_i\csc(\theta_i)\cot(\theta_i)\big]  \\
L_i & = & -\pi\,\erfb{\frac{\bbeta_i}{\sqrt{2}}},\;
\ell_i  = -\frac{\pi}{2}\left[\erfb{\frac{\bbeta_i}{\sqrt{2}}}+\erfb{\frac{\bgamma_i}{\sqrt{2}}}\right]\\
E_i^X & = & -\sqrt{2\pi}\bbeta_i e^{-\frac{\bbeta_i^2}{2}} \erfcb{\frac{Z_i}{\sqrt{2}}}\\
e_i^X& = & -\sqrt{\frac{\pi}{2}}\left[\bbeta_i e^{-\frac{\bgamma_i^2}{2}}\erfcb{\frac{z_i}{\sqrt{2}}}+\bgamma_i e^{-\frac{\bbeta_i^2}{2}}\erfcb{\frac{\wzz_i}{\sqrt{2}}}\right] 
\end{eqnarray*}
In case $i = j$, and $i \in \RS$, define
\begin{multline*}
\begin{aligned}
\overline{\mathcal{U}}_i = &  e_i^X + \bbeta_i\bgamma_i\big(\ell_i + \pi - \wtheta_i\big) + e^{-k_i}\sin(\theta_i)\quad
\overline{\mathcal{V}}_i =  \ell_i + \pi - \wtheta_i \\
\overline{\mathcal{W}}_i =& E_i^X + \bbeta^2_i\big(L_i + \pi - \wTheta_i\big) + e^{-K_i}\sin(\Theta_i)\quad
\overline{\mathcal{Y}}_i =  L_i + \pi - \widetilde{\Theta}_i 
\end{aligned}
\end{multline*}
\begin{Def}
For $i \in \RS$, define
\begin{enumerate}
\item $\ww^i_\sgstar = \sum_{a\in\is{i}_R, a > I(i)} \ww^a$. 
\item $\vv^i_\sgstar = \sum_{a\in\is{i}_R, a > I(i)} \vv^a$.
\end{enumerate}
Both $\ww^i_\sgstar$ and $\vv^i_\sgstar$ are defined to be zero if $i \notin \RS$.
\end{Def}
\begin{rem} 
Unlike when $i \notin \RS$, the vectors $\ww^i_\sgstar$ (resp.~$\vv^i_\sgstar$) are not column averages of the set of rows $\ww$ (resp.~$\vv$) of rowtype $i$ since $\ww^{I(i)}$ (resp.~$\vv^{I(i)}$)
is omitted.  As a result the block of $\ww^i_\sgstar$ defined by the intersection with the diagonal block of $\WW$ with isotropy $\Delta S_{q_i}$ is not a constant block---there are two different coordinate values. \rend
\end{rem}

\begin{prop}\label{prop: sym}
(Notation and assumptions as above.)\\ For $i \in \RS$ (or $i \in \is{R}$), component $i$ of $\grad{\mathcal{L}}(\WW)$ is given by
\begin{equation}
\grad{\mathcal{L}} (\WW)^i =\Gamma_i\ww^i - \is{A}^i +\is{B}^i + \boldsymbol{\Omega}
\end{equation}
where $\Gamma_i, \is{A}^i,\is{B}^i$ and $\boldsymbol{\Omega}$ are defined by:
\begin{scriptsize}
\begin{eqnarray*}
\Gamma_i & = & \frac{X_i}{2}+\frac{1}{2\pi}\left[\frac{(q_i-1)\big(\|\ww^i\|e^{-K_i}\sin(\Theta_i) - \|\vv^i\|e^{-k_i}\sin(\theta_i)\big) - \|\vv^i\|e^{-k_{ii}}\sin(\lambda_{ii})}{\|\ww^i\|}\right]\\
&& + \frac{1}{2\pi}\left[\frac{(q_i-1)\big(\|\ww^i\|E_i^X - \|\vv^i\|e_i^X\big)-\|\vv^i\|e^X_{ii}}{\|\ww^i\|}\right]\\
&& +\frac{\bbeta_i}{2\pi}\left[\frac{(q_i-1)\big[\bbeta_i\|\ww^i\|\big(L_i +(\pi-\wTheta_i) \big) -\bgamma_i\|\vv^i\|\big(\ell_i +(\pi - \wtheta_i)\big)\big]
 - \bgamma_i\|\vv^i\|\big(\ell_{ii} +(\pi-\wlambda_{ii})\big)}{\|\ww^i\|} \right] \\
&&+\frac{1}{2\pi}\left[\sum_{j\in\is{R},j\ne i}q_j\frac{\|\ww^j\| e^{-K_{ij}}\sin(\Lambda_{ij})}{\|\ww^i\|}-\sum_{j\in\is{f},j\ne i}q_j\frac{\|\vv^j\| e^{-k_{ij}}\sin(\lambda_{ij})}{\|\ww^i\|}\right]\\
&&+\frac{1}{2\pi}\left[\sum_{j\in\is{R},j\ne i}q_j\frac{\|\ww^j\|E_{ij}^X}{\|\ww^i\|}-\sum_{j\in\is{f}, j \ne i}q_j\frac{\|\vv^j\|e_{ij}^X}{\|\ww^i\|}\right]\\
&&+\frac{\bbeta_i}{2\pi}\left[\sum_{j\in\is{R},j\ne i}q_j\frac{\bbeta_j\|\ww^j\|\big((\pi-\wLambda_{ij}) + L_{ij}\big)}{\|\ww^i\|} -\sum_{j\in\is{f}, j \ne i}q_j\frac{\bgamma_j\|\vv^j\|\big((\pi-\wlambda_{ij})+\ell_{ij}\big)}{\|\ww^i\|}\right]\\ 
\is{A}^i& = &\frac{1}{2\pi}\left(\sum_{j\in\is{R}, j \ne i}\wLambda_{ij}\ww_\Sigma^j - \sum_{j \in \is{f}, j \ne i}\wlambda_{ij}\vv_\Sigma^j\right) + 
\frac{1}{2\pi}\big(\wTheta_i\ww^i_\sgstar-\wtheta_i\vv^i_\sgstar-\wlambda_{ii}\vv^i\big)\\
\is{B}^i& = & \frac{1}{2\pi}\left(\sum_{j\in\is{R},j\ne i}L_{ij}\ww_\Sigma^j-\sum_{j \in \is{f},j\ne i}\ell_{ij}\vv_\Sigma^j \right)+\frac{1}{2\pi}\big(L_i\ww^i_\sgstar-\ell_i\vv^i_\sgstar-\ell_{ii}\vv^i\big)\\
\bOmega & = & \frac{1}{2}\left(\sum_{j\in\is{R}}\ww_\Sigma^j - \sum_{j \in \is{f}}\vv_\Sigma^j\right)=\frac{1}{2}\left(\sum_{a\in\is{k}}\ww^a - \sum_{b \in \is{d}}\vv^b\right)
\end{eqnarray*}
\end{scriptsize}
\end{prop}
\begin{proof} If $i \in \is{R}$, then
\begin{multline*}
\begin{aligned}
\grad{\mathcal{L}} (\WW)^i =&\left(\frac{\overline{X}_i}{2}+\frac{1}{2\pi}\left[\frac{(q_i-1)\big(\|\ww^i\|\overline{\mathcal{W}}_{i} - \|\vv^i\|\overline{\mathcal{U}}_{i}\big)-\|\vv^i\|\overline{\mathcal{U}}_{ii}}{\|\ww^i\|}\right]\right)\ww^i \\
& + \frac{1}{2\pi}\left(\sum_{j\in\is{R},j \ne i} q_j\frac{\|\ww^j\|\overline{\mathcal{W}}_{ij}}{\|\ww^i\|}  -  \sum_{j \in \is{f},j\ne i} q_j\frac{\|\vv^j\|\overline{\mathcal{U}}_{ij}}{\|\ww^i\|}\right)\ww^i\\
&+\frac{1}{2\pi}\left[\sum_{j\in\is{R},j \ne i} \overline{\mathcal{Y}}_{ij} \ww_\Sigma^j+ \overline{\mathcal{Y}}_i \ww_{\boldsymbol{\Sigma}}^i\right] 
  - \frac{1}{2\pi}\left[\sum_{j \in \is{f}}\overline{\mathcal{V}}_{ij} \vv_\Sigma^j + \overline{\mathcal{V}}_i \vv_{\boldsymbol{\Sigma}}^i\right]
\end{aligned}
\end{multline*}
where the terms involving $\overline{\mathcal{W}}_{i},\overline{\mathcal{U}}_{i}, \overline{\mathcal{Y}}_i,\overline{\mathcal{V}}_i$ are  zero if $i \notin\RS$.
The proof now follows that of Proposition~\ref{prop: asym} but using the additional terms $\overline{\mathcal{W}}_{i}$, $\overline{\mathcal{U}}_{i}$, $\overline{\mathcal{Y}}_{i},\overline{\mathcal{V}}_i$ when $i \in \RS$
(see below and Section~\ref{sec: sec12} for examples). \end{proof}
\begin{rem}
The theorem is valid for all $i \in \is{R}$ and general targets $\VV \in M(k,d)$ (assume no parallel rows). Indeed, if there are no column symmetries then
the statement of the theorem is \emph{simpler}. For example, there will be no second term in the expressions for  $\is{A}^i$ and $\is{B}^i$
and $\ww^i_\Sigma$ gets replaced by $\ww^i$ and $\vv^j_\Sigma$ by $\vv^j$---every row has a unique rowtype. However, if there are no column symmetries then
the computational cost of finding spurious minima increases rapidly with $d$ as there is no scope for reducing the problem to a low dimensional fixed point space.
Granted the benefits of the symmetry assumption, there is necessarily a cost to that assumption: the equations are more complex.  
However, it is possible to break the symmetry in a diagonal block $A$ and use path based methods to examine the effects of forced symmetry breaking within the block $A$. 
As long as the asymmetric block size remains fixed, $d$ can be increased and it is possible to analyze the limiting behaviour as $d \arr \infty$ to obtain FPS for 
families of critical points parametrized by $d$. Put another way,
if we start with a symmetric family of spurious minima defined for $d \ge d_0$ and with isotropy $\Delta S_{d-1}$, it is possible to study numerically and analytically 
forced symmetry breaking to a family with  isotropy $\Delta S_{d-a} = \Delta S_{d-a} \times S_1^a$. Again the equations are more complex than those of the
fully asymmetric system but there is the promise of proving strong analytical results based on FPS---for example, asymptotics of the loss. 
We indicate one or two examples in Section~\ref{Sec: 11}. \rend
\end{rem}

\subsection{The terms $\is{A}^i,\is{B}^i$ when $i \in \RS$ and $\VV = I_d$}\label{sec: AB}
If $i \in \RS$, then the row vectors $\is{A}^i, \is{B}^i$ have a block decomposition induced from the natural block decomposition 
on $M(k,d)^{G^\Delta_{\mathbf{q}}}$. However, if $i \in \RS$, then $q_i > 1$ and 
block $i$ will intersect the $q_i\times q_i$-diagonal block associated to $i$ (decomposition of $\WW\in M(k,d)^{G^\Delta_{\mathbf{q}}}$ into block diagonal form) 
and the entries of this block will not all be equal. 
Thus, the $1\times q_i$ block of $\is{A}^i$ will be $[A^i_i,(A^i_{\star i})^{q_i-1}]$ with
\begin{eqnarray*}
A^i_i & = & \frac{1}{2\pi}\left[(q_i - 1) \wTheta_i \xi^i_{\star i} + \sum_{\alpha \ne i} q_\alpha\wLambda_{i\alpha}\xi^\alpha_{i} - \wlambda_{ii}\right],\;\; q_i \ge 1 \\
& = & \frac{1}{2\pi}\left[\sum_{\alpha \ne i} q_\alpha\wLambda_{i\alpha}\xi^\alpha_{i} - \wlambda_{ii}\right], \;\;\text{if } q_i = 1\\
A^i_{\star i} & = &  \frac{1}{2\pi}\left[\wTheta_i \xi^i_i + (q_i - 2) \wTheta_i \xi^i_{\star i} + \sum_{\alpha \ne i } q_\alpha\wLambda_{i\alpha}\xi^\alpha_{i} - \wtheta_i\right],\;q_i \ge 2 
\end{eqnarray*}
Finally, the angle terms $A^i_j$ when $ j \ne i$. 
\begin{eqnarray*}
 A^i_j & = & \frac{1}{2\pi}\left[\sum_{\alpha\ne j} q_\alpha \wLambda_{i \alpha } \xi^\alpha_j + 
                (q_i-1) \wTheta_i \xi^i_j + \wLambda_{ij}((q_j-1)\xi_{\star j}^j+\xi_j^j) - \wlambda_{ij}\right],\; q_i \ge 2 \\
& = & \frac{1}{2\pi}\left[\sum_{\alpha\ne j} q_\alpha \wLambda_{i\alpha } \xi^\alpha_j  +\wLambda_{ij}((q_j-1)\xi_{\star j}^j+\xi_j^j) - \wlambda_{ij}\right],\; q_i = 1
\end{eqnarray*}
These expressions extend those given for $i \notin \RS$ and apply whenever $i \in\is{R}$, $j \in \is{f}$.

The results for $\is{B}^i$, $i \in \RS$, are similar:  

\begin{eqnarray*}
B^i_i & = & \frac{1}{2\pi}\left[(q_i - 1) L_i \xi^i_{\star i} + \sum_{\alpha \ne i} q_\alpha L_{i\alpha}\xi^\alpha_{i} - \ell_{i}\right],\; q_i \ge 1 \\
& = & \left[\sum_{\alpha \ne i} q_\alpha  L_{i\alpha}\xi^\alpha_{i} - \ell_{i}\right], \;\text{if } q_i = 1\\
B^i_{\star i} & = &  \frac{1}{2\pi}\left[L_i \xi^i_i + (q_i - 2) L_i \xi^i_{\star i} + \sum_{\alpha \ne i } q_\alpha L_{i\alpha}\xi^\alpha_{i} - \ell_{i}\right], \; q_i \ge 2 
\end{eqnarray*}
If $ j \ne i$, then
\begin{eqnarray*}
B^i_j & = & \frac{1}{2\pi}\left[\sum_{\alpha\ne j} q_\alpha L_{i \alpha } \xi^\alpha_j + 
                (q_i-1) L_{i} \xi^i_j + L_{ij}((q_j-1)\xi_{\star j}^j+\xi_j^j) - \ell_{ij}\right]
\end{eqnarray*}

\subsection{Population loss with symmetric target}
The section concludes with expressions for the population loss $\mathcal{L}$ and the $\bbeta$-derivatives of $\mathcal{L}$ when the network is biased and $\WW \in M(k,d)$ has isotropy is $G^\Delta_{\is{q}} \subset \Delta S_d$.
As above, the target $\VV$ is assumed to be $\is{I}_d \in M(d,d)$ (zero bias). $ k \ge d$ and there are $R$ rowtypes, $f$ factors, $r_s$ factors that have symmetry $S_{q_i}$ with $q_i > 1$, and 
$\is{R}\smallsetminus \is{f} = \is{r_m}$ labels
neurons corresponding to the last $k-d$ rows of $\WW$.

Recall that if $\WW \in M(k,d)^{G^\Delta_{\is{q}}}$ and $i \in \is{R}$ is a rowtype, 
then $\ww^i$ denotes the first row in $\WW$ of rowtype $i$ (that is, $\ww^{I(i)}$). 
If $i \in \RS$, let $\widetilde{\ww}^i$ denote any row of rowtype $i$ other than $\ww^i$. For example, $\ww^{I(i)+1}$. If $q_i$ denotes the number of rows of rowtype $i$,  then $q_i = 1$ iff $i \notin \RS$.

The population loss $\mathcal{L}$ comprises terms $f(\ww,\ww')$, $f(\ww,\vv)$ and $f(\vv,\vv')$. Each of these terms can be described in terms of rowtype.

\subsubsection{Terms $f(\ww,\ww')$} 
\begin{itemize}
\item[(1a)] If $i \in \is{R}$, then 
\[
\sum_{a \in I(i)} f(\ww^a,\ww^a) = q_i f(\ww^i,\ww^i)
\]
\item[(1b)] If $i \in \RS$, then 
\[
\sum_{a,b \in I(i), a \ne b} f(\ww^a,\ww^b) = q_i(q_i-1) f(\ww^i,\widetilde{\ww}^i)
\]
\item[(1c)] If $i \ne j, i,j \in \is{R}$, then 
\[
\sum_{a\in I(i), b \in I(j)} \big(f(\ww^a,\ww^b) + f(\ww^b,\ww^a)\big) =  2q_iq_jf(\ww^i,\ww^j)
\]
\end{itemize}
\subsubsection{Terms $f(\ww,\vv)$} 
\begin{itemize}
\item[(2a)] If $i \in \is{f}$, then
\[
\sum_{a \in I(i)} f(\ww^a,\vv^a) = q_if(\ww^i,\vv^i)
\]
\item[(2b)] If $i\in \RS$, then
\[
\sum_{a,b \in I(i), a \ne b} f(\ww^a,\vv^b) = q_i(q_i-1) f(\ww^i,\widetilde{\vv}^i)
\]
\item[(2c)] If $i\in \is{R},j \in \is{f}$, $i \ne j$, then
\[
\sum_{a\in I(i), b \in I(j)} f(\ww^a,\vv^b) =  q_iq_jf(\ww^i,\vv^j)
\]
\end{itemize}
\subsubsection{Terms $f(\vv,\vv')$} 
\begin{itemize}
\item[(3a)] If $i \in \is{f}$, then 
\[
\sum_{a \in I(i)} f(\vv^a,\vv^a) = q_i f(\vv^i,\vv^i)
\]
\item[(3b)] If $i \in \RS$, then 
\[
\sum_{a,b \in I(i), a \ne b} f(\vv^a,\vv^b) = q_i(q_i-1) f(\vv^i,\widetilde{\vv}^i) 
\]
\item[(3c)] If $i < j, i,j \in \is{f}$, then 
\[
\sum_{a\in I(i) b \in I(j)} f(\vv^a,\vv^b) = \sum_{a\in I(i) b \in I(j)} f(\vv^b,\vv^a) = q_iq_jf(\vv^i,\vv^j) 
\]
\end{itemize}
\begin{rem}
It is assumed here that $\{\vv^a\}_{a \in \is{d}}$ is an orthonormal basis of $(\real^d)^\star \approx \real^d$ and that the target network has bias zero, Hence,
\[
\sum_{a\in I(i)}f(\vv^a,\vv^a) = \frac{q_i}{2},\; \sum_{a,b \in I(i), a \ne b} f(\vv^a,\vv^b) = \frac{q_i(q_i-1)}{2\pi}, \; \sum_{a\in I(i)i, b \in I(j)} f(\vv^a,\vv^b) = \frac{q_iq_j}{2\pi}
\]
It is not difficult to compute these sums if the target is the identity with non-zero bias $\boldsymbol{\bgamma}$ fixed by $G_{\is{q}}^\Delta$ (use Example~\ref{EX: on_set_target}). 
\rend
\end{rem}
\begin{prop}[Loss in terms of rowtype]
\mbox{ } \\
(Notation and assumptions as above.)
\begin{eqnarray*}
\mathcal{L}(\WW)& =& \sum_{i \in \is{f}} q_i \big[\frac{f(\ww^i,\ww^i)+f(\vv^i,\vv^i)}{2} - f(\ww^i,\vv^i)\big]\\
&& \hspace*{-0.5in} + \sum_{i\in \RS} q_i(q_i-1)\big[\frac{f(\ww^i,\widetilde{\ww}^i)+f(\vv^i,\widetilde{\vv}^i)}{2} - f(\ww^i,\widetilde{\vv}^i)\big] \\
&&\hspace*{-0.5in} + \sum_{i,j\in \is{f}, i<j} q_iq_j\big[\big(f(\ww^i,\ww^j) + f(\vv^i,\vv^j)\big) -\big(f(\ww^i,\vv^j)+f(\ww^j,\vv^i)\big)\big] \\
&&\hspace*{-0.5in}+ \sum_{i \in \is{r_m}} \left[\frac{f(\ww^i,\ww^i)}{2} +\sum_{j \in \is{f}} q_j \big(f(\ww^i.\ww^j) -f(\ww^i,\vv^j)\big)+ \sum_{j \in \is{r_m},j<i}f(\ww^i,\ww^j)\right] 
\end{eqnarray*}
\end{prop}
\begin{proof} The first, second and third lines of the expression for $\mathcal{L}(\WW)$ follow respectively from the (a), (b) and (c) lines above; the
final line accounts for the contribution from the extra neurons parametrized by $\is{r_m}$ and depends on the (a) and (b) lines. \end{proof} 
\subsection{Derivatives of $\mathcal{L}$ with respect to $\bbbeta$, indexed by rowtype} 
We continue with the assumptions of the previous subsection. In particular, $(\WW,\bbbeta) \in M(k,d+1)^{G_{\is{q}}^\Delta}$ and $\{\vv^j\}_{j \in \is{d}}$ is an orthonormal basis of $(\real^d)^\star$. The formula for $\frac{\partial \mathcal{L}}{\partial \bbeta_i}$ is given 
in Proposition~\ref{prop: lossderivbeta}. However, the indexing is by row number rather than row type and there are no symmetry assumptions.   The next proposition gives the formula for
derivatives of $\mathcal{L}$ with respect to $\bbbeta$ in terms of rowtype.

\begin{prop}\label{prop: betaderv}
(Notation and assumptions as above.) For $i \in \is{R}$
{\small 
\begin{multline*}
\begin{aligned}
q_i^{-1}\frac{\partial \mathcal{L}}{\partial \bbeta_i}& = \frac{\|\ww^i\|^2}{2}\left[\bbeta_i\erfcb{\frac{\bbeta_i}{\sqrt{2}}} - \sqrt{\frac{2}{\pi}}e^{-\frac{\bbeta_i^2}{2}}\right]\\
& +\frac{\|\ww^i\|^2}{2\pi}(q_i-1)\bbeta_i \left(\pi\,\erfcb{\frac{\bbeta_i}{\sqrt{2}}}-\wTheta_{i}\right)\\
&-\frac{\|\ww^i\|}{2\pi}\bgamma_i \left(q_i\frac{\pi}{2}\left[\erfcb{\frac{\bbeta_i}{\sqrt{2}}}+\erfcb{\frac{\bgamma_i}{\sqrt{2}}}\right]-(\wlambda_{ii}+(q_i-1)\wtheta_i)\right),\;(i \in \is{f})\\
& -\frac{\|\ww^i\|^2}{2\sqrt{2\pi}}(q_i-1)\left[e^{-\frac{\bbeta_i^2}{2}}\erfcb{\frac{-\beta_i(-\cot(\Theta_i) + \csc(\Theta_i))}{\sqrt{2}}}(1+\cos(\Theta_i))\right]\\
& +\frac{\|\ww^i\|}{2\sqrt{2\pi}}(q_i-1)\left[e^{-\frac{\bgamma_i^2}{2}}\erfcb{\frac{z_{i}}{\sqrt{2}}}+\cos(\theta_{i})e^{-\frac{\bbeta_i^2}{2}}\erfcb{\frac{\wzz_{i}}{\sqrt{2}}}\right] \\
& +\frac{\|\ww^i\|}{2\sqrt{2\pi}}\left[e^{-\frac{\bgamma_i^2}{2}}\erfcb{\frac{z_{ii}}{\sqrt{2}}}+\cos(\lambda_{ii})e^{-\frac{\bbeta_i^2}{2}}\erfcb{\frac{\wzz_{ii}}{\sqrt{2}}}\right], \;(i \in \is{f})\\
& +\frac{\|\ww^i\|}{2\pi}\sum_{j\in\is{R},j \ne i}q_j\|\ww^j\|\bbeta_j \left(\frac{\pi}{2}\left[\erfcb{\frac{\bbeta_i}{\sqrt{2}}}+\erfcb{\frac{\bbeta_j}{\sqrt{2}}}\right]-\wLambda_{ij}\right)\\
& -\frac{\|\ww^i\|}{2\sqrt{2\pi}}\sum_{j\in\is{R},j \ne i}q_j\|\ww^j\|\left[e^{-\frac{\bbeta_j^2}{2}}\erfcb{\frac{Z_{ij}}{\sqrt{2}}}+\cos(\Lambda_{ij})e^{-\frac{\bbeta_i^2}{2}}\erfcb{\frac{\wZZ_{ij}}{\sqrt{2}}}\right]\\
& -\frac{\|\ww^i\|}{2\pi}\sum_{j\in\is{f}, j \ne i}\bgamma_j q_j\left(\frac{\pi}{2}\left[\erfcb{\frac{\bbeta_i}{\sqrt{2}}}+\erfcb{\frac{\bgamma_j}{\sqrt{2}}}\right]-\wlambda_{ij}\right)\\
& +\frac{\|\ww^i\|}{2\sqrt{2\pi}}\sum_{j\in\is{f}, j \ne i}q_j\left[e^{-\frac{\bgamma_j^2}{2}}\erfcb{\frac{z_{ij}}{\sqrt{2}}}+\cos(\lambda_{ij})e^{-\frac{\bbeta_i^2}{2}}\erfcb{\frac{\wzz_{ij}}{\sqrt{2}}}\right]\\
\end{aligned}
\end{multline*}
} \normalsize 
\end{prop}
\begin{proof} The proof follows directly from Proposition~\ref{prop: lossderivbeta}.  Note that lines 3 \& 6, ending in $(i \in \is{f})$, are defined to be zero if $i \notin \is{f}$. The  lines 2,4 \& 5 are zero if $i \notin\RS$.
\end{proof}
\begin{rem}
The terms $Z_{ij},\wZZ_{ij}, z_{ij}, \wzz_{ij}$ are defined in Section~\ref{SEC: table}, and $z_i, \wzz_{i}$ at the beginning of Section~\ref{subsec: rtype}. \rend
\end{rem}

\begin{exams}
(1) If $\boldsymbol{\bgamma} = \is{0}$---the only case we investigate in this work---then
{\small 
\begin{multline*}
\begin{aligned}
q_i^{-1}\frac{\partial \mathcal{L}}{\partial \bbeta_i}& = \frac{\|\ww^i\|^2}{2}\left[\bbeta_i\erfcb{\frac{\bbeta_i}{\sqrt{2}}} - \sqrt{\frac{2}{\pi}}e^{-\frac{\bbeta_i^2}{2}}\right]\\
& +\frac{\|\ww^i\|^2}{2\pi}(q_i-1)\bbeta_i \left(\pi\,\erfcb{\frac{\bbeta_i}{\sqrt{2}}}-\wTheta_{i}\right)\\
& -\frac{\|\ww^i\|^2}{2\sqrt{2\pi}}(q_i-1)\left[e^{-\frac{\bbeta_i^2}{2}}\erfcb{\frac{-\beta_i(-\cot(\Theta_i) + \csc(\Theta_i))}{\sqrt{2}}}(1+\cos(\Theta_i))\right]\\
& +\frac{\|\ww^i\|}{2\sqrt{2\pi}}(q_i-1)\left[\erfcb{\frac{z_{i}}{\sqrt{2}}}+\cos(\theta_{i})e^{-\frac{\bbeta_i^2}{2}}\erfcb{\frac{\wzz_{i}}{\sqrt{2}}}\right] \\
& +\frac{\|\ww^i\|}{2\sqrt{2\pi}}\left[e^{-\frac{\bgamma_i^2}{2}}\erfcb{\frac{z_{ii}}{\sqrt{2}}}+\cos(\lambda_{ii})e^{-\frac{\bbeta_i^2}{2}}\erfcb{\frac{\wzz_{ii}}{\sqrt{2}}}\right],\;(i \in \is{f})\\
& +\frac{\|\ww^i\|}{2\pi}\sum_{j\in\is{R},j \ne i}q_j\|\ww^j\|\bbeta_j \left(\frac{\pi}{2}\left[\erfcb{\frac{\bbeta_i}{\sqrt{2}}}+\erfcb{\frac{\bbeta_j}{\sqrt{2}}}\right]-\wLambda_{ij}\right)\\
& -\frac{\|\ww^i\|}{2\sqrt{2\pi}}\sum_{j\in\is{R},j \ne i}q_j\|\ww^j\|\left[e^{-\frac{\bbeta_j^2}{2}}\erfcb{\frac{Z_{ij}}{\sqrt{2}}}+\cos(\Lambda_{ij})e^{-\frac{\bbeta_i^2}{2}}\erfcb{\frac{\wZZ_{ij}}{\sqrt{2}}}\right]\\
& +\frac{\|\ww^i\|}{2\sqrt{2\pi}}\sum_{j\in\is{f}, j \ne i}q_j\left[\erfcb{\frac{\bbeta_i\csc(\lambda_{ij})}{\sqrt{2}}}+\cos(\lambda_{ij})e^{-\frac{\bbeta_i^2}{2}}\erfcb{\frac{-\bbeta_i\cot(\lambda_{ij})}{\sqrt{2}}}\right]
\end{aligned}
\end{multline*}
} \normalsize
(2) Suppose $\boldsymbol{\bgamma} = \is{0}$, $k = d$ and $G_{\is{q}}^\Delta = \Delta (S_{q_0} \times S_{q_1})$. where $q_0 \ge q_1 \ge 1$ and $d \ge 3$. There are two rowtypes.  Since $q_0 \ge 2$, 
{\small 
\begin{multline*}
\begin{aligned}
q_0^{-1}\frac{\partial \mathcal{L}}{\partial \bbeta_0}& = \frac{\|\ww^0\|^2}{2}\left[\bbeta_0\erfcb{\frac{\bbeta_0}{\sqrt{2}}} - \sqrt{\frac{2}{\pi}}e^{-\frac{\bbeta_0^2}{2}}\right]
 +\frac{\|\ww^0\|^2}{2\pi}(q_0-1)\bbeta_0 \left(\pi\,\erfcb{\frac{\bbeta_0}{\sqrt{2}}}-\wTheta_{0}\right)\\
& -\frac{\|\ww^0\|^2}{2\sqrt{2\pi}}(q_0-1)\left[e^{-\frac{\bbeta_0^2}{2}}\erfcb{\frac{-\beta_0(-\cot(\Theta_0) + \csc(\Theta_0))}{\sqrt{2}}}(1+\cos(\Theta_0))\right]\\
& +\frac{\|\ww^0\|}{2\sqrt{2\pi}}(q_0-1)\left[\erfcb{\frac{z_{0}}{\sqrt{2}}}+\cos(\theta_{0})e^{-\frac{\bbeta_0^2}{2}}\erfcb{\frac{\wzz_{0}}{\sqrt{2}}}\right] \\
& +\frac{\|\ww^i\|}{2\sqrt{2\pi}}\left[\erfcb{\frac{z_{00}}{\sqrt{2}}}+\cos(\lambda_{00})e^{-\frac{\bbeta_0^2}{2}}\erfcb{\frac{\wzz_{00}}{\sqrt{2}}}\right] \\
& +\frac{\|\ww^0\|}{2\pi}q_1\|\ww^1\|\bbeta_1 \left(\frac{\pi}{2}\left[\erfcb{\frac{\bbeta_0}{\sqrt{2}}}+\erfcb{\frac{\bbeta_1}{\sqrt{2}}}\right]-\wLambda_{01}\right)\\
& -\frac{\|\ww^0\|}{2\sqrt{2\pi}}q_1\|\ww^1\|\left[e^{-\frac{\bbeta_1^2}{2}}\erfcb{\frac{Z_{01}}{\sqrt{2}}}+\cos(\Lambda_{01})e^{-\frac{\bbeta_0^2}{2}}\erfcb{\frac{\wZZ_{01}}{\sqrt{2}}}\right]\\
& +\frac{\|\ww^0\|}{2\sqrt{2\pi}}q_1\left[\erfcb{\frac{\bbeta_0\csc(\lambda_{01})}{\sqrt{2}}}+\cos(\lambda_{01})e^{-\frac{\bbeta_0^2}{2}}\erfcb{\frac{-\bbeta_0\cot(\lambda_{01})}{\sqrt{2}}}\right]
\end{aligned}
\end{multline*}
} \normalsize
If $q_1 = 1$,
{\small 
\begin{multline*}
\begin{aligned}
\frac{\partial \mathcal{L}}{\partial \bbeta_1}& = \frac{\|\ww^1\|^2}{2}\left[\bbeta_1\erfcb{\frac{\bbeta_1}{\sqrt{2}}} - \sqrt{\frac{2}{\pi}}e^{-\frac{\bbeta_1^2}{2}}\right]\\
& +\frac{\|\ww^1\|}{2\pi}q_0\|\ww^0\|\bbeta_0 \left(\frac{\pi}{2}\left[\erfcb{\frac{\bbeta_0}{\sqrt{2}}}+\erfcb{\frac{\bbeta_1}{\sqrt{2}}}\right]-\wLambda_{10}\right)\\
& +\frac{\|\ww^1\|}{2\sqrt{2\pi}}\left[\erfcb{\frac{z_{11}}{\sqrt{2}}}+\cos(\lambda_{11})e^{-\frac{\bbeta_1^2}{2}}\erfcb{\frac{\wzz_{11}}{\sqrt{2}}}\right] \\
& -\frac{\|\ww^1\|}{2\sqrt{2\pi}}q_0\|\ww^0\|\left[e^{-\frac{\bbeta_0^2}{2}}\erfcb{\frac{Z_{10}}{\sqrt{2}}}+\cos(\Lambda_{10})e^{-\frac{\bbeta_1^2}{2}}\erfcb{\frac{\wZZ_{10}}{\sqrt{2}}}\right]\\
& +\frac{\|\ww^1\|}{2\sqrt{2\pi}}q_0\left[\erfcb{\frac{\bbeta_1\csc(\lambda_{10})}{\sqrt{2}}}+\cos(\lambda_{10})e^{-\frac{\bbeta_1^2}{2}}\erfcb{\frac{-\bbeta_1\cot(\lambda_{10})}{\sqrt{2}}}\right]
\end{aligned}
\end{multline*}
} \normalsize
If $q_1 > 1$, the formula can be deduced from the result for $q_0$ by appropriate permutation of the indices.

More details for the simplest case $G_{\is{q}}^\Delta = \Delta S_{d}$ appear at the beginning of Section~\ref{sec: sec12}.
\examend
\end{exams}

\section{Review of results for symmetric target $I_d$}\label{sec: revsymm}
We review parts of the theory for bias free networks when the target is $I_d$ and $k \ge d$. The emphasis is on the critical point equations and examples of infinite families of critical points parametrized by $d$.
Details and proofs can be found in~\cite{ArjevaniField2021x,ArjevaniField2020b,ArjevaniField2021,ArjevaniField2022b}. For each example, the corresponding critical point equations are given for the biased network
under the assumption that the target has zero bias.

\subsection{Fixed point spaces and fractional power series}
Suppose $k_0 - d_0 = m \ge 0$ and $\WW \in M(k_0,d_0)$ is a critical point of $\mathcal{L}$ of isotropy $G^\Delta_{\is{q}}$,
where $q_0 > q_1 \ge \ldots \ge q_{f-1} \ge 1$.  
Taking $q_1,\ldots, q_{f-1}$ as constant and $q_0 = d - \sum_{i>0} q_i$ (variable), it follows from section~\ref{sec: fixsp} that for $d \ge d_0$ there is a natural linear isomorphism
\[
M(k,d)^{G_\mathbf{q}^\Delta} \approx \left(\prod_{i \in \is{R}} \real^f\right)  \times \real^{r_s}
\]
which we use to coordinatize $M(k,d)^{G_\mathbf{q}^\Delta}$. 
Here $r_s$ counts the number of $q_i> 1$ and $R = f+m$. 
If $m > 0$, It is convenient to define $q_j = 1$ for $j \in [f,m+f-1]$.
Denoting coordinates on $\left(\prod_{i \in \is{R}} \real^f\right)  \times \real^{r_s}$ by $(\xi_j^i)_{i \in \is{R},j \in \is{f}}, (\xi_{\star i}^i)_{i\in \RS}$, 
the critical point equations restricted to the flow invariant subspace $M(k,d)^{G_\mathbf{q}^\Delta}$ may be written
in $\bxi,\bxi_\star$ coordinates as
\begin{eqnarray}\label{EQ: CPE1}
\frac{1}{2\pi}\left[\Gamma_i \xi_j^i -A_j^i\right] + \frac{1}{2}\Omega_j& =& 0,\;j\in\is{f}\\
\label{EQ: CPE2}
\frac{1}{2\pi}\left[\Gamma_i \xi_{\star i}^i -A_{\star i}^i\right] + \frac{1}{2}\Omega_i& =& 0,\;\text{if } q_i > 1,
\end{eqnarray}
where 
\begin{eqnarray*}
A^i_j & = & \sum_{\ell\ne j} q_\ell \Lambda_{i \ell } \xi^\ell_j + 
                (q_i-1) \Theta_i \xi^i_j + \Lambda_{ij}((q_j-1)\xi_{\star j}^j+\xi_j^j) - \lambda_{ij},\; i \ne j \\
A^i_i & = & (q_i - 1) \Theta_i \xi^i_{\star i} + \sum_{j} q_j\Lambda_{ij}\xi^j_{i} - \lambda_{ii},\; i \in \is{f} \\
A^i_{\star i} & = &  \Theta_i \xi^i_i + (q_i - 2) \Theta_i \xi^i_{\star i} + \sum_{j } q_j\Lambda_{ij}\xi^j_{i} - \theta_i, \;\; q_i \ge 2 
\end{eqnarray*}
\begin{rems}\label{isomremark}
(1) The terms $\Omega_j$, $j \in \is{f}$, are defined in section~\ref{sec: Uij}.\\
(2) Every solution of (\ref{EQ: CPE1},\ref{EQ: CPE2}) determines a critical point of $\mathcal{L}$ lying in $M(k,d)^{G_\mathbf{q}^\Delta}$.
The equations (\ref{EQ: CPE1},\ref{EQ: CPE2}) on $\left(\prod_{i \in \is{R}} \real^f\right)  \times \real^{r_s}$ are not the gradient of $\mathcal{L}|M(k,d)^{G_\mathbf{q}^\Delta}$ 
with respect to the standard Euclidean inner product on 
$\left(\prod_{i \in \is{R}} \real^f\right)  \times \real^{r_s}$. In order for that the map from $M(k,d)^{G_\mathbf{q}^\Delta}$ 
to $\left(\prod_{i \in \is{R}} \real^f\right)  \times \real^{r_s}$ needs to be an isometry with respect to the
Euclidean (Frobenius) inner product on $M(k,d)^{G_\mathbf{q}^\Delta}$. Relative to the natural block decomposition of
matrices in $\mathcal{L}|M(k,d)^{G_\mathbf{q}^\Delta}$, non diagonal blocks $x I_{q_i,q_j}$ will map to
$x \sqrt{q_i q_j}$ and a $q_i \times q_i$-diagonal block with diagonal entry $x$ and off diagonal entry $y$ will
map to $(\sqrt{q_i}x,\sqrt{q_i(q_i-1)}y)$ (omitted coordinates are zero). However, there is no advantage 
in doing this since the linearization of the 
vector field is similar by a \emph{diagonal} matrix to the Hessian of the loss. In particular, 
Sylvester's criterion applies to the linearization of the vector field. 
Of course, this remark only applies to the set of eigenvalues associated to trivial representation
of $G_\mathbf{q}^\Delta$ on the fixed point space. For the remaining eigenvalues methods based on representation theory 
may be used~\cite{ArjevaniField2021}. \rend
\end{rems}
\subsubsection{Discussion} 
Since the fixed point space is independent of $d \ge d_0$, we may regard $d \ge d_0$ as a \emph{real} parameter in (\ref{EQ: CPE1},\ref{EQ: CPE2}) and consider the solution $\bxi,\bxi_\star$ of
the parametrized family as a function of $d$. The critical point equations are analytic on an open dense semi-analytic subset of 
$M(k,d)^{G_\mathbf{q}^\Delta}$~\cite{ArjevaniField2021x}. This suggests path following $\bxi = \bxi(d), \bxi_\star =\bxi_\star(d)$ as $d$ is increased. 
This will work by the analytic version of the implicit function theorem provided that the Jacobian determinant along the path (in the fixed point space) is non-singular.
Practically speaking, Newton's method is used to find a discrete series of solutions as $d$ is incremented. Note that the trivial solution $\VV$
does not lie in the domain of analyticity of the critical point.  If $k = d$, it may be shown that the loss is $C^3$ at $\WW = \VV$.

Using high precision computation, it is straightforward to use path following methods to obtain critical points for large values of $d$---for example $d = 10^{1000}$. The numerical results indicate that in many cases there
are fractional power series solutions (FPS) in $d ^{-\frac{p}{q}}$ for these families of critical points. If $k = d$, it is often relatively straightforward to prove the existence of FPS solutions for families of critical points
$(\bxi,\bxi_\star)_{d \ge d_1}$, where $d_1 \ge d_0$, as well as obtain low order coefficients in the FPS. Examples of these computations, which depend on elementary methods using the implicit function theorem, are in ~\cite{ArjevaniField2021x,ArjevaniField2022b}.
Typically the denominator $q$ is a power of $2$ and $p=1$.  The existence of these solutions follows from the general theory of subanalytic nets and the Curve Selection Lemma.  
In practice, numerics can be used to give the initial constant terms
of the FPS (often given by $\pm \vv^j$---the rows of the target) and the exponent $q$ and this data suffices to work out low order terms of the FPS and apply the implicit function theorem. 
For the examples given below, this process is usually easy 
if $k = d$, when the FPS are in powers of $d^{-\frac{1}{2}}$, but harder when $m > 0$ or there is significant asymmetry in the isotropy (many terms with $q_i = 1$) when low order coefficients may need to be computed numerically 
and the series may be in powers of $d^{-\frac{1}{4}}$ implying slow convergence. 

Although the existence of FPS follows straightforwardly from general theory in bias free networks, this is \emph{not} so for biased networks since the critical point equations are not 
subanalytic and so the Curve Selection Lemma and generalities
on subanalytic stratifications may not apply. Even in the simplest cases, this is an issue. For example, suppose the target is $T = I_d$, $k = d$, and the biases $\bgamma_i$ 
in the teacher network are all set to $\bgamma\in\real$.
Let $\mathcal{S} = (\bxi(d),\bxi_\star(d))_{d \ge d_0}$ be a family of of spurious minima of isotropy $\Delta (S_{d-q_1}\times S_{q_1})$. Denote the normalized biases by $\bbeta_i(d)$, $i \in \is{2}$. 
A necessary condition for there 
to be an FPS representation of $\mathcal{S}$ is that for $i \in \is{2}$, $\lim_{d\arr\infty} \bbeta_i(d) = \pm \bgamma$ and that $|\bbeta_i(d) - \bgamma| = O(d^{-\alpha})$, for some $\alpha > 0$.
\emph{This is so, even if $\bgamma = 0$}.
If $\mathcal{S} = (\bxi(d),\bxi_\star(d))_{d \ge d_0}$ is family of critical points for the \emph{unbiased network}, 
then it will generally \emph{not} be a family of critical points for the biased network even if the biases are all set to zero for the target (see Section~\ref{Sec: 11}).
\subsection{A family of spurious minima with isotropy $\Delta S_d$, $k = d$}\label{sec: typeAex}
In the bias free case, the FPS is defined for $d \ge 4$ and for $d \ge 8$ defines a curve of spurious minima. The FPS is given by
convergent series for the components $\xi_0^0,\xi_{\star 0}^0$ of the critical point and is a power series in $d^{-\frac{1}{2}}$:
\[
\xi_0^0 = -1 + \sum_{n=2}^\infty c_n k^{-\frac{n}{2}},\quad
\xi_{\star 0} = \sum_{n=2}^\infty e_n k^{-\frac{n}{2}}
\]
where the initial coefficients are given by
\[
\begin{matrix}
c_0  =  -1& e_0 = 0    & c_1 = e_1 = 0 & c_2 = e_2 = 2 & c_3 = e_3 = 0\\
&&c_4 = \frac{8}{\pi}-4&e_4=\frac{4}{\pi}-2 &
\end{matrix}
\]
The series appears to converge to the critical point for all $d$ for which type A solutions exist but the existence proof~\cite[\S 8.3]{ArjevaniField2021x}, based as it 
is on the implicit function theorem, only gives convergence for ``sufficiently large'' $d$.
The family was originally designated as being of type A in \cite{ArjevaniField2021x}. Subsequently, a new definition of type I was introduced for families with isotropy $\Delta (S_{d-q_1} \times S_{q_1})$, where $q_1 \ge 0$. Roughly speaking,
a family is said to be of type I if the constant coefficient in the FPS for $\xi^0_0$ is $-1$; the family is said to be of type II, if the constant 
coefficient in the FPS for $\xi^0_0$ is $1$ and the critical point does not define the global minimum.  A characteristic feature of type I families is that the loss $\mathcal L = \mathcal{L}(d)$ does not decay to zero as $d \arr \infty$. For the family with isotropy $\Delta S_d$ and $k = d$ it is shown in \cite[\S 8.6]{ArjevaniField2021x} that 
$ \mathcal{L}(d) = \frac{1}{2} - \frac{1}{\pi} +O(d^{-\frac{1}{2}})$.

\subsection{Fractional power series (FPS): types I and II}
We give two examples here to illustrate and fill out the details on fractional power series as well as introduce the notion of type I and II critical point families.
\begin{exams}\label{EX: FPSexs}
(1) We give the FPS for a family of spurious minima of type II~\cite[Thm.~8.1]{ArjevaniField2021x} with isotropy $\Delta(S_{d-1} \times S_1) \approx \Delta S_{d-1}$ and $m = 0$. Taking note of Lemma~\ref{LEM: FP}, 
the dimension of the fixed point space is 5 (Example~\ref{EX: simple}, bias free case). 
There is a family $(\bxi(d))_{d \ge 3}$ of critical points of isotropy $\Delta S_{d-1}$ for which there is an FPS in $1/d$ (power series in $d^{-\frac{1}{2}}$) for each component of the critical point $\bxi(d)$:

\begin{align*}
\xi_0^0& = 1 + \sum_{n=4}^\infty c_n d^{-\frac{n}{2}} =  1 + \frac{8}{\pi}d^{-2} -\frac{320\pi}{3\pi^4 (\pi-2)}d^{-\frac{5}{2}} + O(d^{-3})\\
\xi_{\star 0}^0&= \sum_{n=4}^\infty e_n d^{-\frac{n}{2}} = -\frac{4}{\pi}d^{-2} -\frac{32}{\pi^3}d^{-\frac{5}{2}} + O(d^{-3})\\
\xi_1^0&= \sum_{n=2}^\infty f_n d^{-\frac{n}{2}}= 2d^{-1} + 0 d^{-\frac{3}{2}} + O(d^{-2})\\
\xi^1_0&= \sum_{n=2}^\infty g_n d^{-\frac{n}{2}}=\frac{4}{\pi} d^{-1} + \frac{32}{\pi^3}d^{-\frac{3}{2}} + O(d^{-2})\\
\xi_1^1&= -1 + \sum_{n=2}^\infty h_n d^{-\frac{n}{2}}=-1 + 2\left(1 + \frac{2}{\pi}\right)^2d^{-1} +\frac{64 \pi -768}{3\pi^4(\pi-2)}d^{-\frac{3}{2}} + O(d^{-2})
\end{align*}
The loss $\mathcal{L}(\bxi(d)) = \left(\frac{1}{2}-\frac{1}{\pi^2}\right)d^{-1} + O(d^{-\frac{3}{2}})\arr 0$, as $d \arr \infty$~\cite[\S 8.5]{ArjevaniField2021x}.
The series are convergent for sufficiently large $d$. In practice, it seems possible that the series are convergent for small values of $d$, say $d \ge 9$ (implying a radius of convergence of 
at least $1/3$).

A family $(\bxi(d))_{d\ge d_0}$ of critical points with isotropy group $\Delta S_{d_-p} \times S_p$, $d-p > p$, is said to be of \emph{type I} (resp.~\emph{type II}) if
it has an FPS such that as $d \arr \infty$, $\xi_0^0(d) \arr -1 $ (resp.~$+1$). More generally, a critical point $\bxi(d_0)$ is of type I (resp.~type II) if it lies on a family of
critical points of type I (resp.~II). The convergence of the FPS at the specific value $d_0$ is immaterial: what matters is the
behaviour as $d \arr \infty$ and that the critical point lies on a family---as defined above. 

The family of critical points of type II described above is a family of spurious (local, not global) minima for $d \ge 6$. For $d \in [3,5]$, $\bxi(d)$ is not a 
spurious minimum of $\mathcal{L}$ but is a non-degenerate critical point: a bifurcation occurs at about $d = 5.58$---transverse to the fixed point space and associated to the standard representation of $S_{d-1}$~\cite{ArjevaniField2022a}. \\

For families of spurious minima of type II (resp.~type I), the loss $\mathcal{L}(\bxi(d))$ typically converges
to zero (resp.~a strictly positive constant) as $d \arr \infty$. \\
(2) We describe the FPS for a type I family $(\bxi(d))$ of critical points with isotropy $\Delta S_{d-1}$ and $m = 1$. Aside from giving an example of a type I family defining spurious minima, the FPS will be in
powers of $d^{-\frac{1}{4}}$ and the constant terms in the FPS lie in $\{0,\pm \frac{1}{2}, \pm 1\}$. The presence of non-trivial terms $c d^{-\frac{1}{4}}$ implies that convergence is slow.
Setting $\bxi(d) = \bxi \in M(d+1,d)$, we have 
\begin{align*}
        \bxi^0_0 & = -1 + 2d^{-1} + \frac{\pi}{2} d^{-\frac{3}{2}} 
        + O(d^{-\frac{7}{4}})\\
	\bxi^0_{\star 0}&= 2d^{-1} - \sqrt{\pi-2}d^{-\frac{7}{4}}+ O(d^{-2})  \\
        \bxi_1^0 & = d^{-1} - \frac{6+3\pi}{4\pi\sqrt{\pi-2}}d^{-\frac{5}{4}}+ O(d^{-\frac{3}{2}})\\
        \bxi^1_0&= \frac{\sqrt{\pi-2}}{2} d^{-\frac{3}{4}} + O(d^{-1}) \\
        \bxi^1_1 & = \frac{1}{2} + \frac{6+3\pi}{8\pi\sqrt{\pi-2}}d^{-\frac{1}{4}} + O(d^{-\frac{1}{2}}) \\
	\bxi^2_0 &= \frac{\sqrt{\pi-2}}{2}d^{-\frac{3}{4}} + O(d^{-1})\\
        \bxi^2_1 & = -\frac{1}{2} + \frac{6+3\pi}{8\pi\sqrt{\pi-2}}d^{-\frac{1}{4}} + O(d^{-\frac{1}{2}})
\end{align*}
The loss $\mathcal{L}(\bxi(d)) = \frac{1}{2} - \frac{1}{\pi} - \frac{4}{3\pi}d^{-\frac{1}{2}} - \frac{(\pi-2)^{\frac{3}{2}}}{3\pi} d^{-\frac{3}{4}} + O(d^{-1})$.
We refer to~\cite[Appendix A]{ArjevaniField2022b} and also for the computation of the FPS for type II when $m = 1,2$ and isotropy is $\Delta S_{d-1}$.
\examend
\end{exams}

\section{Computing $\frac{\partial \mathcal{L}}{\partial \bbeta}(\WW,\mathbf{0})$: symmetric \& asymmetric targets}\label{Sec: 11}
In this section, results are given on the $\bbeta$-derivatives of $\mathcal{L}$ at points $(\WW,\boldsymbol{\bbeta})\in M(k,d+1)$, where $\WW$ is a critical point of the unbiased network and $\boldsymbol{\bbeta} = \mathbf{0}$.
The bias $\boldsymbol{\bgamma}\in \real^d$ associated to the target $\VV$, is also set to $\mathbf{0}\in \real^d$.
The critical points $\WW$ will define spurious minima in $M(k,d)$ but $(\WW,\mathbf{0})$ will generally \emph{not} be a critical point in 
$M(k,d+1)$ \emph{unless} $\mathcal{L}(\WW)=0$ and so defines the global minimum of $\mathcal{L}$. Under this condition,  
the $\bbeta$-derivative is zero ($\WW$ will be a degenerate critical point if $k > d$~\cite[Appendix A]{ArjevaniField2022b}). The examples presented 
include symmetric spurious minima of types I and II
as described in \cite{ArjevaniField2021x,ArjevaniField2020b,ArjevaniField2021,ArjevaniField2022b}. All of these critical points have are associated 
to the symmetric target $\VV = I_d$ and can be represented by FPS in $\frac{1}{d}$
(typically the initial non-constant term of lowest order will be a constant multiple of $1/d^{\frac{1}{2}}$ or $1/d$, but $1/d^{\frac{1}{4}}$ is possible).
Two additional examples are given one of which is completely asymmetric 
(both target and critical point), the other ``partially asymmetric'' (the target has significant asymmetry and critical points have an FPS representation).

\subsection{Computation of  $\frac{\partial \mathcal{L}}{\partial \boldsymbol{\bbeta}}(\WW,\mathbf{0})$ : a symmetric example}

We start with an example illustrating how the \fps expansion for a family of critical points can be used to estimate $\bbeta$-derivatives at $\boldsymbol{\bbeta} = \mathbf{0}$.
\begin{exam}
Consider the type II family $(\mathfrak{c}(d))_{d \ge 3}$ of critical points with $k=d$ and isotropy $\Delta S_{d-1}$ 
described in Examples~\ref{EX: FPSexs}(1). Using the \fps expansion and the formula given in Remark~\ref{REM: formulabeta},  we estimate the decay of 
$\frac{\partial \mathcal{L}}{\partial \bbeta_0}(\mathfrak{c}(d),\mathbf{0}))$.  
\begin{rem}\label{rem: betaderivlabe}
The symmetry assumption implies that
$\frac{\partial \mathcal{L}}{\partial \bbeta_i}(\mathfrak{c}(d),\mathbf{0})) = \frac{\partial \mathcal{L}}{\partial \bbeta_0}(\mathfrak{c}(d),\mathbf{0}))$, if $i \in \is{d-1}$.
In general, $\frac{\partial \mathcal{L}}{\partial \bbeta_0}(\mathfrak{c}(d),\mathbf{0})) \ne \frac{\partial \mathcal{L}}{\partial \bbeta_{d-1}}(\mathfrak{c}(d),\mathbf{0}))$.
In the sequel, we typically parameterize the $\bbeta$ derivatives by rowtype---indeed, this is what we do when we solve the critical point equations for $\bxi,\bxi_\star,\bbbeta$ 
on the associated seven dimensional fixed point space as well as in the tables presented in the remainder of this section. \rend
\end{rem}

Let $\tau_i$ denote the norm of a row of $\mathfrak{c}(d)$ of rowtype $i$, $i \in \is{2}$. Using the \fps for the 
critical point given in Examples~\ref{EX: FPSexs}(1), we find that
\begin{eqnarray*}
\tau_0 & = & 1 + \left(2+\frac{8}{\pi}\right)d^{-2} - \frac{320\pi}{3\pi^4(\pi-2)}d^{-\frac{5}{2}} + O(d^{-3})\\
\tau_1 & = & 1 -\left(2+\frac{8}{\pi^2}\right)d^{-1} + \frac{320\pi}{3\pi^4(\pi-2)}d^{-\frac{3}{2}} + O(d^{-2})
\end{eqnarray*}
Since $\sum_{j\in\is{d}} \|\ww^j\| = (d-1)\tau_0 + \tau_1$, it follows easily that 
\[
\sum_{j\in\is{d}} \|\ww^j\| = d + O(d^{-2})
\]
and, since $\sum_{j\in\is{d}} \|\vv^j\| = d$, 
\[
\frac{1}{2\sqrt{2\pi}}
\|\ww^i\|\big[\sum_{j\in \is{d}}\|\vv^j\|-\sum_{j\in \is{k}}\|\ww^j\|\big] = O(d^{-2}),\;\; i \in \is{d-1}
\]
For the second term in the formula of Remark~\ref{REM: formulabeta}, we need to estimate the \fps of inner products. Specifically
\subsubsection{Inner products of $\ww^0$ with $\ww_i$ and $\vv^j$, $j \in \is{d}$} Computing, using the \fps,
\begin{eqnarray*}
\langle \ww^0,\ww^0 \rangle & = & \tau_0^2 = 1 + O(d^{-2})\\
(d-2)\langle \ww^0,\ww^1 \rangle & = & \left(-\frac{8}{\pi}\right)d^{-1} - \frac{64}{\pi^3}d^{-\frac{3}{2}} +  O(d^{-2}) \\
\langle \ww^0,\ww^d\rangle & = & \left(\frac{4}{\pi} - 2\right) d^{-1} + \frac{32}{\pi^3} d^{-\frac{3}{2}}  +  O(d^{-2})
\end{eqnarray*}
\begin{eqnarray*}
\langle \ww^0,\vv^0 \rangle & = & 1 +  O(d^{-2}) \\
(d-2)\langle \ww^0,\vv^1 \rangle & = & -\frac{4}{\pi}d^{-1} - \frac{32}{\pi^3}d^{-\frac{3}{2}}  +  O(d^{-2}) \\
\langle \ww^0,\vv^d\rangle & = &  \frac{2}{d} + O(d^{-2})
\end{eqnarray*}
Using these estimates, it follows easily that 
\[
\frac{1}{2\sqrt{2\pi}}
\big[\sum_{j\in \is{d}}\langle \ww^i,\vv^j\rangle - \sum_{j\in \is{k}}\langle\ww^i,\ww^j\rangle\big] = O(d^{-2}).
\]
Hence $\frac{\partial \mathcal{L}}{\partial \bbeta_0}(\mathfrak{c}(d),\mathbf{0}))$ decays at least as fast as $d^{-2}$. With the addition of one further term in each of the five \fps for $\mathfrak{c}(d)$,
the leading term in the \fps for $\frac{\partial \mathcal{L}}{\partial \bbeta_0}(\mathfrak{c}(d),\mathbf{0}))$ can be computed. 
However, it is faster and easier to do a high precision numerical computations to verify that 
the leading term of the \fps for the derivative is a multiple of $d^{-2}$. 

A high precision computation of the derivative taking $d = 10^{512}$ gives the value of the 
leading coefficient $D^0_4$ of $d^{-2}$ in the \fps for the derivative with respect to $\bbeta_0$ to be
\begin{scriptsize}
\begin{eqnarray*}
D^0_4& =& 0.541586362042490362453470777718114374878787486392237162895683618\ldots
\end{eqnarray*}
\end{scriptsize}
(Accurate to 64 decimal places.)
\subsubsection{Computation of $\frac{\partial \mathcal{L}}{\partial \bbeta_0}(\mathfrak{c}(d),\mathbf{0})$ for specific values of $d$}
Using high precision computations, we tabulate values of the derivative for specific choices of $d$: 
\begin{scriptsize}
\begin{table}[h]
\hspace*{-0.5in}
\begin{tabular}{|c||c|c|c|c|c|c|}

\hline
$d = k$          & Value of $\frac{\partial \mathcal{L}}{\partial \bbeta_0}(\mathfrak{c}(d),\mathbf{0})$               \\ \hline \hline
$6$           &$0.71830101072557148145\times 6^{-2}$                 \\ \hline
$10$          &$0.33399917189077800002\times 10^{-2}$                 \\ \hline
$20$          &$0.10183336077334108245\times 10^{-2}$                \\ \hline
$100$         &$0.49862816276459333028\times 10^{-4}$                \\ \hline
$1000$        &$0.53292271608433825762\times10^{-6}$                 \\ \hline
$10^6$        &$0.54139203214105478639 \times 10^{-12}$                \\ \hline
$10^{12}$     &$0.54158617036606773582\times 10^{-24}$                \\ \hline
\end{tabular}
\vspace*{0.05in}
\caption{Numerically computed values of $\frac{\partial \mathcal{L}}{\partial \bbeta_0}(\mathfrak{c}(d),\mathbf{0})$ for a type II critical point with $k = d$ and isotropy $\Delta S_{d-1}$} 
\label{table: estimates}
\end{table}
\end{scriptsize}

Observe that the results for small values of $d$ are smaller than the estimate given by $D^0_4 d^{-2}$.  
The estimate given by $D^0_4/6^2$ is approximately $1.5\times 10^{-2}$. If instead we take $d = 4$, the true value is
approximately $2.8025\times 10^{-2}$ while the estimate $D^0_4d^{-2}$ gives $3.34 \times 10^{-2}$ (the family is not defined for $d = 3$).
\examend
\end{exam}

\subsection{Results for isotropy $\Delta S_d$, type I (type A in \cite{ArjevaniField2021x})}\label{sec: 11.2}
We give results only for $m=0$.

\begin{scriptsize}
\begin{table}[h]
\hspace*{-0.5in}
\begin{tabular}{|c||c|c|c|c|c|c|}

\hline
$d = k$          & Value of $\frac{\partial \mathcal{L}}{\partial \bbeta_0}(\mathfrak{c}(d),\mathbf{0})$               \\ \hline \hline
$10$          &$0.77142020703795495629\times 10^{-2}$                 \\ \hline
$100$          &$0.1220227002721111972\times 10^{-2}$                 \\ \hline
$1000$        &$0.13745083160523485590 \times 10^{-3}$                 \\ \hline
$10^6$        &$0.14472588806870954343 \times 10^{-6}$                \\ \hline
$10^{12}$     &$0.14496749467704066784\times 10^{-12}$                \\ \hline
\end{tabular}
\vspace*{0.05in}
\caption{Numerically computed values of $\frac{\partial \mathcal{L}}{\partial \bbeta_0}(\mathfrak{c}(d),\mathbf{0})$ for a type I critical point with $k = d$ and isotropy $\Delta S_d$}
\label{table: estimates1}
\end{table}
\end{scriptsize}
 
From the table it follows that the derivative decays like $d^{-1}$ and 
high precision computation leads to the value for the initial coefficient $D_2$ in the \fps of the derivative:
\begin{scriptsize}
\begin{eqnarray*}\label{data: typeA}
D^0_2&= & 0.144967736664468798796892840195837787855643034035842961341133133\ldots
\end{eqnarray*}
\end{scriptsize}
\noindent (Correct to the 64 decimal places displayed.)

It is straightforward to obtain an explicit formula for $D^0_2$ in terms of the coefficients of the \fps. 
The initial terms for the \fps of the critical points are easy to derive (see~\cite[\S 8.3]{ArjevaniField2021x}):
\begin{eqnarray*}
\xi^0_0&=&-1 + 2d^{-1} + 2\left(\frac{4}{\pi} -2\right) d^{-2} + O(d^{-\frac{5}{2}})\\
\xi^0_{\star 0}&=& 2d^{-1} + \left(\frac{4}{\pi} -2\right) d^{-2} + O(d^{-\frac{5}{2}})
\end{eqnarray*}
Using the formula of Remark~\ref{REM: formulabeta}, a little algebra gives the exact value of $D^0_2$ and the asymptotics for the derivative: 
\[
\frac{\partial \mathcal{L}}{\partial \bbeta_0}(\mathfrak{c}(d),\mathbf{0}) = \frac{\pi - 2}{\pi\sqrt{2\pi}} d^{-1} + O(d^{-\frac{3}{2}})
\]
The numerically computed vale of $D^0_2$ agrees with $\frac{\pi - 2}{\pi\sqrt{2\pi}}$ to the number of decimal paces shown.

Henceforth, tabular values are only given to 4 significant figures and the coefficients of the leading term for the \fps of 
derivatives are accurate to the $32$ decimal places displayed.
Except where noted, the \fps are series in $d^{-\frac{1}{2}}$ and $D^i_n$ will be the leading non-zero coefficient of $d^{-\frac{n}{2}}$ in the \fps for $\partial \mathcal{L}/\partial \bbeta_i$. 
Thus if $n = 2$, the derivatives decay like $d^{-1}$.

\newpage

\subsection{Results for isotropy $\Delta S_{d-1}$, type I, $m = 0,1,2$} Values of the derivatives are given for $d = 10, 10^3$ and $10^6$.

\begin{scriptsize}
\begin{table}[h]
\hspace*{-0.5in}
\begin{tabular}{|c||c|c|}

\hline
Type I && \\
$d = k$          &  $\partial \mathcal{L}/\partial \bbeta_0$ & $\partial \mathcal{L}/\partial \bbeta_1$              \\  
Isotropy $\Delta S_{d-1}$ &&  \\  \hline \hline
$10$          &$0.7468\times 10^{-2}$    & $0.7283\times 10^{-2}$             \\ \hline
$1000$        &$0.1374 \times 10^{-3}$   & $0.1769\times 10^{-3}$             \\ \hline
$10^6$        &$0.1473 \times 10^{-6}$   & $0.1969\times 10^{-6}$            \\ \hline
\end{tabular}
\vspace*{0.05in}
\caption{Computed values of $\frac{\partial \mathcal{L}}{\partial \bbeta_i}(\mathfrak{c}(d),\mathbf{0})$, $i \in \is{2}$, for a type I critical point with $m = 0$ and isotropy $\Delta S_{d-1}$.} 
\label{table: estimatesx}
\end{table}
\end{scriptsize}

\vspace*{-0.2in}
Initial coefficients for the \fps of derivatives:
\begin{eqnarray*}
D_2^0 & = & 0.144967736664468798796892840195838\ldots \\
D_2^1 & = & 0.197646145812959995056847958971127\ldots
\end{eqnarray*}

\vspace*{-0.01in}

\begin{scriptsize}
\begin{table}[h]
\hspace*{-0.5in}
\begin{tabular}{|c||c|c|c|}

\hline
Type I &&& \\
$k = d+1$          &  $\partial \mathcal{L}/\partial \bbeta_0$ & $\partial \mathcal{L}/\partial \bbeta_1$ &$\partial \mathcal{L}/\partial \bbeta_2$              \\  
Isotropy $\Delta S_{d-1}$& &&  \\  \hline \hline
$10$          &$0.3532\times 10^{-2}$    & $0.4842\times 10^{-2}$ &$0.1070\times 10^{-1}$ \\ \hline
$1000$        &$0.1318 \times 10^{-3}$   & $0.4658\times 10^{-3}$&$0.4885\times 10^{-3}$             \\ \hline
$10^6$        &$0.1446 \times 10^{-6}$   & $0.2495\times 10^{-5}$&$0.2513\times 10^{-5}$            \\ \hline
\end{tabular}
\vspace*{0.05in}
\caption{Computed values of $\frac{\partial \mathcal{L}}{\partial \bbeta_i}(\mathfrak{c}(d),\mathbf{0})$, $i\in \is{3}$, for a type I critical point with $m = 1$ and isotropy $\Delta S_{d-1}$.} 
\label{table: estimatesxx}
\end{table}
\end{scriptsize}
\vspace*{-0.2in}

The \fps is in powers of $d^{-\frac{1}{4}}$ and decay of derivatives is either like $d^{-1}$ ($D^0_4$) or $d^{-\frac{3}{4}}$ ($D^1_3, D^2_3$).
Initial coefficients of the \fps for the  derivatives:
\begin{eqnarray*}
D_4^0 & = & 0.144967736664468798796892840195838\ldots \\
D_3^1 & = & 0.0774456350768987130272615429004301\ldots \\
D_3^2 & = & 0.0774456350768987130272615429004301\ldots
\end{eqnarray*}

\begin{scriptsize}
\begin{table}[h]
\hspace*{-0.5in}
\begin{tabular}{|c||c|c|c|c|}

\hline
Type I &&&& \\
$k = d+2$          &  $\partial \mathcal{L}/\partial \bbeta_0$ & $\partial \mathcal{L}/\partial \bbeta_1$ &$\partial \mathcal{L}/\partial \bbeta_2$ &$\partial \mathcal{L}/\partial \bbeta_3$             \\  
Isotropy $\Delta S_{d-1}$&&&&  \\  \hline \hline
$10$          &$0.2760\times 10^{-2}$    & $0.4500\times 10^{-2}$& $0.4969\times 10^{-1}$& $0.7548\times 10^{-2}$ \\ \hline
$1000$        &$0.1295 \times 10^{-3}$   & $0.4396\times 10^{-3}$& $0.4006\times 10^{-3}$& $0.1937\times 10^{-3}$             \\ \hline
$10^6$        &$0.1449 \times 10^{-6}$   & $0.2483\times 10^{-5}$& $0.2420\times 10^{-5}$& $0.1850i\times 10^{-6}$            \\ \hline
\end{tabular}
\vspace*{0.05in}
\caption{Computed values of $\frac{\partial \mathcal{L}}{\partial \bbeta_i}(\mathfrak{c}(d),\mathbf{0})$, $i\in \is{4}$, for a type I critical point with $m = 2$ and isotropy $\Delta S_{d-1}$.} 
\label{table: estimatesyx}
\end{table}
\end{scriptsize}
The \fps is in powers of $d^{-\frac{1}{4}}$ and decay of derivatives is either like $d^{-1}$ ($D_4^i$) or $d^{-\frac{3}{4}}$ ($D_3^i$).
Initial coefficients of the \fps for the  derivatives:
\begin{eqnarray*}
D_4^0 & = & 0.14496773666446879879689284019584\ldots \\
D_3^1 & = & 0.077445635076898713027261542900430\ldots \\
D_3^2 & = & 0.077445635076898713027261542900430\ldots \\
D_4^3 & = & 0.18171273140268941937100246131777\ldots
\end{eqnarray*}

\subsection{Results for isotropy $\Delta S_{d-1}$, type II, $m = 0,1,2$}
The next three tables give results for type II critical points with isotropy $\Delta S_{d-1}$.

\begin{scriptsize}
\begin{table}[h]
\hspace*{-0.5in}
\begin{tabular}{|c||c|c|}

\hline
Type II && \\
$d = k$          &  $\partial \mathcal{L}/\partial \bbeta_0$ & $\partial \mathcal{L}/\partial \bbeta_1$              \\
Isotropy $\Delta S_{d-1}$ &&  \\  \hline \hline
$10$          &$0.3340\times 10^{-2}$    & $-0.7436\times 10^{-2}$             \\ \hline
$1000$        &$0.5329 \times 10^{-6}$   & $-0.2304\times 10^{-3}$             \\ \hline
$10^6$        &$0.5414. \times 10^{-12}$   & $-0.2371\times 10^{-6}$            \\ \hline
\end{tabular}
\vspace*{0.05in}
\caption{Computed values of $\frac{\partial \mathcal{L}}{\partial \bbeta_i}(\mathfrak{c}(d),\mathbf{0})$, $i \in \is{2}$, for a type II critical point with $m = 0$ and isotropy $\Delta S_{d-1}$.}
\label{table: estimatesIIm0}
\end{table}
\end{scriptsize}

\vspace*{-0.2in}
Initial coefficients for the \fps of derivatives:
\begin{eqnarray*}
D_4^0 & = & 0.541586362042490362453470777718114\ldots \\
D_2^1 & = &-0.237257064180446401333830561616386\ldots
\end{eqnarray*}

\vspace*{-0.01in}

\begin{scriptsize}
\begin{table}[h]
\hspace*{-0.5in}
\begin{tabular}{|c||c|c|c|}

\hline
Type II &&& \\
$k = d+1$          &  $\partial \mathcal{L}/\partial \bbeta_0$ & $\partial \mathcal{L}/\partial \bbeta_1$ &$\partial \mathcal{L}/\partial \bbeta_2$              \\
Isotropy $\Delta S_{d-1}$& &&  \\  \hline \hline
$10$          &$0.2605\times 10^{-2}$    & $-0.6054\times 10^{-2}$ &$-0.2055\times 10^{-2}$ \\ \hline
$1000$        &$-0.1011 \times 10^{-5}$   & $-0.2062\times 10^{-3}$&$-0.7271\times 10^{-5}$             \\ \hline
$10^6$        &$-0.5076 \times 10^{-10}$   & $-0.2131\times 10^{-6}$&$-0.2369\times 10^{-9}$            \\ \hline
\end{tabular}
\vspace*{0.05in}
\caption{Computed values of $\frac{\partial \mathcal{L}}{\partial \bbeta_i}(\mathfrak{c}(d),\mathbf{0})$, $i\in \is{3}$, for a type I critical point with $m = 1$ and isotropy $\Delta S_{d-1}$.}
\label{table: estimatesIIm1}
\end{table}
\end{scriptsize}
\vspace*{-0.2in}

The \fps is in powers of $d^{-\frac{1}{2}}$ and decay of the the derivative is either like $d^{-\frac{3}{2}}$ or $d^{-1}$.
Initial coefficients of the \fps for the  derivatives:
\begin{eqnarray*}
D^0_3 & = & -0.0513640604753390711588488689614188\ldots \\
D^1_2 & = & -0.213238489802580052077245332033039\ldots  \\
D^2_3 & = & -0.237042923895697357373281965133800\ldots 
\end{eqnarray*}
\begin{rem}
$S_2$ acts freely on the last two rows (rowtypes 1, 2). Interchanging these rows will swap the asymptotics. In general, if isotropy is $\Delta (S_{d-1} \times S_1)$ and $m> 0$, there 
will be a free action of $S_m$ on the last $m+1$ rows and corresponding free action on the fixed point space. 
\rend
\end{rem}
\begin{scriptsize}
\begin{table}[h]
\hspace*{-0.5in}
\begin{tabular}{|c||c|c|c|c|}

\hline
Type II &&&& \\
$k = d+2$          &  $\partial \mathcal{L}/\partial \bbeta_0$ & $\partial \mathcal{L}/\partial \bbeta_1$ &$\partial \mathcal{L}/\partial \bbeta_2$ &$\partial \mathcal{L}/\partial \bbeta_3$             \\
Isotropy $\Delta S_{d-1}$&&&&  \\  \hline \hline
$10$          &$0.2590\times 10^{-2}$    & -$0.5638\times 10^{-2}$& $0.1317\times 10^{-2}$& $0.1147\times 10^{-2}$ \\ \hline
$1000$        &$-1.018 \times 10^{-6}$   & $-0.2052\times 10^{-3}$& $-5.682\times 10^{-6}$& $-2.536\times 10^{-6}$             \\ \hline
$10^6$        &$-5.099 \times 10^{-11}$   & $-0.2130\times 10^{-6}$& $-1.86\times 10^{-10}$& $-8.051\times 10^{-11}$            \\ \hline
\end{tabular}
\vspace*{0.05in}
\caption{Computed values of $\frac{\partial \mathcal{L}}{\partial \bbeta_i}(\mathfrak{c}(d),\mathbf{0})$, $i\in \is{4}$, for a type II critical point with $m = 2$ and isotropy $\Delta S_{d-1}$.}
\label{table: estimatesIIm2}
\end{table}
\end{scriptsize}

\vspace*{-0.35in}

The \fps is in powers of $d^{-\frac{1}{2}}$ and derivatives decay  like $d^{-\frac{3}{2}}$ or $d^{-1}$.
Initial coefficients of the \fps for the  derivatives:
\begin{eqnarray*}
D^0_3 & = & -0.0515991045178723939348354300977768\ldots \\
D^1_2 & = & -0.213189482387164650802356811574747\ldots  \\
D^2_3 & = & -0.187109834651767919653399769806042\ldots  \\
D^3_3 & = & -0.0804839616208356631748165154342071\ldots 
\end{eqnarray*}

\subsection{Results for isotropy $G^\Delta_{d-2,2} =\Delta (S_{d-2}\times S_2)$, type I, $m = 0,1,2$}
For this example, as $d \arr \infty$, $\xi_0^0 \arr +1$, $\xi_1^1 \arr -1$ and all off-diagonal entries converge to zero.

\vspace*{-0.15in}
\begin{scriptsize}
\begin{table}[h]
\hspace*{-0.5in}
\begin{tabular}{|c||c|c|}
\hline
Type I && \\
$d = k$          &  $\partial \mathcal{L}/\partial \bbeta_0$ & $\partial \mathcal{L}/\partial \bbeta_1$              \\
Isotropy $G^\Delta_{d-2,2}$ &&  \\  \hline \hline
$10$          &$0.7160\times 10^{-2}$    & $0.6965\times 10^{-2}$             \\ \hline
$1000$        &$0.1374 \times 10^{-3}$   & $0.1769\times 10^{-3}$             \\ \hline
$10^6$        &$0.1447 \times 10^{-6}$   & $0.1969\times 10^{-6}$            \\ \hline
\end{tabular}
\vspace*{0.05in}
\caption{Computed values of $\frac{\partial \mathcal{L}}{\partial \bbeta_i}(\mathfrak{c}(d),\mathbf{0})$, $i \in \is{2}$, for a type I critical point with $m = 0$ and isotropy $G^\Delta_{d-2,2}$.}
\label{table: estimatesz}
\end{table}
\end{scriptsize}

\vspace*{-0.2in}

Initial coefficients of the \fps in $d^{-\frac{1}{2}}$ for derivatives:
\begin{eqnarray*}
D_2^0 & = & 0.144967736664468798796892840195838\ldots \\
D_2^1 & = & 0.197646145812959995056847958971127\ldots
\end{eqnarray*}

\vspace*{-0.01in}

\begin{scriptsize}
\begin{table}[h]
\hspace*{-0.5in}
\begin{tabular}{|c||c|c|c|}

\hline
Type I &&& \\
$k = d+1$          &  $\partial \mathcal{L}/\partial \bbeta_0$ & $\partial \mathcal{L}/\partial \bbeta_1$ &$\partial \mathcal{L}/\partial \bbeta_2$              \\
Isotropy $G^\Delta_{d-2,2}$& &&  \\  \hline \hline
$10$          &$0.2840\times 10^{-2}$    & $0.4030\times 10^{-2}$ &$1.087\times 10^{-2}$ \\ \hline
$1000$        &$0.1304 \times 10^{-3}$   & $0.2585\times 10^{-3}$&$0.4299\times 10^{-3}$             \\ \hline
$10^6$        &$0.1445 \times 10^{-10}$   & $0.3027\times 10^{-6}$&$0.4764\times 10^{-9}$            \\ \hline
\end{tabular}
\vspace*{0.05in}
\caption{Computed values of $\frac{\partial \mathcal{L}}{\partial \bbeta_i}(\mathfrak{c}(d),\mathbf{0})$, $i\in \is{3}$, for a type I critical point with $m = 1$ and isotropy $G^\Delta_{d-2,2}$.}
\mbox{           } \\

\label{table: estimatesIzIm1}
\end{table}
\end{scriptsize}
\vspace*{-0.2in}

The \fps is in powers of $d^{-\frac{1}{2}}$ and decay of the the derivative is like $d^{-1}$.
Initial coefficients of the \fps for the  derivatives:
\begin{eqnarray*}
D^0_2 & = & 0.144967736664468798796892840195838\ldots \\
D^1_2 & = & 0.304176867750552726817930542844234\ldots  \\
D^2_2 & = & 0.477893257041395097510780683943154\ldots 
\end{eqnarray*}

\vspace*{-0.01in}

\begin{scriptsize}
\begin{table}[h]
\hspace*{-0.5in}
\begin{tabular}{|c||c|c|c|c|}

\hline
Type I &&&& \\
$k = d+2$          &  $\partial \mathcal{L}/\partial \bbeta_0$ & $\partial \mathcal{L}/\partial \bbeta_1$ &$\partial \mathcal{L}/\partial \bbeta_2$ &$\partial \mathcal{L}/\partial \bbeta_3$             \\
Isotropy $G^\Delta_{d-2,2}$&&&&  \\  \hline \hline
$10$          &$0.2186\times 10^{-2}$    & $0.3591\times 10^{-2}$& $0.5519\times 10^{-2}$& $0.6705\times 10^{-2}$ \\ \hline
$1000$        &$0.1298 \times 10^{-3}$   & $0.2544\times 10^{-3}$& $0.3176\times 10^{-3}$& $0.1431\times 10^{-3}$             \\ \hline
$10^6$        &$0.1445 \times 10^{-6}$   & $0.2996\times 10^{-6}$& $0.3636\times 10^{-6}$& $0.1346\times 10^{-6}$            \\ \hline
\end{tabular}
\vspace*{0.05in}
\caption{Computed values of $\frac{\partial \mathcal{L}}{\partial \bbeta_i}(\mathfrak{c}(d),\mathbf{0})$, $i\in \is{4}$, for a type I critical point with $m = 2$ and isotropy $G^\Delta_{d-2,2}$.}
\label{table: estimatesIIIm2}
\end{table}
\end{scriptsize}

\vspace*{-0.01in}

The \fps is in powers of $d^{-\frac{1}{2}}$ and derivatives decay  like $d^{-1}$.
Initial coefficients of the \fps for the  derivatives:
\begin{eqnarray*}
D^0_2 & = & 0.144967736664468798796892840195838\ldots \\
D^1_2 & = & 0.301137183047154477921400976012972\ldots  \\
D^2_2 & = & 0.365035379908109906816379948976053e\ldots  \\
D^3_2 & = & 0.134277248901558876075204468076051\ldots  
\end{eqnarray*}

\subsection{Results for isotropy $\Delta (S_{d-2}\times S_2)$, type II, $m = 0,1$}
For this example, as $d \arr \infty$, $\xi_0^0 \arr +1$, $\xi_1^1 \arr -1$ and all off-diagonal entries converge to zero. Only results are given for $m = 0,1$ as at the time of writing, we are unaware of a
type II solution with an \fps which originates from the type II solution with $m = 1$ described below. 

\vspace*{-0.15in}
\begin{scriptsize}
\begin{table}[h]
\hspace*{-0.5in}
\begin{tabular}{|c||c|c|}
\hline
Type II && \\
$d = k$          &  $\partial \mathcal{L}/\partial \bbeta_0$ & $\partial \mathcal{L}/\partial \bbeta_1$              \\
Isotropy $G^\Delta_{d-2,2}$ &&  \\  \hline \hline
$10$          &$0.6278\times 10^{-2}$    & $-0.4145\times 10^{-2}$             \\ \hline
$1000$        &$1.065 \times 10^{-6}$   & $-0.2296\times 10^{-3}$             \\ \hline
$10^6$        &$1.083 \times 10^{-6}$   & $-2.371\times 10^{-3}$            \\ \hline
\end{tabular}
\vspace*{0.05in}
\caption{Computed values of $\frac{\partial \mathcal{L}}{\partial \bbeta_i}(\mathfrak{c}(d),\mathbf{0})$, $i \in \is{2}$, for a type II critical point with $m = 0$ and isotropy $G^\Delta_{d-2,2}$.}
\label{table: estimatesz1}
\end{table}
\end{scriptsize}

\vspace*{-0.2in}

Initial coefficients of the \fps in $d^{-\frac{1}{2}}$ for derivatives:
\begin{eqnarray*}
D_4^0 & = & 1.08317272408498072490694155543623\ldots \\
D_2^1 & = & -0.237257064180446401333830561616386 \ldots
\end{eqnarray*}

\vspace*{-0.01in}

\begin{scriptsize}
\begin{table}[h]
\hspace*{-0.5in}
\begin{tabular}{|c||c|c|c|}

\hline
Type II &&& \\
$k = d+1$          &  $\partial \mathcal{L}/\partial \bbeta_0$ & $\partial \mathcal{L}/\partial \bbeta_1$ &$\partial \mathcal{L}/\partial \bbeta_2$              \\
Isotropy $G^\Delta_{d-2,2}$& &&  \\  \hline \hline
$10$          &$0.4799\times 10^{-2}$    & $-0.3701\times 10^{-2}$ &$-2.594\times 10^{-2}$ \\ \hline
$1000$        &$-1.188 \times 10^{-6}$   & $-0.2042\times 10^{-3}$&$-1.242\times 10^{-5}$             \\ \hline
$10^6$        &$-7.244 \times 10^{-11}$   & $-0.2092\times 10^{-6}$&$-4.060\times 10^{-10}$            \\ \hline
\end{tabular}
\vspace*{0.05in}
\caption{Computed values of $\frac{\partial \mathcal{L}}{\partial \bbeta_i}(\mathfrak{c}(d),\mathbf{0})$, $i\in \is{3}$, for a type II critical point with $m = 1$ and isotropy $G^\Delta_{d-2,2}$.}
\mbox{           } \\

\label{table: estimatesIzIIm2}
\end{table}
\end{scriptsize}
\vspace*{-0.5in}

The \fps is in powers of $d^{-\frac{1}{2}}$ and decay of the the derivative is like $d^{-1}$.
Initial coefficients of the \fps for the  derivatives:
\begin{eqnarray*}
D^0_3 & = & -0.0735719750760874284557949071165949\ldots \\
D^1_2 & = & -0.209265817903196378166307212398253\ldots  \\
D^2_3 & = & 0.406200155553695625169686284999401\ldots 
\end{eqnarray*}

\subsection{Results for isotropy $\Delta (S_{d-p}\times S_p)$, type II, $p > 2$}
Examples of critical point families are given which have \fps representations,  isotropy $G^\Delta_{d-p,p}$, $p > 2$, and are such that diagonal entries converge to $\pm1$. These solutions can all represented in
a two parameter critical point family $(\bxi(p,q))$ with isotropy $G^\Delta_{p,q}$. If $d$ is even, there is a family $(\bxi(p,p))$ with isotropy $G^\Delta_{p,p}$ such that half the diagonal entries converge to $+1$ and half to $-1$. 
Here the isotropy is \emph{not} a proper maximal subgroup of $S_{2p}$---the maximal proper subgroup is $S_2 \wr S_p$.
\begin{scriptsize}
\begin{table}[h]
\hspace*{-0.5in}
\begin{tabular}{|c||c|c|}
\hline
Type II && \\
$d = k$          &  $\partial \mathcal{L}/\partial \bbeta_0$ & $\partial \mathcal{L}/\partial \bbeta_1$              \\
Isotropy $G^\Delta_{d-p,p}$ &&  \\  \hline \hline
$p = 3$ && \\ \hline
$d=10$          &$0.8835\times 10^{-2}$    & $-0.1100\times 10^{-2}$             \\ \hline
$d=10^6$        &$1.624 \times 10^{-12}$   & $-0.2372\times 10^{-6}$            \\ \hline
Initial Coeff.
of \fps  & $D^0_4 = 1.625$ & $D^0_2 = -2.373$ \\ \hline\hline
$p = 4$ && \\ \hline
$d=10$          &$0.8835\times 10^{-2}$    & $-0.1100\times 10^{-2}$             \\ \hline
$d=10^6$        &$1.624 \times 10^{-12}$   & $-0.2371\times 10^{-6}$            \\ \hline
Initial Coeff. of \fps & $D^0_4 = 1.625$ & $D^0_2 = -2.373$ \\ \hline\hline
$d  = 2p$ && \\  \hline 
$p=5$          &$0.5521\times 10^{-2}$    & $0.5329\times 10^{-2}$             \\ \hline
$p=10^6/2$        &$0.2106 \times 10^{-6}$   & $0.1140\times 10^{-6}$            \\ \hline \hline
\end{tabular}
\vspace*{0.05in}
\caption{Computed values of $\frac{\partial \mathcal{L}}{\partial \bbeta_i}(\mathfrak{c}(d),\mathbf{0})$, $i \in \is{2}$, for type II critical points with $p = 3,4$ and  $d-p = p$
If $d-p = p$, the diagonal entries of the upper diagonal block converge to $+1$ as $d \arr \infty$.}
\label{table: estimatesq1}
\end{table}
\end{scriptsize}

\vspace*{0.01in}

If $d-p = p$, the \fps for $\bbeta$ derivatives are in powers of $d^{-\frac{1}{2}}$ and derivatives decay like $d^{-1}$.
Initial coefficients of the \fps for the  derivatives:
\begin{eqnarray*}
D^0_2 & = & 0.2119019935040464\ldots \;\;
D^1_2  =    0.1135488183150048\ldots
\end{eqnarray*}

\mbox{   }

\subsection{Asymmetric targets and fractional power series}
We give two examples where the target is not $I_d$. In the first example the target $\TT\in M(d,d)$  is block diagonal with
the upper diagonal block being $I_{d-6}$ and the remaining block $A \in M(6,6)$ having no symmetries with respect to the action of $S_6^r \times S_6^c$. The off diagonal blocks are matrices of zeroes.
The second example is an asymmetric matrix in $M(10,10)$ which is not fixed by any non-identity element of $S_{10}^r\times S_{10}^c$.

\subsubsection{Path following deformations of the target}
Starting with the target $\VV = I_d$ and a non-degenerate critical point $\bxi$, the basic idea is to path follow $\bxi$ as the target $\VV$ is deformed to the target $\TT\in M(d,d)$.
Specifically, choose the linear path $\TT_\lambda = \lambda \TT + (1-\lambda)\VV$, $\lambda \in [0,1]$, from $\VV$ to $\TT$.  Let $\widetilde{\bxi}(\lambda)$ denote the conjectured continuous path of critical points joining $\wbxi(0)\defoo\bxi$ to the
critical point $\wbxi\defoo\wbxi(1)$ for the target $\TT_1 = \TT$. To this end, break $[0,1]$ into $N$ equal subintervals each of length  $1/N$, $N \in \pint$. 
We are given the critical point $\bxi=\wbxi(0)$ associated to the target $\TT_0 = \VV$. Proceeding inductively, suppose we have computed the critical point $\wbxi(i)$ associated to the target $\TT_{\frac{i}{N}}$, for all $i \le n < N$. 
The next step is to change the target to $\TT_{\frac{n+1}{N}}$, take the approximate solution $\wbxi(n)$ and then check that 
Newton's method gives convergence to a critical point $\wbxi(n+1)$ associated to the target $\TT_{\frac{n+1}{N}}$. The process terminates with $\wbxi(N)$ which will be a critical point associated to the target $\TT_1 = \TT$. 
Suppose, for example, that $\bxi$ a type II critical point. Provided  $d \ge 6$~\cite[\S 8]{ArjevaniField2021x}, $\bxi$ is a non-degenerate spurious minimum. Provided $\|\TT-\VV\|$ is sufficiently small, 
$\wbxi$ will be a non-degenerate spurious minimum since non-degenerate critical points are stable under perturbation of the underlying dynamics. However, since the bifurcation of $\bxi$
from saddle to spurious minimum  occurs near $d = 5.6$, the expectation is that for $d$ near $6$, path following is likely to be difficult in view of the complex 
dynamics close to the bifurcation (see~\cite{ArjevaniField2022a} for more on this). 
However, for type II critical points convergence of eigenvalues as functions of $d$ is fast near the bifurcation~\cite{ArjevaniField2020b} while 
for type I growth is relatively slow (\emph{op.~cit}). As a result,  path following is likely to be harder for type I critical points close to bifurcation 
(near $d = 7$ for type I).   This turns out to be the case (see notes on the first example below).  A consequence is that 
it is sometimes helpful to use piecewise linear paths to avoid points close to  bifurcation and/or Newton fails to converge. 

Even when target $\TT$ and $\wbxi(N)$ have very few or no symmetries, it is still possible to obtain \fps expansions for critical points and this is illustrated in both examples.  

\begin{exam}[Partial asymmetry with a $6\times 6$ asymmetric block]
As asymmetric block, take
the matrix $A\in M(6,6)$ defined by
\begin{scriptsize}
\[
A=
\left[\begin{matrix}
1.03 & 0.04 & -0.02 & 0.07 & 0.0 & -0.02 \\
0.016 & 0.96 & -0.026 & -0.028 & 0.0 & -0.026 \\
0.056 & 0.0276 & 1.04 & 0.0 & 0.076 & 0.0 \\
-0.0397 & 0.0396 & 0.0 & 1.01 & 0.0 & -0.07 \\
0.04 & 0.0 & 0.0 & -0.0277 & 0.97 & 0.01 \\
0.0112 & -0.02 & 0.05 & -0.003 & 0.038 & 1.1 
\end{matrix} \right]
\]
\end{scriptsize}
For $d \ge 8$, extend to the matrix $\TT_d$ defined as the block diagonal matrix
\[
\TT_d = [I_{d-6},A] \in M(d,d)
\]
Take a value of $d \gg 6$---the value of $d$ at which the type II critical point bifurcates to a spurious minimum--for example $d = 100$. Using the path following method described above, compute
the critical point $\wbxi(100)$ associated to the target $\TT_{100}$ and then varying $d$ and using path following compute type II critical points for $\TT_{10}, \TT_{10^n}$, $n \in \{6,7,8\}$. 
As part this process, the 7 derivatives $\partial \mathcal{L}/\partial \bbeta_i$, $i \in \is{7}$ at $\bbeta = \mathbf{0}$ are also computed.
\begin{scriptsize}
\begin{table}[h]
\begin{tabular}{|c||c|c|c|c|c|c|c|}
\hline
Type II, m=0 &&&&&&& \\
          &  $\frac{\partial \mathcal{L}}{\partial \bbeta_0}$ &$\frac{\partial \mathcal{L}}{\partial \bbeta_1}$&$\frac{\partial \mathcal{L}}{\partial \bbeta_2}$&$\frac{\partial \mathcal{L}}{\partial \bbeta_3}$& $\frac{\partial \mathcal{L}}{\partial \bbeta_4}$ &$\frac{\partial \mathcal{L}}{\partial \bbeta_5}$&$\frac{\partial \mathcal{L}}{\partial \bbeta_6}$  \\
Isotropy $\Delta S_{d-6}$& &&&&&&  \\  \hline \hline
$d=10$          &$3.8\text{e-3}$    & $3.4\text{e-3}$ &$3.2\text{e-3}$&$4.0\text{e-3}$&$2.9\text{e-3}\text{e-3}$&$4.4\text{e-3}$&$-8.4\text{e-3}$ \\ \hline
$d=1000$        &$6.1\text{e-7}$   & $-1.6\text{e-6}$&$-8.0\text{e-6}$ &$1.0\text{e-5}$&$-1.4\text{e-5}$&$9.0\text{e-6}$&$-2.8\text{e-4}$        \\ \hline
$d=10^6$        &$6.2\text{e-13}$   & $-2.2\text{e-9}$&$-8.7\text{e-9}$&$9.7\text{e-9}$&$-1.5\text{e-8}$&$8.5\text{e-9}$&$-2.8\text{e-7}$     \\ \hline
\end{tabular}
\vspace*{0.05in}
\caption{Computed values of $\frac{\partial \mathcal{L}}{\partial \bbeta_i}(\mathfrak{c}(d),\mathbf{0})$, $i\in \is{3}$, for a type II critical point with $m = 1$ and isotropy $G^\Delta_{d-2,2}$.}
\mbox{           } \\

\label{table: estimatesIzIIam2}
\end{table}
\end{scriptsize}
Referring to Table~\ref{table: estimatesIzIIam2}, the second column gives the derivative with respect to $\bbeta_0$ and the rapid decay reflects the $\Delta S_{d-6}$ symmetry of the $I_{d-6}$ block. All the decay rates are characteristic of type II.
The results of the table suggest an \fps for the critical point in powers of $d^{-\frac{1}{2}}$. On this basis, estimate numerically lower order coefficients in the \fps and then
compute the associated critical point for the target with $d=10^{512}$.
We have the following estimate for the initial non-zero coefficient of the 
seven derivatives. This was computed taking $d = 10^{512}$ and the terms are accurate to the 16 significant figures displayed and consistent with the values shown in Table~~\ref{table: estimatesIzIIam2},
\[
\begin{matrix}
D^0_4& =  &    0.6211900643271606 & &  \\
D^1_2& =  & -0.002193417454242254  & D^2_2& = -0.008698020018753445\\
D^3_2&=   & -0.008698020018753445  & D^4_2& =  0.009751822809799718\\
D^5_2& = &   -0.014474815696883097 & D^6_2& =     -0.2878453022368189
\end{matrix}
\]
The loss decays like $d^{-1}$ with initial coefficient $0.36076059...$. 

\begin{rems}
(1)
Although we have not given numerical approximations for the critical points, it is worth remarking that as $d \arr \infty$, the constant terms in the \fps converge to
$[I_{d-6},\TT]$ \emph{except that the last row of components converges to the negative of the last row of $\TT$}. That is, if we set $\wbxi(d) = \wbxi$, then $\widetilde{\bxi^{6}_j} \arr -t_{6j}$, as $d\arr \infty$, $j = 0,\ldots,6$.
Similar statements are valid for asymmetric critical points resulting from deformation of critical points of type I.   \\
(2) There is no difficulty in constructing a path from $d=10$ to $d=9$. However, $d = 8$ provides a challenge and even with a piecewise linear path with several linear segments, we were unable to reach $d = 8$.
\rend
\end{rems}
\end{exam}

\begin{exam}[Full asymmetry for $d=k=10$]
Take as target the matrix $\TT_{10} = \TT \in M(10,10)$ defined by 
\begin{scriptsize}
\[
\hspace*{-0.3in}\left[\begin{matrix} 
1.09 & 0.1 & -0.04& 0.02& 0.07& -0.1& 0.04& -0.08& 0.1& 0.01 \\
0.05& 1.07& 0.04& 0.01& 0.111& -0.00876& 0.02& 0.09& -0.03& 0.111 \\
0.05& 0.03& 1.12& 0.04& -0.043& 0.04& 3.25\text{e-03}& -0.1& 0.008& -0.024 \\
0.04& 0.008& 0.1& 0.97& 0.005& 0.01& 0.032& -0.004& 0.00756& 0.85\\
0.02& 5.078\text{e-05} & -0.03& 0.06& 1.1117& -0.06& 0.002& 0.01& 5.78\text{e-05} & 0.02\\
0.006& 0.09& 0.13& 0.002& 5.6\text{e-06} & 1.03609& 0.07& 0.005& -0.13& 7.77\text{e-04}  \\
-0.05& 0.0909& 0.02& -0.007& 0.088& 0.012& 1.2& 0.0091& -0.07& 0.0009  \\
0.0& -0.002& 0.02& 0.041& -0.140& 0.07& 0.0& 1.3& 0.003& 0.01 \\
0.0008& -0.1& 0.0006& -0.02& 0.002& 0.007& 0.04& 0.003& 1.14& 0.004  \\
0.0039& 0.1& 0.007& 0.002& 0.0007& 0.05& -0.03& 0.005& 0.0& 0.95 
\end{matrix} \right]
\]
\end{scriptsize}
This matrix has no permutation symmetries. It is also not a small perturbation of $I_{10}$---note in particular the matrix entry $t_{3,9}$ (indexing from $0$ to $9$).

Starting with the target $\VV = I_{10}$ and the spurious minimum the critical point of type II ($k = d = 10$ and isotropy $\Delta S_9$), path follow $\bxi$ as the target $\VV$ is deformed to the target $\TT$ as described previously. 
Using 40 iterations per step, Newton's method converges at every step and we obtain a path connecting $\bxi$ to a critical point $\wbxi(10) =\wbxi$ associated to the target $\TT$.
The path following procedure was implemented taking $N = 100$ and $N= 500$ and a log kept of the convergence at the beginning of each step as well the error at the end of the step measured by the
size of the components of $\grad{\mathcal{L}}(\wbxi(i))$.  Computations were done with $2024$-bit precision and
the errors at the end of each step were always $10^{-616}$ or better. There were no indications of bifurcation or change in stability along the path and so it is likely that $\wbxi$ is a non-degenerate spurious minimum.

Assuming zero bias, the 10 derivatives $\partial \mathcal{L}/\partial \bbeta_i$ are given to 16 significant figures by 
\[
\begin{matrix}
D^0 = 1.603536471231742\text{e-03} & D^1 = 2.714453782566770\text{e-03}\\
D^2 = 1.352298945591707\text{e-03} & D^3 = 8.185795884148845\text{e-03}\\
D^4=  1.603140416072457\text{e-03} & D^5 = 1.956506389478921\text{e-03} \\
D^6 = 1.484721381112565\text{e-03} & D^7 = 1.950916845556952\text{e-03} \\
D^8=  1.810232124741424\text{e-03} & D^9 =-1.610119156961948\text{e-03}
\end{matrix}
\]
\begin{rem}
Even though the target is not close to being symmetric, the derivatives with $d = 10$ are all smaller than $10^{-2}$, as in previous symmetric examples. Note however $D^3$, which is the largest derivative. 
The larger
value of $D^3$ is likely a consequence of the big symmetry breaking perturbation of $0.85$ seen in the final entry of the fourth row of $\TT$. \rend
\end{rem}

Although $\TT_{10}$ and $\wbxi(10)$ have no permutation symmetries (that is, trivial isotropy under the action of $S^r_{10} \times S^c_{10}$), it is possible to regard $\wbxi(10)$ as analytically connected to a family 
$\wbxi(d)$ defined for $d \ge 12$ and such that $\wbxi(d)$ has isotropy $\Delta S_{d-10}$. We explain the basic idea. Starting with the matrix $\TT_{10}$ we define $\TT_{10}^\star$ to be the 
$10 \times 10$ matrix with diagonal blocks $I_2$ and $\TT_8$, where $\TT_8$ is the $8 \times 8$-matrix defined by the last 8 rows and columns of $\TT_{10}$.  The two off diagonal blocks of $\TT_{10}^\star$ are zero matrices.
We take the linear path from $\TT_{10}$ to $\TT_{10}^\star$ and let $\wbxi^\star(10)$ denote the corresponding critical point obtained by path following. The critical point $\wbxi^\star$ has isotropy $\Delta S_2$.
The fixed point space for the action of  $\Delta S_2$ on $M(10,10)$ is of dimension $82$. On account of the $S_2$ symmetry, we can path follow
$\wbxi^\star(d)$ as $d$ is increased from $10$ to $12$. The isotropy of $\wbxi^\star(12)$ will be $\Delta S_4$. Regard the $10 \times 10$ lower diagonal block 
of $\TT_{12}^\star$ as asymmetric keeping the $\Delta S_2$ symmetric block $I_2$ from the
first two rows and columns. Now deform  $\TT_{12}^\star$ to $\TT_{12} = [I_2,\TT_{10}]$. In practice, it is easier to take $d=12$ and path deform the type II critical point $\bxi(12)$ with isotropy
$\Delta S_{11}$ to the critical point in $M(12,12)$ associated to the asymmetric target $\TT_{12}$ and then further path deform this critical point to a critical point in $M(12,12)$ with isotropy
$\Delta S_4$, target $\TT_{12}^\star$. When computed, this critical point is $\wbxi^\star(12)$ and 
so the required path connection from $\wbxi(10)$ to $\wbxi(12)$ is the composition of the path from $\wbxi(10)$ to $\wbxi^\star(12)$ 
and the reversed path from $\wbxi^\star(12)$ to $\wbxi(12)$.

While there is no \fps for the asymmetric example ($d$ is fixed),  an \fps can be constructed
if one extends to $d> 12$ as described above. 
Taking $d = 10^{128}$, we find that
the initial non-zero terms of the \fps for the derivatives are either multiples of 
$d^{-1}$ or $d^{-\frac{1}{2}}$ and, assuming\footnote{The result still holds even if the \fps is in
powers of $d^{-\frac{1}{2^n}}$, $n > 1$} the \fps are power series in $d^{-\frac{1}{2}}$,
the initial non-zero coefficients are:  
\[
\begin{matrix}
D^0_4&=&   0.2805800311544216&&&\\ 
D^1_2&=& 0.001873922117631949 & D^2_2 & =&   0.03788347013872621 \\
D^3_2&=& -0.003893028296665770 & D^4_2 & =&    0.1570405665520425\\
D^5_2&=&  0.001836908105484236&  D^6_2 & =&   0.01117268780308763\\
D^7_2&=& -0.005708669773701929& D^8_2&=&   0.002806305904052021\\
D^9_2&=&  -0.0004776786377640696&  D^{10}_2&= & -0.2088224377148879
\end{matrix}
\]
\end{exam}
\section{Families of spurious minima in biased ReLU networks}\label{sec: 15}  
In this section, the question of whether or not families of critical points and spurious minima that are seen in 
unbiased ReLU networks with $d$-inputs and target matrix $\VV = I_d$~\cite{ArjevaniField2021x,ArjevaniField2021,ArjevaniField2022b}
extend naturally to biased ReLU networks. This problem is examined in the simplest case of ReLU networks for which the number of neurons equals the number of inputs ($k = d$).
In particular, no investigation is made of the over-specified case ($k > d$) nor of the spectrum of the Hessian at critical points. 
This is partly for reasons of length but primarily to 
keep the focus on the naturality of the extension (and what that means).  The arguments given are based on high precision numerics and implicitly
assume the \fps exist. 
However, for the families presented below it is
likely that more formal arguments, based on the implicit function, are relatively simple and similar to those given for 
unbiased networks~\cite{ArjevaniField2021x}, \cite[Appendix A, remark 4]{ArjevaniField2022b}.

\emph{It is assumed throughout that the target matrix has zero bias}---that is, $\bgamma_j = 0$, $j \in \is{d}$. 
Including examples where the target has bias would distract from the main question on the naturality of the extension. 

\subsection{Isotropy $\Delta S_d$, $d = k$: Type I in  a network with bias}\label{sec: sec12}
Two basic definitions inspired by high precision numerical results,  are introduced in this subsection 
and play an important role in the remaining subsections.  Indeed, 
one definition motivates the formulation of an \emph{ansatz} that leads to a formal proof for the 
existence of \fps for all the families of critical points in biased networks described below. 

We start with the simplest case: the family of spurious minima of type I (or A) described in~\cite{ArjevaniField2021x}.  
For this family there is exactly one rowtype $0$ and $q_0 = d$.  

Critical points with isotropy $\Delta S_d$ in the type I family for the unbiased network (resp.~biased network) are denoted by $\bxi(d)=\bxi\in M(d,d)^{\Delta S_d}$ 
(resp.~$\wbxi(d) = \wbxi \in M(d,d+1)^{\Delta S_d}$)---the previously introduced notation
$\mathfrak{c}(d)$ (resp.~$\mathfrak{c}^\star(d)$) is used when membership of $M(d,d)$ (resp.~$M(d,d+1)$) is emphasized rather 
than coordinatization of the fixed points space. 

Denote coordinates for $\bxi$ (resp.~$\wbxi$) in the associated fixed point space by
$(\xi_0,\xi_{\star 0})$ (resp.~$((\xi_0,\xi_{\star 0}),\bbeta)$.  Write $\bxi(d)$, $\wbxi(d)$ if the dependence on $d$ needs to be emphasized. For either of these families, 
the index identifying rowtype is usually omitted since the isotropy is $\Delta S_d$ and there is only one rowtype. For example, write 
$\|\ww^i\| = \tau_0 =\tau$, $i \in \is{d}$. The bias 
for the network with parameter matrix $\WW$ 
is denoted by $\bbeta$ and angles generally have no subscript \emph{except} that we keep `$00$' terms such as $\lambda_{00}$ (there is no $\lambda_{01}$ term since we assume $\Delta S_d$ symmetry).

Let $\vv^0,\ldots,\vv^{d-1}$ define the standard orthonormal basis of $(\real^d)^\star$ (identified with the dual space $\real^d$ using the identification provided by the Euclidean inner product). 
Define the three angles
\begin{eqnarray*}
\Theta&=&\cos^{-1}\big( \langle \ww^i,\ww^j\rangle/\tau^2\big), \; i,j\in \is{d}, i \ne j \\
\theta&=&\cos^{-1} \big(\langle \ww^i,\vv^j\rangle/\tau\big), \; i,j\in \is{d}, i \ne j \\
\lambda_{00}&=&\cos^{-1} \big(\langle \ww^i,\vv^i\rangle/\tau\big),\; i \in \is{d}
\end{eqnarray*}
Following the conventions of section~\ref{subsec: rtype}, we have  
\begin{eqnarray*}
Z =\wZZ & = & \bbeta (-\cot(\Theta)+\csc(\Theta)), \quad
z  =  \bbeta\csc(\theta),\quad \wzz = -\bbeta \cot(\theta)\\
K & = & \bbeta^2\csc(\Theta)(\csc(\Theta) - \cot(\Theta))\quad
k  =  \frac{1}{2} \big( \bbeta^2 \csc^2(\theta) \big) \\
L & = & -\pi\erfb{\frac{\bbeta}{\sqrt{2}}},\quad \ell =-\frac{\pi}{2}\erfb{\frac{\bbeta}{\sqrt{2}}} \\
E^X & = & -\sqrt{2\pi}\bbeta e^{-\frac{\bbeta^2}{2}}\erfcb{\frac{Z}{\sqrt{2}}} \quad
e^X  =  -\sqrt{\frac{\pi}{2}}\bbeta \erfcb{\frac{z}{\sqrt{2}}}\\
\wTheta & = & \frac{\pi}{2} + S(\bbeta,\bbeta,\Theta), \quad
\wtheta  =  \frac{\pi}{2} + S(\bbeta,0,\theta),\quad
\wlambda_{00}  = \frac{\pi}{2} + S(\bbeta,0,\lambda_{00}) \\
z_{00} & = & \bbeta \csc(\lambda_{00}),\quad \wzz_{00}  =  -\bbeta \cot(\lambda_{00}) \\
k_{00} & = & \frac{1}{2} \bbeta^2\csc^2(\lambda_{00}) \\
e^X_{00} & = & -\sqrt{\frac{\pi}{2}}\bbeta \erfcb{\frac{z_{00}}{\sqrt{2}}}\quad\
\ell_{00}  =  -\frac{\pi}{2}\erfb{\frac{\bbeta}{\sqrt{2}}}\\
\end{eqnarray*}
Applying Proposition~\ref{prop: sym} and the computations of section~\ref{sec: AB} gives  
\begin{eqnarray*}
\Gamma & = & \frac{1}{2}\left[\bbeta^2 - (1+\bbeta^2)\erfb{\frac{\bbeta}{\sqrt{2}}}\right]- \frac{1}{\sqrt{2\pi}}\bbeta e^{-\frac{\bbeta^2}{2}} \\ 
&&+\frac{1}{2\pi}\left[\frac{(d-1)\big(\tau e^{-K}\sin(\Theta) - e^{-k}\sin(\theta)\big) - e^{-k_{00}}\sin(\lambda_{00})}{\tau}\right]\\
&& + \frac{1}{2\pi}\left[\frac{(d-1)\big(\tau E^X - e^X\big)-e^X_{00}}{\tau}\right]
 +\frac{\bbeta^2}{2\pi}(d-1)\big[L  +(\pi-\wTheta)\big]   \\
A_0 & = & \frac{1}{2\pi}\big((d-1)\wTheta\xi_{\star 0}  - \wlambda_{00}\big) \quad
A_{\star 0}  =  \frac{1}{2\pi}\big(\wTheta \xi_0 + (d-2)\xi_{\star 0} \wTheta - \wtheta\big)\\
B_0 & = & \frac{1}{2\pi}\big((d-1)L\xi_{\star 0}  - \ell_{00}\big)\quad
B_{\star 0}  =  \frac{1}{2\pi}\big(L(\xi_0  + (d-2)\xi_{\star 0}) - \ell_{00}\big)\\
\Omega_0 & = & \frac{1}{2}\big( \xi_0 + (d-1)\xi_{\star 0} -1\big)
\end{eqnarray*}

Also needed is the derivative of the loss $\mathcal{L}$ with respect to $\boldsymbol{\bbeta}$.  
\begin{prop}
(Notation and assumptions as above; in particular, $\bgamma = 0$.)
\begin{multline*}
\begin{aligned}
\hspace*{0.2in}\frac{1}{d}\frac{\partial \mathcal{L}}{\partial \bbeta}&= \frac{\tau^2}{2}\left[\bbeta\, \erfcb{\frac{\bbeta}{\sqrt{2}}} - \sqrt{\frac{2}{\pi}}e^{-\frac{\bbeta^2}{2}}\right] +  
(d-1)\frac{\tau^2\bbeta}{2\pi}\left[\pi\,\erfcb{\frac{\bbeta}{\sqrt{2}}} - \wTheta \right]\\
&\hspace*{-0.6in} +\frac{\tau}{2\sqrt{2\pi}}\left[\erfcb{\frac{\bbeta\csc(\lambda_{00})}{\sqrt{2}}}+\cos(\lambda_{00})e^{-\frac{\bbeta^2}{2}}\erfcb{\frac{-\bbeta\cot(\lambda_{00})}{\sqrt{2}}}\right]\\
&\hspace*{-0.6in} + (d-1)\frac{\tau}{2\sqrt{2\pi}}\left[\erfcb{\frac{\bbeta\csc(\theta)}{\sqrt{2}}}+\cos(\theta)e^{-\frac{\bbeta^2}{2}}\erfcb{\frac{-\bbeta\cot(\theta)}{\sqrt{2}}}\right] \\
&\hspace*{-0.6in} - (d-1)\frac{\tau^2}{2\sqrt{2\pi}}\left[e^{-\frac{\bbeta^2}{2}}\erfcb{\frac{\bbeta(-\cot(\Theta) +\csc(\Theta))}{\sqrt{2}}}\right]\big(1 + \cos(\Theta)\big)
\end{aligned}
\end{multline*}
\end{prop}
\begin{proof} 
This follows from Proposition~\ref{prop: betaderv}. For a simple direct proof, it follows from $\Delta S_d$ symmetry that
\[
\mathcal{L} = d\frac{f(\ww^0,\ww^0)}{2} + d(d-1) \frac{f(\ww^0,\ww^1)}{2} - df(\ww^0,\vv^0) - d(d-1) f(\ww^0,\vv^1) + K
\]
where $\ww^0, \ww^1$ are the first and second rows of $\WW$---any pair of distinct rows suffice---and the term $K$ depends on $\VV$ and is independent of $\WW$ and $\bbeta$. Hence
\[
\frac{1}{d}  \frac{\partial \mathcal{L}}{\partial \bbeta} =  \frac{\partial}{\partial \bbeta}\left[\frac{f(\ww^0,\ww^0)}{2} - f(\ww^0,\vv^0) + (d-1)\frac{f(\ww^0,\ww^1)}{2}-(d-1)f(\ww^0,\vv^1) \big)\right]
\]
Using Lemma~\ref{EQ:betaderv}, Proposition~\ref{prop: lossbeta}(2), and Example~\ref{EX: fwwderivex} \Refb{EQ: iibeta_deriv}, 
the partial derivative in $\bbeta$ is seen to be that of the statement of the proposition.
\end{proof}
In future set 
\[
\mathcal{L}^{0} = \frac{f(\ww^0,\ww^0)}{2} - f(\ww^0,\vv^0) + (d-1)\left[ \frac{f(\ww^0,\ww^1)}{2}-f(\ww^0,\vv^1)\right]
\]
and regard the $\mathcal{L}$ superscript as an index by rowtype (see the remark following for the generalization to families with isotropy $\Delta (S_{q0} \times S_{q1})$). 

The equations for critical points with isotropy $\Delta S_d$, target $\VV = I_d$ and $k = d$, may be written 
\begin{eqnarray*}
\Gamma \xi_0 - A_0 + B_0 + \Omega_0& =& 0 \\
\Gamma \xi_{\star 0} - A_{\star 0} + B_{\star 0} + \Omega_0& =& 0 \\
\frac{\partial \mathcal{L}^{0}}{\partial \bbeta_0}&=&0 
\end{eqnarray*}
\begin{rem}[Needed only in subsequent sections]
The definition of $\mathcal{L}^{0}$ extends to problems with symmetry group $G^{\Delta}_{\is{q}}$ acting on $M(k,d)$, where 
$q_0\ge q_1 \ge \cdots q_{f-1} \ge 1$ and $k \ge d$ (see section~\ref{sec: fixsp}). Here there will be $R = f + m$ rowtypes, where $m = k-d$. Let
$\RS \subset \is{f}$ denote the set of $q_i > 1$ so that $q_0,\ldots,q_{R_s - 1} \ge 2$ and the remaining $q_i$ are all equal to one.
Recall the conventions on superscripts: $\ww^i$ will denote the \emph{first} row of $\WW\in M(k,d)$ which is of rowtype $i$ and $\widetilde{\ww}^i$ will denote any row of
$\WW$ which is of rowtype $i$ but is not $\ww^i$ (that is, not the first row of rowtype $i$). Similar conventions hold for rows of $\VV$. 

For $i \in \RS$, define 
\begin{multline*}
\begin{aligned}
\mathcal{L}^{i}=& (q_i-1)\left[\frac{f(\ww^i,\widetilde{\ww}^i)}{2} - f(\ww^i,\widetilde{\vv}^i)\right]  + \left[\frac{f(\ww^i,\ww^i)}{2} - f(\ww^i,\vv^i)\right]  \\
& + \sum_{j \ne i} q_j\left[f(\ww^i,\ww^j)-f(\ww^i,\vv^j)\right]
\end{aligned}
\end{multline*}
For $i \in \is{f}\smallsetminus\RS$, define 
\begin{eqnarray*}
\mathcal{L}^{i}&=& \left[\frac{f(\ww^i,\ww^i)}{2} - f(\ww^i,\vv^i)\right]  
 + \sum_{j \ne i} q_j\left[f(\ww^i,\ww^j)-f(\ww^i,\vv^j)\right]
\end{eqnarray*}
If $i \in \is{R}\smallsetminus \is{f}$, define
\begin{eqnarray*}
\mathcal{L}^{i}&=& \frac{f(\ww^i,\ww^i)}{2} + \sum_{j \ne i} q_j\left[f(\ww^i,\ww^j)-f(\ww^i,\vv^j)\right]
\end{eqnarray*}
The critical point equations associated to the $\bbeta$-derivatives of $\mathcal{L}$ are then 
\begin{eqnarray*}
\frac{\partial \mathcal{L}^{i}}{\partial \bbeta_i}&=&0,\; i \in \is{R}
\end{eqnarray*}
In each case,  $\frac{\partial \mathcal{L}^{i}}{\partial \bbeta_i} = \frac{1}{q_i}\frac{\partial \mathcal{L}}{\partial \bbeta_i}$. \rend
\end{rem}

\subsection{Numerics} As shown in Section~\ref{sec: 11.2}, $\frac{\partial \mathcal{L}}{\partial \bbeta} \ne 0$ 
at a type I critical point of the bias free network.  

Starting with the biased network, with bias set to zero, and initial point $(\bxi(d),0)\in M(d,d+1)^{\Delta S_d}\approx \real^3$, 
Newton's method is applied to the critical point equations given above. 
To this end, \emph{MPFR} code is used, always with at least $2048$-bit precision (higher for large $d$). Newton's method converges rapidly for all $d \ge 4$ 
to a critical point $\wbxi(d)\in M(d,d+1)^{\Delta S_d}\approx \real^3$ and is not a point defining the global minimum. 
The smallest value of $d$ tested was $4$, the largest $10^{512}$.  
\begin{rem}
Although we do not analyze the spectrum of the Hessian at $\wbxi(d)$ here (see below), it is likely 
that the critical point defines a spurious minimum for large enough $d$ (most likely $d \ge 8$). 
Indeed, for values of $8 \le d < 30$, small random perturbation of the critical point in
$M(d,d+1)$ (not $M(d,d+1)^{\Delta S_d}$!) always increase the loss. This is a necessary condition for a non-degenerate local minimum. 
A more formal analysis of the Hessian spectrum can be made using the representation theoretic methods 
of~\cite{ArjevaniField2021,ArjevaniField2022b}: adding the single bias adds exactly one trivial
representation and one standard representation to the isotypic decomposition of the $\Delta S_d$-action
on $M(d,d)$ (bias free parameter space). The isotypic decomposition of $\Delta S_d$ on $M(d,d+1)$, 
has 3 copies of the trivial representation (the dimension of the fixed point space), 
4 copies of the standard representation and one copy each of the exterior square 
representation and the representation associated to the partition $(d-2,2)$.  
The spectrum associated to the standard representations and the exterior square representation are of greatest interest~\cite{ArjevaniField2022a,ArjevaniField2022b}.
Since the trivial factor is given by the fixed point space, it is unlikely (see below) that there is any change in the spectrum associated 
to this factor. In any case, explicit computation of the 3 eigenvalues associated to the trivial representation is easily done using Newton's method.
An example is given below for the type I family with $\Delta S_d$ symmetry. 
\rend
\end{rem}

\subsubsection{The loss}
In the unbiased network, the asymptotics of the loss~\cite[\S 8.6]{ArjevaniField2021x} are given by
\begin{eqnarray*}
\mathcal{L}(\bxi(d))& = &  \left(\frac{1}{2} - \frac{1}{\pi}\right) + O(d^{-{\frac{1}{2}}}) \\
			& \arr & \left(\frac{1}{2} - \frac{1}{\pi}\right) \approx 0.181690113816\ldots,\;\text{as } d\arr\infty 
\end{eqnarray*}
Letting $d \arr \infty$ in the biased network, numerical investigation up to $d = 10^{128}$ strongly suggests that 
\begin{eqnarray*}
\mathcal{L}(\wbxi(d))&\arr& \left(\frac{1}{2} - \frac{1}{\pi}\right) + O(d^{-{\frac{1}{2}}}) 
\end{eqnarray*}

Moreover, even though the behaviour as $d \arr \infty$ appears the same in the biased and unbiased network, the loss at $\wbxi(d)$ is \emph{always strictly less} than 
the loss at $\bxi(d)$---see the tables below and note that if this were not always true, then either the computations would be in error or any connection between $\bxi(d)$ and 
$\wbxi(d)$ would likely involve complex landscape geometry.  

Define 
\[
\Delta_{ub}(d) = \mathcal{L}(\bxi(d))-\mathcal{L}(\wbxi(d)).
\]

\begin{scriptsize}
\begin{table}[h]
\begin{tabular}{|c||c|c|c|}
\hline
Type I & Bias free network& Biased network& Loss decrease  \\
$d = k$          &  &   &        with bias    \\
Isotropy $\Delta S_{d}$ & $\mathcal{L}(\bxi(d)) $& $\mathcal{L}(\wbxi(d))$ & $\Delta_{ub}(d)$ \\  \hline \hline
$4$          &$3.7653\times 10^{-2}$   & $3.6323\times 10^{-2}$ & $1.3301 \times 10^{-2}$ \\ \hline
$10$        &$7.6637 \times 10^{-2}$   & $7.4862\times 10^{-2}$  & $1.7751\times 10^{-3}$          \\ \hline
$100$        &$1.4207 \times 10^{-1}$   & $1.4171\times 10^{-1}$ & $3.6221\times 10^{-4}$          \\ \hline
$10^3$       &$1.6853 \times 10^{-1}$  & $1.6849\times 10^{-1} $ & $4.2197\times 10^{-5}$            \\ \hline 
$10^4$        &$1.77472 \times 10^{-1}$   & $1.77467\times 10^{-1}$ & $4.41056\times 10^{6}$           \\ \hline
$10^5$        &$1.80351 \times 10^{-1}$   & $1.80350\times 10^{-1}$ & $4.4712\times 10^{-7}$           \\ \hline
$10^6$        &$1.8126595 \times 10^{-1}$   & $1.8126591\times 10^{-1}$ & $4.4904\times 10^{-8}$          \\ \hline
$10^7$        &$1.81555928 \times 10^{-1}$   & $1.81555923\times 10^{-1}$ & $4.4965\times 10^{-9}$          \\ \hline
$10^8$        &$1.816476750 \times 10^{-1}$   & $1.816476746\times 10^{-1}$  & $4.4984 \times 10^{-10}$       \\ \hline
$10^9$        &$1.8167669295\times 10^{-1}$ &$1.8167669290\times 10^{-1}$ &$4.4990\times 10^{-11}$ \\ \hline
\end{tabular}
\vspace*{0.05in}
\caption{Computed values of the loss for type I critical point of isotropy  $\Delta S_{d}$.
Loss is given to 5 or more significant figures and shows how the addition of bias to the student network always reduces the loss.
The third column gives the decrease of the loss when the network is biased.  This is always positive.
Included are the values obtained when $d = 4$.  However,
$\bxi(d)$ is not a spurious minima for $d \le 7$; the same is likely true for $\wbxi(4)$.}
\end{table}
\end{scriptsize}

The data presented above strongly suggests that the critical point families described not only have \fps representations but 
that lowest order terms are the same for unbiased and biased networks. This is perhaps the strongest argument that $\wbxi(d)$ is a natural extension of the
family $\bxi(d)$. There is more. 
\begin{exam}
Taking $d = 10^{128}$, a high precision computation gives
\[
\Delta_{ub}(10^{128}) = 0.044992875434563\ldots \times 10^{-128} > 0
\]
indicating that the first \emph{two} terms in the \fps expansions for $\mathcal{L}(\bxi(d))$ and $\mathcal{L}(\wbxi(d))$ are equal. In particular,
the coefficients of $d^{-\frac{1}{2}}$ are \emph{equal and non-zero} (easily verified analytically~\cite{ArjevaniField2021x} for the unbiased network). That is,
\[
\Delta_{ub}(d) = O(d^{-1})
\]
and not $O(d^{-\frac{1}{2}})$ as one might have expected.
\examend
\end{exam}
Granted this example, it is natural to ask about the coefficients of initial terms in the \fps of coordinates of $\bxi$ and $\wbxi$.
All the examples considered here have \fps in $d^{-\frac{1}{2}}$ (this may not be so if there is over-specification~\cite{ArjevaniField2022b}). 

\begin{Def}
Let $\sum_{j\ge p_i\ge 0} \alpha_j^i d^{-\frac{j}{2}}$, $i \in \is{2}$, be a pair of fractional power series (scalar valued) with
$\alpha^i_{p_i} \ne 0$, $i\in \is{2}$. 
The series are said to be \emph{$n$-matched} if
\begin{enumerate}
\item $p_0 = p_1\defoo p$.
\item $\alpha^0_j = \alpha^1_j$, $j = p,\ldots,p+n-1$.
\end{enumerate}
Two vector valued \fps are $n$-matched if they have the same range and  all pairs of corresponding coordinates are $n$ matched (the values $p(i)$ may and do vary with the coordinate chosen).
It is implicit that $n$ is the largest integer for which the statement is true. 

More generally, given two  $\real^Q$ valued  \fps such that the $i$th coordinates are $n_i$ matched with $n_i \ge 0$, for all $i \in \is{Q}$,
the \fps are \emph{$\is{n}$-matched} with matching vector $\is{n} = (n_0,\ldots,n_{Q-1})$.
\end{Def}
\begin{exams}
(1) The \fps for $\mathcal{L}(\bxi)$ and $\mathcal{L}(\wbxi)$ are 2-matched but not 3-matched .\\
(2) Making use of computations of $\bxi$ and $\wbxi$ when $d = 10^{128}$ and ignoring the $\bbeta$ -coordinate, we have
\begin{eqnarray*} 
\xi_0 -  \widetilde{\xi_0}& =& c_0 10^{-129}\\
\xi_{\star 0} - \widetilde{\xi}_{\star 0}&  = & c_{\star 0} 10^{-256}
\end{eqnarray*}
where $c_0 \approx 7.78$ and $c_{\star 0} \approx 1.27$.
Hence the \fps are $2$-matched (or $\is{2} = (2,2)$-matched)\footnote{We have ignored $\bbeta(d)$ but since $\bbeta(d) = c d^{-1}$, $c \approx -6.21$, this term is ``$2$-matched" to $\bxi$, with $p_0 = p_1 = 0$, if we define the $\bbeta$-component to be identically zero.}.  It follows from section~\ref{sec: typeAex} that the coefficient of $d^{-\frac{1}{2}}$ in the \fps for $\widetilde{\xi}_0$ and the 
coefficient of $d^{-\frac{3}{2}}$ in the \fps of $\widetilde{\xi}_{\star 0}$ are both zero.

The existence of \fps for this unbiased network is proved using elementary implicit function arguments~\cite{ArjevaniField2021x}. The numerics for $\wbxi$ strongly suggest the ansatz that
the pair $\bxi,\wbxi$ is $2$-matched. If so, a significant number of the initial coefficients of the \fps for $\wbxi$ are known exactly without having to do any computations 
on the coefficients of the biased network (see results in section~\ref{sec: typeAex}). As a result, 
the expectation is that the implicit function theorem method can be extended 
to allow for dependence on a parameter---in this case $\bbeta$.  Along the line joining $0$ to $\bbeta(d)$, the corresponding \fps should always be 2-matched to the \fps for $\bxi$.  \examend
\end{exams}

\begin{rems}
(1) Based on high precision numerics, the initial terms of the \fps expansion for $\bbeta(d)$  is
\[
\bbeta(d) = -0.62072950119515416243932018851173\ldots d^{-1} + O(d^{-\frac{3}{2}})
\]
where the first 32 significant figures of the first non-zero coefficient are shown and the coefficient of $d^{-\frac{3}{2}}$ is strictly negative.\\
(2) For both biased and unbiased networks, two of the eigenvalues are of the form $c_0 + c_1 d + O(d^{-\frac{1}{2}})$ where in one case $c_1 = 1/4$ for both biased and unbiased networks and in the other case
$ c_1 > 0$ but the value is different from that in the unbiased network. \\
(3) If $d = 4$, the three eigenvalues are $0.8811246\ldots, 0.0242543\ldots$ and $0.0661600\ldots$. All are strictly positive (the case for all $d  \ge 4$).
\rend
\end{rems}

Not only are these strong results unexpected (to the author at least) but they suggest that explicit computation of the initial terms of the \fps along the lines described in \cite[\S 8]{ArjevaniField2021x} 
may not be as challenging as it seems at first sight. In particular, since the low order terms are the same, the implicit function theorem based method described in \cite[\S 8]{ArjevaniField2021x} should extend if we regard
$\bbeta$ as a parameter. Alternatively, using estimates from sections 4 and 5, there is the possibility of showing that a general proof of existence of an \fps in the biased case can be done within a cuspidal 
domain on which the loss is subanalytic and so existence can be deduced from the Curve Selection Lemma~\cite{BierMil1998,Milnor1968}.

\subsection{Path based analysis}
One criticism of the arguments presented above might be that we start with a critical point for the bias free network and then apply 
Newton to the biased network. It is possible that on the first iteration of Newton there is convergence to a nearby critical point 
which is a saddle point for the loss.  Ideally one would like a gradient descent path from $\bxi(d)$ to $\wbxi(d)$ along which the loss is strictly monotone decreasing.
However, once $\wbxi(d)$ has been computed, the variation in the loss along a straight line path from $\wbxi(d)$ to $\bxi(d)$  can be analyzed. 
View the equations
\begin{eqnarray} \label{A1X}
\Gamma \xi_0 - A_0 + B_0 + \Omega_0& =& 0 \\
\label{A2X}
\Gamma \xi_{\star 0} - A_{\star 0} + B_{\star 0}  + \Omega_0& =& 0 
\end{eqnarray}
as equations for an `unbiased network' depending on the parameter $\widetilde{\beta}\in [0,\bbeta]$. Now $\bxi(d) = (\xi_0,\xi_{\star 0},0)$,
$\wbxi(d) = (\widetilde{\xi}_0, \widetilde{\xi}_{0,\star}, \widetilde{\beta})$ and $\bxi(d)$ (resp.~$(\widetilde{\xi}_0, \widetilde{\xi}_{0,\star})$) are
solutions of these equations when the parameter is zero (resp.~$\widetilde{\beta}$). Test for the path by dividing $[0, \bbeta]$ into
$N$ equal increments. Set $\delta = \bbeta/N$ and  use Newton to move from $\bbeta^\star$ back to zero in $N$ steps of length $\delta$.
Examine the output at each step, especially for the initial iterations of Newton, to detect any anomalies that are suggestive of a bifurcation, 
change of stability or `jump' in the critical point.  We have done this for several values of $d$. 
In all cases convergence was to the original critical point in the bias free network. If $\wbxi$ is a non-degenerate spurious minimum, then the loss should be 
strictly monotone increasing near $\wbxi$ (it is) but possibly decreasing near $\bxi$ (for this example, it is not). 
Of course, there is no reason why a curve of steepest descent linking $\bxi$ to $\wbxi$ 
should be a straight line. However, here there is only one rowtype, and 
it is likely that there is a unique trajectory joining $(\bxi(d),0)$ to $\wbxi(d)$ along which the loss is strictly decreasing. 
\subsection{Isotropy $\Delta (S_{d-1}\times S_1) $, $q_1 = 1$, $d=k$, Types I and II}

\subsubsection{Type I family of spurious with isotropy $\Delta (S_{d-1} \times S_1)\defoo \Delta S_{d-1}$} 
Let $\bxi(d)$ denote the type I family of critical points in $M(d,d)^{\Delta S_{d-1}}$ which give spurious minima for $d \ge 8$. 
Using Newton's method, a corresponding family $\wbxi(d) \in M(d,d+1)^{\Delta S_{d-1}}$ is constructed for the biased network. 

Recall  that $\Delta_{ub}(d) = \mathcal{L}(\bxi(d))-\mathcal{L}(\wbxi(d))$. 
\begin{scriptsize}
\begin{table}[h]
\begin{tabular}{|c||c|c|c|}
\hline
Type I & Bias free network& Biased network& Loss decrease  \\
$d = k$          &  &   &        with bias    \\
Isotropy $\Delta S_{d-1}$ & $\mathcal{L}(\bxi(d)) $& $\mathcal{L}(\wbxi(d))$ & $\Delta_{ub}(d)$ \\  \hline \hline
$10$        &$7.2057 \times 10^{-2}$   & $7.0390\times 10^{-2}$  & $1.667\times 10^{-3}$          \\ \hline
$100$        &$1.4157 \times 10^{-1}$   & $1.4121\times 10^{-1}$ & $3.621\times 10^{-4}$          \\ \hline
$10^3$       &$1.6847 \times 10^{-1}$  & $1.6843\times 10^{-1} $ & $4.221\times 10^{-5}$            \\ \hline
$10^4$        &$1.77465 \times 10^{-1}$   & $1.77461\times 10^{-1}$ & $4.4108\times 10^{-6}$           \\ \hline
$10^5$        &$1.803499 \times 10^{-1}$   & $1.803494\times 10^{-1}$ & $4.4712\times 10^{-7}$           \\ \hline
$10^6$        &$1.8126589 \times 10^{-1}$   & $1.8126584\times 10^{-1}$ & $4.4904\times 10^{-8}$          \\ \hline
$10^7$        &$1.81555921 \times 10^{-1}$   & $1.81555917\times 10^{-1}$ & $4.4965\times 10^{-9}$          \\ \hline
$10^8$        &$1.816476744 \times 10^{-1}$   & $1.816476739\times 10^{-1}$ & $4.4984\times 10^{-9}$          \\ \hline
$10^9$        &$1.8167669288\times 10^{-1}$ &$1.8167669283\times 10^{-1}$ &$4.4990\times 10^{-11}$ \\ \hline
\end{tabular}
\vspace*{0.05in}
\caption{Computed values of the loss for type I critical point of isotropy  $\Delta S_{d-1}$.
Loss is given to 5 or more significant figures and shows how the addition of bias to the student network always reduces the loss.
The third column gives the decrease of the loss when the network is biased.  This is always positive. The table entries are 
close to the corresponding values for the type I, isotropy $\Delta S_d$ example of the preceding section.}
\end{table}
\end{scriptsize}
\begin{exam}
Taking $d = 10^{128}$, a high precision computation gives
\[
\Delta_{ub}(10^{128}) = 0.044992875434563\ldots \times 10^{-128} > 0
\]
indicating that the first two terms in the FPS expansions of  $\mathcal{L}(\bxi(d))$ and $\mathcal{L}(\wbxi(d))$ are equal.
A more detailed analysis indicates that if the subscripts $A$ and $I$ are used to distinguish the type A and type I network, then 
\[
|\Delta^I_{ub}(d) - \Delta^A_{ub}(d)| = c d^{-2} + O(d^{-\frac{5}{2}}),
\]
where $c \approx 0.0242080$. 
\examend
\end{exam}
The families $\bxi(d)$ and $\wbxi(d)$ are $(2,0,2,2,2)$-matched. The $0$-matched term corresponds to 
$\xi^0_1(d)$ and $\widetilde{\xi}^0_1(d)$ which both have initial non-zero term in their \fps a multiple of $d^{-2}$. Hence 
all corresponding terms $c d^{-\frac{p}{2}}$ in every component of the coordinates are equal provided that $p < 4$.  
\begin{rem}
Both $\bbeta$ derivatives of the loss at $(\bxi,0)$ are strictly positive so the expectation is that both values of $\bbeta$ at the 
critical point $\wbxi$ should be strictly negative. This turns out to be true for all $d \ge 10$. The initial terms of the \fps for $\bbeta(d)$ are given by
\begin{eqnarray*}
\bbeta_0(d) & = & -0.62072950119515416243932018851173\dots d^{-1} + O(d^{-\frac{3}{2}})\\
\bbeta_1(d) & = & -1.2006004478530293576268915492951\ldots d^{-1} + O(d^{-\frac{3}{2}})
\end{eqnarray*} \rend
\end{rem}

\noindent \emph{Type II family of spurious with isotropy $\Delta S_{d-1}$}
\begin{scriptsize}
\begin{table}[h]
\begin{tabular}{|c||c|c|c|}
\hline
Type II & Bias free network& Biased network& Loss decrease  \\
$d = k$          &  &   &        with bias    \\
Isotropy $\Delta S_{d-1}$ & $\mathcal{L}(\bxi(d)) $& $\mathcal{L}(\wbxi(d))$ & $\Delta_{ub}(d)$ \\  \hline \hline
$10$        &$1.8702 \times 10^{-2}$   & $1.7477\times 10^{-2}$  & $1.226\times 10^{-3}$          \\ \hline
$100$        &$2.7417 \times 10^{-3}$   & $2.7123\times 10^{-3}$ & $2.936\times 10^{-5}$          \\ \hline
$10^3$       &$2.9255 \times 10^{-4}$  & $2.9224\times 10^{-4} $ & $3.159\times 10^{-7}$            \\ \hline
$10^4$        &$2.96126 \times 10^{-5}$   & $2.96095\times 10^{-5}$ & $3.1280\times 10^{-9}$           \\ \hline
$10^5$        &$2.96998 \times 10^{-6}$   & $2.96994\times 10^{-6}$ & $3.1087\times 10^{-11}$           \\ \hline
$10^6$        &$2.972468 \times 10^{-7}$   & $ 2.972463 \times 10^{-7}$ & $3.1016\times 10^{-13}$          \\ \hline
$10^7$        &$2.9732287 \times 10^{-8}$   & $2.9732284\times 10^{-8}$ & $3.0993\times 10^{-15}$          \\ \hline
$10^8$        &$2.97346669 \times 10^{-9}$   &$2.97346665\times 10^{-9}$ & $3.0985\times 10^{-17}$          \\ \hline
$10^9$        &$2.973541685\times 10^{-10}$ & $2.973541682\times 10^{-10}$ &$3.09829\times 10^{-19}$ \\ \hline
\end{tabular}
\vspace*{0.05in}
\caption{Computed values of the loss for type II critical point of isotropy  $\Delta S_{d-1}$.
Loss is given to 5 or more significant figures and shows how the addition of bias to the student network always reduces the loss.
The third column gives the decrease of the loss when the network is biased.  This is always positive. }
\end{table}
\end{scriptsize}
\begin{exams}
(1) Taking $d = 10^{128}$, a high precision computation gives
\[
\Delta_{ub}(d) =  c d^{-2} + O(d^{-\frac{5}{2}}),
\]
where $c \approx 0.619636515$ and it is assumed that $\bxi,\wbxi$ have FPS.\\
(2) The initial terms of an \fps expansion for $\bbeta_i(d)$, $i \in \is{2}$ are
\begin{eqnarray*}
\bbeta_0(d) & = & -3.9147584367443569485868567728717\dots d^{-2} + O(d^{-\frac{5}{2}})\\
\bbeta_1(d) & = &  2.6116672965535862283319971186917\ldots d^{-1} + O(d^{-\frac{3}{2}})
\end{eqnarray*}
Note that, unlike type I, the initial terms are of opposite sign.  \\
(3) The families $\bxi(d)$ and $\wbxi(d)$ are $(4,2,2,4,2)$-matched. The $4$-matched terms correspond to the ``diagonal'' coordinates
$\xi^0_0, \widetilde{\xi}^0_0$ and $\xi^1_1, \widetilde{\xi}^1_1$ and so in the corresponding \fps terms match up to the coefficients of $d^{-\frac{3}{2}}$ but not
$d^{-2}$. The same is true for $\xi^0_1(d), \widetilde{\xi}^0_1(d)$  and $\xi^1_0(d), \widetilde{\xi}^1_0(d)$
On the other hand, the families
$\xi^0_{\star 0}(d)$ and $\widetilde{\xi}^0_{\star 0}(d)$  agree up to and including terms in $d^{-2}$.
\examend
\end{exams}

\subsection{Isotropy $\Delta (S_{d-2}\times S_2)$, $q_1 = 2$, $d=k$, Types I and II}
\subsubsection{Type I family of spurious minima with isotropy $\Delta (S_{d-2} \times S_2)$}

Let $\bxi(d)$ denote the type I family of critical points in $M(d,d)^{\Delta (S_{d-2}\times S_2)}$ which give spurious minima for sufficiently large $d$. 
Using Newton's method, a corresponding family $\wbxi(d) \in M(d,d+1)^{\Delta (S_{d-2}\times S_2)}$ is constructed for the biased network. 
\begin{scriptsize}
\begin{table}[h]
\begin{tabular}{|c||c|c|c|}
\hline
Type I & Bias free network& Biased network& Loss decrease  \\
$d = k$          &  &   &        with bias    \\
Isotropy $\Delta (S_{d-2}\times S_2)$ & $\mathcal{L}(\bxi(d)) $& $\mathcal{L}(\wbxi(d))$ & $\Delta_{ub}(d)$ \\  \hline \hline
$10$        &$6.6993\times 10^{-2}$   & $6.5457\times 10^{-2}$  & $1.5356\times 10^{-3}$          \\ \hline
$100$        &$1.4107 \times 10^{-1}$   & $1.4071\times 10^{-1}$ & $3.6209\times 10^{-4}$          \\ \hline
$10^3$       &$1.6841 \times 10^{-1}$  & $1.6837 \times 10^{-1} $ & $4.2223\times 10^{-5}$            \\ \hline
$10^4$        &$1.77459 \times 10^{-1}$   & $1.77454`\times 10^{-1}$ & $4.4110\times 10^{-6}$           \\ \hline
$10^5$        &$1.803492 \times 10^{-1}$   & $1.803487\times 10^{-1}$ & $4.4712\times 10^{-7}$           \\ \hline
$10^6$        &$1.8126582 \times 10^{-1}$   & $i1.8126578\times 10^{-1}$ & $4.4904\times 10^{-8}$          \\ \hline
$10^7$        &$1.81555915 \times 10^{-1}$   & $1.81555910\times 10^{-1}$ & $4.4965\times 10^{-9}$          \\ \hline
$10^8$        &$1.816476737 \times 10^{-1}$   & $1.816476732\times 10^{-1}$ & $4.4984\times 10^{-10}$          \\ \hline
$10^9$        &$1.8167669281 \times 10^{-1}$ &  $1.8167669277\times 10^{-1}$ &$4.4990\times 1^{-11}$ \\ \hline
\end{tabular}
\vspace*{0.05in}
\caption{Computed values of the loss for type I critical point of isotropy  $\Delta (S_{d-2}\times S_2)$.
Loss is given to 5 or more significant figures and shows how the addition of bias to the student network always reduces the loss.
The third column gives the decrease of the loss when the network is biased.}
\end{table}
\end{scriptsize}
\begin{exams}
(1) Taking $d = 10^{128}$, a high precision computation gives
\[
\Delta_{ub}(10^{128}) = 0.044992875434563089568371515024571\ldots \times 10^{-128} 
\]
indicating that the first two terms in the \fps expansions of  $\mathcal{L}(\bxi(d))$ and $\mathcal{L}(\wbxi(d))$ are equal. \\
(2) The initial terms of an \fps expansion for $\bbeta_i(d)$, $i \in \is{2}$ are
\begin{eqnarray*}
\bbeta_0(d) & = & -0.620729501195\ldots d^{-1} + O(d^{-\frac{3}{2}})\\
\bbeta_1(d) & = & -1.2006004478530\ldots d^{-1} + O(d^{-\frac{3}{2}})
\end{eqnarray*}
(3) As for type I, $q_1 = 1$, both initial terms are negative and the 
families $\bxi(d)$ and $\wbxi(d)$ are $(2,0,2,2,2,2)$-matched.
\examend
\end{exams}

\subsubsection{Type II family of spurious minima  with isotropy $\Delta (S_{d-2} \times S_2)$}
Let $\bxi(d)$ denote the type II family of critical points in $M(d,d)^{\Delta (S_{d-2}\times S_2)}$ which give spurious minima for $d \ge 8$. 
Using Newton's method, a corresponding family $\wbxi(d) \in M(d,d+1)^{\Delta (S_{d-2}\times S_2)}$ is constructed for the biased network. 
\begin{scriptsize}
\begin{table}[h]
\begin{tabular}{|c||c|c|c|}
\hline
Type II & Bias free network& Biased network& Loss decrease  \\
$d = k$          &  &   &        with bias    \\
Isotropy $\Delta (S_{d-2}\times S_2)$ & $\mathcal{L}(\bxi(d)) $& $\mathcal{L}(\wbxi(d))$ & $\Delta_{ub}(d)$ \\  \hline \hline
$10$        &$3.5327\times 10^{-2}$   & $3.3175\times 10^{-2}$  & $2.1518\times 10^{-3}$          \\ \hline
$100$        &$5.4487 \times 10^{-3}$   & $5.3914\times 10^{-3}$ & $5.7244\times 10^{-5}$          \\ \hline
$10^3$       &$5.8473 \times 10^{-4}$  & $5.8410\times 10^{-1} $ & $6.298\times 10^{-7}$            \\ \hline
$10^4$        &$5.9221 \times 10^{-5}$   & $5.9215\times 10^{-1}$ & $6.2546\times 10^{-9}$           \\ \hline
$10^5$        &$5.93992 \times 10^{-6}$   & $5.93985\times 10^{-6}$ & $6.2173\times 10^{-11}$           \\ \hline
$10^6$        &$5.944931 \times 10^{-7}$   & $5.944925\times 10^{-1}$ & $6.2032\times 10^{-13}$          \\ \hline
$10^7$        &$5.9464570 \times 10^{-8}$   & $5.9464563\times 10^{-8}$ & $6.1986\times 10^{-15}$          \\ \hline
$10^8$        &$5.94693333 \times 10^{-9}$   & $5.94693327\times 10^{-1}$ & $6.1971\times 10^{-17}$          \\ \hline
$10^9$        &$5.947083367\times 10^{-10}$ &$5.947083361\times 10^{-10}$ &$6.1966\times 10^{-19}$ \\ \hline
\end{tabular}
\vspace*{0.05in}
\caption{Computed values of the loss for type II critical point of isotropy  $\Delta (S_{d-2}\times S_2)$.
Loss is given to 5 or more significant figures and shows how the addition of bias to the student network always reduces the loss.
The third column gives the decrease of the loss when the network is biased.}
\end{table}
\end{scriptsize}
\begin{exams}
(1) Taking $d = 10^{128}$, a high precision computation gives
\[
\Delta_{ub}(10^{128}) = 0.6196365153963871523492\ldots \times 10^{-256} 
\]
indicating that the first two terms in the FPS expansions of  $\mathcal{L}(\bxi(d))$ and $\mathcal{L}(\wbxi(d))$ are equal.\\
(2) The initial terms of an \fps expansion for $\bbeta_i(d)$, $i \in \is{2}$ are 
\begin{eqnarray*}
\bbeta_0(d) & = & -7.829516873488713897737135457431\ldots d^{-2} + O(d^{-\frac{5}{2}})\\
\bbeta_1(d) & = &  2.611667296553586228331997118692\ldots d^{-1} + O(d^{-\frac{3}{2}})
\end{eqnarray*}
and have opposite sign.\\
(3) Similar to the type II family of critical points with isotropy $\Delta S_{d-1}$, the families $\bxi(d), \wbxi(d)$ are $(4,2,2,4,2,2)$-matched.
\examend
\end{exams}
\subsection{Isotropy $\Delta (S_{d-3}\times S_3)$, $q_1 = 3$, $d=k$, type II}
The final family of examples has isotropy $\Delta (S_{d-3}\times S_3)$. 
Only the type II case is considered and this follows the pattern described above for $q_1 = 1,2$. 
Similar remarks hold for type I when $q_1 = 3$. 

The `seed' for the type II family is the critical point with $d = 10$ 
and isotropy $\Delta (S_{7}\times S_3)$.
This is obtained by evolving along a path in the fixed point space in the unbiased network from the type II point with 
isotropy $\Delta (S_{8}\times S_2)$ to the type II critical point with isotropy $\Delta (S_{7}\times S_3)$. That is, if we write the
isotropy as $\Delta (S_{q_0} \times S_{q1})$, we vary $(q_0,q_1)$ linearly along a line segment so that $q_0 + q_1 = d  = 10$,
$q_0$ decreases from $8$ to $7$, and $q_1$ increases from $2$ to $3$. This path lies in the 6-dimensional fixed point space $M(10,10)^{\Delta (S_{7}\times S_3)}$. 
If $q_1 = 1$, the fixed point space is 5-dimensional, the critical point equations are different and we cannot path evolve from $q_1 = 1$ to $q_1 = 2$.
The same method can be used for the biased network and this provides a simple consistency check (is the result of incrementing first and then evolving the same as evolving first then incrementing?).
\begin{scriptsize}
\begin{table}[h]
\begin{tabular}{|c||c|c|c|}
\hline
Type II & Bias free network& Biased network& Loss decrease  \\
$d = k$          &  &   &        with bias    \\
Isotropy $\Delta (S_{d-3}\times S_3)$ & $\mathcal{L}(\bxi(d)) $& $\mathcal{L}(\wbxi(d))$ & $\Delta_{ub}(d)$ \\  \hline \hline
$10$        &$5.0240\times 10^{-2}$   & $4.7284\times 10^{-2}$  & $2.9556\times 10^{-3}$          \\ \hline
$100$        &$8.1326 \times 10^{-3}$   & $8.0485\times 10^{-3}$ & $8.4097\times 10^{-5}$          \\ \hline
$10^3$       &$8.7666 \times 10^{-4}$  & $8.7572\times 10^{-4} $ & $9.4278\times 10^{-7}$            \\ \hline
$10^4$        &$ 8.8828 \times 10^{-5}$   & $8.8818\times 10^{-1}$ & $9.3801\times 10^{-9}$           \\ \hline
$10^5$        &$8.9098 \times 10^{-6}$   & $8.9097\times 10^{-6}$ & $9.3258\times 10^{-11}$           \\ \hline
$10^6$        &$8.91739 \times 10^{-7}$   & $8.91738\times 10^{-1}$ & $9.3048\times 10^{-13}$          \\ \hline
$10^7$        &$8.919685 \times 10^{-8}$   & $8.919684\times 10^{-8}$ & $9.2978\times 10^{-15}$          \\ \hline
$10^8$        &$8.9204000\times 10^{-9}$   & $8.9203999\times 10^{-1}$ & $9.2956\times 10^{-17}$          \\ \hline
$10^9$        &$8.92062505\times 10^{-10}$ &$8.92062504\times 10^{-10}$ &$9.2949\times 10^{-19}$ \\ \hline
\end{tabular}
\vspace*{0.05in}
\caption{Computed values of the loss for type II critical point of isotropy  $\Delta (S_{d-3}\times S_3)$.
Loss is given to 5 or more significant figures and shows how the addition of bias to the student network always reduces the loss.
The third column gives the decrease of the loss when the network is biased.}
\end{table}
\end{scriptsize}
\begin{exams}
(1) Using the result for $d = 10^9$ and assuming the existence of the FPS, 
\[
\Delta_{ub}(10^{9}) = 0.9295\ldots \times 10^{-18} + O(d^{-\frac{3}{2}})
\]
indicating that the first two terms in the FPS expansions of  $\mathcal{L}(\bxi(d))$ and $\mathcal{L}(\wbxi(d))$ are equal.\\
(2) Using computations for $d = 10^9$, the initial terms of an \fps expansion for $\bbeta_i(d)$, $i \in \is{2}$, are
\begin{eqnarray*}
\bbeta_0(d) & = & -11.74\ldots d^{-2} + O(d^{-\frac{5}{2}})\\
\bbeta_1(d) & = &  2.612\ldots d^{-1} + O(d^{-\frac{3}{2}})
\end{eqnarray*}
and have opposite sign.\\
(3) The families $\bxi(d)$ and $\wbxi(d)$ are $(4,2,2,4,2,2)$-matched.
\examend
\end{exams}

\subsection{Comments on numerical methods: comparison with bias free networks}
Rather than using exact Newton's method, we avoid computing the many second derivatives of the loss function by using 
approximations to the second derivatives. Specifically, partial derivatives of $f(\xx) = \grad{\mathcal{L}}(\xx)$
are approximated by $[f(\xx+h \is{e}_i) - f(\xx)]/h$ where
$\xx \in \real^n$, $(\is{e}_1,\ldots,\is{e}_n)$ is an orthonormal basis of $\real^n$ and $h$ is a ``small'' scalar, dependent on the precision.  This method works well 
and reliably for bias-free ReLU networks. 
However, matters are more challenging when bias is introduced. There are two potential problems. \\
(1) The presence of terms such as $-\bgamma_i\cot(\lambda_{ii}) + \bbeta_i \csc(\lambda_{ii})$ when $\lambda_{ii}$
converges to $0$ or $\pi$ as $d \arr \infty$ (as it does in all the examples). This only has a happy ending if $ |\bgamma_i -  \bbeta_i |$ decays fast enough to zero as $d \arr \infty$.  In particular, if
$\bgamma_i = 0$, then $\bbeta_i$ needs to have a fast enough decay to offset the singular behaviour of either the cotangent or cosecant at $0$ or $\pi$ (see also Section~\ref{Sec: Owen}).\\
(2) Formula~\Refb{EQ: GENT} seems unlikely to behave well with respect to the approximate Newton method sketched above.

With a little care in choosing the increment $h$, the second problem can be managed within the framework of the approximate Newton method, at least for values of $d$ that 
are not too large and if the $\bgamma$ terms are assumed to be zero. 
As an example, if $d = 10^{512}$ and we take the type I family of maximum isotropy spurious minima for the unbiased network and try to use Newton's method, then
the iteration fails if $h = 10^{-480}$ but converges to the correct critical point if $h = 10^{-600}$. Precision in both case was 16384 bit --- approximately 
4900 decimal places.  The iteration converges significantly faster if $h = 10^{-1000}$ with errors (value of LHS of the associated critical point equation) ranging from about $10^{-4400}$ to $10^{-4934}$, much the same as with $h = 10^{-600}$.

As it is, 
there is a simple resolution of the problem obtained by using the easily computed exact derivatives of the integral that is being represented using the $T$-function 
(see Propositions~\ref{lem: Hderiv} and \ref{prop: tildederiv}).

The first problem appears harder to manage but more interesting. One possibility is that the result of using noisy data is not to have convergence to spurious minima 
when the $\bbeta_i$ terms are not decaying
fast enough in $d$ (this is potentially a mechanism for removing spurious minima when $k > d$). 

Summarizing. if approximate Newton routines are used, then it is important to choose the increments $h$ carefully so as not to leave the region of subanalyticity. 
This is potentially a serious problem 
for approximating derivatives of terms depending on the $T$-function on account of the ratios of $\bbeta$ terms that occur. 

The complete set of the critical point data used in this section is available on request as a ``data.h" file that can be incorporated into a high precision program.
Also available is an \emph{MPFR} program that includes tested code for the computation of the $T$-function and the generalized $T$-function (see Appendix C).  

\appendix
\section{Review of the error function $\erf{x}$}\label{append: error}
We state a number of elementary results on the error function that are used in Section~\ref{sec: bias}.

Recall the definitions of the error and complementary error functions:
\begin{eqnarray}
\label{EQ0a}
\erf{x}& = & \frac{2}{\sqrt{\pi}}\int_{0}^x e^{-t^2}dt,\quad
\erfc{x}  =  1 - \erf{x}
\end{eqnarray}
\begin{rem}\label{rem: erfc}
The error function is an \emph{odd} function of $x$ but the complementary error function is neither odd nor even. Integrals such as $\int e^{-(at+b)^2} dt$ have invariance under the symmetry $(a,b) \leftrightarrow (-a,-b)$ and this carries
through to the error function defined by the integral (see~\Refb{EQ1}---note the appearance of $a$ twice). This is not the case if we work with the complementary error function. \rend
\end{rem}
Related to the error function are $G'(x)=\frac{1}{\sqrt{2\pi}} e^{-\frac{x^2}{2}}$, the univariate density of the standard normal distribution, 
and $G(x) = \int_{-\infty}^x G'(t)dt$ the cumulative density function:
\begin{eqnarray}
\label{EQ0b}
G(x) & = & \frac{1}{2}\left[1 + \erfb{\frac{x}{\sqrt{2}}}\right] \\
\label{EQ0c}
\frac{d}{dx}\erfb{\frac{x}{\sqrt{2}}}& =& 2 G'(x) = \sqrt{\frac{2}{\pi}} e^{-\frac{x^2}{2}} 
\end{eqnarray}

Next, some elementary results used in Section~\ref{sec: bias} about integrals related to the error function (see~\cite{Korotkov2019,NgGeller1969,NgGeller1971} for general tables).
To this end, $a \ne 0$ and $b$ will denote real constants or real functions depending on a parameter.
\begin{eqnarray}
\label{EQ1}
\int e^{-(ax + b)^2}dx&=&\frac{\sqrt{\pi}}{2a} \erf{ax+b}
\end{eqnarray}
and so
\begin{equation}
\label{EQ1a}
\frac{d}{dx}\erf{ax+b} = \frac{2a}{\sqrt{\pi}}e^{-(ax+b)^2}
\end{equation}
More generally, if $f$ is continuously differentiable,
\begin{equation}
\label{EQ1b}
\frac{d}{dx}\erf{f(x)} = \frac{2f'(x)}{\sqrt{\pi}}e^{-f(x)^2}
\end{equation}

\begin{exam}
In terms of definite integrals, we have
\begin{eqnarray}
\label{EQ2a}
\int_{0}^x e^{-(at + b)^2}dt&=&\frac{\sqrt{\pi}}{2a} \big(\erf{ax+b}-\erf{b}\big) \\
\label{EQ2b}
\int_x^\infty  e^{-(at+b)^2}dt&=& \frac{\sqrt{\pi}}{2a} \big(\sgn(a) - \erf{ax+b}\big) 
\end{eqnarray}
If $\sgn(a) = 1$, then $\frac{\sqrt{\pi}}{2a}\big(\sgn(a) - \erf{ax+b}\big)=\erfc{ax+b}$. \examend
\end{exam}
\begin{equation}
\label{EQ3}
\int x e^{-(ax+b)^2} dx=-\frac{1}{2a^2}\big[e^{-(ax+b)^2} + \sqrt{\pi} b\,\erf{ax+b}\big] 
\end{equation}
\begin{exam}\label{Ex: EQ3}
Using~\Refb{EQ3}, we have
\[
\int_X^\infty t e^{-(at+b)^2} dt = \frac{1}{2a^2}\left[e^{-(aX+b)^2} + \sqrt{\pi}b\big[\erf{aX+b}- \sgn(a)\big]\right]
\]
Note the invariance under the symmetry $(a,b) \leftrightarrow (-a,-b)$. \examend
\end{exam}
\begin{multline}
\label{EQ4}
\int x^2 e^{-(ax+b)^2} dx= \\
\frac{1}{2a^3}\left[(b-ax)e^{-(ax+b)^2} + \frac{\sqrt{\pi}(2b^2+1)}{2} \erf{ax+b}\right]  
\end{multline}

\begin{multline}
\label{EQ4a}
\int x^3e^{-(ax+b)^2} dx=\\
\begin{aligned}
& -\frac{\sqrt{\pi}}{2a^4}\big(b^3 + \frac{3b}{2}\big) \erf{ax+b}  
 -\frac{1}{2a^4}e^{-(ax+b)^2}\left[a^2x^2 - abx + b^2 + 1\right] 
\end{aligned}
\end{multline}
Integrals involving $\erf{ax+b}$ can often be computed using integration by parts and~\Refb{EQ1a}. For example,
\begin{multline}
\label{EQ4b}
\int_0^x \erf{at +  b}dt  =  a^{-1}\left[(ax+b)\erf{ax+b} -b\,\erf{b}\right]\\
\begin{aligned}
 +\frac{1}{a\sqrt{\pi}}\left[e^{-(ax+b)^2} - e^{-b^2}\right] 
\end{aligned}
\end{multline}
The following result, which uses Example~\ref{Ex: EQ3}, is used for computing $f(\ww,\vv)$ when there is non-zero bias.
\begin{multline}
\label{EQ5a}
\int_X^\infty t^2\erfb{\frac{at+b}{\sqrt{2}}}e^{-\frac{t^2}{2}} dt = Xe^{-\frac{X^2}{2}}\erfb{\frac{aX+b}{\sqrt{2}}}+\\
\begin{aligned}
& \frac{a^2b}{(a^2+1)^{\frac{3}{2}}} e^{-\frac{b^2}{2(1+a^2)}} \left[\erfb{\frac{(1+a^2)X+ba}{\sqrt{2(1+a^2)}}}-1\right] + \\
& \sqrt{\frac{2}{\pi}}\frac{a}{a^2+1} e^{-\frac{b^2+((1+a^2)X+ba)^2}{2(1+a^2)}} + \int_X^\infty e^{-\frac{t^2}{2}}\erfb{\frac{at+b}{\sqrt{2}}}dt
\end{aligned}
\end{multline}
\begin{rem}
For $X \in (-\infty,\infty)$, the integral $\int_X^\infty e^{-\frac{t^2}{2}}\erfb{\frac{at+b}{\sqrt{2}}}dt$ cannot be evaluated in
terms of the error function and exponentials. However, the integral can be represented in terms of the \emph{Owen T-function}~\cite{Owen1956,Owen1980}---see Section~\ref{sec: bias}. \rend
\end{rem}

\section{High precision computation of Owen's $T$-function}\label{appendB: OwenT}
High precision computations were done using the \emph{MPFR} package which uses \emph{GMP} and \emph{GCC}. 
As far as we know, Owen's $T$-function is currently not available as a 
standard built-in function in \emph{MPFR} or any other high precision package. However, it is available
in several (low precision) packages; for example, the statistical software \emph{R}. Recently, a package was written by Komelj that uses 
the \emph{Rmpmr} high precision extension of \emph{R}.  
This package uses his method for the computation of the $T$ function~\cite{Komelj2023} and gives a fast high precision computation of
$T(h,a)$ for all values of $h,a$. Based on the analysis in~\cite{Komelj2023}, we have written \emph{MPFR} code for high precision computation of the $T$-function with
a view to applications to biased ReLU networks.  We give the basic method below and refer the reader to Komelj's article~\cite{Komelj2023} for details we omit.

The aim is to compute $T(h,a)$ for $h, a \in \real$. We use the two algorithms (40a) and a slightly modified (40c) from~\cite[Theorem 5]{Komelj2023}\footnote{If $a < -1$, apply the second algorithm to compute $T(h,-a)$ and then $T(h,a) =-T(h,-a)$. Note that both (40a,40c) converge for all $a \in \real$ but that (40a) (resp. (40c)) converges fastest if $|a| \le 1$ (resp.~$a > 1$).
In the proof of ~\cite[Theorem 5]{Komelj2023}, reference is made to Owen's formula (17) and that should be (15) (\emph{op.~cit.}). As stated, \cite[(40c)]{Komelj2023} is not quite correct---this does not affect any of the numerical 
results in~\cite{Komelj2023} which depend on an application of (15).}.
\begin{scriptsize}
\begin{eqnarray*}
T(h,a)& =& \frac{tan^{-1}(a)}{2\pi} - \frac{a}{2\pi(1+a^2)}\sum_{k=0}^\infty \frac{(2k)!!}{(2k+1)!!}P(k+1,q)\left(\frac{a^2}{1+a^2}\right)^k,\; \text{if } |a| \le 1, \\
& =& U(h,a) -\frac{1}{4} + \frac{tan^{-1}(a)}{2\pi} + \frac{a}{2\pi(1+a^2)}\sum_{k=0}^\infty \frac{(2k)!!}{(2k+1)!!}P(k+1,q)\left(\frac{1}{1+a^2}\right)^k,\; \text{if } a > 1.
\end{eqnarray*}
\end{scriptsize}
where  
\begin{eqnarray*}
q& =& \frac{1}{2}(1+a^2)h^2\\
U(h,a)&=& \frac{1}{4}\left[1 -\erfb{\frac{h}{\sqrt{2}}}\erfb{\frac{ha}{\sqrt{2}}}\right],\; a \ge 0\\
\end{eqnarray*}
and $P(k+1,q)$ is the \emph{lower regularized Gamma function} defined for $k \ge 0$ by
\[
P(k+1,q) = \frac{\int_0^q t^{-k}e^{-t}\,dt}{\Gamma(k+1)} \in (0,1)
\]
Noting the elementary identities $P(k+1,q) = P(k,q) - \frac{q^ae^{-q}}{\Gamma(k+1)}$, $P(1,q) = 1 - e^{-q}$, 
recursively define sequences $(B_k)_{k \ge 0},(D_k)_{k \ge 0} = (P(k+1,q)_{k\ge 0}$ by
\begin{enumerate}
\item $B_0 = e^{-q}, D_0 = 1 -B_0$.
\item $B_{k+1} = \frac{qB_k}{k+1}$, $D_{k+1} = D_k - B_{k+1}$
\end{enumerate}
This leads directly to a recursive definition of the terms in the infinite series 
$\sum_{k=0}^\infty \frac{(2k)!!}{(2k+1)!!}P(k+1,q)\left(\frac{a^2}{1+a^2}\right)^k$ and the corresponding partial sums $S_k$. 
Noting the positivity of the terms,
and following~\cite{Komelj2023}, the computation is terminated when $S_{k+1} < S_k$  and the sum is taken to be $S_k$.  We refer to~\cite{Komelj2023} for more details. In our implementation, the partial sums are slightly different from those used in~\cite{Komelj2023}, On initialization, we also define the
sequence $(\frac{(2k)!!}{(2k+1)!!})_{k \le 10,000}$ of double factorials.  
\begin{rem}
Although the upper and lower regularized Gamma functions are used in~\cite{Komelj2023}, 
the \emph{MPFR} implementation of the upper and lower regularized Gamma functions is not used. 
Indeed, as is commented in the \emph{MPFR} manual~\cite[mpfr\_gamma\_inc, page 32]{GNUMPFR}, the \emph{MPFR} implementation of the incremental gamma 
function can be slow for large values of the arguments.
\rend
\end{rem}

\section{Generalized $T$-function}\label{AppendixC}
Recall from Section~\ref{sec: bias} that the generalized $T$-function~\cite{Prezmo} is defined by
\begin{equation}
T(h,a,b)  = \frac{1}{2\sqrt{2\pi}}\int_h^\infty e^{-\frac{t^2}{2}}\erfb{\frac{at + b}{\sqrt{2}}}dt  
\end{equation}
and that
\begin{equation*}
2\pi T(\bgamma,-cot(\theta),\bbeta \csc(\theta)) = \sqrt{\frac{\pi}{2}}\int_{\bgamma}^\infty e^{-\frac{t^2}{2}}\erfb{\frac{-t\cot(\theta) + \bbeta\csc(\theta)}{\sqrt{2}}}dt
\end{equation*}
It follows from~\cite{Prezmo} that
\begin{multline*}\label{EQ: fullT0}
\begin{aligned}
\hspace*{-0.13in}T(h,a,b)  =  &\frac{1}{2\pi} \left[\tan^{-1}(a) -\tan^{-1}\left(a+\frac{b}{h}\right) - \tan^{-1}\left(a+\frac{h(1+a^2)}{b}\right)\right] \\
& + T\left(h,a+\frac{b}{h}\right) +T\left(\frac{b}{\sqrt{1+a^2}},a +\frac{h(1+a^2)}{b}\right) \\
& + \frac{1}{4}\erfb{\frac{b}{\sqrt{2(1+a^2)}}} 
\end{aligned}
\end{multline*}

On the other hand, Owen~\cite{Owen1980} gives a formula in terms of the CDF
\begin{multline*}
\int G'(x) G(ax +b) dx = T\left(x,\frac{b}{x\sqrt{1+a^2}}\right)\\ 
\begin{aligned}
&+ T\left(\frac{b}{\sqrt{1+a^2}},\frac{x\sqrt{1+a^2}}{b}\right)  -T\left(x,\frac{ax +b}{x}\right)\\
& - T\left(\frac{b}{\sqrt{1+a^2}},\frac{ab+ x(1+a^2)}{b}\right) + G(x)G\left(\frac{b}{\sqrt{1+a^2}}\right)
\end{aligned}
\end{multline*}
Since $G'(t) =\frac{1}{\sqrt{2\pi}}e^{-\frac{t^2}{2}}$ and $G(t) = \frac{1}{2}\left[1+\erf{\frac{t}{\sqrt{2}}}\right]$,
\begin{equation*}
\int_h^\infty G'(t)G(at+b)dt = T(h,a,b) + \frac{1}{4}\erfcb{\frac{h}{\sqrt{2}}}
\end{equation*}
and so
\begin{equation}
T(h,a,b) = \int_h^\infty G'(t)G(at+b)dt -\frac{1}{4} \erfcb{\frac{h}{\sqrt{2}}}
\end{equation}
Using Owen's formula for the indefinite integral $\int G'(x) G(ax +b) dx $, the contribution at $x = \infty$ is 
$$\frac{1}{2}\left[1+\erfb{\frac{b}{\sqrt{2(1+a^2}}}\right]$$ and that at $x = h$ is
\begin{eqnarray*}
&&T\left(h,\frac{b}{h\sqrt{1+a^2}}\right) + T\left(\frac{b}{\sqrt{1+a^2}},\frac{h\sqrt{1+a^2}}{b}\right)\\
&&  -T\left(h,a+\frac{b}{h}\right) - T\left(\frac{b}{\sqrt{1+a^2}},a +\frac{h(1+a^2)}{b}\right)\\
&& + \frac{1}{4}\left(1 + \erfb{\frac{h}{\sqrt{2}}}\right)\left(1 + \erfb{\frac{b}{\sqrt{2(1+a^2))}}}\right)
\end{eqnarray*}
Hence
\begin{eqnarray*}
T(h,a,b) & =  & T\left(h,a+\frac{b}{h}\right) + T\left(\frac{b}{\sqrt{1+a^2}},a +\frac{h(1+a^2)}{b}\right)\\
	&& -T\left(h,\frac{b}{h\sqrt{1+a^2}}\right) - T\left(\frac{b}{\sqrt{1+a^2}},\frac{h\sqrt{1+a^2}}{b}\right)\\
	&& +\frac{1}{4}\erfb{\frac{b}{\sqrt{2(1+a^2)}}}\erfcb{\frac{h}{\sqrt{2}}}
\end{eqnarray*}

Unlike the formula used in the text, this formula uses 4 instances of the $T$ function. However, using \cite[Table II, 2.8,2.9]{Owen1980}, it can be shown that 
$-T\left(h,\frac{b}{h\sqrt{1+a^2}}\right) - T\left(\frac{b}{\sqrt{1+a^2}},\frac{h\sqrt{1+a^2}}{b}\right)$ can be expressed in terms of inverse tangents and error functions
and so obtain the formula given in~\cite{Prezmo}.

As well as analytically, we have verified that the two expressions for the generalized $T$-function give the same value for $T(h,a,b)$ to 615 decimal places when working in 2048 bit precision using \emph{MPFR}.


\begin{thebibliography}{99}
\bibitem{ArjevaniField2019a} Y Arjevani and M Field. `On the principle of least symmetry breaking in shallow relu models'. arXiv:1912.11939.
\bibitem{ArjevaniField2020b} Y Arjevani and M Field. `Analytic characterization of the hessian in shallow relu models: A tale of symmetry', \emph{NeurIPS 2020}.
\bibitem{ArjevaniField2021}  Y Arjevani and M Field. `Analytic study of families of spurious minima in two-layer ReLU neural networks: a tale of symmetry II', \emph{Proc.NeurIPS 2021}, 15162-–15174
\bibitem{ArjevaniField2021x}  Y Arjevani and M Field.`Symmetry and critical points for a model shallow neural network', \emph{Physica D} {\bf 427} (2021). arXiv:2003.10576, 2020.
\bibitem{ArjevaniField2022b} Y Arjevani and M Field. `Annihilation of Spurious Minima in Two-Layer ReLU Networks', \emph{Proc.NeurIPS 2022}
\bibitem{ArjevaniField2022a} Y Arjevani and M Field. `Equivariant bifurcation, quadratic equivariants, and symmetry breaking for the standard representation of $S_k$', \emph{Nonlinearity} {\bf 35}(6) (2022), 2809. 
\bibitem{BierMil1998} E Bierstone and P Milman. `Subanalytic Geometry', \emph{Model theory, algebra and geometry (ed. Haskell, D et al.)} MSRI, 39 (1998), 151--172
\bibitem{BrutzkusGloberson2017} A Brutzkus and A Globerson. `Globally optimal gradient descent for a convnet with gaussian inputs', \emph{Proc. of the 34th Int. Conf. on Machine Learning} {\bf 70} (2017), 605--614.
\bibitem{BrutzkusGloberson2018} A Brutzkus, A Globerson, E Malach, \& S Shalev-Shwartz. `SGD Learns Over-parameterized Networks that Provably
Generalize on Linearly Separable Data', (in \emph{6th International Conference on Learning Representations}, ICLR 2018, Vancouver, BC, Canada, April 30-- May 3, 2018, Conf. Track Proc., 2018).
\bibitem{ChooSaul2009} Y Cho and L K Saul. `Kernel Methods for Deep Learning', \emph{Advances in neural information processing systems} (2009), 342--350.
\bibitem{Komelj2023} J Komelj. `The Bivariate Normal Integral via Owen's $T$ Function as a Modified Euler's Arctangent series', \emph{Amer. Jnr. of Comp. Math.} {\bf 13}(4) (2023), 476--504 (also
arXiv:2312.00011v1).
\bibitem{Korotkov2019} N E Korotkov. \emph{Table of integrals related to error function} (based on Russian text by N E Korotkov, edited by A N Korotkov, 2019).
\bibitem{GNUMPFR} GNU MPFR `The Multiple Precision Floating-Point Reliable Library', (version 4.2.2, available at \text{mpfr.org/mpfr-current/\#doc}).
\bibitem{Milnor1968} J Milnor. \emph{Singular points of complex hypersurfaces} (Annals of Math.~Studies~61, Princeton University Press, 1968).
\bibitem{NgGeller1969} E W Ng and M Geller. `A Table of Integrals of the Error Function', \emph{Jnr. Research of the Nat. Bureau of Stand.---B. Math. Sci.} {\bf 73B}(1) (1969), 1--26.
\bibitem{NgGeller1971} E W Ng and M Geller. `A Table of Integrals of the Error Function. II. Additions and Corrections.', \emph{Jnr. Research of the Nat. Bureau of Stand.---B. Math. Sci.} {\bf 75B}(3) (1971), 149--163.
\bibitem{Owen1956} D B Owen. `Tables for computing bivariate normal probabilities', \emph{Ann. Math. Statist.} {\bf 27} (1956), 1075--1090
\bibitem{Owen1980} D B Owen. `A table of normal integrals', \emph{Commun. Statist.--Simula.  and Comput.} 9:4 (1980), 389--419.
\emph{9th Innovations in Theoretical Computer Science Conference, ITCS 2018}, January 11-14, 2018, Cambridge, MA, {USA}, 2018), 22:1--22:19.
\bibitem{Parusinski1994} A Parusi\'{n}ski. `Subanaytic functions', \emph{Trans. Amer. Math. Soc.}, {\bf 344}(2) (12994), 583--595.
\bibitem{Pinkus1999} A Pinkus, `Approximation theory of the MLP model in neural networks', \emph{Acta Numer.} {\bf 8} (1999), 143--195.
\bibitem{Prezmo} Przemo (https://math.stackexchange.com/users/99778/przemo), Generalized Owen's T function, URL (version: 2019-03-28): https://math.stackexchange.com/q/3087504
\bibitem{SafranShamir2018} I Safran and O Shamir. `Spurious Local Minima are Common in Two-Layer ReLU Neural Networks',
\emph{Proc. of the 35th Int. Conf. on Machine Learning} {\bf 80} (2018), 4433--4441 (for data sets, see https://github.com/ItaySafran/OneLayerGDconvergence).
\bibitem{Safranetal2021} I Safran, G Yehudai, \& P Shamir. `The Effects of Mild Over-parametrization on the Optimization Landscape of Shallow ReLU Neural Networks',
\emph{Proc. of Machine Learning Research} {\bf 134} (2021), 1-–46.
\bibitem{ZSL2024} Y Zhang, A Saxe, \& P E Latham. `When are bias-free ReLU networks effectively linear networks?' \emph{arXiv:2406.12615v3}.
\end{thebibliography}
\end{document}